\documentclass[10pt,journal,compsoc]{IEEEtran}
\usepackage{bm}

\ifCLASSOPTIONcompsoc
  \usepackage[nocompress]{cite}
\else
  \usepackage{cite}
\fi
\ifCLASSINFOpdf
\else
\fi
\usepackage{amsmath, amssymb}
\usepackage{amsthm}
\usepackage[dvipsnames]{xcolor}

\usepackage{annotate-equations}

\let\appendices\relax
\usepackage{titletoc}
\usepackage[page,header]{appendix}

\usepackage{microtype}
\usepackage{graphicx}
\usepackage{subfigure}
\usepackage{booktabs} 
\usepackage{multirow}
\usepackage{colortbl}

\usepackage{hwemoji}

\usepackage{bm}
\usepackage{nicefrac}
\usepackage{amsmath}
\usepackage{amssymb}
\usepackage{mathtools}
\usepackage{amsthm}
\usepackage{ragged2e}
\usepackage{thm-restate}

\definecolor{ballblue}{HTML}{338EA7}
\definecolor{lightseagreen}{HTML}{759D39}
\definecolor{lightred}{HTML}{DD7769}
\definecolor{org}{HTML}{F8A145}
\definecolor{blu}{HTML}{63ACE5}
\definecolor{c1}{HTML}{41B3A3}
\definecolor{c2}{HTML}{3500D3}
\definecolor{test}{HTML}{E5B8FD}
\definecolor{bdy}{HTML}{FDADAD}
\definecolor{corr}{HTML}{B2E1EA}
\definecolor{wrng}{HTML}{CFE8BD}
\usepackage[colorlinks,linkcolor=blue,citecolor=org,urlcolor=org, linktoc=all]{hyperref}

\usepackage[capitalize,noabbrev]{cleveref}
\usepackage{algorithm}
\usepackage{algorithmic}
\usepackage{bm}

\newtheorem{defi}{Definition}

\newtheorem{prop}{Proposition}

\newtheorem{lem}{Lemma}
\newtheorem{thm}{Theorem}
\newtheorem{asm}{Assumption}
\newtheorem{col}{Corollary}

\newtheorem{clm}{Claim}

\theoremstyle{definition}
\newtheorem{rem}{Remark}

\renewcommand{\eqnhighlightshade}{9}

\newcommand{\colt}[1]{{\color{wrng}{#1}}}

\newcommand{\eqng}[1]{{{#1}}}
\newcommand{\eqnp}[1]{{{#1}}}
\newcommand{\eqnb}[1]{{{#1}}}

\def \ri {\eqng{\rho^\mathcal{I}}}

\def \rm {\eqnp{\rho^\mathcal{M}}}

\def \tb#1{\text{\color{blue}{#1}}}

\def \dir {\mathsf{Dir}}
\def \mf {C_M}

\def \x {\bm{x}}

\def \P {\mathbb{P}}
\def \pair{(\x,y)}

\def \s {\mathsf{softmax}}
\def \R {\mathcal{R}}
\def \ours{$\mathsf{DirMixE}$}
\def \vp {{\mathbb{V}}_+(\ell_{\mathsf{LA}}) }
\def \v {{\mathbb{V}}(\ell_{\mathsf{LA}}) }
\def \la {{\ell_{\mathsf{LA}}}}

\def \vnn{\mathbb{V}_-[\ell]}
\def \vpp{\mathbb{V}_+[\ell]}
\def \vv{\V[\ell]}
\def \ld{{\ell_{{\color{lightseagreen}\mathcal{D}},\mathcal{E}}}}
\def \ls{{\hat{\ell}_{{\color{ballblue}\mathcal{S}},\mathcal{E}}}}

\def \lee {\lesssim}
\def \gee {\gtrsim}
\def  \ee{\asymp}
\def \els {\Ec_{\mathcal{E}}[\ls]}
\def \eld {\Ec_{\mathcal{E}}[\ld]}
\def \eldp {\Ec_{\mathcal{E}}[\ld']}

\def \eldt {\Ec_{\mathcal{E}’}[\ell_{\mathcal{D},\mathcal{E}'}]}

\def  \vls {\Vc_{\mathcal{E}}[\ls]}
\def  \vld {\Vc_{\mathcal{E}}[\ld]}
\def  \vpld {\Vc_{{\color{ballblue}+},\mathcal{E}}[\ld]}

\def \ehls {\Ehc_{\mathcal{E}}[\ls]}
\def \ehlsp {\Ehc_{\mathcal{E}}[\ls']}

\def \ehld {\Ehc_{\mathcal{E}}[\ld]}
\def  \vhls {\Vhc_{\mathcal{E}}[\ls]}
\def  \vphls {\Vhc_{{\color{ballblue}+},\mathcal{E}}[\ls]}

\def  \vhld {\Vhc_{\mathcal{E}}[\ld]}
\def  \vphld {\Vhc_{{\color{ballblue}+},\mathcal{E}}[\ld]}

\def  \vphld {\Vhc_{{\color{ballblue}+},\mathcal{E}}[\ld]}

\def  \ninf {\mathcal{N}_{\infty}(\mathcal{F},\epsilon,n)}
\def \el {\E[\ell]}
\def \ela {\E[\ell|\mathcal{A}(\rho)]}

\def \ehl {\Eh[\ell]}
\def \ehla {\Eh[\ell|\mathcal{A}(\rho)]}
\def \FA{ \mathcal{F}_{\mathcal{A}}(\rho,p)}

\def \d {\mathcal{D}}

\newcommand{\thi}[1] {\bm{\Theta}^{(#1)}}
\newcommand{\this}[1] {\bm{\Theta}^{(#1),\star}}
\def \Th {\bm{\Theta}}
\def \Tha {\bm{\Theta}_{\mathrm{all}}}

\def  \Ail {\bm{A}^{(i)}_l}
\def  \Ails {\bm{A}^{(i),\star}_l}

\def  \Bil {\bm{B}^{(i)}_l}
\def  \Bils {\bm{B}^{(i),\star}_l}

\def  \Uil {\bm{W}^{(i)}_{\uparrow,l}}
\def  \Uils {\bm{W}^{(i),\star}_{\uparrow,l}}

\def  \Dil {\bm{W}^{(i)}_{\downarrow,l}}
\def  \Dils {\bm{W}^{(i),\star}_{\downarrow,l}}

\newcommand{\linf}[1]{\left\|{#1}\right\|_{\infty}}

\newcommand{\thit}[1] {\bm{\tilde{\Theta}}^{(#1)}}
\newcommand{\thist}[1] {\bm{\tilde{\Theta}}^{(#1),\star}}

\newcommand{\newsec}[1]{#1}

\newcommand{\newcont}[1]{#1}

\newcommand{\infn}[1]{\|{#1}\|_\infty}
\newcommand{\dif}[1]{\Delta {#1}}

\newcommand{\supf}[1]{\sup_{ f \in \FA}\left[  {#1}\right]}
\newcommand{\supB}[1]{\sup_{ f \in \mathcal{B}_j}\left[  {#1}\right]}
\newcommand{\ninff}[2] {\mathcal{N}_{\infty}(\mathcal{F},{#1},{#2})}
\newcommand{\nls}[1]{\left\{ #1 \right\}_{l=1}^{n_L}}

\newcommand{\ppt}[1]{\P_{tr}\left[{#1}\right]}
\newcommand{\ppte}[1]{\P_{te}\left[{#1}\right]}

\DeclareMathOperator*{\E}{\mathbb{E}}

\DeclareMathOperator*{\Eh}{\hat{\mathbb{E}}}

\DeclareMathOperator*{\V}{\mathbb{V}}
\DeclareMathOperator*{\Vh}{\hat{\mathbb{V}}}

\DeclareMathOperator*{\Ec}{{\color{lightseagreen}\mathbb{E}}}

\DeclareMathOperator*{\Ehc}{{\color{ballblue}\hat{\mathbb{E}}}}

\DeclareMathOperator*{\Vc}{{\color{lightseagreen}\mathbb{V}}}
\DeclareMathOperator*{\Vhc}{{\color{ballblue}\hat{\mathbb{V}}}}

\DeclareMathOperator*{\argmax}{argmax}
\DeclareMathOperator*{\argmin}{argmin}

\usepackage[textsize=tiny]{todonotes}

\usepackage[framemethod=tikz]{mdframed}

\newenvironment{textbo}{}{}

\begin{document}
%
\title{DirMixE: Harnessing Test Agnostic Long-tail Recognition with Hierarchical Label Variations}

\author{Zhiyong~Yang,
        Qianqian~Xu*,~\IEEEmembership{Senior ~Member,~IEEE,}
        Sicong Li,
        Zitai Wang,
        \\
        Xiaochun~Cao,~\IEEEmembership{Senior ~Member,~IEEE,}
        and~Qingming~Huang*,~\IEEEmembership{Fellow,~IEEE}
        \IEEEcompsocitemizethanks{
        
            \IEEEcompsocthanksitem *: corresponding authors
            \IEEEcompsocthanksitem This work was supported in part by the National Natural Science Foundation of China: 62525212 and U23B2051, in part by the Fundamental Research Funds for the Central Universities Grant No. E4EQ1101, in part by the National Natural Science Foundation of China: 62236008, 62441232, 62025604, 62521007, U21B2038, 62576332, 92370102, and 62502500, in part by the Youth Innovation Promotion Association CAS, in part by the Strategic Priority Research Program of the Chinese Academy of Sciences, Grant No. XDB0680201, in part by the project ZR2025ZD01 supported by Shandong Provincial Natural Science Foundation, in part by the China National Postdoctoral Program for Innovative Talents under Grant BX20240384, in part by Beijing Natural Science Foundation under Grant No. L252144, and in part by the General Program of the Chinese Postdoctoral Science Foundation under Grant No. 2025M771558.
            \IEEEcompsocthanksitem Zhiyong Yang is with School of Computer Science and Technology, University of Chinese Academy of Sciences, Beijing 100049, China (email: \texttt{ yangzhiyong21@ucas.ac.cn}).
            \IEEEcompsocthanksitem Qianqian Xu is with the State Key Laboratory of AI Safety, Institute of Computing Technology, Chinese Academy of Sciences, Beijing 100190, China, (email: \texttt{xuqianqian@ict.ac.cn}). 
            \IEEEcompsocthanksitem Sicong Li is with Institute of Information Engineering, Chinese Academy of Sciences, Beijing 100093, China, and also with School of Cyber Security, University of Chinese Academy of Sciences, Beijing 100049, China (email: \texttt{lisicong24@mails.ucas.ac.cn}).
            \IEEEcompsocthanksitem Zitai Wang is with the State Key Laboratory of AI Safety, Institute of Computing Technology, Chinese Academy of Sciences, Beijing 100190, China, (email: \texttt{wangzitai@ict.ac.cn}).
            \IEEEcompsocthanksitem Xiaochun Cao is with School of Cyber Science and Technology, Shenzhen Campus, Sun Yat-sen University, Shenzhen 518107, China, (email:\texttt{caoxiaochun@mail.sysu.edu.cn}).
            \IEEEcompsocthanksitem Qingming Huang is with the School of Computer Science and Technology, University of Chinese Academy of Sciences, Beijing 101408, China, also with the Key Laboratory of Big Data Mining and Knowledge Management (BDKM), University of Chinese Academy of Sciences, Beijing 101408, China, also with the State Key Laboratory of AI Safety, Institute of Computing Technology, Chinese Academy of Sciences, Beijing 100190, China, and also with Peng Cheng Laboratory, Shenzhen 518055, China (e-mail: \texttt{qmhuang@ucas.ac.cn}).
            \IEEEcompsocthanksitem 
        }
}

\markboth{Submitted to IEEE TRANSACTIONS ON PATTERN ANALYSIS AND MACHINE INTELLIGENCE}%
{Shell \MakeLowercase{\textit{et al.}}: Bare Demo of IEEEtran.cls for Computer Society Journals}
%

\IEEEtitleabstractindextext{%
    \begin{abstract}
    \justifying
    This paper explores test-agnostic long-tail recognition, a challenging long-tail task where the test label distributions are unknown and arbitrarily imbalanced. We argue that the variation in these distributions can be broken down hierarchically into global and local levels. The global ones reflect a broad range of diversity, while the local ones typically arise from milder changes, often focused on a particular neighbor. Traditional methods predominantly use a Mixture-of-Expert (MoE) approach, targeting a few fixed test label distributions that exhibit substantial global variations. However, the local variations are left unconsidered. To address this issue, we propose a new MoE strategy, $\mathsf{DirMixE}$, which assigns experts to different Dirichlet meta-distributions of the label distribution, each targeting a specific aspect of local variations. Additionally, the diversity among these Dirichlet meta-distributions inherently captures global variations. This dual-level approach also leads to a more stable objective function,  allowing us to sample different test distributions better to quantify the mean and variance of performance outcomes. Building on this idea, we develop a general Latent Skill Finetuning (LSF) framework for parameter-efficient finetuning of foundation models. We provide implementations based on LoRA and Adapter. Theoretically, we derive upper bounds on the generalization error for both standard learning and PEFT. Under mild assumptions, we show that the variance-based regularization helps tighten these bounds. Furthermore, we prove that the covering number of the PEFT hypothesis class scales with the number of trainable parameters. Finally, extensive experiments on CIFAR-10-LT, CIFAR-100-LT, ImageNet-LT, and iNaturalist validate the effectiveness of $\mathsf{DirMixE}$.    
\end{abstract}
        
}
  \maketitle
  \IEEEdisplaynontitleabstractindextext
  \IEEEpeerreviewmaketitle

\section{Introduction}
Traditional machine learning (ML) methods are usually designed for data with a balanced distribution \cite{krizhevsky2009learning,DBLP:conf/cvpr/DengDSLL009}. In contrast, most real-world data exhibits a long-tailed nature. Specifically, a few classes have a large number of samples (\textit{i.e.}, head classes), while many others have far fewer (\textit{i.e.}, tail classes) \cite{DBLP:journals/pami/ZhangKHYF23}. This pattern produces profound studies in many tasks, including species classification \cite{DBLP:conf/cvpr/HornASCSSAPB18,DBLP:journals/natmi/MiaoLGPYG21}, face recognition \cite{DBLP:conf/iccv/ZhangFWLQ17,DBLP:conf/cvpr/ZhongDWHPTH19}, medical image diagnosis \cite{DBLP:conf/miccai/GaldranCB21}, social image understanding  \cite{tang2016tri,li2018deep}, semantic segmentation \cite{zhong2023understanding,NEURIPS2020_07211688}, and object detection \cite{DBLP:conf/iccv/LinGGHD17}, recommendation system \cite{DBLP:conf/ausai/YeL23, ye2025towards, DBLP:journals/tist/YeL25}.

Hitherto, considerable methods have been developed to address this issue seeing its widespread nature. To evaluate model performance with long-tail distributions, most existing studies train their models on long-tailed data and test them on a balanced dataset \cite{DBLP:conf/cvpr/CuiJLSB19,DBLP:conf/nips/CaoWGAM19,DBLP:conf/iclr/KangXRYGFK20,DBLP:conf/iclr/MenonJRJVK21,kini2021labelimbalanced,DBLP:conf/iccv/CuiZ00J21,DBLP:conf/nips/RangwaniAMR22,DBLP:conf/cvpr/AlshammariWRK22,ddc,DBLP:conf/icml/LiXYWZCH25}. This protocol poses an implicit assumption: the test distribution is balanced and fixed. 
However, this assumption is often invalid, as the test distribution is typically unknown and can change over time.

\begin{figure}[t]  
    \centering
    \includegraphics[width=0.95\columnwidth]{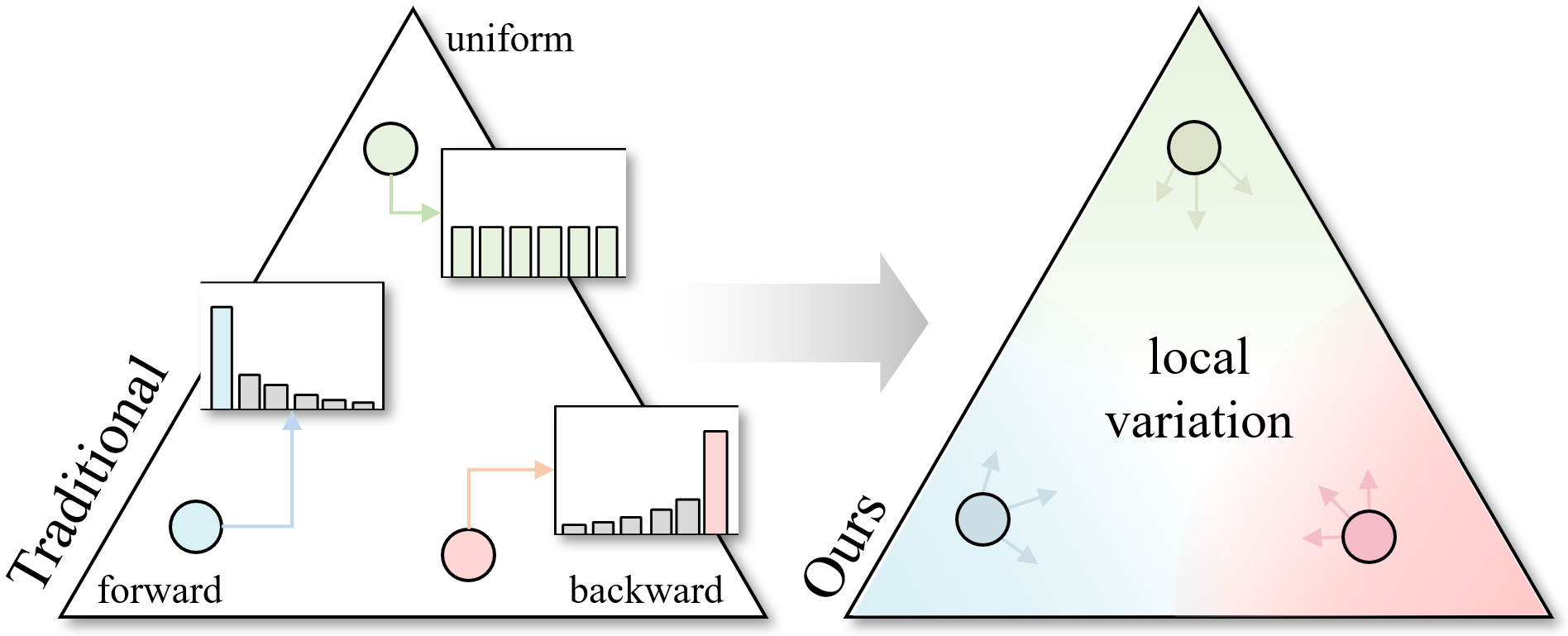} 
    \caption{\label{fig:frame}Traditional methods can only capture \textit{global} variations of label distribution. By contrast, our $\mathsf{DirMixE}$ learns from both \textit{global} and \textit{local} variations, covering more test distributions.}
\end{figure}
To reduce the bias in test distribution, \cite{DBLP:conf/nips/ZhangHHF22} recently proposed a more practical setting called test-agnostic long-tailed recognition. This setting aims at ensuring model efficacy across \textit{heterogeneous} test distributions, each exhibiting varying levels of imbalance (ranging from long-tail and nearly uniform to inversely long-tail distributions). To address this challenge, they propose a novel mixture of experts approach called SADE, crafting distinct experts for each type of distribution. A typical feature of SADE is that each expert therein is assigned to \textit{a fixed} label distribution. We argue that the random variation of the unseen test label distribution could be decomposed into {\textbf{global}} and {\textbf{local}} variations. The global variations capture the heterogeneous diversity across different distributions (say, the significant gap between long-tail and inverse long-tail distribution). The local variations capture the milder changes mostly concentrated within a neighborhood. Back to SADE, the fixed assignment scheme could capture global variations by choosing sufficiently diverse target distributions for different experts. However, the local variations of label distributions can hardly be discovered by fixed points. In this sense, how to comprehensively utilize the hierarchical random variations remains an open problem. To address this issue, we propose a new mixture-of-experts strategy with the following contributions.

First, we explore how to better model the randomness of  the unknown test label distributions in a hierarchical manner. To this end, we propose a probabilistic framework where the label distributions are presumed to be sampled from a meta-distribution. To capture the \textbf{global} variations, we formulate the meta-distribution as a mixture form of heterogeneous distributions. Herein, each component is a Dirichlet distribution capturing the \textbf{local} variations of a specific factor disentangled from the meta-distribution. Furthermore, we allocate a specific expert for each component of the meta-distribution to ensemble expert models enjoying diverse skills. The overall effect is illustrated in Fig.\ref{fig:frame}.
  
Secondly, we take a step further to explore how to evaluate the \textit{overall performance on the meta-distribution} and learn from it. We argue that the average performance is not stable enough for the complex random variations of the test label distributions. In this sense, we develop a Monte Carlo method for estimating the mean and variance of the averaged training loss obtained across different label distributions. From the optimization perspective, minimizing the variance term might prevent the model from performing better than the average level. We thus construct a semi-variance regularization as a surrogate for the empirical variance, selectively penalizing performances below the average level. The attained estimation forms the basis for our objective function in training the expert ensemble.

\newcont{On top of this, to support long-tail learning in foundation models, we propose Latent Skill Finetuning (LSF), a PEFT-based approach for test-agnostic recognition. Each expert adapts to a meta-distribution using PEFT and shares a latent skill space. Assigning proper skills across tasks helps improve performance.}

{We also study the \textit{theoretical generalization ability} of our method. Using variance-based regularization, we derive an upper bound on generalization error based on covering numbers. \newcont{We further assume that the trained models perform well over induced subclasses, which lets us apply conditional concentration inequalities to tighten the bound.} We show that Monte Carlo sampling achieves better generalization than fixed test distributions. \newcont{We implement the upper bound for LSF by leveraging a local Taylor expansion and the geometric properties of the low-rank manifold to improve the bound on the covering number.}}

Finally, we conduct extensive experiments on CIFAR-10, CIFAR-100, ImageNet, and iNaturalist, using both standard deep models and foundation models. Results demonstrate the effectiveness of our proposed methods.

\newcont{This paper extends our ICML 2024 paper \cite{oursnew}. The new version includes several major improvements:

\begin{enumerate}
  \item \textbf{New method for foundation models.} In Sec.\ref{sec:found}, we introduce LSF to improve test-agnostic performance via PEFT. Experts adapt to different meta-distributions by assigning proper latent skills. We implement this with both LoRA and Adapter.
  \item \textbf{New theoretical results.} In Sec.\ref{sec:theo_gen}, we improve our generalization bounds by introducing a prior over well-trained models, leading to induced subclasses. This tighter bound is derived in App.\tb{F.3}-\tb{F.4}. In Sec.\ref{sec:theo_found}, we analyze LSF using a sharp upper bound on covering number, dependent only on trainable parameters, based on Taylor expansion and low-rank geometry (see App.\tb{G}).
  \item \textbf{New experiments.} We evaluate LSF on CIFAR-10-LT, CIFAR-100-LT, ImageNet-LT, and iNaturalist. Sec.\ref{sec:imp} describes implementation details. Results and comparisons are in Sec.\ref{sec:compete}, Sec.\ref{sec:res}, and Sec.\ref{sec:foundexp}. Due to space, full results appear in App.\tb{N.3} and App.\tb{N.4}. An ablation study of our SVD-based initialization is in App.\tb{N.9}.
\end{enumerate}}


\section{Related Work}
This paper will provide a brief overview of several areas closely related to our study.

\textbf{Loss Modification.} The main problem with traditional methods is their poor performance on tail classes. A direct solution is modifying the loss function to enhance performance for these classes. Increasing the weights of tail-class losses is a common tactic \cite{DBLP:conf/icml/MorikBJ99,DBLP:conf/cvpr/CuiJLSB19} along this course. While seemingly reasonable, it can lead to unstable optimization \cite{DBLP:conf/cvpr/CuiJLSB19,DBLP:conf/nips/CaoWGAM19,ddc} and ruin the overall performance. \cite{DBLP:conf/nips/CaoWGAM19} suggests using weighted terms mainly in the later stages of training. To further address the optimization issues caused by unequal weights, another line of research turns to adjust the logits through class-specific operations.
For instance, \cite{DBLP:conf/cvpr/TanWLLOYY20} observes that positive samples of one class can be negative for others, causing discouraging gradients for tail samples. To address this issue they proposed the equalized loss, adding class-dependent operations to ignore these gradients. LDAM imposes a larger margin penalty on tail-class logits for stronger regularization \cite{DBLP:conf/nips/CaoWGAM19}. The LA loss \cite{DBLP:conf/iclr/MenonJRJVK21} uses additive adjustments on the logits to maintain the consistency of balanced error. Later, \cite{DBLP:conf/nips/KiniPOT21} combines additive and multiplicative terms \cite{DBLP:journals/corr/abs-2001-01385} to form a unified approach called the VS loss. Most recently, \cite{ddc} provides a detailed generalization analysis of these loss-modification methods using a local contraction lemma.


\noindent \textbf{Experts Ensembling.} In our paper, we address the long-tail issue using a mixture of experts strategy. This approach is a specific form of ensemble learning that combines multiple models to enhance overall model generalization. Early long-tail ensemble learning methods \cite{DBLP:conf/cvpr/ZhouCWC20,DBLP:conf/cvpr/0002021} use two network branches to handle long-tail and uniform distributions separately. During training, these branches' outputs are dynamically merged, shifting the focus progressively from head to tail classes. \cite{DBLP:conf/cvpr/LiWKTWLF20} finds that balanced datasets often outperform long-tail ones. As a result, they split the long-tail dataset into several more balanced subsets and designed networks with multiple branches. Since model diversity is crucial in ensemble learning, recent developments have turned their focus to improving experts with different skills. For instance, in ACE \cite{DBLP:conf/iccv/Cai0H21}, three experts are created to learn from subsets containing all classes, middle+tail classes, and tail classes, respectively. RIDE \cite{DBLP:conf/iclr/WangLM0Y21} uses a KL-divergence based loss to encourage diversity among experts. SADE \cite{DBLP:conf/nips/ZhangHHF22} introduces a concept called test-agnostic long-tail recognition, which demands generalization across various test distributions. This is achieved by combining diverse experts trained with different logit adjustments using self-supervision. Finally, BalPoE \cite{DBLP:conf/cvpr/AimarJFK23} suggests a balanced product of experts approach, blending multiple logit-adjusted experts using a product rule to suit the chosen test distribution.


In this paper, we construct a distributional mixture of experts model for the test-agnostic long-tail recognition task. Specifically, instead of using fixed test distributions, we adopt a Dirichlet mixture distribution to simultaneously model the global and local variations of the unseen label distribution. Moreover, our method also enjoys a sharp generalization bound.

\section{Problem Formulation}
In this study, we focus on the challenge of training a good model in the presence of a \textbf{long-tail} data distribution. We define our training data as $\mathcal{S} = \left\{ \x_i, y_i \right\}_{i=1}^N$, sampled from distribution $\mathcal{D}$, with a total of $N$ instances. Here, $\x_i \in \mathcal{X} \subseteq \mathbb{R}^d$ represents the raw input feature of dimensionality $d$ for instance $i$, and $y_i \in {1, 2, \cdots, C}$ denotes the corresponding label in a $C$-class classification problem. Under this setting, if we decompose the joint data distribution as $\d = \ppt{\x|y} \cdot \ppt{y}$, then the label distribution is often highly skewed in the sense that  $\max_i{\ppt{y=i}}/\min_{j}{\ppt{y=j}} >> 1$. This imbalance necessitates the design of specialized learning strategies to ensure adequate performance in the tail classes.

 \textbf{Test Agnostic Long-tail Recognition.}  Contrary to the assumption of a uniform test label distribution, we adopt the test agnostic long-tail recognition  in \cite{DBLP:conf/nips/ZhangHHF22}. This approach evaluates model performance across \textbf{multiple} test sets, each subjects to a different label distribution. From a distributional standpoint, we can frame this as a {label shift problem}, where i) $\mathbb{P}_{tr}\left[ \x|y \right] = \mathbb{P}_{te}\left[ \x|y \right]$ and ii) $\mathbb{P}_{tr}\left[ y \right] \neq \mathbb{P}_{te}\left[ y \right]$ generally. To model the uncertainty of the test label distributions, we assume that the labels are sampled in the following hierarchical process:
\begin{enumerate}
    \item[1)] Due to label shift problem, the test label distribution $\ppte{y}$ should be more than just a constant function.  In this paper, we assume that $\ppte{y}$ is rather sampled from a \textbf{meta-distribution} $\mathcal{E}$ over the simplex $\Delta^C$, ~\textit{i.e.},~ $\P_{te}\left[ y \right] \sim \mathcal{E}$. Moreover, the selected meta-distribution should: i) reflect the diversity of the label distributions with significant different degrees of imbalance (global variations); ii) reflect the local variations of the label distributions to make more stable predictions (local variations). 
    \item[2)] The test data is then sampled from the observed test distribution: $\pair \sim \ppte{\x,y} = \ppt{\x|y} \cdot \ppte{y}$.        
\end{enumerate}

\textbf{Goal.} We aim to develop a well-trained model that demonstrates good performance across the entire spectrum of test label distributions within $\mathcal{E}$. Moreover, if we denote the performance of a model $f$ on test data $\mathcal{D}_{te}$ by a loss function $\ell$, the goal can be expressed as:
\begin{align*}
\E_{\pair \sim \ppte{\x,y}} \ell\left(y,f(\x)\right)~\text{is small for most}~ \P_{te}\left[ y \right] \sim \mathcal{E}.
\end{align*}
\section{DirMixE for traditional Deep Learning Models}
\subsection{Preliminaries}
To initiate our discussion, we revisit the approach to addressing the label shift problem in scenarios where the test label distribution $\P_{te}\left[ y \right]$ is fixed but differs from the training distribution.

Following the convention of the classification problem, we want to learn a logit function $\{f_{\theta,1},\cdots,f_{\theta, C}\}:\mathbb{R}^d\rightarrow \Delta^C$ by minimizing the CE loss over the training data:
\begin{align*}
    \ell_{CE}(f_\theta(\x),y) = -\log\left(\s\left(f_{\theta,y}(\x)\right)\right).
\end{align*}
With the adjusted logits \cite{DBLP:conf/iclr/MenonJRJVK21}, if the softmax function is calibrated,  we have $\s(f_{\theta,y}(\x)) \propto \P_{tr}\left[ y|\x  \right]$. We can then achieve the Fisher consistency in the sense that:
\begin{align*}
    \hat{y} = \argmax_{i \in \{1,\cdots,C\}} f_{\theta,i}(\x) \approx \argmax_{i \in \{1,\cdots,C\}} \P_{tr}\left[ y=i|\x  \right]
\end{align*}
However, due to the label shift problem, the prediction rule suffers a distributional bias since $\ppt{y|\x} \neq \ppte{y|\x}$. To address this issue, \cite{lade} proposes to employ the logit adjustment method \cite{DBLP:conf/iclr/MenonJRJVK21} to modify the loss function. Specifically, one can use \textbf{the following loss function denoted as $\ell_{\mathsf{LA}}\left(f_\theta(\x),y; \P_{te}\right)$}:
\begin{align}\label{eq:la}
 \ell_{CE}\left(f_{\theta,y}(\x)-  \log\left(\frac{\ppte{y}}{\ppt{y}}\right)\right).
\end{align}
To see how the new logarithm factor eliminates the bias, we can again check the probability correspondence when calibration of the softmax score is achieved:
\begin{align*}
    \s(f_{\theta,y}(\x))  \propto \P_{tr}\left[ y|\x \right]\cdot \frac{\ppte{y}}{\ppt{y}} \propto \P_{te}\left[ y|\x \right]
\end{align*} 
This implies that when the model is effectively trained, it is possible to re-establish the Bayes rule.

\subsection{\ours:~Learning with a Dirichlet Mixture of the Experts}

We are ready to formulate our method, called \ours, for unknown label distributions $\P_{te}\left[ y \right]$. Seeing the complicated random variation, \textit{minimizing the mean error is not enough for a stable prediction}. In this sense, we aim to optimize the model performance from a distributional perspective. In other words, we want to ensure the average error is low for most possible choices of $\P_{te}\left[ y \right]$. To do this, we control the average error's first (mean) and second (variance) moment to be small. Mathematically, we try to minimize the following objective function:
\begin{align}\label{eq:obj}
    \E_{ \P_{te} \sim \mathcal{E}}\left[ \R(f;\P_{te})\right] +  \lambda \cdot   \V_{ \P_{te} \sim \mathcal{E}}\left[ \R(f;\P_{te})\right]
\end{align}
where $\R(f;\P_{te}) = \E_{\pair \in \P_{tr}} \ell_{\mathsf{LA}}(f_\theta(\x),y; \P_{te})$ is the expected logit adjustment loss (see Eq.\ref{eq:la}) given the label distribution $\P_{te}$, and $\lambda$ is a trade-off coefficient. 

It is hard to simultaneously minimize all the training losses over different distributions by learning a single model. To address this issue, we propose an ensemble learning strategy to deal with the heterogeneity among test distributions with different experts. We will elucidate the technical aspects of our approach by addressing three key questions: 1) How to model the distribution $\mathcal{E}$ to capture the test distribution of interest effectively, 2) How to assign different experts to distinct test distributions, and 3) How to approximate the objective function using sampling methods and empirical data.  

\textbf{Targeted $\mathcal{E}$.} The target $\mathcal{E}$ is designed to capture the global and local variations of the label distribution, where  we resort to the hierarchical structure of \textbf{mixture distributions}. Specifically, we define $K$ as the number of components in this mixture distribution, using a discrete distribution to select a component $i$ with probability $p_i$. Given the commonality of the Dirichlet distribution for sampling discrete probabilities (label distributions are discrete probabilities ), we define the $i$-th component itself as a Dirichlet distribution with parameters $\bm{\alpha}^{(i)} = [ \alpha^{(i)}_1,\cdots, \alpha^{(i)}_C ]$. This mixture distribution framework facilitates the hierarchical sampling of $\P_{te}$, offering a structured approach to model the variation of test label distributions:
    \begin{align}
        &\P_{te}|i ~\sim~ \mathsf{Dir}\left({\alpha}^{(i)}_1,\cdots, {\alpha}^{(i)}_C\right) \label{eq:p}\\ 
        & i|\bm{p}    ~\sim~ \mathsf{Discrete}\left(p_1,\cdots,p_K\right)\label{eq:i}
    \end{align}
In the meta-distributions, the diversity across different distributions captures the global variations. Moreover, within each component, the Dirichlet distribution makes sure that the local variations concentrated on the mean can be effectively characterized. In this sense, we can better utilize the hierarchy of the variations by virtue of the meta-distribution.

    Note that prior research \cite{DBLP:conf/cvpr/AimarJFK23,lade,DBLP:conf/nips/ZhangHHF22} can be regarded as a special case of the proposed approach, wherein the Dirichlet distribution is substituted with a fixed point (Dirac delta distribution). However, as we will demonstrate in the subsequent section, this approach limits the generalization capacity of the resulting model. This limitation arises because such fixed distribution schemes inadequately represent the randomness inherent in the test label distributions.
    \begin{figure}[t]  
        \centering
        \includegraphics[width=0.95\columnwidth]{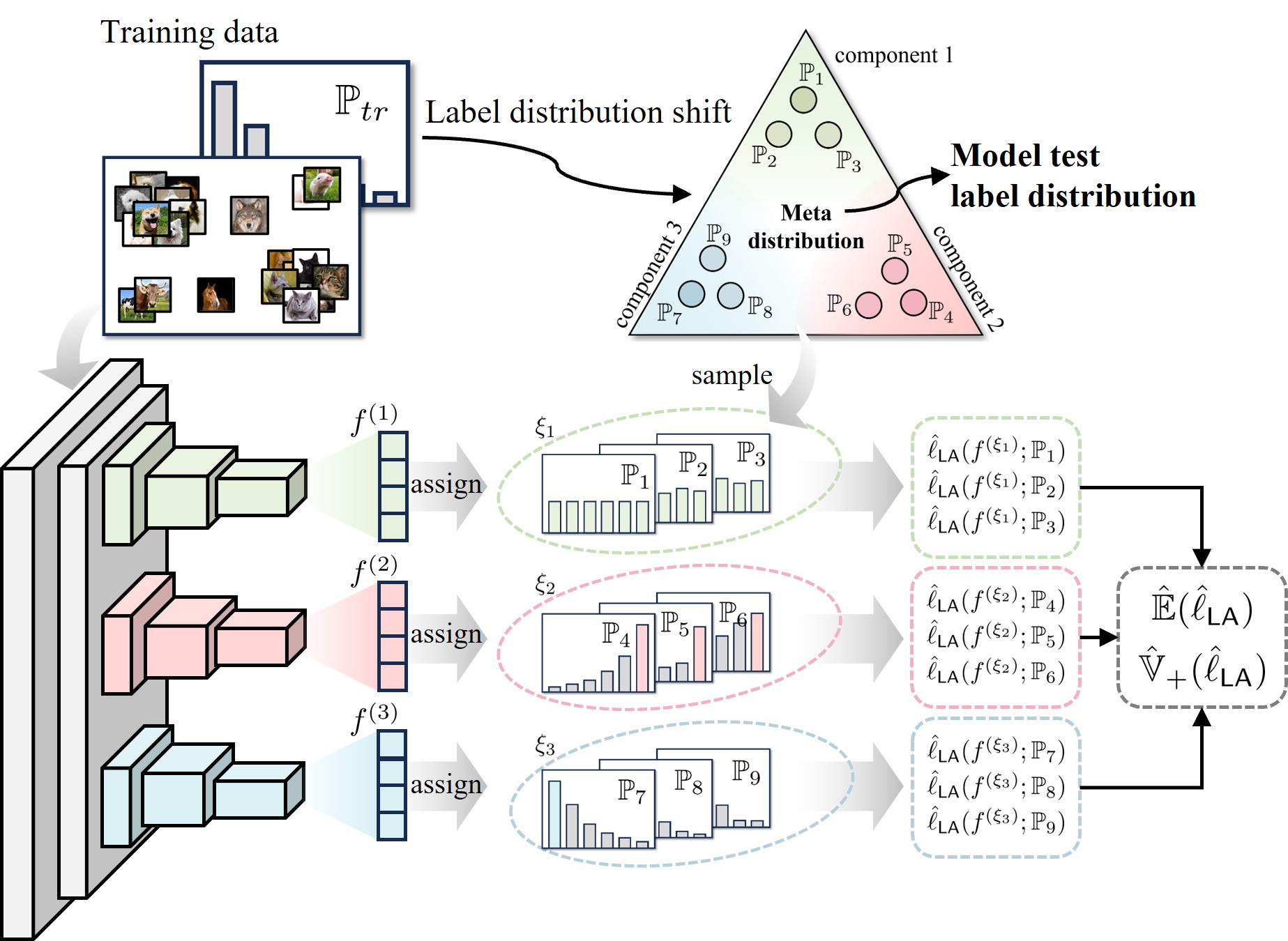} 
        \caption{\label{fig:train}Illustration of the training process, where a hierarchical sampling process is employed to esitmate the empirical risk.}
    \end{figure}
\textbf{Mixture of Experts Strategy.} The overall model is expressed in a compositional structure 
\begin{align*}
    f^{(i)}_\theta(\cdot) = g^{(i)} \circ \psi(\cdot),~ i = 1,2,\cdots,K,
\end{align*}
where $f^{(i)}$ is the model allocated for the $i$-th component, $g^{(i)}$ reflects the distributional-specific feature, $\psi$ reflects the invariant part and is shared across components. 

In the upcoming discussions, we will see how $\mathcal{E}$ and the MoE strategy will help us construct an estimation of the loss distribution. 

\textbf{Empirical Approximation of the Objective Function.} To approximate the population-level objective function (Eq.\ref{eq:obj}), we utilize a Monte Carlo method to estimate the expectation and variance over $\mathcal{E}$. This is achieved by generating empirical data for test label distributions. Each time, we sample a test label distribution ${\P}$ and a component $\xi$ from the hierarchical model as outlined in equations \eqref{eq:p}-\eqref{eq:i}. The generated data set is represented as $\mathcal{P} = \{{\P}_j, \xi_j\}_{j=1}^M$. In this process, each expert $f_\theta^{(\xi_j)}$ is assigned to the sampled pair $({\P}_j, \xi_j)$. The performance of these experts is then evaluated based on the following average loss:
\begin{align*}
    \hat{\ell}_{{\mathsf{LA}}_j} = \frac{1}{N}\sum_{\pair \in \mathcal{D}} \la\left(f_\theta^{(\xi_j)}(\x),y;\P_{j}\right).
\end{align*}
From all of the above, we get the Monte Carlo estimation of the population mean and variance in \eqref{eq:obj} as:
\begin{align*}
    &\hat{\mathbb{E}}(\ell_{\mathsf{LA}}) = \frac{1}{M}\sum_{j=1}^M \hat{\ell}_{{\mathsf{LA}}_j},~\hat{\mathbb{V}}(\ell_{\mathsf{LA}}) = \frac{1}{M}\sum_{j=1}^M \left(\hat{\ell}_{{\mathsf{LA}}_j} -  \hat{\mathbb{E}}(\ell_{\mathsf{LA}})\right)^2. 
\end{align*}
Practically, minimizing the variance term tends to punish $\hat{\ell}_{{\mathsf{LA}}_j}$ being smaller than its empirical mean. This over-regularization can impede the training process, particularly for test distributions that are relatively easier to learn. To circumvent this challenge, we replace the variance penalty with a semi-variance term, which only penalizes $\hat{\ell}_{{\mathsf{LA}}_j}$ that are larger than the mean.
\begin{align*}
    &\hat{\mathbb{V}}_+(\ell_{\mathsf{LA}}) = \frac{1}{M}\sum_{j=1}^M \left(\left(\hat{\ell}_{{\mathsf{LA}}_j} -  \hat{\mathbb{E}}(\ell_{\mathsf{LA}})\right)_+\right)^2. 
\end{align*}
Please see Appendix \tb{A} for the sampling  process in details.

Putting all together, we come to the final objective function:
\begin{align*}
    \min_{g^{(1)},\cdots,g^{(K)},~\psi} ~ \hat{\mathbb{E}}(\ell_{\mathsf{LA}}) + \lambda \cdot \hat{\mathbb{V}}_+(\ell_{\mathsf{LA}}).
\end{align*}
During testing, we adopt the self-supervision strategy as proposed in SADE \cite{DBLP:conf/nips/ZhangHHF22} to learn a set of model averaging weights $\omega^{(1)}_{te}, ~\omega^{(2)}_{te}, ~\cdots, ~\omega^{(K)}_{te}$. The prediction for a test data point is then determined by combining these weights, leading to a weighted average of the outputs from the individual models expressed as follows: 
\begin{align}\label{eq:merge}
    f_{te}(\x)= \sum_{i=1}^K \omega^{(i)}_{te} \cdot f^{(i)}(\x).
\end{align}

\section{\newsec{\ours~ for Finetuning Foundation Models}}\label{sec:found}

In the preceding section, we have examined our framework for traditional deep learning methods. The emergence of foundation models presents a compelling opportunity to advance our understanding through an in-depth study of larger-scale models. To capitalize on this, we will investigate the application of the \ours~ principle to facilitate efficient adaptation in foundation models. More specifically, we aim to enhance \ours~ by developing a parameter-efficient fine-tuning framework on top of our setting.

\subsection{Enhancing DirMixE with Latent Skill Finetuning (LSF)}

We adopt a standard additive model to finetune the pretrained foundation model. Specifically, let the pretrained parameters be denoted as $\bm{\theta}$. We update the model using $\bm{\theta} + \Delta{\bm{\theta}}$, where $\bm{\theta}$ remains frozen and we only learn a simple structured matrix $\Delta{\bm{\theta}}$. In our task, we need to simultaneously update the models $f^{(1)}_{\theta_1}, f^{(2)}_{\theta_2}, \cdots, f^{(C)}_{\theta_C}$ corresponding to each component of the meta-distribution. For convenience, we denote the pretrained parameter of $f^{(i)}_{\theta_i}$ as $\bm{\theta}^{(i)}$, and its corresponding update as $\Delta{\bm{\theta}^{(i)}}$. To achieve parameter-efficient finetuning, $\Delta{\bm{\theta}^{(i)}}$ should be structured to reduce model complexity and the number of parameters that need to be updated.

\textbf{Note.} \textbf{For the sake of simplicity, we only discuss our method for a single layer of the model. All the methods could be easily extended to the real-case in an layer-wise manner. }

For clarity, we treat the learning of each $\Delta{\bm{\theta}^{(i)}}$ as a task $i$. A naive approach is to train each $\Delta{\bm{\theta}^{(i)}}$ independently for each meta-distribution. However, since the meta-distributions differ only in their label distributions, they share substantial common knowledge. Thus, training each task separately is inefficient. To address this, we adopt a latent skill-based knowledge-sharing strategy, inspired by how humans learn similar tasks in parallel. Specifically, we assume all tasks share a set of $n_K$ latent skills, each with parameters $\zeta_1, \zeta_2, \cdots, \zeta_{n_K}$. For each task, an assignment function $ass^{(i)}$ maps these skills to:
\begin{align*}
  \Delta{\bm{\theta}^{(i)}} = ass^{(i)}\left(\zeta_1, \zeta_2, \cdots, \zeta_{n_K} \right)
\end{align*}
In this way, knowledge from latent skills can be shared effectively across tasks via proper assignment functions.

This general framework can be instantiated in various ways depending on the form of $\zeta$ and $ass$. In this paper, we consider two specific implementations based on two popular PEFT frameworks: LoRA and Adapter.

\textbf{LoRA-LSF.} In the original LoRA method, $\Delta{\bm{\theta}}$ is modeled with a low-rank structure:
\begin{align*}
  \Delta{\bm{\theta}} = \bm{A}\bm{B}^\top
\end{align*}
where $ \Delta{\bm{\theta}} \in \mathbb{R}^{m\times n}, \bm{A} \in \mathbb{R}^{m\times r}, \bm{B} \in \mathbb{R}^{n\times r}$. When $r \ll \min\{m, n\}$, LoRA significantly reduces the number of trainable parameters. In our LSF setting, the latent skills are parameterized by different choices of $\bm{A}$ and $\bm{B}$, i.e., $\zeta = \{\bm{A}_{(1)}, \cdots, \bm{A}_{(n_K)}\}$ and $\{\bm{B}_{(1)}, \cdots, \bm{B}_{(n_K)}\}$. We assume that $\bm{A}^{(i)}$ and $\bm{B}^{(i)}$ for each task lie in the linear subspaces spanned by $\{\bm{A}_{(j)}\}_{j=1}^{n_K}$ and $\{\bm{B}_{(j)}\}_{j=1}^{n_K}$, respectively. The assignment function $ass^{(i)}$ then produces:
\begin{align}\label{eq:align}
  \bm{A}^{(i)} = \sum_{j} w^i_{A,j} \cdot \bm{A}_{(j)}, \quad \bm{B}^{(i)} = \sum_{j} w^i_{B,j} \cdot \bm{B}_{(j)}
\end{align}
By learning the weights $w^i_j$, we determine the latent skill composition for each task. The final update is then computed as $\Delta{\bm{\theta}}^{(i)} = \bm{A}^{(i)}\bm{B}^{(i)^\top}$. To ensure the latent skill space is expressive enough, we encourage diversity among $\bm{A}_{(1)}, \cdots, \bm{A}_{(n_K)}$ and $\bm{B}_{(1)}, \cdots, \bm{B}_{(n_K)}$. One simple way is to make them mutually orthogonal, i.e., $tr(\bm{A}_{(i)}^\top \bm{A}_{(j)}) = 0$ for all $i \neq j$, and similarly for $\bm{B}$. This leads to the mutual orthogonal regularization term $\mathcal{R}_{org}$:
\begin{align}\label{eq:orgtho}
  \mathcal{R}_{org}(\bm{A},\bm{B}) = \sum_{i \neq j} tr\left(\bm{A}_{(i)}^\top \bm{A}_{(j)}\right) + tr\left(\bm{B}_{(i)}^\top \bm{B}_{(j)}\right)
\end{align}
With the LoRA-based update rule and regularization, our final objective becomes:
\begin{align*}
  \min_{\{\bm{A}_{(i)}\}_{i=1}^{n_K}, \{\bm{B}_{(i)}\}_{i=1}^{n_K}, w_A, w_B} \mathcal{L}_{DirMixE} + \lambda_2 \cdot \mathcal{R}_{org}(\bm{A}, \bm{B})
\end{align*}
where $\mathcal{L}_{DirMixE}$ is the original loss of DirMixE, updated layer-wise by $\bm{W}_0 + \bm{A}^{(i)}\bm{B}^{(i)^\top}$, and $\lambda_2$ is the trade-off coefficient.

\textbf{Adapter-LSF.} AdaptFormer is another popular PEFT method that replaces the MLP layer in Transformer-based models with an AdaptMLP layer. As shown in the figure, the AdaptMLP layer has two branches: the left branch retains the original MLP, while the right branch introduces a scalable module with a \texttt{down-projection$\rightarrow$RELU$\rightarrow$up-projection} structure to produce an intermediate variable $\tilde{\bm{x}}_\ell$:
\begin{align*}
  \tilde{\x}_\ell = \mathsf{RELU}\left( \mathsf{LN}(\x) \cdot \bm{W}_{\downarrow} \right) \cdot \bm{W}_{\uparrow}
\end{align*}
Here, $\x \in \mathbb{R}^d$ is the output from self-attention, $\bm{W}_{\downarrow} \in \mathbb{R}^{d \times \tilde{d}}$ is the down-projection, and $\bm{W}_{\uparrow} \in \mathbb{R}^{\tilde{d} \times d}$ is the up-projection. To reduce parameters, $\tilde{d}$ is set much smaller than $d$. This allows AdaptFormer to incrementally update the pretrained model as:
\begin{align*}
  \x_\ell = \underbrace{\mathsf{MLP}(\mathsf{LN}(\x)) + \x}_{\text{original}} + \underbrace{s \cdot \tilde{\x}_\ell}_{\text{new}}
\end{align*}
where $s$ is a scaling factor. The updated parameters $\Delta \theta$ thus include only $\bm{W}_{\downarrow}$ and $\bm{W}_{\uparrow}$. Similar to LoRA, for each task $i$, we define $\Delta \theta^{(i)} = \{ \bm{W}^{(i)}_{\downarrow}, \bm{W}^{(i)}_{\uparrow} \}$. The latent skill parameters are then $\zeta = \{\bm{W}_{\downarrow,(j)}\}_{j=1}^{n_K}, \{\bm{W}_{\uparrow,(j)}\}_{j=1}^{n_K}$, and each task-specific parameter lies in the subspace spanned by them:
\begin{align*}
  \bm{W}^{(i)}_{\downarrow} = \sum_{j} w^i_{\downarrow, j} \cdot \bm{W}_{\downarrow,(j)}, \quad \bm{W}^{(i)}_{\uparrow} = \sum_{j} w^i_{\uparrow, j} \cdot \bm{W}_{\uparrow,(j)}
\end{align*}
We also apply the mutual orthogonal regularization $\mathcal{R}_{org}(\bm{W}_{\downarrow}, \bm{W}_{\uparrow})$, defined similarly as in Eq.~\ref{eq:orgtho}. The final objective becomes:
\begin{align*}
  \min_{\{\bm{W}_{\downarrow,(i)}\}_{i=1}^{n_K}, \{\bm{W}_{\uparrow,(i)}\}_{i=1}^{n_K}, w_{\downarrow}, w_{\uparrow}} \mathcal{L}_{DirMixE} + \lambda_2 \cdot \mathcal{R}_{org}(\bm{W}_{\downarrow}, \bm{W}_{\uparrow})
\end{align*}

\subsection{Improving the Initialization Regime for LoRA-LSF}

For LoRA-based latent skill finetuning, we explore an improved initialization scheme to enhance gradient feedback in the early training stage. In the standard LoRA setting, $\bm{A}$ is initialized using the Kaiming scheme, while $\bm{B}$ is initialized to $\bm{0}$. A straightforward way to initialize LoRA-LSF is to apply the same setting to each $\bm{A}_{(i)}$ and $\bm{B}_{(i)}$. 

For a given task (meta-distribution) $t$, let its loss function be $\ell^{(t)}$. Then, the initial gradients are:

\begin{align*}
  \eqnb{\frac{ \partial{\ell^{(t)}} } {\partial\bm{A}_{(i)}}} &= \sum_{j=1}^{n_K} w^t_{A,i} \cdot w^t_{B,j} \cdot \bm{X} \cdot \nabla_{\bm{W}}(\ell^{(t)})\cdot \bm{B}_{(j)} \\ 
  &= \eqnb{\bm{0}}, \\ 
  \frac{ \partial{\ell^{(t)}} } {\partial\bm{B}_{(i)}} &= \sum_{j=1}^{n_K} w^t_{A,j} \cdot w^t_{B,i} \cdot \nabla_{\bm{W}}(\ell^{(t)})^\top \cdot \bm{X} \cdot \bm{A}_{(j)}, \\
  \eqnb{\frac{ \partial{\ell^{(t)}} } {w^t_{A,i}}} &= \sum_{j=1}^{n_K} w^t_{B,j} \cdot tr\left(\nabla_{\bm{W}}(\ell^{(t)})^\top \cdot \bm{A}_{(i)}\bm{B}^\top_{(j)}\right) \\
  &= \eqnb{0}, \\
  \eqnb{\frac{ \partial{\ell^{(t)}} } {w^t_{B,i}}} &= \sum_{j=1}^{n_K} w^t_{A,j} \cdot tr\left(\nabla_{\bm{W}}(\ell^{(t)})^\top \cdot \bm{A}_{(j)}\bm{B}^\top_{(i)}\right) \\
  &= \eqnb{0}.
\end{align*}

\begin{figure*}[t]  
  \centering
  
    \subfigure[Density of Gamma w.r.t $\alpha$ ]{
    
      \includegraphics[width=0.48\columnwidth]{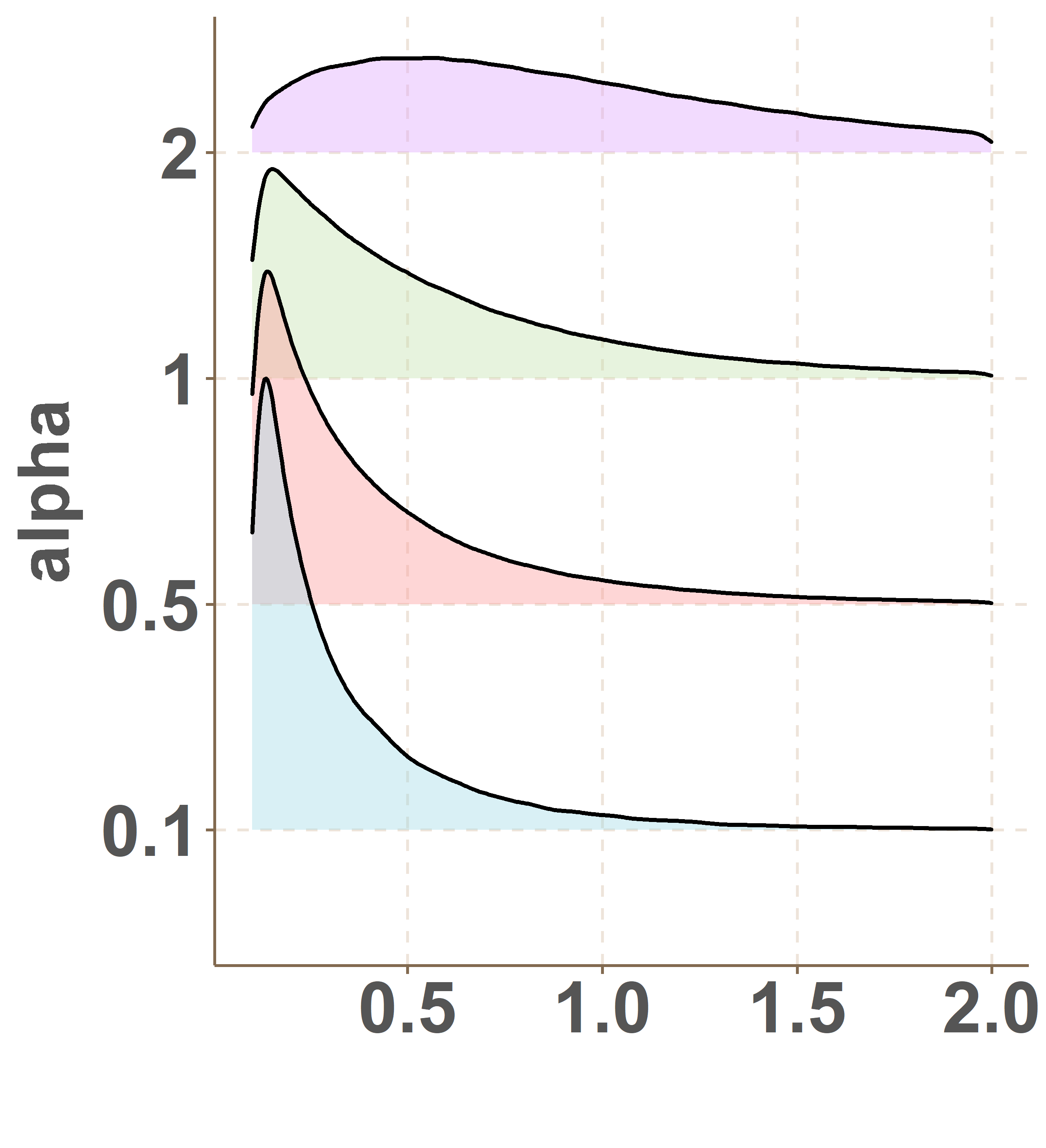} 
    }
    \subfigure[$\rho$ Bound for the Gamma ]{
      \includegraphics[width=0.48\columnwidth]{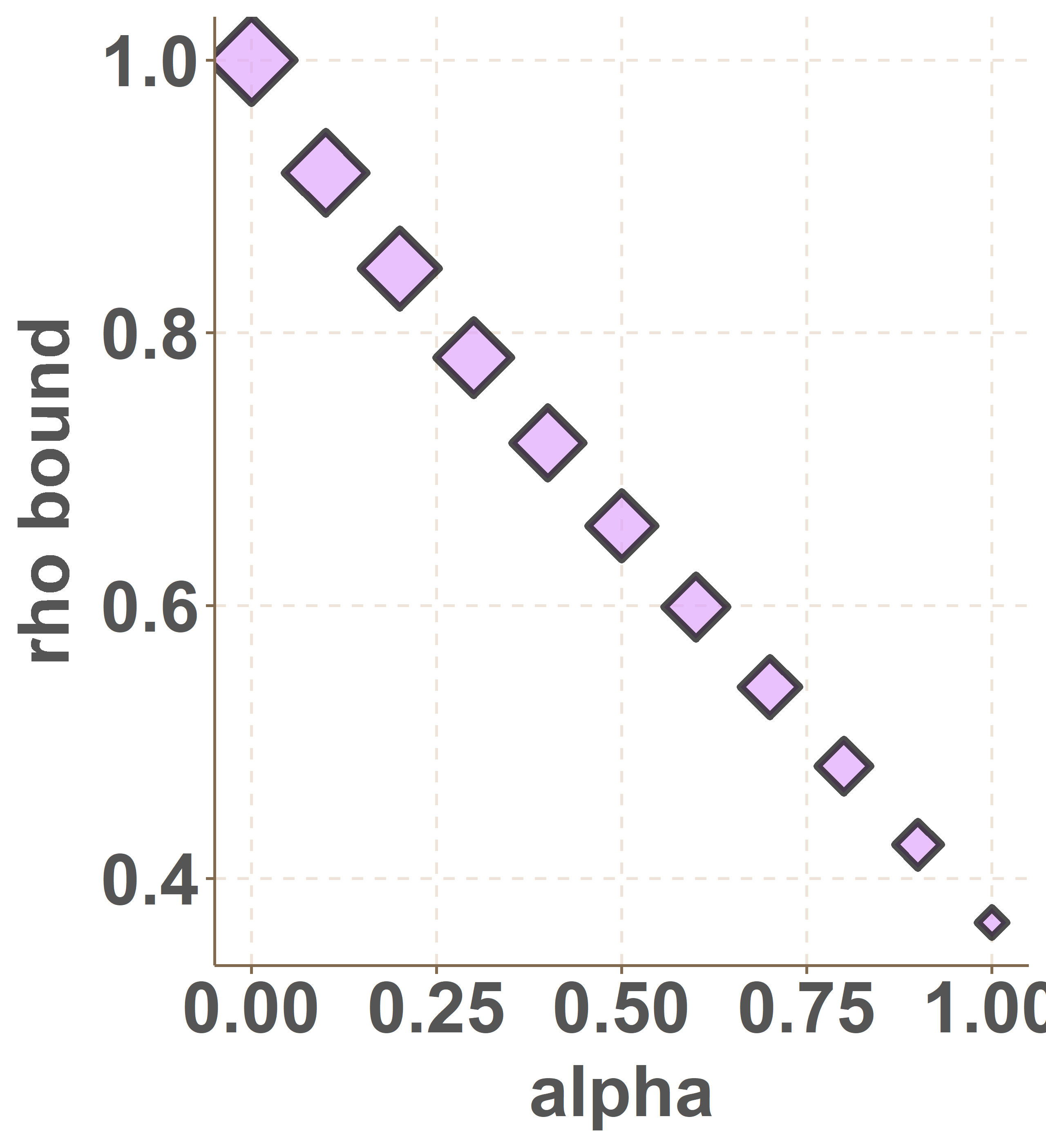} 
    }
    \subfigure[Density of Pareto w.r.t $\theta$ ]{
    
      \includegraphics[width=0.48\columnwidth]{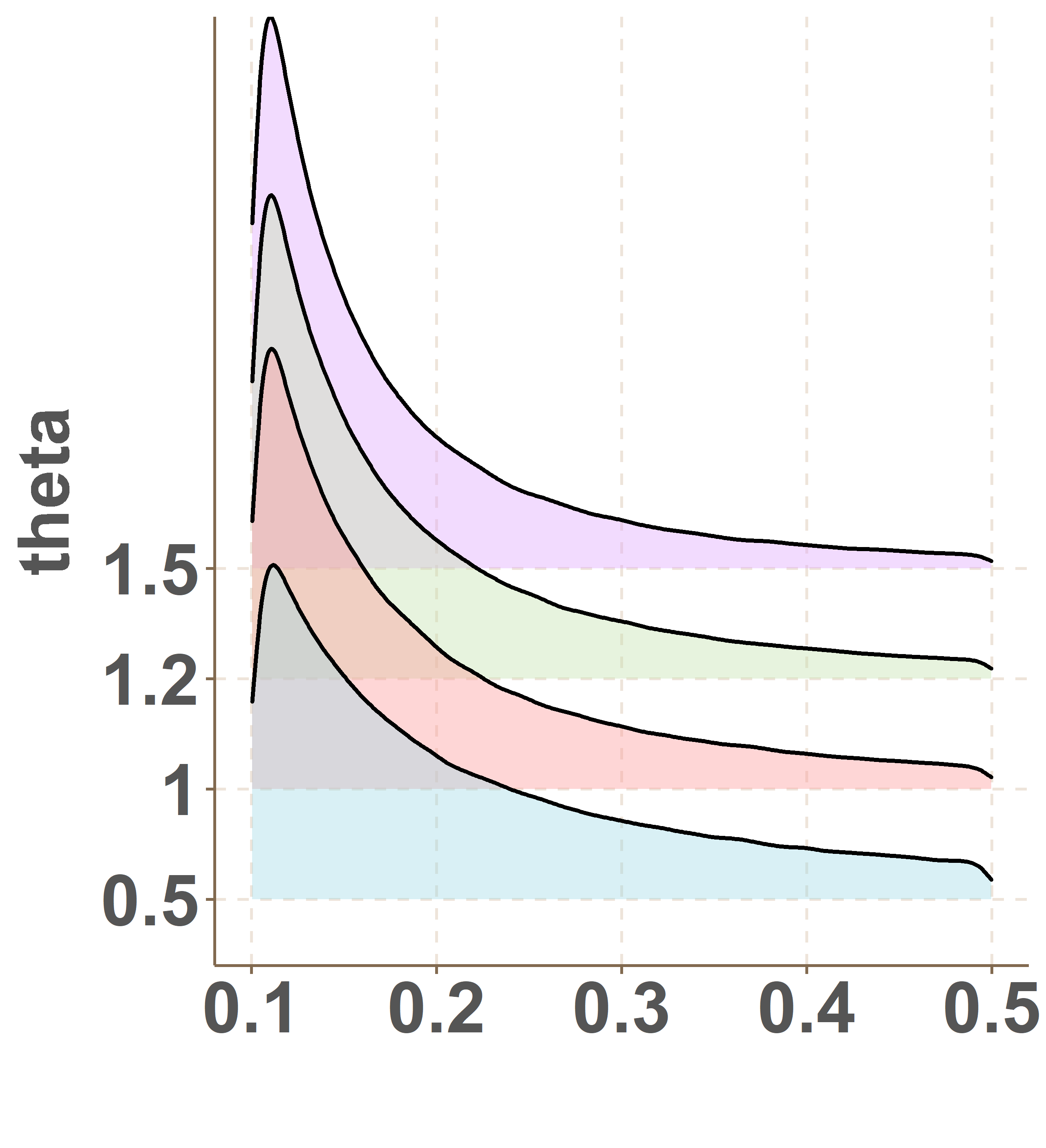} 
    }
    \subfigure[$\rho$ Bound for the Pareto]{
      \includegraphics[width=0.48\columnwidth]{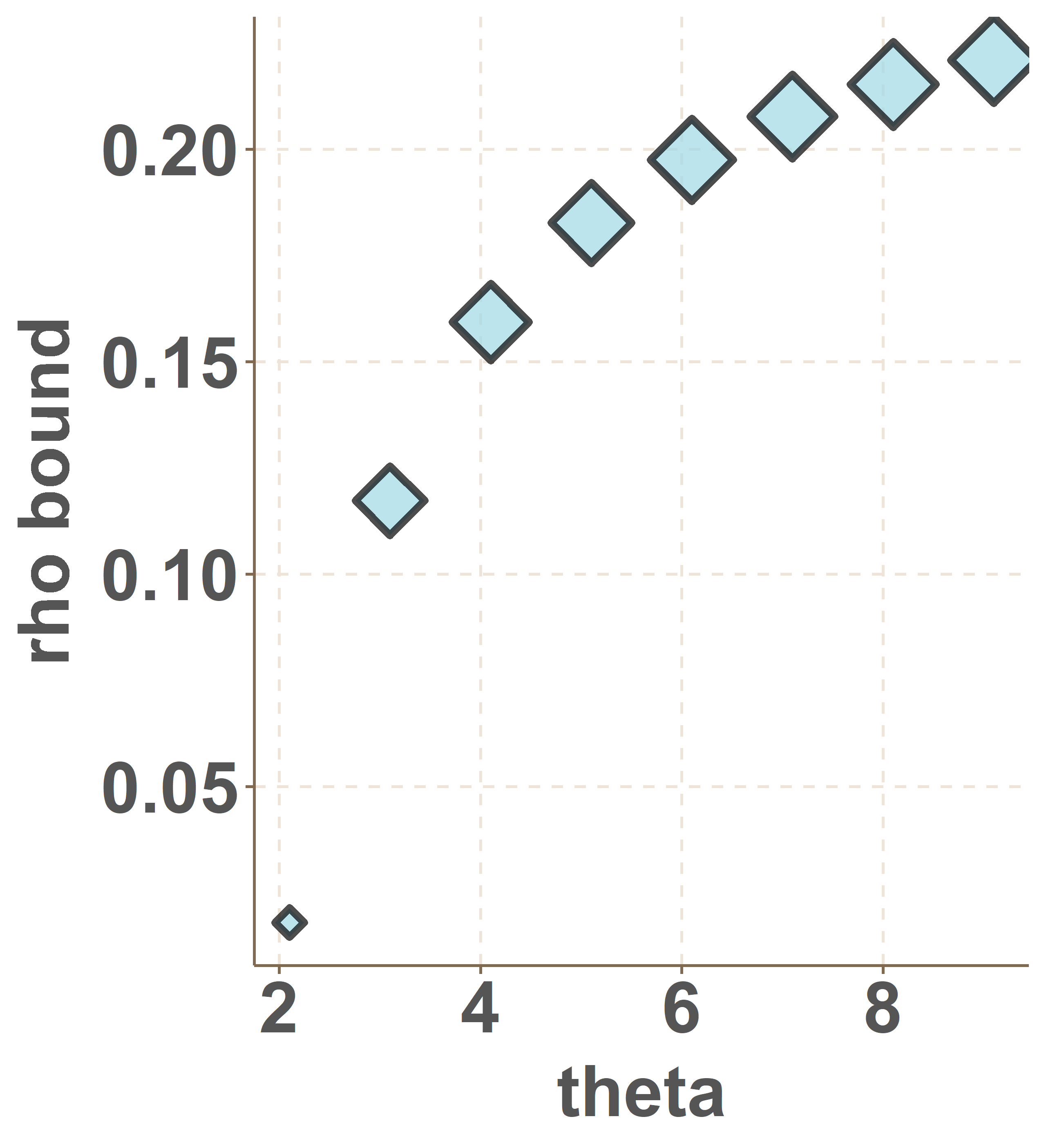} 
    }
    \caption{\label{fig:dist} \textbf{The Upper Bound for $\rho$ for different distributions.} (a) shows the distribution density of the Gamma distribution with $\beta$ fixed as 0.5 and $\alpha$ varied. The corresponding distribution has a tailed shape when $\alpha<1$. (b) shows the upper bound $\rho$ for $\alpha <1$. One can observe that $\rho$ is roughly greater than $0.4$ in this range.  (c) shows the distribution density of the Pareto distribution with $\ell_m$ fixed as $0.1$ and $\theta$ varied. (d) shows the $\rho$ value for $2<\theta \le 10$, we can see that $\rho$ is slightly greater $0.2$ for large $\theta$. In all these observed cases, the assumption $\V_{+} \ee \V$ is admissible.}
  \end{figure*}

Clearly, all gradients except for $\partial \bm{B}_{(i)}$ are initialized to zero. Moreover, the gradients for the assignment weights $w^t_{A,i}$ and $w^t_{B,i}$ are always zero regardless of their initialization. This extremely sparse gradient flow may hinder training in the early stages.

To address this issue, we follow a recent approach \cite{meng2024pissa} that proposes non-trivial initializations using singular value and singular vector reallocation.

Let the singular value decomposition of the pretrained parameter $\bm{W}_0$ be:
\begin{align*}
  \bm{W}_0 = \bm{U} \Sigma V^\top = \sum_{i=1}^{k} \sigma_i \bm{u}_i \bm{v}_i^\top
\end{align*}

Our goal is to reallocate $\sigma_i$, $\bm{u}_i$, and $\bm{v}_i$ to construct the initial parameters $W_{res}$, $\bm{A}^{init}_{(i)}$, and $\bm{B}^{init}_{(i)}$ such that, for each task $t$, we have:
\begin{align*}
  \bm{W}_0 = W_{res} + \bm{A}^{(t,init)} \bm{B}^{(t,init)^\top}
\end{align*}

To do this, we assign singular values to $W_{res}$ and $\bm{A}^{init}_{(i)}, \bm{B}^{init}_{(i)}$. Let $\mathcal{R}_{0}$ denote the index set of singular values (and vectors) used for $W_{res}$, and $\mathcal{R}_{i}$ denote the index set for $\bm{A}^{init}_{(i)}\bm{B}^{init^\top}_{(i)}$, defined as:
\begin{align*}
  &\mathcal{R}_{0} = \{\text{indices of singular values assigned to } \bm{W}_{res}\} \\
  &\mathcal{R}_{i} = \{\text{indices of singular values assigned to } \bm{A}_{(i)}\bm{B}^\top_{(i)}\}
\end{align*}
To make this allocation valid, the index sets must satisfy:
\begin{align*}
  \cup_{i=0}^k \mathcal{R}_{i} = [1:d], \quad \mathcal{R}_{i} \cap \mathcal{R}_{j} = \emptyset
\end{align*}

For notational convenience, for a matrix $\bm{Y}$ and index set $\mathcal{R}$, define $\bm{Y}^{\mathcal{R}} = \bm{Y}[:,\mathcal{R}]$, i.e., the columns of $\bm{Y}$ at indices in $\mathcal{R}$. Based on this, we define the initialization:

\begin{align}\label{eq:res}
  &\bm{W}_{res} = \bm{U}^{\mathcal{R}_{0}} \bm{\Sigma}^{\mathcal{R}_{0}} V^{{\mathcal{R}_{0}}^\top} \tag{initR} \\
  &\bm{A}^{init}_{(i)} = n_K \cdot [\bm{0}, \cdots, \bm{U}^{\mathcal{R}_{i}} \cdot \bm{\Sigma}^{\mathcal{R}_{i}^{1/2}}, \cdots, \bm{0}] \tag{initA} \\
  &\bm{B}^{init}_{(i)} = n_K \cdot [\bm{0}, \cdots, \bm{\Sigma}^{\mathcal{R}_{i}^{1/2}} \cdot \bm{V}^{\mathcal{R}_{i}^\top}, \cdots, \bm{0}]^\top \tag{initB} \\
  &w^k_{A,i} = w^k_{B,i} = \sqrt{\frac{1}{n_K}} \tag{initW}
\end{align}

Under this setting:
\begin{align}\label{eq:orthognal}
  &\bm{A}^{init}_{(i)} \bm{A}^{init^\top}_{(j)} = \bm{A}^{init^\top}_{(i)} \bm{B}^{init^\top}_{(j)} = \bm{B}^{init}_{(i)} \bm{B}^{init^\top}_{(j)} = \bm{0} \\
  &\bm{A}^{init}_{(i)} \bm{B}^{init^\top}_{(i)} = \bm{U}^{\mathcal{R}_{i}} \bm{\Sigma}^{\mathcal{R}_{i}} \bm{V}^{\mathcal{R}_{i}^\top}
\end{align}

Finally, we set:
\begin{align*}
  \bm{A}^{(i)} \bm{B}^{(i)^\top} = \sum_{j} w^i_{A,j} w^i_{B,j} \bm{A}_{(j)} \bm{B}_{(j)}^\top = \bm{U}^{\mathcal{R}_{1:n_K}} \bm{\Sigma}^{\mathcal{R}_{1:n_K}} V^{\mathcal{R}_{1:n_K}^\top}
\end{align*}
Combining this with equations $\eqnb{(initR)}–\eqnb{(initW)}$, we recover:
\begin{align*}
  \bm{W}_{0} = \bm{W}_{res} + \bm{A}^{(i)} \bm{B}^{(i)^\top}
\end{align*}
demonstrating that this initialization can reconstruct the original pretrained parameter. 

Next, we finish this topic with an interesting side-product of this initialization scheme. Specifically Eq.(\ref{eq:orthognal}) automatically ensures that the $\bm{A}^{init}_{(i)} \perp \bm{A}^{init}_{(j)}$ and $\bm{B}^{init}_{(i)} \perp \bm{B}^{init}_{(j)}$. In this way, the latent parameters are ensured to be diversified during the beginning phase.

\section{Theoretical Analysis}

For the sake of simplicity, we will adopt \textbf{asymptotic notations} to perform magnitude comparison. Specifically, $f(n) \lee g(n)$ denotes $\exists C >0, f(n) \le C \cdot g(n)$, and $f(n) \gee g(n)$ \textit{vice versa}. $f(n) \ee g(n)$ means $f(n) \lee g(n)$ and  $f(n) \gee g(n)$  hold simultaneously. All the proofs are deferred to Appendix \tb{D} and \tb{F}.

    \subsection{Upper Bounding Variance with Semi-Variance}\label{sec:semi}
    In \ours, we replace the variance in empirical estimation with the semi-variance of loss to enhance optimization. As shown in Thm.\ref{thm:gen}, the stochastic error associated with this approach remains small if $\V(\la) \ee \V_{+}(\la)$. Given that semi-variance is inherently smaller than variance, our analysis primarily focuses on validating the condition $\V(\la) \lee \V_{+}(\la)$.
    
    Note that, for a well-trained model, there should be a negative correlation between the probability density and the magnitude of loss. Based on this assumption, we show that ${\mathbb{V}}(\ell_{\mathsf{LA}}) \lee {\mathbb{V}}_+(\ell_{\mathsf{LA}})$
    for certain types of distributions satisfying the negative correlation, ranging from light-tail  (such as the exponential distribution) to heavy-tail ones (such as the Pareto distribution) distributions. For clarity and simplicity, we will henceforth use\textbf{ $\ell$ as shorthand for $\la$} in this subsection.
    
    \begin{thm}     \label{thm:rho}
        Let $\rho = \frac{\vpp}{\vv}$, then we have the following results for different distributions:
    
    \begin{enumerate}
        \item[a)] If $\ell$ subjects to an \textbf{exponential} distribution, \textit{i.e} the p.d.f  $p(\ell) \propto \exp(-c\ell)$, then we have: $\rho \ge \exp(-1)$.
        \item[b)] If $\ell$ subjects a \textbf{Gamma} distribution with parameters $\alpha,\beta$, \textit{i.e}, the p.d.f  $p(\ell) \propto \ell^{\alpha -1} \cdot \exp(-\beta \cdot\ell)$, then we have: $\rho \ge 1- \alpha \cdot \frac{\Gamma^\uparrow(\alpha,\alpha)}{\Gamma(\alpha)}$, where $\Gamma^\uparrow$ is the lower incomplete gamma : $\Gamma^\uparrow(s,x) = \frac{1}{\Gamma(\alpha)} \cdot \int_0^x t^{s -1} \cdot \exp(-t) dt  $. 
        \item[c)] If $\ell$ subjects to a \textbf{Pareto} distribution with parameter $\theta$, i.e the p.d.f  $p(\ell) \propto (\ell/\ell_m)^{-\theta}$ for some $\ell_m >0$, then we have: 
        \begin{align*}
            \rho=&   (\theta -1)^2\cdot\left[ 1- \phi(\theta)^{2-\theta} \right] +\\ 
            &\theta\cdot (\theta -2) \cdot \left[ 2\cdot \phi(\theta)^{1-\theta} - \phi(\theta)^{-\theta} -1 \right]
        \end{align*}
        where $\phi(\theta) = {\theta}/({\theta-1})$. 
    \end{enumerate}
    \end{thm}
          
         
    According to the theorem, we present practical observation on the Gamma and Pareto distribution in Fig.\ref{fig:dist}-(a). The gamma distributions are tail-shaped when $\alpha <1$. In this sense, we plot the theoretical lower bound $\rho$ for $\alpha <1$ in Fig.\ref{fig:dist}-(b). The results show that $\V_+/\V > 0.4$ in most cases. Similar Fig.\ref{fig:dist}-(c) shows that the Pareto distributions are generally tail-shaped. We then plot $\rho$ for $2<\theta<10$ in Fig.\ref{fig:dist}-(d). The results show that the  $\V_+/\V$ roughly resides in $[0.1,0.3]$ when $\theta \ge 3$.  Moreover, for the exponential distribution, we have a universal upper bound for $\rho$ as $\exp(-1)$, which is greater than $0.2$. In this sense, the proposed claim $\V_{+} \ee \V$ holds for a wide span of tail-shaped distributions.

    In practical terms, $\P_{te}$ is sampled from a mixture distribution. Consequently, it is probable that the loss function is subject to a mixture distribution rather than a single distribution. The following theorem shows that $\rho$ for a mixture distribution can also be lower bounded by the mean of the respective $\rho_i$ values of its component distributions.

    \begin{thm}     \label{thm:multi}
            Let $\ell$ be sampled from a \textbf{mixture} distribution such that $p(\ell) = \sum_{k=1}^K \omega_k \cdot p_k(\ell)$, where $p_k(\ell)$ is the p.d.f. for the component $k$. Furthermore, for each component, we denote: $\E_{k}[\ell] = \mu_k, \V_k[\ell] = \sigma^2_k$. Under the assumption that 1) $\P\left[ \ell \neq \E[\ell] \right]$ with probability one w.r.t the mixture distribution,
            2) $\vv/\sigma_i^2  \ge \max\left\{  \frac{\vnn_{i,j}}{\vnn_{i }} ,1\right\}$, then $\rho$ can be upper bounded by the average of $\rho_i$ in the sense that:
        \begin{align*}
            \rho \le \sum_{i} \omega_i \cdot \rho_i,
        \end{align*}
        where $\vnn =  \vv- \vpp$;~ $\rho_i,\vnn_i$ are the corresponding versions of $\rho,\vnn$ on the $i$-th component, and
        \begin{align*}
            \vnn_{i,j} = \int_{0}^\infty \left((\ell - \mu_i)_-\right)^2\cdot p_{j}(\ell) \cdot d\ell.
        \end{align*}
    \end{thm}
    
    This implies that controlling $\rho$ is feasible if we can manage each $\rho_i$. Ultimately, this leads to the conclusion that $\v \asymp \vp$ holds under mild conditions.

\subsection{\newsec{Generalization Ability of \ours}}\label{sec:theo_gen}
In this section, we begin exploring the generalization ability of the proposed framework.

\begin{textbo}
The theoretical framework developed in this work addresses several fundamental challenges that are not fully explored by existing studies, particularly in the context of test-agnostic long-tail recognition and the fine-tuning of foundation models.
\begin{enumerate}
\item[a)] Classical generalization theory assumes a fixed data distribution. Recent methods (e.g., SADE) use multiple experts for fixed test distributions but lack a stochastic model for test-label distribution, preventing analysis when distributions change. We model the test distribution as drawn from a meta-distribution $\mathcal{E}$ (Dirichlet mixture). How to derive a uniform generalization bound over hierarchical generalization error is a key challenge.

\item[b)] Traditional worst-case uniform convergence bounds are loose for over-parameterized models, as they ignore that training finds models in good regions. It is crucial to mathematically incorporate such regions into the hypothesis space for sharper bounds.

\item[c)] Variance-penalized objectives yield sharp generalization bounds under fixed test-label distributions. But it is still challenging to extend these to our complex scenario.

\item[d)] PEFT updates only a small parameter subset of foundation models, but traditional complexity theory ties complexity to total parameters. It is necessary to improve the conventional bound to bridge this gap.

\end{enumerate}
\end{textbo}

\textbf{Please See App.\tb{B}, for a detailed discussion.
}

To demonstrate the generalization ability of our proposed method, we aim to answer the question that how well does the proposed training method generalize to unseen label distributions. Specifically, we address this by presenting an upper bound of the expected loss over all test label distributions sampled from $\mathcal{E}'$ (which may differ from $\mathcal{E}$) and $\mathcal{D}$. 

\textbf{The Hierarchy of Stochastic Error.} Recall that we adopt a stratified sampling process for the Dirichlet mixture distribution. For each mixture component $j$, we sample $M$ label distributions:
\begin{align*}
    \mathcal{P} \sim \mathcal{E}^M, \quad \mathcal{P} = \left\{\P_1,\cdots,\P_M \right\}
\end{align*}
This forms an empirical sample of the test label distributions. We denote the corresponding overall estimations as $\ehld, \vhld$. In addition, we have a fixed training dataset:
\begin{align*}
    \mathcal{S} \sim \mathcal{D}^N, \quad \mathcal{S} = \{\x_i, y_i\}_{i=1}^N,
\end{align*}
which serves as an empirical approximation of the training data distribution. We denote empirical (and population) estimations using the subscript $\mathcal{S}$ (and $\mathcal{D}$, respectively). To estimate the excess risk, we analyze the hierarchical stochastic error caused by both label distribution and data sampling. 

In our proof, we use the following intermediate empirical estimations.

The first set relates to a fixed test label distribution $\P_i \in \mathcal{P}$. Specifically, we define $\ell_{\mathcal{D},\mathcal{E},i}$ and $\ell_{\mathcal{S},\mathcal{E},i}$ as the expected loss over $\mathcal{D}$ and the empirical average loss over $\mathcal{S}$, respectively:
\begin{align}\label{eq:quant1}
    &\ell_{\mathcal{D},\mathcal{E},i} = \E_{\mathcal{D}}\left[ \la(f^{(\xi_i)}(\x), y; \P_i) \right] \notag \\
    &\ell_{\mathcal{S},\mathcal{E},i} = \Eh_{\mathcal{S}}\left[ \la(f^{(\xi_i)}(\x), y; \P_i) \right]
\end{align}
We further define $\ehls$ as the empirical average of $\ell_{\mathcal{S},\mathcal{E},i}$ over sampled label distributions in the Monte Carlo process; $\eld$ as the expected loss over the meta-distribution of label distributions; and $\vphls$ as the empirical semi-variance of $\ell_{\mathcal{S},\mathcal{E},i}$ over the sampled distributions. Our goal is to upper-bound the excess risk:
\begin{align*}
  \sup_{f \in \mathcal{H}} \left[\eld - \ehls \right]
\end{align*}
where $\mathcal{H}$ is a predefined hypothesis set. We aim to explore two important properties of this upper bound. First, since we use semi-variance as a regularization term, the bound should decrease as the regularization diminishes. Second, the choice of $\mathcal{H}$ is crucial. It is well known that worst-case uniform convergence leads to loose bounds for deep learning and foundation models. To avoid purely uniform arguments, we introduce a natural prior: our target models are well-trained and outperform worst-case assumptions. Based on this, we consider a refined subset of hypotheses, defined below.

\textbf{Refined Hypotheses.} Given a hypothesis class $\mathcal{F}$, we define the following induced subclass.

\begin{defi}[\textbf{Induced Subclass}]     \label{def: induced subclass}
  For each $f \in \mathcal{F}$, define $\mathcal{A}_f(\rho)$ (or simply $\mathcal{A}(\rho)$) as:
  \begin{align*}
    \mathcal{A}^{\mathcal{I}}(\rho) = \left\{\x: 0 \le \E_{\mathcal{E}}[\ell(f(\x),y,P_i)] \le \ri, \forall{p_i} \right\}
  \end{align*}
  \begin{align*}
    \mathcal{A}^{\mathcal{M}}(\rho) = \left\{P_i: 0 \le \E_{\mathcal{D}}[\ell(f(\x),y,P_i)] \le \rm \right\}
  \end{align*}
Given a fixed $\rho$, the induced subclass $\mathcal{F}_{\mathcal{A}}$ is:
\begin{align*}
  \mathcal{F}_{\mathcal{A}}(\rho,p) = \mathcal{F} \cap \underbrace{ \left\{f: \P\left[ \mathcal{A}^{\mathcal{I}}(\rho)^c \right] \le p^\mathcal{I}, \P\left[ \mathcal{A}^{\mathcal{M}}(\rho)^c \right] \le p^\mathcal{M}\right\} }_{\text{high probability regions}}
\end{align*}
\end{defi}

\begin{textbo}

\begin{rem}

Traditional uniform generalization bounds, such as those derived from VC-dimension or Rademacher complexity, tend to be loose because they evaluate the worst-case behavior across the entire hypothesis space. Def.\ref{def: induced subclass} focuses on restricting attention to the good regions of the hypothesis space. Specifically, $\mathcal{A}^{\mathcal{I}}$ captures sample-level performance stability across test distributions, while $\mathcal{A}^{\mathcal{M}}$ reflects distribution-level consistency in expected risk. By intersecting these high-probability regions, the analysis targets well-optimized models, leading to tighter and more realistic generalization guarantees in practice.

\end{rem}

\end{textbo}


Given the effectiveness of foundation models, it is reasonable to assume that our models are well-trained. Thus, it is natural to focus on subclasses where $\rho^\mathcal{M}, \rho^\mathcal{I}, p^\mathcal{I}, p^\mathcal{M}$ are all small. If for all $f \in \mathcal{F}$ we have $\ell(f(x), y) \le B$, the following theorem provides a tighter generalization bound where $B$ is replaced by $\rho^\mathcal{M}, \rho^\mathcal{I}$ or small residues $\Delta^{\mathcal{M}}, \Delta^{\mathcal{I}}$.

\begin{thm}[\textbf{Generalization Bound for the Induced Subclass}]\label{thm:gen}
Let $\mathcal{E}$ be the true meta-distribution and $\mathcal{P}$ the empirical distribution of label distributions sampled via Monte Carlo. Let the training data $\mathcal{S}$ be sampled i.i.d. from $\mathcal{D}$. Assume that $\ninff{\epsilon}{M} \eqnb{\lesssim} \left(\frac{r}{\epsilon}\right)^\nu$, $\ell(\cdot) \in [0, B]$, and $\vhld \ee \vphld$ as defined in Tab. \tb{7}. Then, for any possible meta-distribution $\mathcal{E}'$ over the probability simplex $\mathbb{S}^{c-1}$, the following bound holds uniformly for all $f \in \FA$, with probability at least $1 - \delta$ over the randomness of $\mathcal{S}, \mathcal{P}$:
\begin{align*}
    \eld \lee \ehls + \mathfrak{Reg} + \mathfrak{Err}_{approx} + \mathfrak{Err}_{sto} + \Delta^\mathcal{I} + \Delta^\mathcal{M}
\end{align*}
where $\hat{p}^\mathcal{I}$ and $\hat{p}^\mathcal{M}$ are uniform upper bounds on the empirical frequency of $\left(A^{\mathcal{I}}\right)^c(\rho)$ and $\left(A^{\mathcal{M}}\right)^c(\rho)$, respectively, over the training data, for all $f \in \FA$.

\begin{flalign*}
  & \mathfrak{Err}_{approx} = \frac{B \cdot \|\mathcal{E} -\mathcal{E}'\|_{\infty} }{C!}, ~~~C_M =  2M +2,, 
  \\[2pt] 
  & \zeta_1 = \left(C_MBM/\delta\right)^{(1/\nu)} \cdot r, ~~~\zeta_2 = \left(C_MN/\delta\right)^{(1/\nu)} \cdot r,\\[2pt] 
&\Delta^\mathcal{I} = \frac{p^{\mathcal{I}}B^2}{1-p^{\mathcal{I}}}  +  \frac{\hat{p}^\mathcal{I}B^2}{1-\hat{p}^\mathcal{I}}, ~~\Delta^\mathcal{M} = \frac{p^{\mathcal{M}}\cdot B}{1-p^{\mathcal{M}}}  +  \frac{\hat{p}^\mathcal{M}\cdot B}{1-\hat{p}^\mathcal{M}},  \\[6pt] 
& \mathfrak{Err}_{sto} =  \rm \cdot \frac{  \nu \cdot \log\left( \zeta_1 \right)}{(1-\hat{p}^\mathcal{M})\cdot M} + \ri \cdot \sqrt{\frac{ \nu \cdot \log( \zeta_2)}{(1-\hat{p}^\mathcal{I})\cdot N}}\\ 
&\mathfrak{Reg}  = \sqrt{\frac{ \nu \cdot \vphls \cdot \log\left(\zeta_1\right)}{(1-\hat{p}^\mathcal{M})^2 \cdot M}}
\end{flalign*}
\end{thm}


\begin{textbo}

\begin{rem}[\textbf{Validity of the Covering Number Assumption}]

The covering number assumption $\mathcal N_{\infty}(\mathcal F,\varepsilon,M)\lesssim \Bigl(\frac{r}{\varepsilon}\Bigr)^{\nu}$ is mild and generally holds for bounded high-dimensional spaces under the $\ell_\infty$ metric. It follows from the standard result that an $\ell_\infty$ ball in $\mathbb{R}^{\nu}$ can be covered by a uniform grid of this form \cite{DBLP:conf/iclr/LongS20}. This condition is satisfied by a wide range of common architectures, including linear models, fully connected networks, and convolutional neural network models \cite{anthony2009neural, barron1999risk, pollard2012convergence, DBLP:conf/iclr/LongS20, bach2024learning}. Detailed explanations and proofs of this assumption are provided in App.\tb{H}.

\end{rem}




\begin{rem}[\textbf{Model Complexity Bound}]

The model complexity is reflected by $\nu$ in the bound. According to App.\tb{I}, $\nu$ is roughly proportional to the total number of activated weight entries. Moreover, the number of classes $C$ only shows up in the last layer of the neural network, resulting in a \textbf{linear dependency}. However, since both the number of training samples $N$ and target distributions $M$ are typically much larger than $C$, this dependency has a negligible effect on the overall bound. Our main \textbf{contribution lies in achieving tight bounds w.r.t. $M$ and $N$}. More interestingly, as shown in Sec.\ref{sec:theo_found}, for \textbf{DirMixE-LSF}, we can get rid of this dependency based on the local linearization assumption, where the generalization bound depends solely on the trainable parameters, not on $C$.

\end{rem}

\end{textbo}

The bound can be decomposed into three parts:

\textbf{Empirical Error Term}: The first term, $\ehls$, in the theorem represents the empirical error, which can be directly optimized during training. With proper training, this term naturally becomes negligible.

\textbf{Regularization Term}: The term $\mathfrak{Reg}$ captures a regularization effect, which scales as
\begin{align*}
  \mathfrak{Reg}  \ee \left(\left(\mathfrak{Comp} \cdot \vphls / M \right)^{1/2} \right)
\end{align*}
Here, $\mathfrak{Comp}$ denotes the complexity of the hypothesis space, measured by the logarithm of the covering number (also known as metric entropy). This aligns with our semi-variance regularization strategy. Minimizing the semi-variance $\vphls$ helps reduce this term, and since $1/M$ is already small, this term tends to vanish with sufficient training. Therefore, the main generalization error arises from the remaining terms.

\textbf{Approximation Error from Meta-Distribution Shift}: The term, $\mathfrak{Err}_{approx}$, provides an upper bound on the error due to the shift from the original meta-distribution $\mathcal{E}$ to a different one $\mathcal{E}'$. According to \cite{uni1,uni2}, using a sufficiently large number of components in the Dirichlet mixture can reduce the difference $||\mathcal{E} - \mathcal{E}'||_{\infty}$ to near zero.

\textbf{Hierarchical Stochastic Error}: The third group of terms, $\mathfrak{Err}_{sto} + \Delta^\mathcal{I} + \Delta^\mathcal{M}$, quantifies the stochastic error caused by using empirical samples $(\mathcal{S}, \mathcal{P})$ instead of the true distributions $(\mathcal{D}, \mathcal{E})$. With a well-behaved loss distribution, this error can be sharply bounded by $O(\ri \cdot (\mathfrak{Comp} / N)^{1/2} + \rm \cdot (\mathfrak{Comp} / M))$, which is significantly tighter than the standard bound $O(\eqnb{B} \cdot (\mathfrak{Comp} / N)^{1/2} + \eqnb{B} \cdot (\mathfrak{Comp} / M)^{1/2})$. Note that with a proper loss distribution, both $\ri$ and $\rm$ are much smaller than the constant $\eqnb{B}$. Furthermore, the theorem implies that when the test label distribution is not fixed, it is important to sample a large number of test label distributions (i.e., use a large $M$) to reduce the stochastic error. This highlights the advantage of \ours~against traditional Mixture-of-Experts (MoE) approaches, such as those in \cite{DBLP:conf/cvpr/AimarJFK23,lade,DBLP:conf/nips/ZhangHHF22}, which typically use a fixed assignment with $M = O(1)$. 

Finally, we present a \textbf{data-dependent sharp bound}, assuming that the loss decays faster than an exponential distribution.

\begin{textbo}

\begin{col}[\textbf{Tigher Bound with Light-tail Loss Decay}]   \label{col: exponential loss decay}
  Based on Thm.3 in our paper, suppose that the loess satisfies an asymptotic exponential-tail condition: there exists $\alpha_0 > 0$, such that for all $\tau \ge \tau_0 := \min(N^{-1 / \alpha_0}, M^{-1 / \alpha_0})$, 
  \begin{align}
  &\mathbb{P}_x \left[ \mathbb{E}_{\mathcal{E}}[\ell(f(x), y, P_i)] \ge \tau \right] \lesssim \exp(-\lambda_1 \cdot \tau)   \label{eq: per-loss assumption} \\
  &\mathbb{P}_{\mathcal{E}}\left[ \mathbb{E}_{\mathcal{D}}[\ell(f(x), y, P_i)] \ge \tau \right] \lesssim \exp(-\lambda_2 \cdot \tau)\\ 
  &\hat{\mathbb{P}}_x \left[ \mathbb{E}_{\mathcal{E}}[\ell(f(x), y, P_i)] \ge \tau \right] \lesssim \exp(-\hat{\lambda}_1 \cdot \tau)   \label{eq: per-loss assumption 2}\\
  &\hat{\mathbb{P}}_{\mathcal{E}}\left[ \mathbb{E}_{\mathcal{D}}[\ell(f(x), y, P_i)] \ge \tau \right] \lesssim \exp(-\hat{\lambda}_2 \cdot \tau)
\end{align}
  and $\lambda_1,\hat{\lambda}_1 \asymp \log(N) \cdot N^{1/\alpha_0}, \lambda_2, \hat{\lambda}_2 \asymp \log(M) \cdot M^{1/\alpha_0}$, then we can pick:
  \begin{align*}
    &\rho^\mathcal{M} \asymp   M^{-1/\alpha_0}, \rho^\mathcal{I} \asymp  N^{-1/\alpha_0}, p^\mathcal{I} \asymp \frac{1}{N}, p^\mathcal{M} \asymp \frac{1}{M} \\ 
    &\hat\rho^\mathcal{M} \asymp   M^{-1/\alpha_0}, \hat\rho^\mathcal{I} \asymp  N^{-1/\alpha_0}, \hat p^\mathcal{I} \asymp \frac{1}{N}, \hat p^\mathcal{M} \asymp \frac{1}{M}
  \end{align*}
  such that $\Delta^{\mathcal{I}} \lesssim \frac{1}{N}, \Delta^{\mathcal{M}} \lesssim \frac{1}{M}$, and
  \begin{align*}
   \mathfrak{Err}_{sto} =  &    \frac{\nu \cdot \log\left( \zeta_1 \right)}{(1-\hat{p}^\mathcal{M})\cdot {M^{{1+1/\alpha_0}}}} + \sqrt{\frac{ \nu \cdot \log( \zeta_2)}{(1-\hat{p}^\mathcal{I})\cdot {N^{{1+2/\alpha_0}}}}}
  \end{align*}
\end{col}

   
    

\end{textbo}


\begin{textbo}

\begin{rem}[\textbf{Validity of the Exponential-tail Assumption}]
    We verify this assumption empirically on both the \textbf{CIFAR-100-LT} and \textbf{ImageNet-LT} datasets. As a representative example, Fig.\ref{fig: exponential loss decay} shows the per-sample loss distributions on these datasets, validating Eq.\ref{eq: per-loss assumption} and Eq.\ref{eq: per-loss assumption 2}.
    By fitting the survival function $\log \mathbb{P}[\ell \ge \tau]$ against $\tau$ via \textbf{log-linear regression}, we observe a clear linear trend in the upper tail, with a high coefficient of determination ($R^2 = 0.95$ for CIFAR-100-LT and $R^2 = 0.94$ for ImageNet-LT), confirming exponential decay behavior.
    Full experimental details and comprehensive results for all settings are provided in App.\tb{J}.

\end{rem}

\end{textbo}

\begin{figure}[h]     
    \begin{textbo}
        
    \centering
    
    \subfigure[CIFAR-100-LT]{         
    \includegraphics[width=0.42\columnwidth]{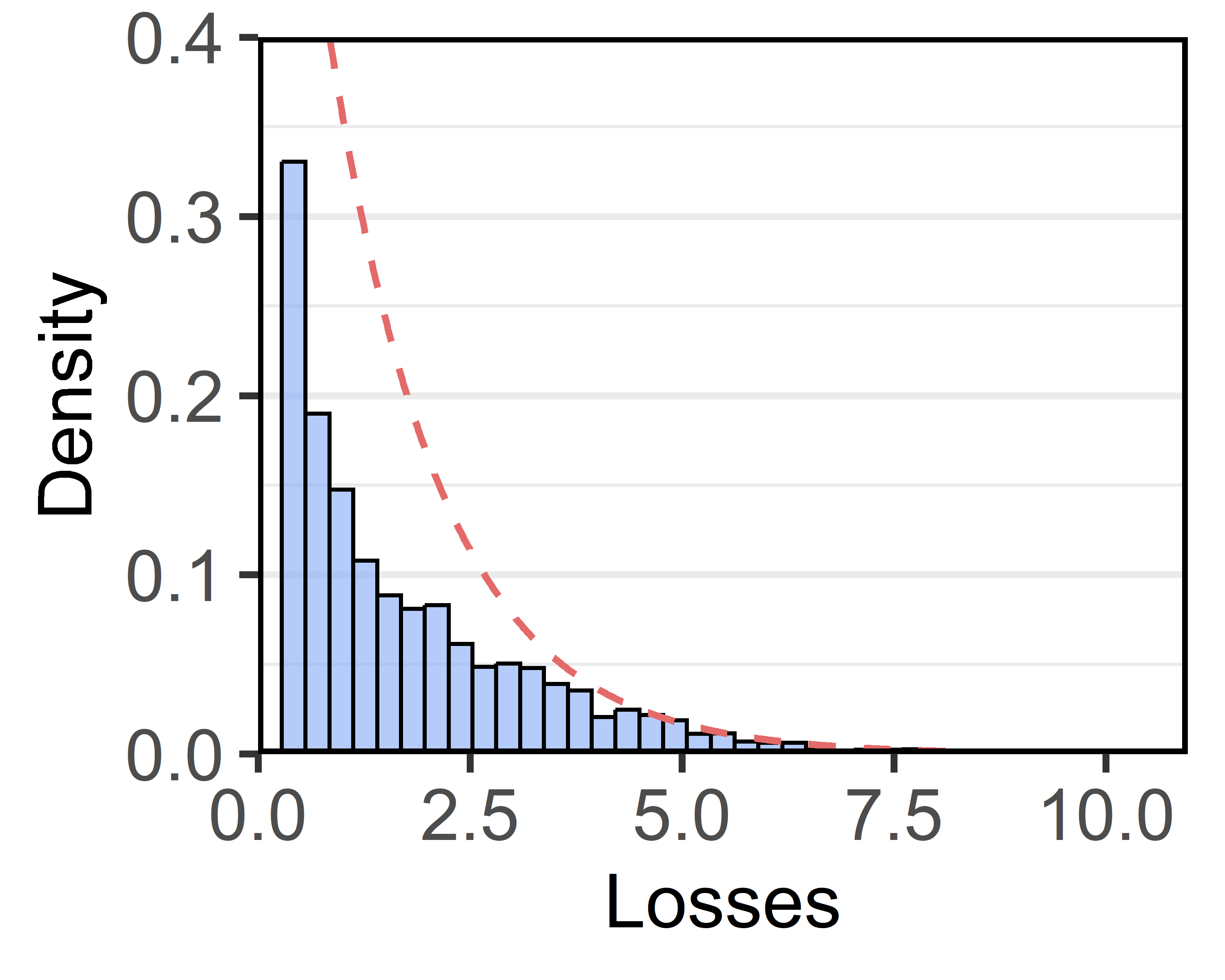} 
    }%
    \subfigure[ImageNet-LT]{
    \includegraphics[width=0.42\columnwidth]{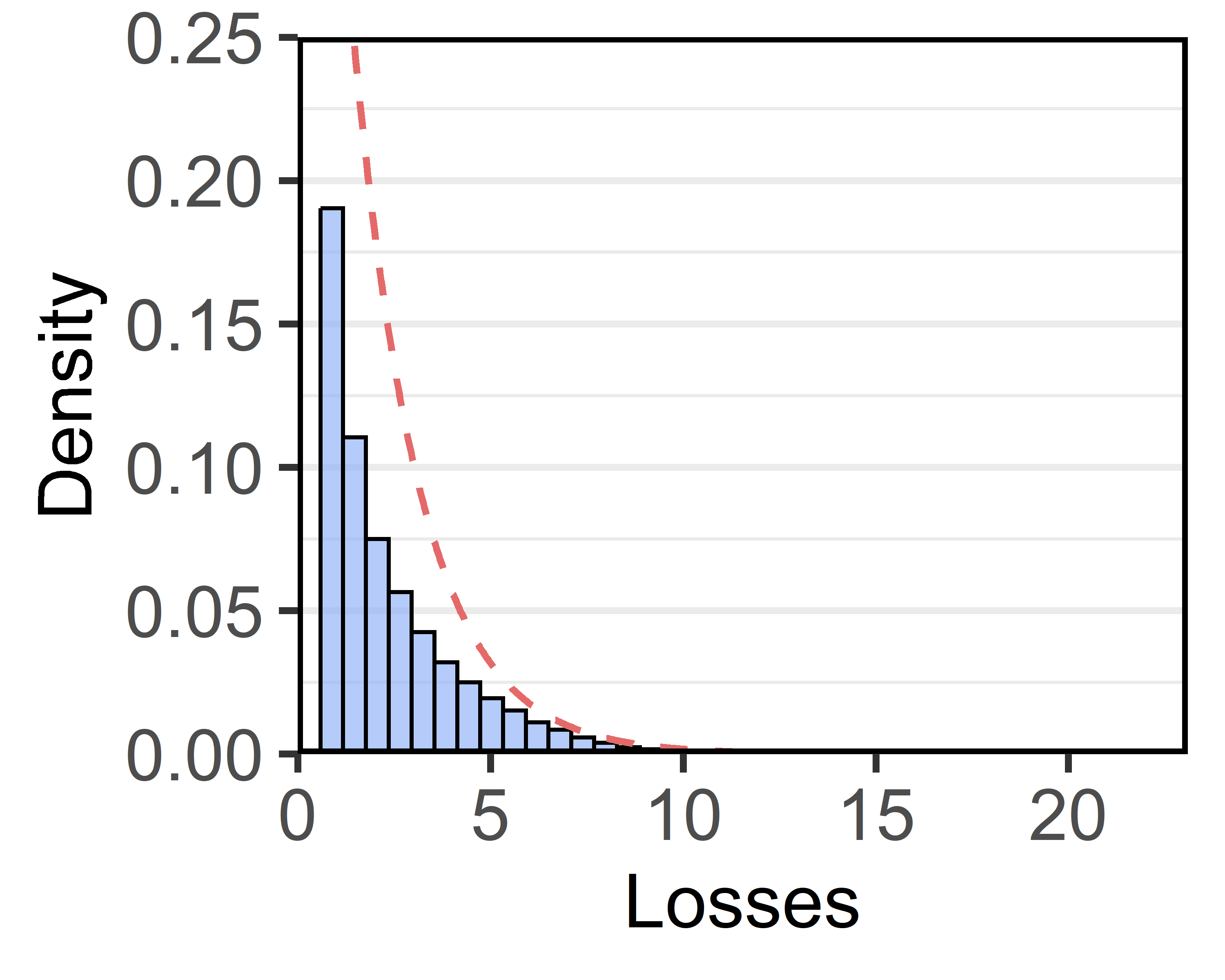} 
    }
    
    \caption{Empirical verification of the exponential-tail loss decay assumption. The plots show the per-sample loss distributions on (a) \textbf{CIFAR-100-LT} and (b) \textbf{ImageNet-LT}, along with exponential fit curves.}

    \label{fig: exponential loss decay} 

    \end{textbo}
    
\end{figure}

In this result, we improve the effect of sample size and test distribution size from $N^{-1/2}, M^{-1}$ to $N^{-(1 - 2/\alpha)/2}, M^{-1 - 1/\alpha}$. Moreover, the influence of $\Delta^{\mathcal{I}}, \Delta^{\mathcal{M}}$ becomes negligible with sufficiently large datasets. Overall, the induced subclass achieves a sharper generalization bound under this distributional setting.

\subsection{\newsec{Generalization Analysis for DirMixE-LSF}}\label{sec:theo_found}

\textbf{Notation Clarifications.}  
In this subsection, we introduce several notational conventions that will be consistently used to simplify the formulation and analysis presented later. These notations help clarify the model structure and facilitate a precise theoretical derivation.

\begin{enumerate}
  \item \textbf{Per-task Notations.}  
  Recall that we define a distinct prediction task corresponding to each of the $n_K$ meta-distributions. For a specific task indexed by $i$, we use $\bm{X}^{(i)}$ to denote the parameters associated with that task. The logit output for task $i$, learned from training, is denoted by $f^{(i)}$.

  \item \noindent \textbf{Layer-wise Notations.}  
  Given a general model parameter $\bm{X}$, we define $\bm{X}_l$ as the portion of the parameter located in the $l$-th layer of the model. This reflects the standard layered architecture in deep neural networks, where certain layers are fine-tuned.

  \item \noindent \textbf{Overall Parameters.}  
  As the parameter space in our setting involves both task-specific and layer-wise components, we introduce a simplified notation to capture their structure. We denote $\thi{i}$ as the \textbf{layer-wise concatenation} of the parameters associated with task $i$ across all layers. The norm $\|\thi{i}\|$ represents the reduced vector $\ell_2$-norm, which computes the root of the sum of squares for the concatenated vector. In the PEFT setting with LSF, the model includes two kinds of parameters: a fixed backbone and a set of incremental parameters subject to optimization. Let $\bm{\Theta}_0$ denote the original parameter configuration, which remains \emph{frozen} during fine-tuning. The learnable, incremental parameters are denoted by $\dif{\thi{i}}$, which corresponds to the layer-wise concatenation of changes from $\bm{\Theta}_0$. The collection of parameters across all tasks is represented as $\Tha =\left\{  \thi{1},\cdots, \thi{n_K} \right\}$.

  \item \noindent \textbf{Local Solutions.}  
  The notation $\dif{\this{i}}$ refers to a properly constructed combination of local solutions, as defined in Eq.(\ref{eq:solution}). \textbf{Here, we choose a subset of such solutions, denoted as $\mathsf{Sol}$, such that each distinct pair of elements herein are seperated apart}. For a given parameter $\thi{i}$, the corresponding $\dif{\this{i}}$ is selected as the composition of local solutions that minimizes its distance to $\dif{\thi{i}}$ under the reduced $\ell_2$-norm:
  \begin{align}\label{eq:solution}
    \dif{\this{i}} = \argmin_{
    \begin{subarray}{l}
      \bm{\Theta}~\text{is a proper composition} \\ 
      \text{of the local solutions in~}  \mathsf{Sol}
    \end{subarray}}
  \|\bm{\Theta} - \dif{\thi{i}} \|
  \end{align} 
  Formally, $\dif{\this{i}} = \dif{\this{i}}(\dif{\thi{i}})$ is a function of the incremental parameters. However, for notational simplicity, we denote it directly as $\dif{\this{i}}$.
\end{enumerate}

\begin{textbo}
    
\textbf{Taylor Expansions.}  
A key characteristic of the PEFT framework is that the majority of the backbone model remains fixed during fine-tuning. As a result, the model’s behavior is primarily governed by the trainable incremental parameters $\dif{\thi{i}}$. To isolate the influence of these parameters and enable sharper theoretical bounds, we consider a Taylor expansion around an appropriate reference point. 

Given that LSF encourages well-trained models where the learned incremental parameters are close to certain local solutions, we select the reference point as $\Th_0 + \dif{\this{i}}$, where $\dif{\this{i}}$ is the closest composition of local solutions in the subset $\mathsf{Sol}$ as defined above. This choice enables a more accurate linear approximation of the model's logit function. Specifically, the first-order Taylor expansion becomes:
\begin{align}\label{eq:lin}
 f^{(i)}_{\thi{i}}(\x) = & f^{(i)}_{\this{i}}(\x) + \langle \nabla_{\this{i}}f^{(i)}_{\this{i}}(\x),  \Delta \thi{i} - \Delta\this{i} \rangle + \\
 &r(i,\x)
\end{align}
where
\begin{align*}
  &\this{i}=\Th_0+\Delta\this{i}, \tag{for LoRA-LSF}\\[4pt] 
  &\this{i}=\left\{ \Th_0;\Delta\this{i} \right\}, \tag{for Adapter-LSF}\\[4pt]  
 &\Delta\thi{i} = \nls{\Ail{\Bil}^\top}, ~~\Delta\this{i} = \nls{\Ails{\Bils}^\top} \tag{for LoRA-LSF}\\[4pt] 
 &\Delta\thi{i} = \nls{\Uil, \Dil}, ~~\Delta\this{i} = \nls{\Uils, \Dils} \tag{for Adapter-LSF}\\[4pt] 
 &r(i,\x) \\ 
 &=\left(\mathrm{vec}(\Delta\thi{i}_\xi - \Delta\this{i})\right)^\top \bm{H}_{\theta}(\x) ~\mathrm{vec}(\Delta\thi{i}_\xi - \Delta\this{i}),
\end{align*}
where $\mathrm{vec}$ denote the vectorization operator, $\Delta\thi{i}_\xi \in \mathsf{LinSeg}$, $\mathsf{LinSeg}$ denotes the set of all linear segments between $\Delta\thi{i}$ and $\Delta\this{i}$, while $\bm{H}_{\theta}(\x)$ is the corresponding Hessian matrix.
\end{textbo}

We emphasize the following key insights:
\textbf{1)} When the residual term $\mathcal{R}(i,\x)$ is small, the logit function $f^{(i)}_{\thi{i}}(\x)$ becomes approximately linear with respect to the incremental parameters, which significantly simplifies subsequent generalization analysis.  
\textbf{2)} The reference $\this{i}$ is a partial solution in terms of the incremental parameters, and the backbone $\bm{\Theta}_0$ remains frozen. In this sense, the gradient $\nabla_{\this{i}}f^{(i)}_{\this{i}}(\x)$ is generally non-zero.  
\textbf{3)} The second-order residual term $r(i,\x)$ is influenced by both the curvature of the function (measured by the operator norm of the Hessian) and the distance between $\dif{\thi{i}}$ and $\dif{\this{i}}$. Given that LSF promotes well-trained solutions and smooth optimization landscapes, we expect both factors to be small in practice.

The function $\dif{\this{i}}$ in Eq.~\ref{eq:solution} depends on the current incremental parameters and may vary between models. In our theoretical analysis, the complexity is measured by covering number. Under this setting, we only need to ensure that $\dif{\this{i}}$ remains constant within sufficiently small open balls.  To measure this property, we introduce the notion of Voronoi diagram to conduct a partition of the entire space of the incremental parameters.

\begin{defi}[\textbf{Voronoi Diagram Induced Parameter Space Partition}]    \label{def: voronoi diagram}
Consider task $k$ with local parameter solutions $\bm{\Theta}^{(k),\star}_1, \bm{\Theta}^{(k),\star}_2, \cdots$ in a solution subset $\mathsf{Sol}$. Define the Voronoi cells $\mathcal{V}^{(k)}_i$ such that the union of all cells covers the entire parameter space:
\begin{align*}
  \mathcal{V}^{(k)}_i = \left\{ \bm{\Theta} ~\middle|~ \|\bm{\Theta} - \Delta\bm{\Theta}^{(k),\star}_i \| \le \|\bm{\Theta} - \Delta\bm{\Theta}^{(k),\star}_j\|, \forall j \neq i \right\}
\end{align*}
The boundary between two cells $\mathcal{V}^{(k)}_i$ and $\mathcal{V}^{(k)}_j$ is defined as:
\begin{align*}
  \bm{\Gamma}^{(k)}_{i,j} = \left\{ \bm{\Theta} ~\middle|~ \|\bm{\Theta} - \Delta\bm{\Theta}^{(k),\star}_i \| = \|\bm{\Theta} - \Delta\bm{\Theta}^{(k),\star}_j\| \right\}
\end{align*}
\end{defi}

If the current parameters are located at least $\delta$ away from any decision boundary $\bm{\Gamma}^{(k)}_{i,j}$, then the closest local solution $\dif{\this{i}}$ remains invariant within an open ball of radius $\delta$. This gives rise to the concept of $\delta$-compactness.

\begin{defi}[\textbf{$\delta$-Compact Parameterization}]     \label{def: delta compact parameterization}
The incremental parameter hypothesis class $\dif{\Th}_{\mathrm{all}} = (\Delta\bm{\Theta}^{(1)}, \cdots, \Delta\bm{\Theta}^{(N_K)})$ is said to be $\delta$-compact if, there exists a subset of local solutions denoted as $\mathsf{Sol}$ such that $\forall (\dif{\Th}_{\mathrm{all},1}, \dif{\Th}_{\mathrm{all,2}}) \in \mathsf{Sol}$, $d\left( \dif{\this{i}}_1,  \dif{\this{i}}_2 \right) >> \delta$, and that for all $\dif{\Th}_{\mathrm{all}}$ in the target hypothesis space, we have:
\[
\min_{i\neq j}d\left(\bm{\Theta}^{(k)}, \Gamma^{(k)}_{i,j}\right) >\delta, \quad \forall k.
\],
where $d$ is measured with reduced vector 2-norm.
\end{defi}

\begin{textbo}

\begin{rem}

Def.\ref{def: voronoi diagram} and Def.\ref{def: delta compact parameterization} jointly establish a \textbf{local linearization framework} for the model’s logits around nearby minima, which underlies the covering-number-based bound.
As shown in Fig.\ref{fig:voi}, Def.\ref{def: voronoi diagram} partitions the parameter space into Voronoi cells, each associated with one local minimum. Within a cell, all parameters share the same local solution, enabling a consistent linear approximation of Eq.~(\ref{eq:lin}).

Def.\ref{def: delta compact parameterization} introduces the $\delta$-compact parametrization, ensuring that each covering ball lies entirely inside a single cell and avoids crossing boundaries.
This assumption is mild, since PEFT starts from a pre-trained model near a local minimum and updates parameters slightly, so choosing $\delta = O(1/N)$ readily satisfies the condition.

\begin{figure}[H]
\centering
\includegraphics[width=0.4\textwidth]{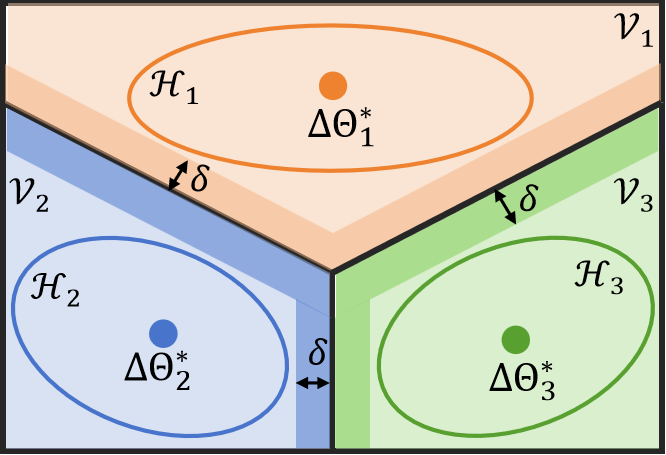}
\caption{Voronoi partition and $\delta$-compact parametrization. The $\delta$-margin ensures each covering ball stays within one cell.}
\label{fig:voi}
\end{figure}

\end{rem}

\end{textbo}

\begin{asm}
We make the following assumptions:
\begin{enumerate}\label{asm:lsf}
  \item \textbf{(Gradient Lipschitz)} The gradient norm of the \textbf{logit} function in Eq.~\ref{eq:lin} is uniformly bounded:
  \[
  \sup_{i,\x \in \mathcal{X}}\|\nabla_{\this{i}}f^{(i)}_{\this{i}}(\x)\|_F \le L_f
  \]
  \item \textbf{(Compact Parameterization)} In the LSF framework, the parameter space is $\max\left\{2/(N \cdot L_f), 2/(M \cdot L_f)\right\}$-compact for both LoRA and Adapter cases.
\end{enumerate}
\end{asm}

\begin{textbo}

\begin{rem}
    
Although CE loss $\ell_{CE}$ is not Lip. continuous,  by composing with the Softmax operator, $\ell = \ell_{CE} \circ \mathsf{Softmax}$ is Lip. continuous. In this sense, we only need to focus on the complexity of the logit function $f_{\Theta}$. In conclusion, it is unnecessary to require the loss function itself to be Lipschitz continuous. 

\end{rem}

\end{textbo}

These assumptions are mild. Assumption 1.1 is a standard smoothness assumption in optimization theory. Assumption 1.2 is also reasonable because $\max\left\{2/(N \cdot L_f), 2/ (M \cdot L_f) \right\}$ is a small radius given sufficiently large sample size. 

\begin{thm}[\textbf{Generalization Bound for the LSF Scheme}]\label{thm:gen1}
Let $\mathcal{E}$ be the true meta-distribution and $\mathcal{P}$ be the empirical distribution obtained via Monte Carlo sampling. Suppose the training data $\mathcal{S}$ is drawn i.i.d. from $\mathcal{D}$. Under Assumption~\ref{asm:lsf}, and assuming $\ell(\cdot) \in [0, B]$ and $\vhld \ee \vphld$ as defined in Tab. \tb{7}, for any $\|\dif{\thi{i}}\| \le R$ and meta-distribution $\mathcal{E}'$ over the simplex $\mathbb{S}^{c-1}$, the following generalization bound holds with probability at least $1 - \delta$ over the randomness in $\mathcal{S}$ and $\mathcal{P}$:
\begin{align*}
  \eld \lee & \ehls + \mathfrak{Reg} + \mathfrak{Err}_{approx} + \mathfrak{Err}_{sto} + \Delta^\mathcal{I} + \Delta^\mathcal{M} \\ 
  &+ \mathbb{E}_{\mathcal{E},  \mathcal{D}}\left[ \mathcal{R}(i,x) \right] + \hat{\mathbb{E}}_{\mathcal{P},  \mathcal{S}}\left[ \mathcal{R}(i,x) \right]
\end{align*}
where all symbols are consistent with Thm.~\ref{thm:gen}, except:
\begin{align*}
  &r = L_f \cdot n_L \cdot R, \quad \nu = n_K \cdot n_L \cdot Kr \cdot (m+n - Kr) \quad \text{for LoRA} \\
  &\nu = 2 n_K \cdot n_L \cdot (m+n - r) \quad \text{for Adapter}
\end{align*}
\end{thm}

When $n_K$ is on the order of $O(1)$, $\nu$ roughly corresponds to the number of trainable parameters, which is much smaller than the full backbone size. Although two residual terms remain due to Taylor expansion, our choice of reference point near local solutions makes them moderate in magnitude.  

\subsection{Test Time Generalization}
In this subsection, we further explore how well we can merge the experts during test time. To explore the answer, we check the performance of the self-supervised model averaging result $f_{te}$ in \eqref{eq:merge}.

\begin{thm}[\textbf{Model Averaging Error}]\label{thm:cov} 
   Under the same setting as Thm.\ref{thm:gen}, the following results hold:
    \begin{align*}
        \E(\la(f_{te})) \lee& \sum_{i\in[K]} \E(\omega_i) \cdot\E(\la(f^{(i)})) \\ 
        &+ \sum_{i \in [K]} \mathbb{COV}(\omega_i, \la(f^{(i)})),
    \end{align*}
where $\mathbb{COV}$ is the covariance operator and all the expectations are taken over the joint distribution of $\mathcal{E} \otimes \mathcal{D}$.
\end{thm} 
The proof of Thm.\ref{thm:cov} follows \cite{weight} and thus is omitted.

The resulting upper bounds consist of two parts. The first part could be regarded as a weighted loss, which can be guaranteed to have a similar concentration as in Thm.\ref{thm:cov}. The second term represents the covariance between the model weights and the corresponding losses. A well-designed model aggregation method should ideally assign lower weights to models that yield higher losses. In this study, we utilize the self-supervised method described in \cite{DBLP:conf/nips/ZhangHHF22}, which effectively maximizes the mutual information between predictions and ground truth under certain conditions. Consequently, we anticipate a negative correlation between model weight and loss (negative covariance), significantly reducing model averaging error.

  \begin{table*}[t]
      \centering
    \caption{Performance Comparison on CIFAR-100-LT When Training ResNet Models (\textbf{Ours Setting)}}
      \renewcommand\arraystretch{1.0}
        \begin{tabular}{lcccccccccc}
        \toprule
        \multirow{1.5}[4]{*}{\textbf{Method}} & \multicolumn{3}{c}{\textbf{Forward-LT}} & \multicolumn{3}{c}{\textbf{Uniform}} & \multicolumn{3}{c}{\textbf{Backward-LT}} & \multirow{1.5}[4]{*}{\textbf{Mean}} \\
              \cmidrule(lr){2-4}\cmidrule(lr){5-7}\cmidrule(lr){8-10}     & \textbf{1} & \textbf{2} & \textbf{3} & \textbf{1} & \textbf{2} & \textbf{3} & \textbf{1} & \textbf{2} & \textbf{3} &  \\
    \toprule
    LDAM  & \cellcolor[rgb]{ .855,  .922,  .808} 64.17  & \cellcolor[rgb]{ .863,  .925,  .82} 63.59  & \cellcolor[rgb]{ .859,  .925,  .812} 66.43  & \cellcolor[rgb]{ 1,  1,  1} 38.16  & \cellcolor[rgb]{ 1,  1,  1} 37.35  & \cellcolor[rgb]{ 1,  1,  1} 38.07  & \cellcolor[rgb]{ 1,  1,  1} 17.32  & \cellcolor[rgb]{ 1,  1,  1} 15.04  & \cellcolor[rgb]{ 1,  1,  1} 16.95  & \cellcolor[rgb]{ 1,  1,  1} 39.68$_{\color{blue}(\pm 19.78)}$\\
        LA    & \cellcolor[rgb]{ .973,  .988,  .965} 58.71  & \cellcolor[rgb]{ .965,  .98,  .953} 58.77  & \cellcolor[rgb]{ .973,  .984,  .961} 61.39  & \cellcolor[rgb]{ .847,  .918,  .8} 44.93  & \cellcolor[rgb]{ .855,  .922,  .808} 45.35  & \cellcolor[rgb]{ .882,  .937,  .847} 43.75  & \cellcolor[rgb]{ .882,  .937,  .843} 31.94  & \cellcolor[rgb]{ .902,  .945,  .871} 29.49  & \cellcolor[rgb]{ .886,  .941,  .851} 30.96  & \cellcolor[rgb]{ .914,  .953,  .886} 45.03$_{\color{blue}(\pm 11.81)}$\\
        VS    & \cellcolor[rgb]{ .984,  .992,  .98} 58.20  & \cellcolor[rgb]{ .973,  .984,  .965} 58.30  & \cellcolor[rgb]{ 1,  1,  1} 59.99  & \cellcolor[rgb]{ .976,  .988,  .969} 39.32  & \cellcolor[rgb]{ .929,  .965,  .906} 41.24  & \cellcolor[rgb]{ .918,  .957,  .89} 42.13  & \cellcolor[rgb]{ .918,  .957,  .894} 27.37  & \cellcolor[rgb]{ .937,  .965,  .914} 24.53  & \cellcolor[rgb]{ .91,  .953,  .882} 28.14  & \cellcolor[rgb]{ .961,  .98,  .949} 42.14$_{\color{blue}(\pm 13.21)}$\\
        LADE  & \cellcolor[rgb]{ 1,  1,  1} 57.42  & \cellcolor[rgb]{ .984,  .992,  .98} 57.76  & \cellcolor[rgb]{ .976,  .988,  .969} 61.06  & \cellcolor[rgb]{ .882,  .937,  .847} 43.35  & \cellcolor[rgb]{ .898,  .945,  .867} 42.92  & \cellcolor[rgb]{ .898,  .945,  .867} 43.05  & \cellcolor[rgb]{ .878,  .933,  .839} 32.40  & \cellcolor[rgb]{ .894,  .945,  .863} 30.34  & \cellcolor[rgb]{ .867,  .929,  .824} 33.47  & \cellcolor[rgb]{ .922,  .957,  .894} 44.64$_{\color{blue}(\pm 11.01)}$\\
        DDC   & \cellcolor[rgb]{ .961,  .98,  .949} 59.26  & \cellcolor[rgb]{ 1,  1,  1} 56.89  & \cellcolor[rgb]{ .98,  .992,  .976} 60.90  & \cellcolor[rgb]{ .847,  .918,  .796} 45.05  & \cellcolor[rgb]{ .863,  .925,  .82} 44.86  & \cellcolor[rgb]{ .847,  .918,  .796} 45.59  & \cellcolor[rgb]{ .875,  .933,  .835} 32.59  & \cellcolor[rgb]{ .882,  .937,  .847} 32.05  & \cellcolor[rgb]{ .859,  .925,  .816} 34.41  & \cellcolor[rgb]{ .902,  .949,  .871} 45.73$_{\color{blue}(\pm 10.68)}$\\
        RIDE  & \cellcolor[rgb]{ .875,  .933,  .835} 63.34  & \cellcolor[rgb]{ .882,  .937,  .843} 62.72  & \cellcolor[rgb]{ .875,  .933,  .835} 65.68  & \cellcolor[rgb]{ .839,  .914,  .788} 45.30  & \cellcolor[rgb]{ .855,  .922,  .808} 45.35  & \cellcolor[rgb]{ .8,  .89,  .733} \underline{47.81} & \cellcolor[rgb]{ .89,  .941,  .855} 30.82  & \cellcolor[rgb]{ .914,  .953,  .886} 27.61  & \cellcolor[rgb]{ .89,  .941,  .855} 30.75  & \cellcolor[rgb]{ .89,  .941,  .855} 46.60$_{\color{blue}(\pm 14.02)}$\\
        SADE  & \cellcolor[rgb]{ .812,  .898,  .753} 66.13  & \cellcolor[rgb]{ .816,  .902,  .757} 66.00  & \cellcolor[rgb]{ .816,  .902,  .757} 68.31  & \cellcolor[rgb]{ .804,  .894,  .741} \underline{46.89} & \cellcolor[rgb]{ .804,  .894,  .741} \underline{48.05} & \cellcolor[rgb]{ .847,  .918,  .796} 45.59  & \cellcolor[rgb]{ .8,  .89,  .733} \underline{41.99} & \cellcolor[rgb]{ .808,  .894,  .745} \underline{42.91} & \cellcolor[rgb]{ .812,  .898,  .753} \underline{40.06} & \cellcolor[rgb]{ .804,  .894,  .741} \underline{51.77}$_{\color{blue}(\pm 10.90)}$\\
        BalPoE & \cellcolor[rgb]{ .776,  .878,  .706} \textbf{67.75} & \cellcolor[rgb]{ .776,  .878,  .706} \textbf{67.80} & \cellcolor[rgb]{ .776,  .878,  .706} \textbf{69.98} & \cellcolor[rgb]{ .847,  .918,  .796} 45.05  & \cellcolor[rgb]{ .827,  .906,  .769} 46.86  & \cellcolor[rgb]{ .776,  .878,  .706} \textbf{48.81} & \cellcolor[rgb]{ .898,  .945,  .867} 29.80  & \cellcolor[rgb]{ .922,  .961,  .898} 26.32  & \cellcolor[rgb]{ .886,  .937,  .851} 31.17  & \cellcolor[rgb]{ .863,  .925,  .82} 48.17$_{\color{blue}(\pm 16.19)}$\\
        \toprule
        \textbf{DirMixE} & \cellcolor[rgb]{ .796,  .89,  .733} \underline{66.85} & \cellcolor[rgb]{ .808,  .894,  .745} \underline{66.40} & \cellcolor[rgb]{ .792,  .886,  .725} \underline{69.44} & \cellcolor[rgb]{ .776,  .878,  .706} \textbf{47.99} & \cellcolor[rgb]{ .776,  .878,  .706} \textbf{49.41} & \cellcolor[rgb]{ .875,  .933,  .835} 44.21  & \cellcolor[rgb]{ .776,  .878,  .706} \textbf{44.41} & \cellcolor[rgb]{ .776,  .878,  .706} \textbf{47.01} & \cellcolor[rgb]{ .776,  .878,  .706} \textbf{44.35} & \cellcolor[rgb]{ .776,  .878,  .706} \textbf{53.34}$_{\color{blue}(\pm 10.22)}$\\
        \bottomrule
        \end{tabular}%
      \label{tab:Cifar100_Our}%
    \end{table*}%

    \begin{table*}[htbp]
      \centering
      \caption{Performance Comparison on ImageNet-LT When Training ResNet Models (\textbf{Ours Setting)}}
      \renewcommand\arraystretch{1.0}
      \small
        \begin{tabular}{lcccccccccc}
        \toprule
        \multirow{1.5}[4]{*}{\textbf{Method}} & \multicolumn{3}{c}{\textbf{Forward-LT}} & \multicolumn{3}{c}{\textbf{Uniform}} & \multicolumn{3}{c}{\textbf{Backward-LT}} & \multirow{1.5}[4]{*}{\textbf{Mean}} \\
              \cmidrule(lr){2-4}\cmidrule(lr){5-7}\cmidrule(lr){8-10}     & \textbf{1} & \textbf{2} & \textbf{3} & \textbf{1} & \textbf{2} & \textbf{3} & \textbf{1} & \textbf{2} & \textbf{3} &  \\
        \toprule
        LDAM  & \cellcolor[rgb]{ .996,  .957,  .933} 61.74  & \cellcolor[rgb]{ .996,  .953,  .925} 62.22  & \cellcolor[rgb]{ .996,  .949,  .918} 61.49  & \cellcolor[rgb]{ 1,  1,  1} 47.51  & \cellcolor[rgb]{ 1,  1,  1} 47.37  & \cellcolor[rgb]{ 1,  1,  1} 48.63  & \cellcolor[rgb]{ 1,  1,  1} 32.67  & \cellcolor[rgb]{ 1,  1,  1} 32.54  & \cellcolor[rgb]{ 1,  1,  1} 32.00  & \cellcolor[rgb]{ 1,  1,  1} 47.35$_{\color{blue}(\pm 12.02)}$\\
        LA    & \cellcolor[rgb]{ 1,  .973,  .957} 60.94  & \cellcolor[rgb]{ 1,  .988,  .976} 60.32  & \cellcolor[rgb]{ 1,  .976,  .961} 59.78  & \cellcolor[rgb]{ .992,  .941,  .906} 50.82  & \cellcolor[rgb]{ .992,  .941,  .906} 50.86  & \cellcolor[rgb]{ .996,  .957,  .929} 50.86  & \cellcolor[rgb]{ .992,  .933,  .894} 40.39  & \cellcolor[rgb]{ .992,  .933,  .894} 40.16  & \cellcolor[rgb]{ .992,  .933,  .894} 40.14  & \cellcolor[rgb]{ .996,  .957,  .929} 50.47$_{\color{blue}(\pm \phantom{0}8.22)}$\\
        VS    & \cellcolor[rgb]{ .996,  .969,  .949} 61.14  & \cellcolor[rgb]{ 1,  1,  .996} 59.60  & \cellcolor[rgb]{ 1,  .988,  .984} 59.00  & \cellcolor[rgb]{ .992,  .918,  .871} 52.04  & \cellcolor[rgb]{ .992,  .918,  .867} 52.22  & \cellcolor[rgb]{ .988,  .91,  .855} 53.18  & \cellcolor[rgb]{ .988,  .902,  .843} 44.03  & \cellcolor[rgb]{ .988,  .894,  .831} 44.47  & \cellcolor[rgb]{ .988,  .91,  .855} 43.02  & \cellcolor[rgb]{ .992,  .933,  .894} 52.08$_{\color{blue}(\pm \phantom{0}6.60)}$\\
        LADE  & \cellcolor[rgb]{ .992,  .922,  .878} 63.58  & \cellcolor[rgb]{ .996,  .953,  .922} 62.29  & \cellcolor[rgb]{ .992,  .941,  .906} 61.92  & \cellcolor[rgb]{ .988,  .89,  .827} 53.48  & \cellcolor[rgb]{ .988,  .914,  .863} 52.38  & \cellcolor[rgb]{ .988,  .906,  .851} 53.31  & \cellcolor[rgb]{ .992,  .922,  .878} 41.53  & \cellcolor[rgb]{ .988,  .914,  .863} 42.31  & \cellcolor[rgb]{ .992,  .925,  .878} 41.25  & \cellcolor[rgb]{ .992,  .929,  .886} 52.45$_{\color{blue}(\pm \phantom{0}8.56)}$\\
        DDC   & \cellcolor[rgb]{ 1,  1,  1} 59.42  & \cellcolor[rgb]{ 1,  1,  1} 59.45  & \cellcolor[rgb]{ 1,  1,  1} 58.30  & \cellcolor[rgb]{ .992,  .929,  .89} 51.36  & \cellcolor[rgb]{ .992,  .922,  .878} 51.83  & \cellcolor[rgb]{ .992,  .945,  .91} 51.51  & \cellcolor[rgb]{ .988,  .914,  .863} 42.50  & \cellcolor[rgb]{ .988,  .898,  .835} 44.09  & \cellcolor[rgb]{ .988,  .906,  .851} 43.47  & \cellcolor[rgb]{ .996,  .945,  .91} 51.33$_{\color{blue}(\pm \phantom{0}6.43)}$\\
        RIDE  & \cellcolor[rgb]{ .988,  .89,  .824} 65.32  & \cellcolor[rgb]{ .992,  .918,  .867} 64.28  & \cellcolor[rgb]{ .988,  .914,  .863} 63.49  & \cellcolor[rgb]{ .984,  .859,  .776} 55.18  & \cellcolor[rgb]{ .984,  .867,  .788} 55.02  & \cellcolor[rgb]{ .98,  .855,  .769} 55.98  & \cellcolor[rgb]{ .988,  .906,  .851} 43.49  & \cellcolor[rgb]{ .988,  .894,  .831} 44.52  & \cellcolor[rgb]{ .988,  .91,  .855} 43.16  & \cellcolor[rgb]{ .988,  .898,  .839} 54.49$_{\color{blue}(\pm \phantom{0}8.47)}$\\
        SADE  & \cellcolor[rgb]{ .976,  .808,  .694} 69.64  & \cellcolor[rgb]{ .976,  .82,  .71} \underline{69.77} & \cellcolor[rgb]{ .973,  .796,  .678} \textbf{70.35} & \cellcolor[rgb]{ .973,  .796,  .678} \textbf{58.51} & \cellcolor[rgb]{ .973,  .796,  .678} \textbf{58.96} & \cellcolor[rgb]{ .973,  .796,  .678} \textbf{58.69} & \cellcolor[rgb]{ .976,  .816,  .71} \underline{53.54} & \cellcolor[rgb]{ .976,  .816,  .71} \underline{53.13} & \cellcolor[rgb]{ .976,  .82,  .714} \underline{53.82} & \cellcolor[rgb]{ .976,  .808,  .698} \underline{60.71}$_{\color{blue}(\pm \phantom{0}6.86)}$\\
        BalPoE & \cellcolor[rgb]{ .976,  .804,  .69} \underline{69.82} & \cellcolor[rgb]{ .976,  .827,  .725} 69.32  & \cellcolor[rgb]{ .976,  .8,  .682} 70.26  & \cellcolor[rgb]{ .976,  .804,  .686} 58.26  & \cellcolor[rgb]{ .976,  .8,  .686} 58.78  & \cellcolor[rgb]{ .976,  .804,  .686} \underline{58.47} & \cellcolor[rgb]{ .98,  .827,  .729} 52.08  & \cellcolor[rgb]{ .98,  .831,  .733} 51.46  & \cellcolor[rgb]{ .98,  .831,  .733} 52.36  & \cellcolor[rgb]{ .976,  .82,  .714} 60.09$_{\color{blue}(\pm \phantom{0}7.37)}$\\
        \toprule
        \textbf{DirMixE} & \cellcolor[rgb]{ .973,  .796,  .678} \textbf{70.13} & \cellcolor[rgb]{ .973,  .796,  .678} \textbf{70.88} & \cellcolor[rgb]{ .976,  .8,  .682} \underline{70.29} & \cellcolor[rgb]{ .976,  .8,  .682} \underline{58.38} & \cellcolor[rgb]{ .976,  .8,  .682} \underline{58.85} & \cellcolor[rgb]{ .976,  .812,  .702} 58.02  & \cellcolor[rgb]{ .973,  .796,  .678} \textbf{55.59} & \cellcolor[rgb]{ .973,  .796,  .678} \textbf{55.09} & \cellcolor[rgb]{ .973,  .796,  .678} \textbf{56.25} & \cellcolor[rgb]{ .973,  .796,  .678} \textbf{61.50}$_{\color{blue}(\pm \phantom{0}6.43)}$\\
        \bottomrule
        \end{tabular}%
      \label{tab:ImageNet_Our}%
    \end{table*}%

    \begin{table*}[htbp]
      \centering
      \caption{Performance Comparison on iNaturalist When Training ResNet Models (\textbf{Ours Setting)}}
      \renewcommand\arraystretch{1.0}
      \small
        \begin{tabular}{lcccccccccc}
        \toprule
        \multirow{1.5}[4]{*}{\textbf{Method}} & \multicolumn{3}{c}{\textbf{Forward-LT}} & \multicolumn{3}{c}{\textbf{Uniform}} & \multicolumn{3}{c}{\textbf{Backward-LT}} & \multirow{1.5}[4]{*}{\textbf{Mean}} \\
                \cmidrule(lr){2-4}\cmidrule(lr){5-7}\cmidrule(lr){8-10}     & \textbf{1} & \textbf{2} & \textbf{3} & \textbf{1} & \textbf{2} & \textbf{3} & \textbf{1} & \textbf{2} & \textbf{3} &  \\
        \midrule
        LDAM  & \cellcolor[rgb]{ .945,  .973,  .929}61.33  & \cellcolor[rgb]{ .941,  .969,  .922}61.36  & \cellcolor[rgb]{ .941,  .969,  .922}60.31  & \cellcolor[rgb]{ .98,  .988,  .973}61.58  & \cellcolor[rgb]{ .957,  .976,  .945}63.08  & \cellcolor[rgb]{ .961,  .98,  .949}63.59  & \cellcolor[rgb]{ .957,  .976,  .945}64.10  & \cellcolor[rgb]{ .976,  .988,  .969}62.14  & \cellcolor[rgb]{ .969,  .984,  .957}63.42  & \cellcolor[rgb]{ .965,  .98,  .953}62.32$_{\color{blue}(\pm 1.14)}$  \\
        LA    & \cellcolor[rgb]{ .937,  .965,  .914}61.88  & \cellcolor[rgb]{ .941,  .969,  .922}61.46  & \cellcolor[rgb]{ .945,  .969,  .925}60.19  & \cellcolor[rgb]{ .976,  .988,  .973}61.61  & \cellcolor[rgb]{ .984,  .992,  .98}62.04  & 62.09  & 62.33  & 60.89  & 61.84  & \cellcolor[rgb]{ .98,  .992,  .976}61.59$_{\color{blue}(\pm 0.60)}$  \\
        VS    & 58.58  & 58.00  & \cellcolor[rgb]{ .984,  .992,  .98}57.80  & 60.72  & 61.46  & 62.20  & \cellcolor[rgb]{ .992,  .996,  .988}62.79  & \cellcolor[rgb]{ .973,  .988,  .965}62.25  & \cellcolor[rgb]{ .984,  .992,  .98}62.56  & 60.71$_{\color{blue}(\pm 1.82)}$  \\
        LADE  & \cellcolor[rgb]{ .878,  .933,  .839}64.81  & \cellcolor[rgb]{ .902,  .949,  .871}63.56  & \cellcolor[rgb]{ .89,  .941,  .855}63.31  & \cellcolor[rgb]{ .914,  .953,  .882}64.11  & \cellcolor[rgb]{ .929,  .965,  .91}64.13  & \cellcolor[rgb]{ .922,  .957,  .894}65.07  & \cellcolor[rgb]{ .922,  .961,  .898}65.50  & \cellcolor[rgb]{ .937,  .969,  .918}64.06  & \cellcolor[rgb]{ .933,  .965,  .914}64.93  & \cellcolor[rgb]{ .918,  .957,  .89}64.39$_{\color{blue}(\pm 0.65)}$  \\
        DDC   & 58.58  & \cellcolor[rgb]{ .992,  .996,  .992}58.46  & 56.80  & \cellcolor[rgb]{ .988,  .992,  .98}61.31  & 61.40  & \cellcolor[rgb]{ .988,  .996,  .984}62.56  & \cellcolor[rgb]{ .941,  .969,  .922}64.81  & \cellcolor[rgb]{ .937,  .969,  .918}64.00  & \cellcolor[rgb]{ .949,  .973,  .929}64.33  & \cellcolor[rgb]{ .988,  .992,  .984}61.36$_{\color{blue}(\pm 2.57)}$  \\
        RIDE  & \cellcolor[rgb]{ .827,  .906,  .773}67.33  & \cellcolor[rgb]{ .82,  .902,  .761}68.28  & \cellcolor[rgb]{ .827,  .906,  .773}66.96  & \cellcolor[rgb]{ .788,  .886,  .722}\underline{68.76}  & \cellcolor[rgb]{ .82,  .902,  .765}68.39  & \cellcolor[rgb]{ .808,  .898,  .745}69.27  & \cellcolor[rgb]{ .835,  .91,  .78}69.05  & \cellcolor[rgb]{ .843,  .914,  .792}68.72  & \cellcolor[rgb]{ .843,  .918,  .792}69.04  & \cellcolor[rgb]{ .827,  .906,  .773}68.42$_{\color{blue}(\pm 0.71)}$  \\
        SADE  & \cellcolor[rgb]{ .796,  .89,  .729}68.94  & \cellcolor[rgb]{ .784,  .882,  .718}\underline{70.13}  & \cellcolor[rgb]{ .784,  .882,  .718}\underline{69.52}  & \cellcolor[rgb]{ .792,  .886,  .725}68.64  & \cellcolor[rgb]{ .808,  .894,  .745}68.92  & \cellcolor[rgb]{ .792,  .886,  .725}69.82  & \cellcolor[rgb]{ .812,  .898,  .753}69.95  & \cellcolor[rgb]{ .827,  .906,  .773}69.44  & \cellcolor[rgb]{ .808,  .894,  .745}70.73  & \cellcolor[rgb]{ .8,  .894,  .737}69.57$_{\color{blue}(\pm 0.60)}$  \\
        BalPoE & \cellcolor[rgb]{ .784,  .882,  .714}\underline{69.46}  & \cellcolor[rgb]{ .804,  .894,  .741}69.03  & \cellcolor[rgb]{ .827,  .906,  .773}66.99  & \cellcolor[rgb]{ .788,  .886,  .722}68.70  & \cellcolor[rgb]{ .784,  .882,  .718}\underline{69.73}  & \cellcolor[rgb]{ .78,  .882,  .714}\underline{70.22}  & \cellcolor[rgb]{ .776,  .878,  .706}\textbf{71.30} & \cellcolor[rgb]{ .78,  .882,  .714}\underline{71.75}  & \cellcolor[rgb]{ .784,  .882,  .714}\underline{71.84}  & \cellcolor[rgb]{ .792,  .89,  .725}\underline{69.89}$_{\color{blue}(\pm 1.42)}$  \\
        \midrule
        \textbf{DirMixE} & \cellcolor[rgb]{ .776,  .878,  .706}\textbf{69.75} & \cellcolor[rgb]{ .776,  .878,  .706}\textbf{70.49} & \cellcolor[rgb]{ .776,  .878,  .706}\textbf{69.88} & \cellcolor[rgb]{ .776,  .878,  .706}\textbf{69.13} & \cellcolor[rgb]{ .776,  .878,  .706}\textbf{70.00} & \cellcolor[rgb]{ .776,  .878,  .706}\textbf{70.34} & \cellcolor[rgb]{ .78,  .882,  .71}\underline{71.24}  & \cellcolor[rgb]{ .776,  .878,  .706}\textbf{71.91} & \cellcolor[rgb]{ .776,  .878,  .706}\textbf{72.02} & \cellcolor[rgb]{ .776,  .878,  .706}\textbf{70.53}$_{\color{blue}(\pm 0.89)}$ \\
        \bottomrule
        \end{tabular}%
      \label{tab:inat-ours}%
    \end{table*}%

\begin{textbo}
    
\subsection{Discussions}

\subsubsection{Novelty of the Theoretical Analysis}

The novelty of the theoretical analysis is four-fold:

\textbf{Hierarchical Decomposition}: Thm.\ref{thm:gen} introduces a hierarchical concentration framework decomposing generalization error. It accounts for stochasticity from two sources: finite training data $\mathcal{S}$ and finite sampling of label distributions $\mathcal{P}$ from meta-distribution $\mathcal{E}$. The proof decomposes excess risk into: (i) meta-distribution shift $\mathfrak{Err}_{approx}$, (ii) Monte Carlo sampling error $\mathfrak{Err}_{sto}$, and (iii) data estimation error.

\textbf{Induced Subspace for Sharper Bound}: To avoid loose worst-case bounds, we refine the hypothesis space via the Induced Subclass (Def.\ref{def: induced subclass}), ensuring models reside in a well-behaved region. Novel conditional concentration inequalities (Lem. \tb{6}-\tb{8}) replace the global bound $B$ with a data-dependent $\rho$ in stochastic error terms, leading to $\mathfrak{Err}_{sto} = O(\rho^{\mathcal{I}}/\sqrt{N} + \rho^{\mathcal{M}}/M)$, tighter than standard $O(B/\sqrt{N} + B/\sqrt{M})$.

\textbf{Semi-variance Regularization}: Semi-variance regularization yields a sharper bound. Thm.\ref{thm:rho} and Thm.\ref{thm:multi} show that for light- and heavy-tailed losses (Exponential, Gamma, Pareto, mixtures), semi-variance $\mathbb{V}_{+}$ is comparable to full variance $\mathbb{V}$. This allows using semi-variance $\hat{\mathbb{V}}_{+}$ in $\mathfrak{Reg}$ and extending Bernstein-like inequalities to hierarchical problems.

\textbf{PEFT for Foundation Models}: Thm.\ref{thm:gen1} builds on Thm.\ref{thm:gen} to derive a sharp generalization bound for PEFT, where model complexity scales only with trainable parameters, not total parameters. To do this, we first use first-order Taylor expansion around reference points $\Theta^{(i),*}$ and $\delta$-compact parametrization (Def.\ref{def: delta compact parameterization}) to partition parameter space into Voronoi cells. This ensures functions in each ball $\mathcal{B}_i$ share a reference point, and replace the entire backbone with the trainable parameters $\Delta \Theta^{(i)}$. Then, we apply covering number bounds for low-rank matrices (Lem.\tb{10}) and Stiefel manifolds (Lem.\tb{9}), depending only on LoRA/AdaptFormer dimensions. Thus, the complexity $\nu$ scales with trainable parameters, yielding tight guarantees.

\subsubsection{Existing Results used in the Proof}

The theoretical derivations in this work are built upon a solid foundation of established mathematical tools, which are strategically extended and combined to address the novel challenges of our setting.

\textbf{Covering Numbers and Uniform Convergence.} The use of covering numbers to construct uniform convergence bounds is a standard technique in statistical learning theory. We adopt the standard assumption that the covering number of our hypothesis class $\mathcal{F}$ scales polynomially. This is a conventional and widely adopted assumption \cite{anthony2009neural, barron1999risk, pollard2012convergence, DBLP:conf/iclr/LongS20, bach2024learning} that holds for many common models. 

\textbf{Extension of Basic Concentration Inequalities.} The core tools we build upon are the classic Hoeffding's inequality (for bounded random variables) and the Empirical Bernstein Inequality (as in Maurer \& Pontil \cite{bern} ), which provides a data-dependent bound based on the empirical variance.

\textbf{Parameter Space Partition via Voronoi Diagram in Theorem \ref{thm:gen1}.} The analysis of the LSF scheme in Thm.\ref{thm:gen1} requires controlling the complexity of the linearized hypothesis class. To achieve this, we introduce the Voronoi Diagram to induce Parameter Space Partition (Def.\ref{def: voronoi diagram}). 

\textbf{Variance-based Concentration.} For variance-penalized objectives, our results are built upon existing studies like Maurer \& Pontil \cite{bern}. This provides a bound of the form $\mathbb{E}[\ell] \lesssim \hat{\mathbb{E}}[\ell] + \sqrt{\hat{\mathbb{V}}[\ell] / N} + 1/N$. We extend it to deal with conditional expectations and hierarchical probabilistic models.

\textbf{Complexity Bound for Homogeneous Space.} To bound the complexity of the PEFT-based hypothesis classes in LSF, we rely on advanced results from geometric functional analysis. Specifically, we apply the covering number bounds for homogeneous space from Szarek \cite{metric} to construct the covering number of the low-rank Stiefel manifold.

\end{textbo}

  \section{Experiments}
   
  In this section, we conduct a series of empirical studies to demonstrate the effectiveness of our proposed algorithm. \textbf{Due to space limitations, please refer to Appendix \tb{M} and \tb{N} for more details and experiments.}

  \begin{table*}[htbp]
    \centering
    \caption{Performance Comparison on ImageNet-LT When Fine-tuning with LoRA (\textbf{Ours Setting)}}
        \begin{tabular}{lcccccccccc}
        \toprule
        \multirow{1.5}[4]{*}{\textbf{Method}} & \multicolumn{3}{c}{\textbf{Forward-LT}} & \multicolumn{3}{c}{\textbf{Uniform}} & \multicolumn{3}{c}{\textbf{Backward-LT}} & \multirow{1.5}[4]{*}{\textbf{Mean}} \\
    \cmidrule(lr){2-4}\cmidrule(lr){5-7}\cmidrule(lr){8-10}          & \textbf{1} & \textbf{2} & \textbf{3} & \textbf{1} & \textbf{2} & \textbf{3} & \textbf{1} & \textbf{2} & \textbf{3} &  \\
        \midrule
        LIFT-LoRA & 78.80 & 78.32 & 79.08 & \cellcolor[rgb]{ .973,  .796,  .678}\textbf{75.03} & \cellcolor[rgb]{ .976,  .808,  .694}\underline{75.62} & \cellcolor[rgb]{ .973,  .796,  .678}\textbf{75.61} & 72.50 & 71.49 & 72.27 & 75.41$_{\color{blue}(\pm 2.74)}$ \\
        SADE-LoRA & \cellcolor[rgb]{ .98,  .839,  .745}82.06 & \cellcolor[rgb]{ .976,  .82,  .718}81.77 & \cellcolor[rgb]{ .98,  .847,  .757}81.71 & 73.47 & 74.65 & 74.05 & \cellcolor[rgb]{ 1,  .98,  .965}73.06 & \cellcolor[rgb]{ .992,  .937,  .898}73.16 & \cellcolor[rgb]{ .996,  .969,  .945}73.09 & \cellcolor[rgb]{ .992,  .933,  .898}76.34$_{\color{blue}(\pm 3.93)}$ \\
        DirMixE-LoRA & \cellcolor[rgb]{ .98,  .831,  .733}\underline{82.21} & \cellcolor[rgb]{ .976,  .82,  .714}\underline{81.79} & \cellcolor[rgb]{ .976,  .816,  .71}\underline{82.19} & \cellcolor[rgb]{ .992,  .941,  .906}73.93 & \cellcolor[rgb]{ 1,  .992,  .984}74.70 & \cellcolor[rgb]{ 1,  .984,  .973}74.19 & \cellcolor[rgb]{ .976,  .824,  .722}\underline{76.74} & \cellcolor[rgb]{ .976,  .804,  .69}\underline{76.48} & \cellcolor[rgb]{ .976,  .8,  .686}\underline{76.97} & \cellcolor[rgb]{ .98,  .839,  .741}\underline{77.69}$_{\color{blue}(\pm 3.26)}$ \\
        \midrule
        \textbf{DirMixE-LoRA-LSF} & \cellcolor[rgb]{ .973,  .796,  .678}\textbf{82.87} & \cellcolor[rgb]{ .976,  .808,  .694}\textbf{82.03} & \cellcolor[rgb]{ .973,  .796,  .678}\textbf{82.51} & \cellcolor[rgb]{ .976,  .808,  .694}\underline{74.96} & \cellcolor[rgb]{ .973,  .796,  .678}\textbf{75.66} & \cellcolor[rgb]{ .984,  .867,  .788}\underline{75.08} & \cellcolor[rgb]{ .973,  .796,  .678}\textbf{77.36} & \cellcolor[rgb]{ .973,  .796,  .678}\textbf{76.61} & \cellcolor[rgb]{ .973,  .796,  .678}\textbf{77.06} & \cellcolor[rgb]{ .973,  .796,  .678}\textbf{78.24}$_{\color{blue}(\pm 3.10)}$ \\
        \bottomrule
        \end{tabular}%
    \label{tab:performance_comparison_imagenet_lora}%
  \end{table*}%

  \begin{table*}[htbp]
    \centering
    \caption{Performance Comparison on ImageNet-LT When Fine-tuning with AdaptFormer (\textbf{Ours Setting)}}
    \small
        \begin{tabular}{lcccccccccc}
        \toprule
        \multirow{1.5}[4]{*}{\textbf{Method}} & \multicolumn{3}{c}{\textbf{Forward-LT}} & \multicolumn{3}{c}{\textbf{Uniform}} & \multicolumn{3}{c}{\textbf{Backward-LT}} & \multirow{1.5}[4]{*}{\textbf{Mean}} \\
    \cmidrule(lr){2-4}\cmidrule(lr){5-7}\cmidrule(lr){8-10}          & \textbf{1} & \textbf{2} & \textbf{3} & \textbf{1} & \textbf{2} & \textbf{3} & \textbf{1} & \textbf{2} & \textbf{3} &  \\
        \midrule
        LIFT-AF & 79.86 & 79.89 & 80.63 & \cellcolor[rgb]{ .973,  .796,  .678}\textbf{76.80} & \cellcolor[rgb]{ .973,  .796,  .678}\textbf{77.51} & \cellcolor[rgb]{ .973,  .796,  .678}\textbf{77.06} & 74.01 & 73.48 & 73.74 & 77.00$_{\color{blue}(\pm 2.62)}$ \\
        SADE-AF & \cellcolor[rgb]{ .976,  .812,  .702}83.03 & \cellcolor[rgb]{ .976,  .82,  .714}82.98 & \cellcolor[rgb]{ .976,  .808,  .698}\underline{83.38} & \cellcolor[rgb]{ .992,  .925,  .882}\underline{76.46} & 75.54 & \cellcolor[rgb]{ .996,  .965,  .945}76.40 & \cellcolor[rgb]{ .992,  .933,  .894}75.36 & \cellcolor[rgb]{ .988,  .89,  .824}75.58 & \cellcolor[rgb]{ .992,  .922,  .871}75.31 & \cellcolor[rgb]{ .988,  .894,  .835}78.16$_{\color{blue}(\pm 3.49)}$ \\
        DirMixE-AF & \cellcolor[rgb]{ .976,  .8,  .682}\underline{83.25} & \cellcolor[rgb]{ .976,  .808,  .694}\underline{83.19} & \cellcolor[rgb]{ .976,  .812,  .702}83.35 & \cellcolor[rgb]{ 1,  .996,  .988}76.28 & \cellcolor[rgb]{ .98,  .847,  .757}77.04 & 76.26 & \cellcolor[rgb]{ .98,  .839,  .745}\underline{77.24} & \cellcolor[rgb]{ .976,  .824,  .718}\underline{76.82} & \cellcolor[rgb]{ .98,  .839,  .745}\underline{76.84} & \cellcolor[rgb]{ .976,  .827,  .725}\underline{78.92}$_{\color{blue}(\pm 3.09)}$ \\
        \midrule
        \textbf{DirMixE-AF-LSF} & \cellcolor[rgb]{ .973,  .796,  .678}\textbf{83.27} & \cellcolor[rgb]{ .973,  .796,  .678}\textbf{83.35} & \cellcolor[rgb]{ .973,  .796,  .678}\textbf{83.54} & 76.26 & \cellcolor[rgb]{ .98,  .839,  .745}\underline{77.12} & \cellcolor[rgb]{ .992,  .929,  .89}\underline{76.54} & \cellcolor[rgb]{ .973,  .796,  .678}\textbf{78.04} & \cellcolor[rgb]{ .973,  .796,  .678}\textbf{77.27} & \cellcolor[rgb]{ .973,  .796,  .678}\textbf{77.63} & \cellcolor[rgb]{ .973,  .796,  .678}\textbf{79.22}$_{\color{blue}(\pm 2.99)}$ \\
        \bottomrule
        \end{tabular}%
    \label{tab:performance_comparison_imagenet_adaptformer}%
  \end{table*}%

  \begin{table*}[htbp]
    \centering
    \caption{Performance Comparison on iNaturalist When Fine-tuning with AdaptFormer (\textbf{Ours Setting)}}
    \small
        \begin{tabular}{lcccccccccc}
        \toprule
        \multirow{1.5}[4]{*}{\textbf{Method}} & \multicolumn{3}{c}{\textbf{Forward-LT}} & \multicolumn{3}{c}{\textbf{Uniform}} & \multicolumn{3}{c}{\textbf{Backward-LT}} & \multirow{1.5}[4]{*}{\textbf{Mean}} \\
    \cmidrule(lr){2-4}\cmidrule(lr){5-7}\cmidrule(lr){8-10}          & \textbf{1} & \textbf{2} & \textbf{3} & \textbf{1} & \textbf{2} & \textbf{3} & \textbf{1} & \textbf{2} & \textbf{3} &  \\
        \midrule
        LIFT-AF & \cellcolor[rgb]{ .922,  .953,  .98}\underline{75.69} & \cellcolor[rgb]{ .941,  .965,  .988}76.15 & 74.06 & \cellcolor[rgb]{ .898,  .941,  .976}\underline{77.43} & \cellcolor[rgb]{ .878,  .925,  .969}77.18 & \cellcolor[rgb]{ .937,  .961,  .984}77.92 & \cellcolor[rgb]{ .957,  .976,  .992}79.06 & 78.18 & 78.80 & 77.16$_{\color{blue}(\pm 1.52)}$ \\
        SADE-AF & \cellcolor[rgb]{ .741,  .843,  .933}\textbf{76.68} & \cellcolor[rgb]{ .757,  .851,  .937}\underline{76.92} & \cellcolor[rgb]{ .741,  .843,  .933}\textbf{75.18} & 76.74 & 76.66 & 77.68 & 78.72 & \cellcolor[rgb]{ .957,  .976,  .992}78.52 & \cellcolor[rgb]{ .961,  .976,  .992}79.04 & \cellcolor[rgb]{ .953,  .973,  .988}77.35$_{\color{blue}(\pm 1.17)}$ \\
        DirMixE-AF & 75.25 & 75.90 & \cellcolor[rgb]{ .878,  .929,  .969}\underline{74.59} & \cellcolor[rgb]{ .925,  .957,  .98}77.25 & \cellcolor[rgb]{ .867,  .918,  .969}\underline{77.23} & \cellcolor[rgb]{ .788,  .871,  .945}\underline{78.47} & \cellcolor[rgb]{ .859,  .914,  .965}\underline{79.82} & \cellcolor[rgb]{ .761,  .855,  .941}\underline{80.07} & \cellcolor[rgb]{ .816,  .89,  .953}\underline{79.87} & \cellcolor[rgb]{ .89,  .933,  .973}\underline{77.61}$_{\color{blue}(\pm 1.96)}$ \\
        \midrule
        \textbf{DirMixE-AF-LSF} & \cellcolor[rgb]{ .933,  .961,  .984}75.64 & \cellcolor[rgb]{ .741,  .843,  .933}\textbf{76.97} & \cellcolor[rgb]{ .741,  .843,  .933}\textbf{75.18} & \cellcolor[rgb]{ .741,  .843,  .933}\textbf{78.47} & \cellcolor[rgb]{ .741,  .843,  .933}\textbf{77.74} & \cellcolor[rgb]{ .741,  .843,  .933}\textbf{78.63} & \cellcolor[rgb]{ .741,  .843,  .933}\textbf{80.69} & \cellcolor[rgb]{ .741,  .843,  .933}\textbf{80.21} & \cellcolor[rgb]{ .741,  .843,  .933}\textbf{80.29} & \cellcolor[rgb]{ .741,  .843,  .933}\textbf{78.20}$_{\color{blue}(\pm 1.89)}$ \\
        \bottomrule
        \end{tabular}%
    \label{tab:performance_comparison_inaturalist_adaptformer}%
  \end{table*}%

  \subsection{Implentation Details.}\label{sec:imp}

  \textbf{Training ResNet Models.} 
  Following~\cite{DBLP:conf/nips/ZhangHHF22, DBLP:conf/cvpr/AimarJFK23}, we employ ResNeXt-50~\cite{xie2017aggregated} and ResNet-32~\cite{he2016deep} as the backbones for ImageNet-LT and CIFAR-LT datasets, respectively. In CIFAR-LT experiments, we train the model for $200$ epochs using Stochastic Gradient Descent (SGD)~\cite{robbins1951stochastic}. The initial learning rate is set at $0.1$, with $0.9$ momentum rate and $128$ batch size. Moreover, a step learning rate schedule is adopted, which reduces the learning rate by a factor of $10$ at the $160$-th and $180$-th epoch, respectively. Regarding the ImageNet-LT dataset, the model is trained $180$ epochs using SGD. Here, the initial learning rate is $0.025$ with $0.9$ momentum and $64$ batch size. Then, the learning rate is adjusted through a cosine annealing schedule, which gradually declines from $0.025$ to $0$ over $180$ epochs. Finally, for the sake of fair comparisons, we re-implement the above methods using their publicly available code and conduct experiments on the same device.

  \newcont{

  \vspace{1\baselineskip}

  \noindent \textbf{Fine-tuning Foundation Models.} 
  We follow the protocol described in LIFT \cite{DBLP:conf/icml/Shi00SH024}. Specifically, we fine-tune the image encoder of CLIP \cite{DBLP:conf/icml/RadfordKHRGASAM21} with a ViT-B/16 \cite{DBLP:conf/iclr/DosovitskiyB0WZ21} backbone. We apply two PEFT methods: LoRA \cite{DBLP:conf/iclr/HuSWALWWC22} and AdaptFormer \cite{DBLP:conf/nips/ChenGTWSWL22}. A cosine classifier is added on top of the image encoder, with its weights initialized using the text features of class prompts (\textit{e.g.}, "a photo of a [CLASS]") extracted by the CLIP text encoder. After initialization, the text encoder is removed. We use stochastic gradient descent (SGD) as the optimizer, with a batch size of 128 and a momentum of 0.9. The initial learning rate varies by dataset: 0.1 for CIFAR-LT, 0.025 for ImageNet-LT, and 0.01 for iNaturalist. A cosine learning rate scheduler is used to gradually reduce the learning rate to zero during training. The model is trained for 10 epochs on CIFAR-LT and ImageNet-LT, and for 20 epochs on iNaturalist. 

  }

  \subsection{Datasets}
  \textbf{Dataset Descriptions.} We conduct experiments on four popular benchmark datasets for imbalanced learning: \textbf{(a) CIFAR-10-LT and CIFAR-100-LT} datasets. The original CIFAR-10 and CIFAR-100 datasets~\cite{krizhevsky2009learning} have $50,000$ images for training and $10,000$ images for validation with $10$ and $100$ categories, respectively. Following~\cite{DBLP:conf/cvpr/CuiJLSB19}, we use the long-tailed version of CIFAR 10 and CIFAR 100 datasets with imbalanced ratio $\rho = N_{\max}/N_{\min} = 100$. \textbf{(b) ImageNet-LT} dataset. We adopt the ImageNet-LT dataset proposed by~\cite{openlongtailrecognition}, which is sampled from ImageNet~\cite{DBLP:conf/cvpr/DengDSLL009} following the \textit{Pareto} distribution. Briefly, it consists of $115.8K$ images from $1000$ classes, with $1280$ images in the most frequent class and $5$ images in the minority. \newcont{ \textbf{(c) iNaturalist} dataset. We use the iNaturalist 2018 dataset~\cite{DBLP:conf/cvpr/HornASCSSAPB18}, which contains $437K$ images from $8142$ classes. The most frequent class contains $1,000$ images, while the least frequent class has only $2$ images. }

  \subsection{Competitors}\label{sec:compete}

  \textbf{Training ResNet Models.} We compare our method with several baselines for training ResNet models, including:

  \begin{itemize}
  \item \textbf{Label-Distribution-Aware Margin (LDAM)} ~\cite{DBLP:conf/nips/CaoWGAM19} improves the performance of minority classes by encouraging larger margins for minority classes. 
  \item \textbf{Logit Adjustment (LA)} ~\cite{DBLP:conf/iclr/MenonJRJVK21} advances the conventional softmax cross-entropy by ensuring Fisher consistency in minimizing the balanced error.
  \item \textbf{Vector Scaling (VS)} ~\cite{kini2021labelimbalanced} proposes to leverage both multiplicative and additive logit adjustments to address label imbalance problems.
  \item \textbf{LAbel distribution DisEntangling (LADE)} ~\cite{lade} X regards the long-tailed learning as a label shift problem and aims to disentangle the source label distribution from the model prediction to match the target label distribution during training.
  
  \item \textbf{Data Dependent Contraction (DDC)} ~\cite{ddc} designs a Deferred Re-Weighting scheme to boost performance for imbalanced learning, which is also compatible with the VS loss.
  \item \textbf{RoutIng Diverse Experts (RIDE)} ~\cite{DBLP:conf/iclr/WangLM0Y21} proposes a distribution-aware multiple expert routing system, which can efficiently reduce both model bias and variance for imbalanced learning.
  
  \item \textbf{Self-supervised Aggregation of Diverse Experts (SADE)} ~\cite{DBLP:conf/nips/ZhangHHF22} is a state-of-the-art test-agnostic long-tailed learning method, which trains multiple diverse experts to tackle different class distributions and adopts self-supervision to aggregate the decisions of multiple experts to adapt unknown test distributions.
  \item \textbf{Balanced Product of Experts (BalPoE)} ~\cite{DBLP:conf/cvpr/AimarJFK23} is another SOTA long-tailed learning algorithm that successfully extends logit adjustment to the mixture of experts.

  \end{itemize}

  \newcont{

  \noindent \textbf{Fine-tuning Foundation Models.} We also compare our method with several competitors for fine-tuning foundation models, including:

  \begin{itemize}
    \item \textbf{Self-supervised Aggregation of Diverse Experts (SADE)} ~\cite{DBLP:conf/nips/ZhangHHF22} We re-implement SADE based on the official code to fine-tune the CLIP foundation model using LoRA and AdaptFormer.

    \item \textbf{LIghtweight Fine-Tuning (LIFT)} \cite{DBLP:conf/icml/Shi00SH024} is a state-of-the-art method that applies PEFT techniques to adjust the image encoder of CLIP. It uses LA loss as the training objective and leverages semantic knowledge from CLIP's text encoder to initialize the classifier added during fine-tuning.
  \end{itemize}

  }

  \subsection{The Choice of meta-distribution and Experts}\label{sec:meta}
  We briefly introduce the choice of \eqref{eq:p} and \eqref{eq:i} to define the mixture distribution. Detailed implementations are shown in the appendix. Drawing inspiration from the skill-diverse expert learning approach in prior art, we employ a three-component mixture model for \eqref{eq:p} to encapsulate three critical skills. Each component is determined by a specific choice of $\alpha$, the parameter for Dirichlet distribution. The \underline{forward} component $\bm{\alpha}^{(f)}$ aligns with the training label distribution, indicative of performance in the head classes. The \underline{uniform} component $\bm{\alpha}^{(u)}$ corresponds to a uniform distribution, reflecting adherence to the conventional long-tail testing protocol.
  The \underline{backward} component represents an inverse long-tail distribution of the training set (where head classes are transformed into tail classes and \textit{vice versa}), signifying performance in the tail distribution. Furthermore, for \eqref{eq:i},  we chose a uniform distribution ($p_1 = p_2 = p_3 = 1/3$) in our model, considering the equal significance of all three skills. This prevents the unfair oversight of any particular skill.
  

  \subsection{Experiment Protocols}
  
  \textbf{Evaluation Protocols.} We evaluate the performance of various methods across multiple test datasets. Specifically, we employ \textbf{two regimes} to generate test datasets: \textbf{(a) Ours Setting.} We generate test data by sampling from the perturbed version of the meta-distribution in Sec.\tb{M}. Subsequently, for each Dirichlet distribution, we sample \textbf{three label distributions} for testing on three forward/uniform/backward LT distributions, respectively (As in Tab.\ref{tab:Cifar100_Our}). \textbf{(b) SADE's Setting}:~\cite{DBLP:conf/nips/ZhangHHF22}. Following SADE, the test datasets usually fall into one of three distribution types: (forward) \textit{long-tail, uniform, and backward long-tail}, each defined by a different imbalance degree $\rho$. \textbf{Please see Appendix \tb{M} for more details.}

  \subsection{\newsec{Overall Performance When Training ResNet Models}}
  \label{sec:res}
  \newcont{Tab.\ref{tab:Cifar100_Our}-\ref{tab:inat-ours} compare the overall performance on CIFAR-100, ImageNet and iNaturalist for Our's setting when training ResNet models, while the results for SADE's setting are shown in Tab.\tb{10}-\tb{12} in Appendix \tb{N}}. For fairness, we only compare the performance using MoE scheme and do not use mixup for all the competitors. Moreover, we adopt the self-supervised aggregation method of SADE to align test-time operations for all the MoE-based Models (SADE, BalPoE, \ours) except RIDE. The rationale for the exception is that all the experts in RIDE are designed for the same distribution, and there is already a routing strategy to choose the experts.
  We have the following observations on the results:
  
  \begin{enumerate}
      \item[a)] \textbf{Overall empirical trends}: 1) Our method shows similar performance to SOTA methods like BalPoE and SADE for \textbf{Forward-LT} and \textbf{Uniform}. 2) However, in \textbf{Backward-LT}, our improvements are more substantial. For instance, in Tab.\ref{tab:Cifar100_Our} (CIFAR-100, our setting), the performance gain varies from 2.4 to 4.1. In Tab.\ref{tab:ImageNet_Our} (ImageNet-LT, our setting), it ranges from 2.0 to 2.4. Thanks to these gains in \texttt{Backward-LT}, our method consistently achieves the best average performances, demonstrating its effectiveness.
      \item[b)] \textbf{Performance differences across datasets}: CIFAR series performances are comparable, likely due to their similar data distributions and scales. In contrast, the ImageNet dataset, with its distinct data distribution and scale, shows slightly different trends. Yet, the results still follow the pattern noted in a). Even though our method slightly lags behind SOTA methods on the ImageNet dataset for \texttt{Uniform}, the performance differences are mostly less than 0.5. Conversely, the improvement in \texttt{Backward} distributions is more significant, ensuring our method maintains the best average performances.
      \item[c)] \textbf{Explaining the differences}: The distinct distribution of the ImageNet dataset compared to CIFAR-10 and CIFAR-100 may account for these discrepancies. The greater variation between different label distributions means the model must focus more on \texttt{Backward} distributions, slightly compromising performance in \texttt{Uniform} and \texttt{Forward} distributions to lead in overall performance.
  \end{enumerate}
  See Appendix \tb{N.7} for comparisions with other baselines.

  \begin{figure*}[t]  
      \centering
      \subfigure[Expert 1 (Forward)]{
        \includegraphics[width=0.31\textwidth]{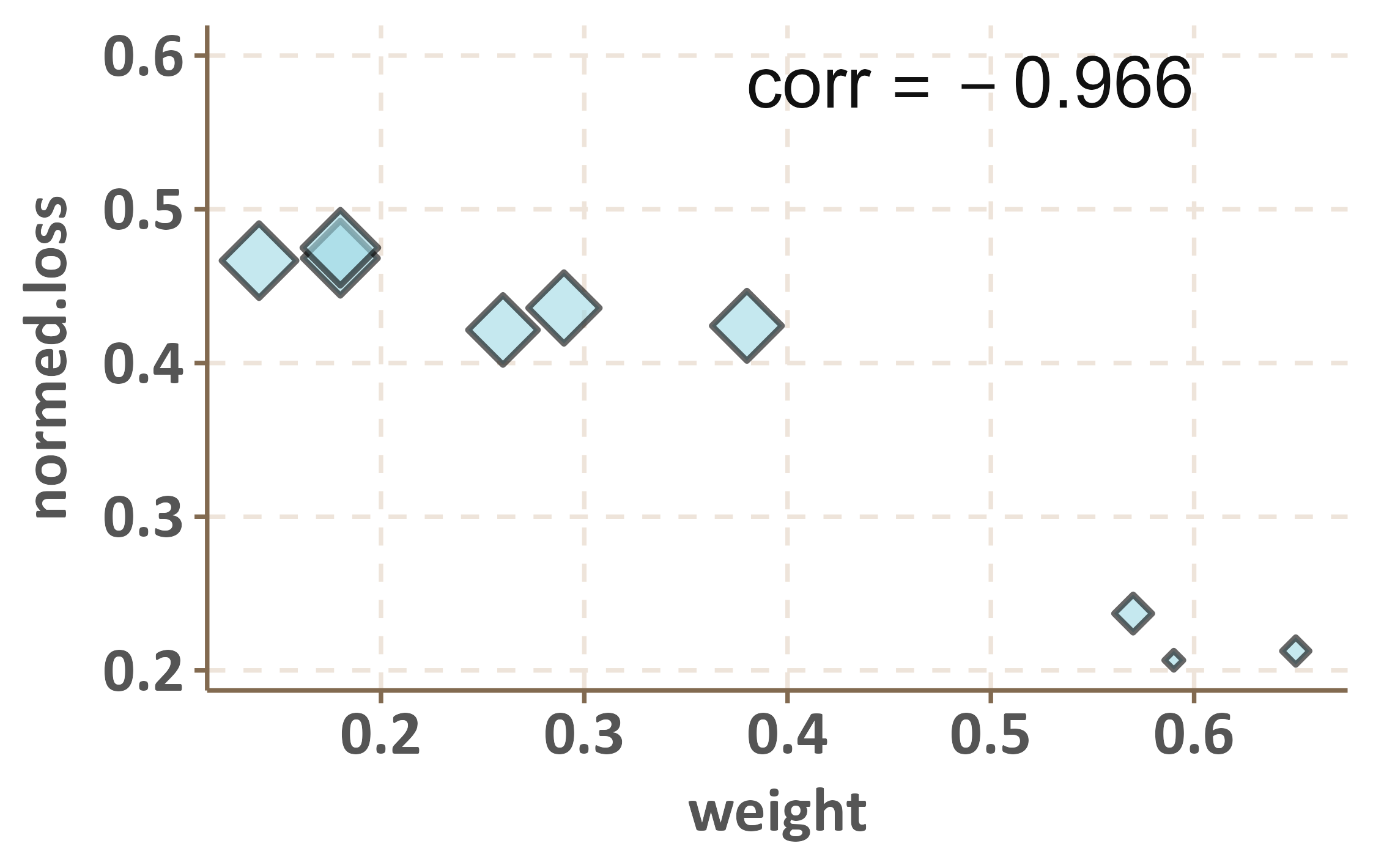} 
      }
      \subfigure[Expert 2 (Uniform)]{
        \includegraphics[width=0.31\textwidth]{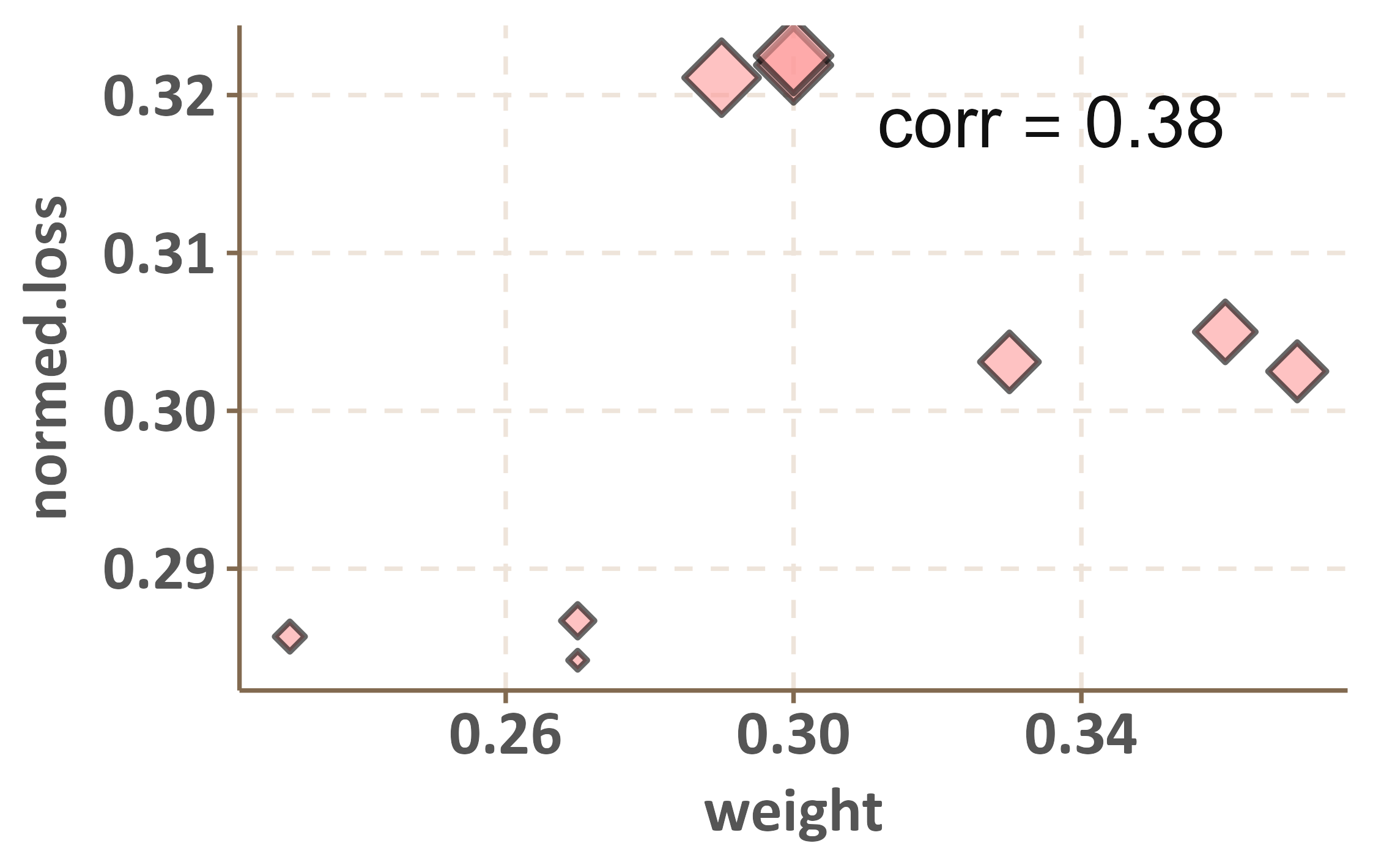} 
      }
      \subfigure[Expert 3 (Backward)]{
        \includegraphics[width=0.31\textwidth]{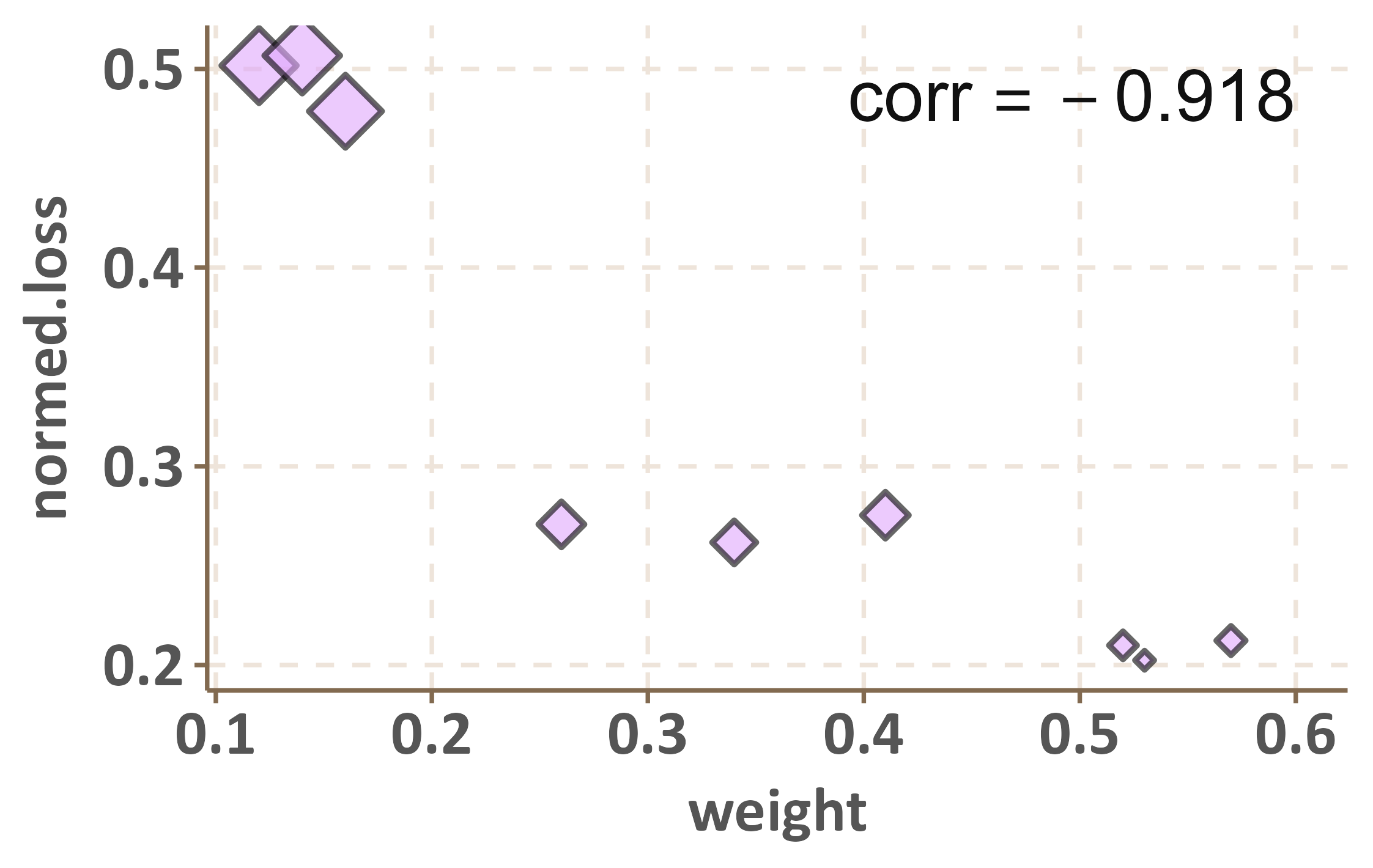} 
      }
      \caption{\label{fig:corr} \textbf{The Correlation between Expert Weights and Loss.}}
  \end{figure*}

  \begin{figure}[h]  
    \centering
    \subfigure[Our's Setting]{
    
      \includegraphics[width=0.42\textwidth]{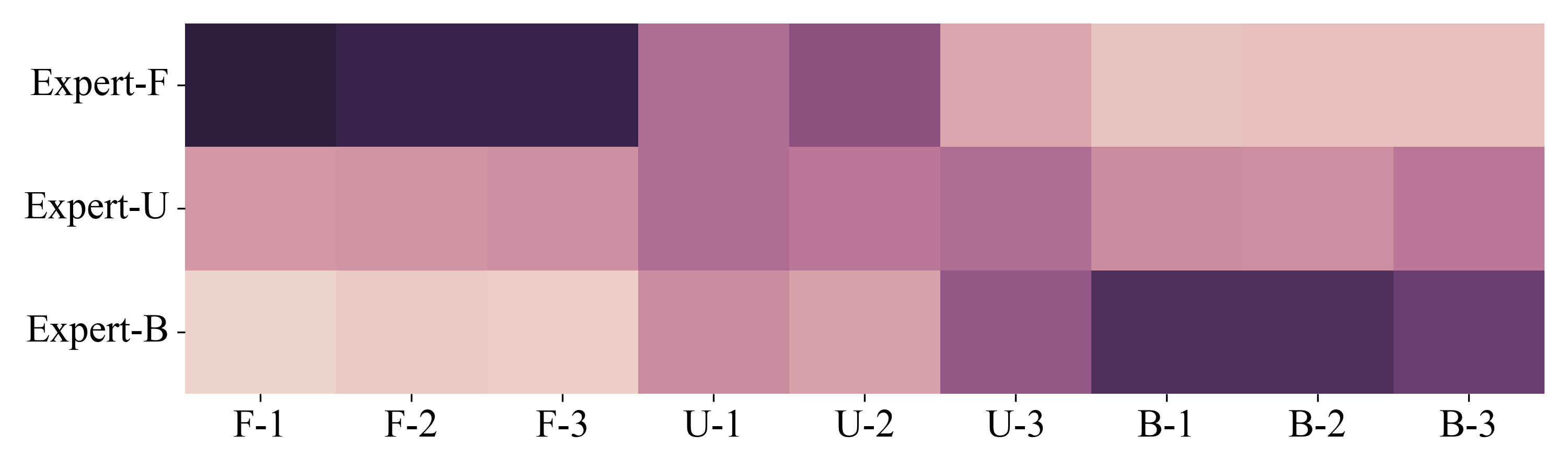} 
    }
    \subfigure[SADE's Setting]{
      \includegraphics[width=0.42\textwidth]{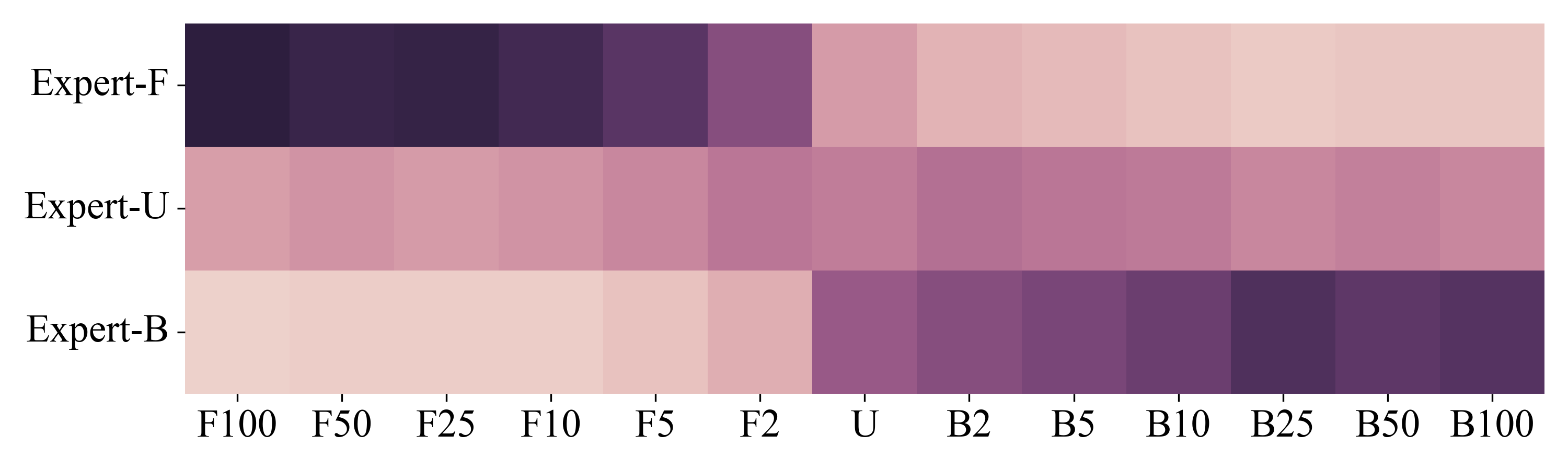} 
    }
    
    \caption{\label{fig:weights} \textbf{Weight Assignment in the Self-supervised Aggregations on CIFAR-100.} \textbf{F,U,B} represent the forward, uniform and backward distributions. For the x-axis, \textbf{F-1,F-2,F-3} in \textbf{(a)} denote the three observed label distributions in the test data respectively. \textbf{F-2,F-5,$\cdots$,F-100} represents the corresponding imbalance ratio of the forward distribution under SADE's setting. The case for \textbf{the suffices of U and B} are similar. \textbf{Expert-F,U,B} represents the experts assigned to the forward, backward, uniform Dirichlet distribution, respectively.}
  \end{figure}

  \subsection{\newsec{Overall Performance When Fine-tuning Foundation Models}}\label{sec:foundexp}
  \newcont{
  Tab.\ref{tab:performance_comparison_imagenet_lora}-\ref{tab:performance_comparison_inaturalist_adaptformer} show the overall performance on ImageNet and iNaturalist for Ours setting when fine-tuning foundation models using LoRA and AdaptFormer. The results on CIFAR-10 and CIFAR-100 for Ours setting, and on CIFAR-10, CIFAR-100, ImageNet, and iNaturalist for SADE's setting, are provided in Appendix \tb{N.3} and \tb{N.4}, respectively. From these results, we observe that DirMixE consistently outperforms the other baselines across all datasets when fine-tuning with the same PEFT method (LoRA or AdaptFormer), demonstrating its effectiveness. We further observe that DirMixE-LSF achieves better performance than DirMixE across all datasets. This suggests that the latent skill fine-tuning (LSF) strategy can further improves DirMixE's performance, confirming the effectiveness of LSF.
  
  }

  \subsection{Experts Assignment}
  In this part, we validate the ability of the test-time self-supervised aggregation by visualizing the weight assignments for different label distributions in Fig.\ref{fig:weights}. The forward and backward experts always tend to have a significant weight for their corresponding distributions. Uniform distributions tend to utilize all three experts. This is because tail and head classes are equally crucial for uniform distribution. Please see Appendix \tb{N.5} for more results.

  \subsection{Correlation between Weights and Losses}
  Fig.\ref{fig:corr} shows the correlation between expert weights of \ours~ during the test phase and their corresponding loss. We normalize the losses to align the magnitude of the loss on different distributions, where the normalized loss of the expert $i$ is $\ell_i /\sum_{i=1}^3 \ell_i$. The results show a strong negative correlation on the forward and backward experts, and a much weaker positive correlation for uniform ones. This is because uniform distributions do not have a significant bias on head/tail classes, producing a relatively stable average performance across different distributions. Above all, in most cases, we can observe negative correlations between loss and expert weight. According to Thm.\ref{thm:cov}, the negative correlation tends to reduce the generalization error of the test-time aggregation scheme, validating the reasonability of observed performance advantage.
  \subsection{The Semi-Variance/Variance Ratio}
  Recall Thm.\ref{thm:gen}, we adopt the assumption that $\V(\la) \ee V_+(\la)$. To validate this assumption, we calculate the semi-variance/variance ratio ($\rho$) in CIFAR-10 and CIFAR-100. We find that $\rho =0.503, 0.509$, respectively for CIFAR-10 and 100, which obviously aligns with the assumption.
  

  \subsection{Fine-grained Performance}
  In addition to the overall accuracy of test datasets, we also examine the effectiveness of DirMixE across many-shot, medium-shot, and few-shot classes. For CIFAR 100-LT and ImageNet-LT, classes are divided into many-shot ($> 100$), medium-shot ($20 \sim 100$), and few-shot ($< 20$) categories. In the case of CIFAR 10-LT, classes are split into many-shot ($> 1000$), medium-shot ($200 \sim 1000$), and few-shot ($< 200$) categories. All these statistics are based on the training label distribution. The results are illustrated in heat maps in Fig.\ref{fig:fine-cifar100} for CIFAR-100. Please see Appendix \tb{N.6} for more results.

  \begin{figure}[t]  
    \centering
    \subfigure[Our's Setting]{
    
      \includegraphics[width=0.86\columnwidth]{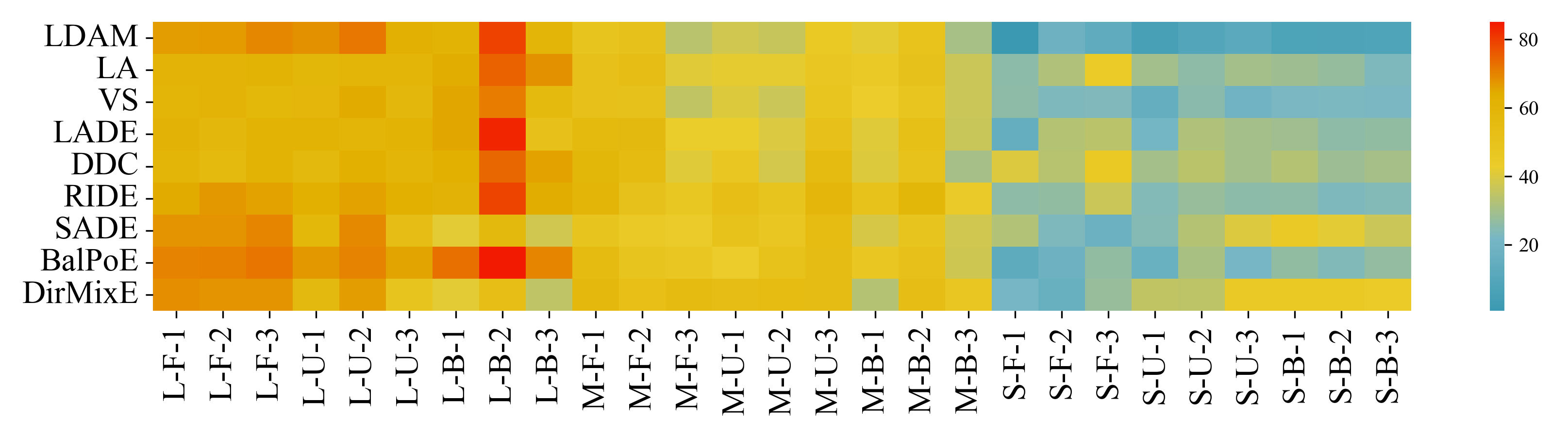} 
    }
    \subfigure[SADE's Setting]{
      \includegraphics[width=0.86\columnwidth]{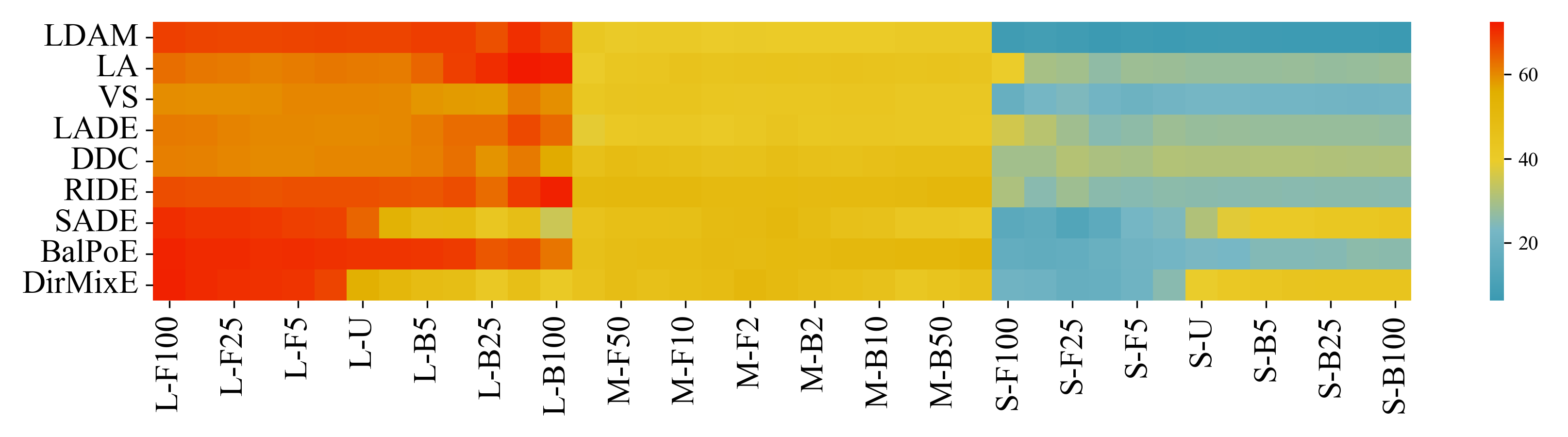} 
    }
    \caption{\label{fig:fine-cifar100} \textbf{Fine-grained Performance of DirMixE on CIFAR-100-LT} \textbf{F,U,B} represent the forward, uniform and backward distributions. For the x-axis, \textbf{F-1,F-2,F-3} in \textbf{(a)} denote the three observed label distributions in the test data respectively. \textbf{F2,F5,$\cdots$,F100} in \textbf{(b)} represents the corresponding imbalance ratio of the forward distribution under SADE's setting. The case for \textbf{the suffixes of U and B} are similar. \textbf{L, M, S} denote the many-shot, medium-shot and few-shot classes respectively. For example, \textbf{L-B100} indicates the performance of medium-shot classes on the backward distribution with an imbalance ratio of 100.}
  \end{figure}

  \section{Conclusion}
  We consider the hierarchy of the global and local variations of the test label distributions in test agnostic long-tail recognition. To this end, we propose a Dirichlet MoE method named \ours. The label distributions are sampled from a meta-distribution, characterized by a mixture of Dirichlet distribution. In the proposed MoE strategy, each expert is assigned to a local Dirichlet distribution for a specific skill. The global and local variations are then captured by inter- and intra-component variations of the meta-distribution. This  also leverages a stable objective function minimizing the mean and semi-variance of the loss with the help of Monte Carlo method. Furthermore, we extend this method to PEFT of foundation models to improve model performance.  When $\V_+(\la) \ee \V(\la)$, we show that the proposed objective function enjoys an sharper bound by semi-variance regularization. Finally, extensive experiments demonstrate the efficacy of \ours.
  
  \bibliographystyle{IEEEtran}
  \bibliography{example_paper}

\begin{IEEEbiography}
	[{\includegraphics[width=1in,height=1.25in,clip,keepaspectratio]{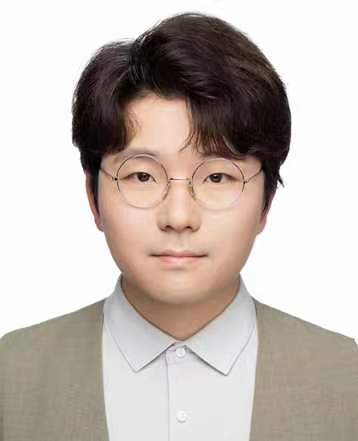}}]{Zhiyong Yang} received his M.S. degree in computer science and technology from the University of Science and Technology Beijing (USTB) in 2017, and Ph.D. degree from the University of Chinese Academy of Sciences (UCAS) in 2021. He is currently an Associate Professor at the University of Chinese Academy of Sciences. His research interests include trustworthy machine learning, long-tail learning, and optimization frameworks for complex metrics. He is one of the key developers of the XCurve learning framework (\url{https://xcurveopt.github.io/}), designed to address decision biases between model trainers and users. His work has been recognized with various awards, including Top 100 Baidu AI Chinese Rising Stars Around the World, Top-20 Nomination for the Baidu Fellowship, Asian Trustworthy Machine Learning (ATML) Fellowship, and the China Computer Federation (CCF) Doctoral Dissertation Award. He has authored or co-authored over 60 papers in top-tier international conferences and journals, including more than 30 papers in T-PAMI, ICML, and NeurIPS. He has also served as an Area Chair (AC) for NeurIPS 2024/ICLR 2025/ICML 2025, a Senior Program Committee (SPC) member for IJCAI 2021, and as a reviewer for several prestigious journals and conferences, such as T-PAMI, IJCV, TMLR, ICML, NeurIPS, and ICLR.
\end{IEEEbiography}

\begin{IEEEbiography}
	[{\includegraphics[width=1in,height=1.25in,clip,keepaspectratio]{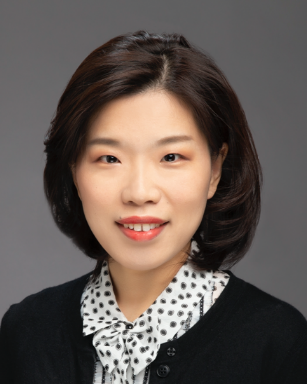}}]{Qianqian Xu} received the B.S. degree in computer science from China University of Mining and Technology in 2007 and the Ph.D. degree in computer science from University of Chinese Academy of Sciences in 2013. She is currently a Professor with the Institute of Computing Technology, Chinese Academy of Sciences, Beijing, China. Her research interests include statistical machine learning, with applications in multimedia and computer vision. She has authored or coauthored 100+ academic papers in prestigious international journals and conferences (including T-PAMI, IJCV, T-IP, NeurIPS, ICML, CVPR, AAAI, etc). Moreover, she serves as an associate editor of IEEE Transactions on Circuits and Systems for Video Technology, IEEE Transactions on Multimedia, and ACM Transactions on Multimedia Computing, Communications, and Applications.
\end{IEEEbiography}

\begin{IEEEbiography}
	[{\includegraphics[width=1in,height=1.25in,clip,keepaspectratio]{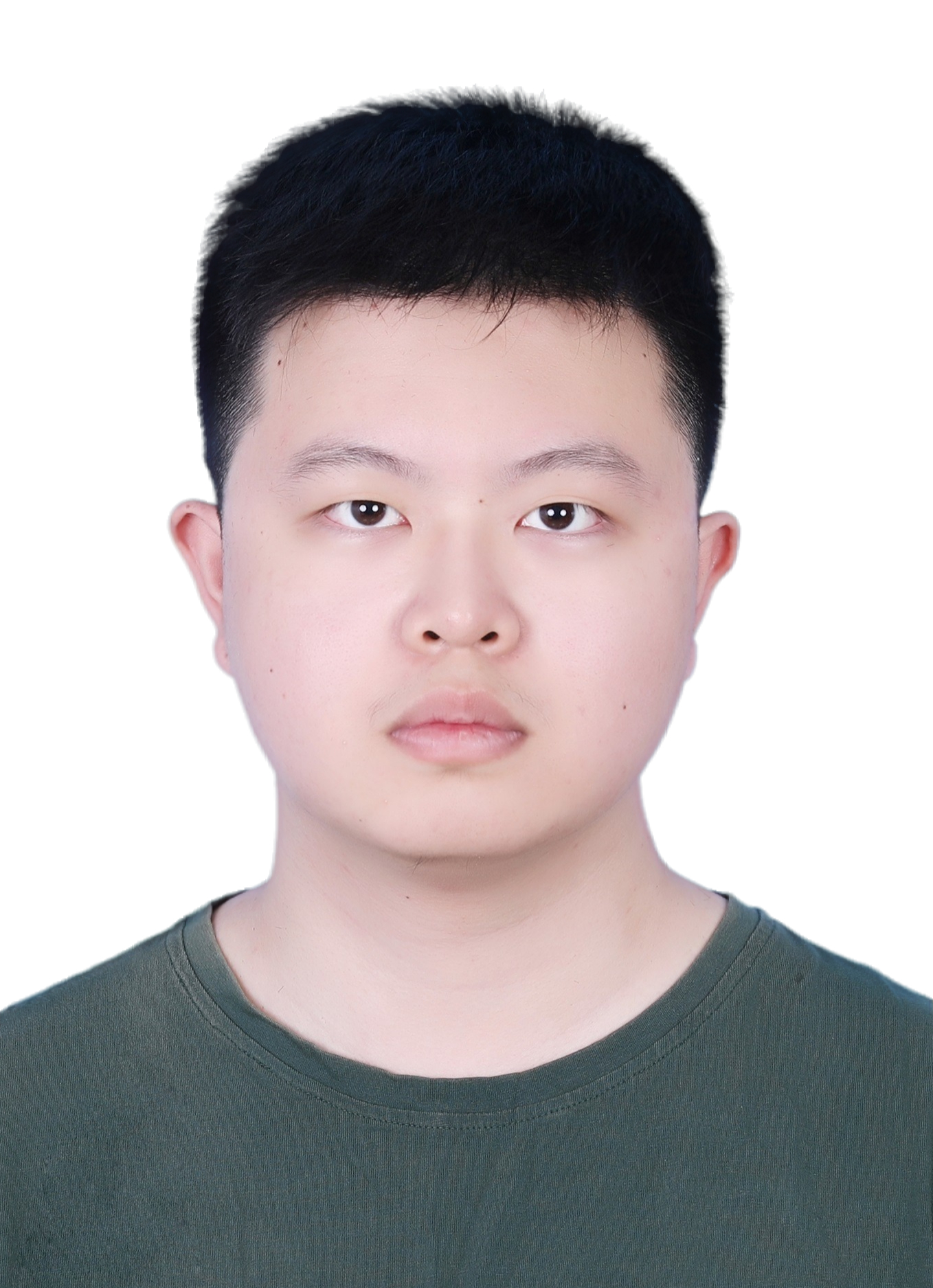}}]{Sicong Li}  received the B.E. degree in Computer Science and Technology from Beihang University in 2024. He is currently pursuing his Ph.D. degree at the University of Chinese Academy of Sciences. His research interests include machine learning and computer vision.
\end{IEEEbiography}

\begin{IEEEbiography}
    [{\includegraphics[width=1in,height=1.25in,clip,keepaspectratio]{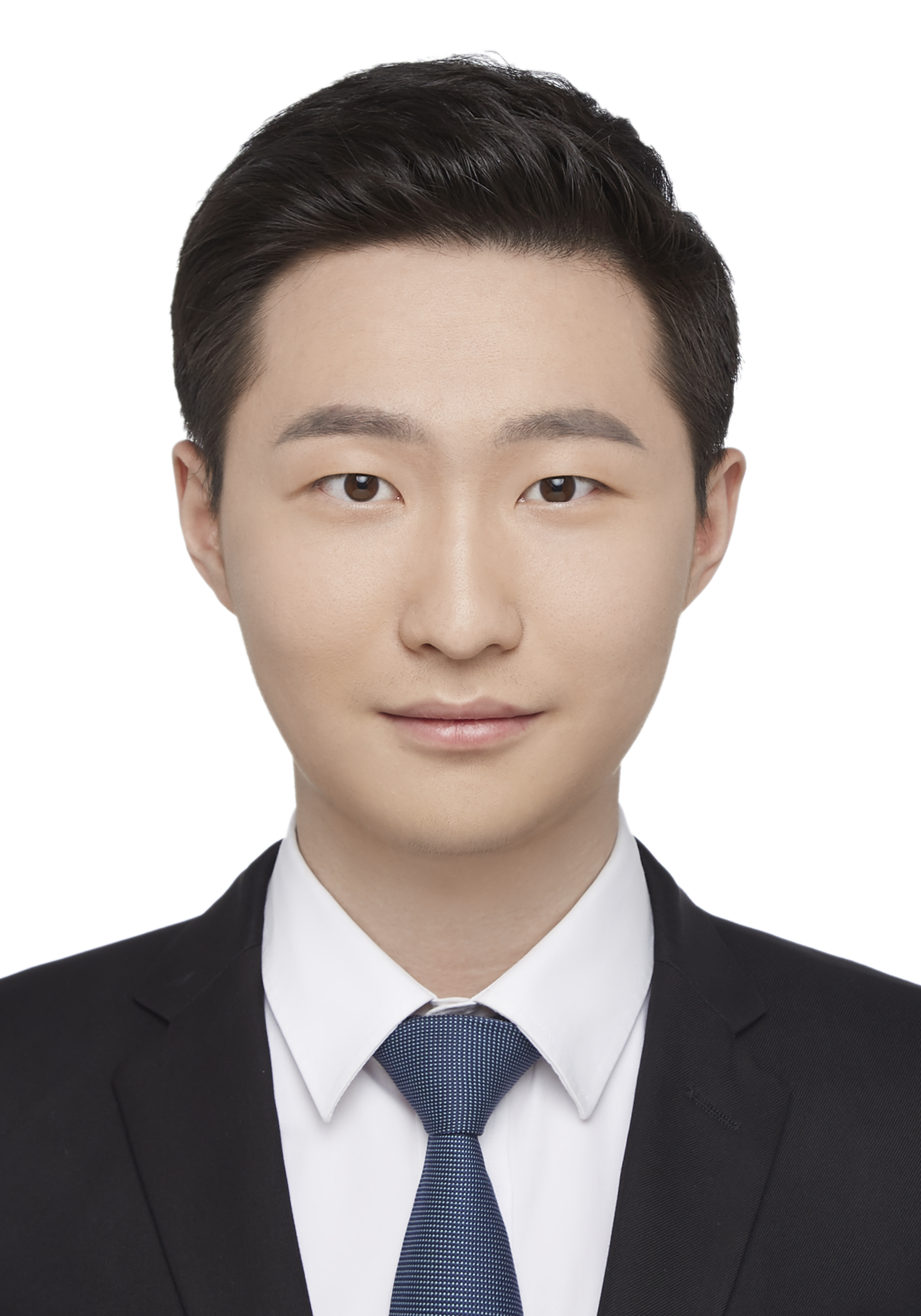}}]{Zitai Wang} received the B.E. degree in computer science and technology from Beijing Jiaotong University in 2019 and the Ph.D. degree in University of Chinese Academy of Sciences in 2024. He is currently a post-doctoral research fellow with the Institute of Computing Technology, Chinese Academy of Sciences. His research interests include machine learning and data mining. He has authored or coauthored 10+ academic papers in top-tier international conferences and journals including T-PAMI, IJCV, ICML, NeurIPS, AAAI, and ACM Multimedia. He served as a reviewer for several top-tier journals and conferences such as T-PAMI, T-CSVT, NeurIPS, ICLR, CVPR, AISTATS, and AAAI.
\end{IEEEbiography}

\begin{IEEEbiography}
	[{\includegraphics[width=1in,height=1.25in,clip,keepaspectratio]{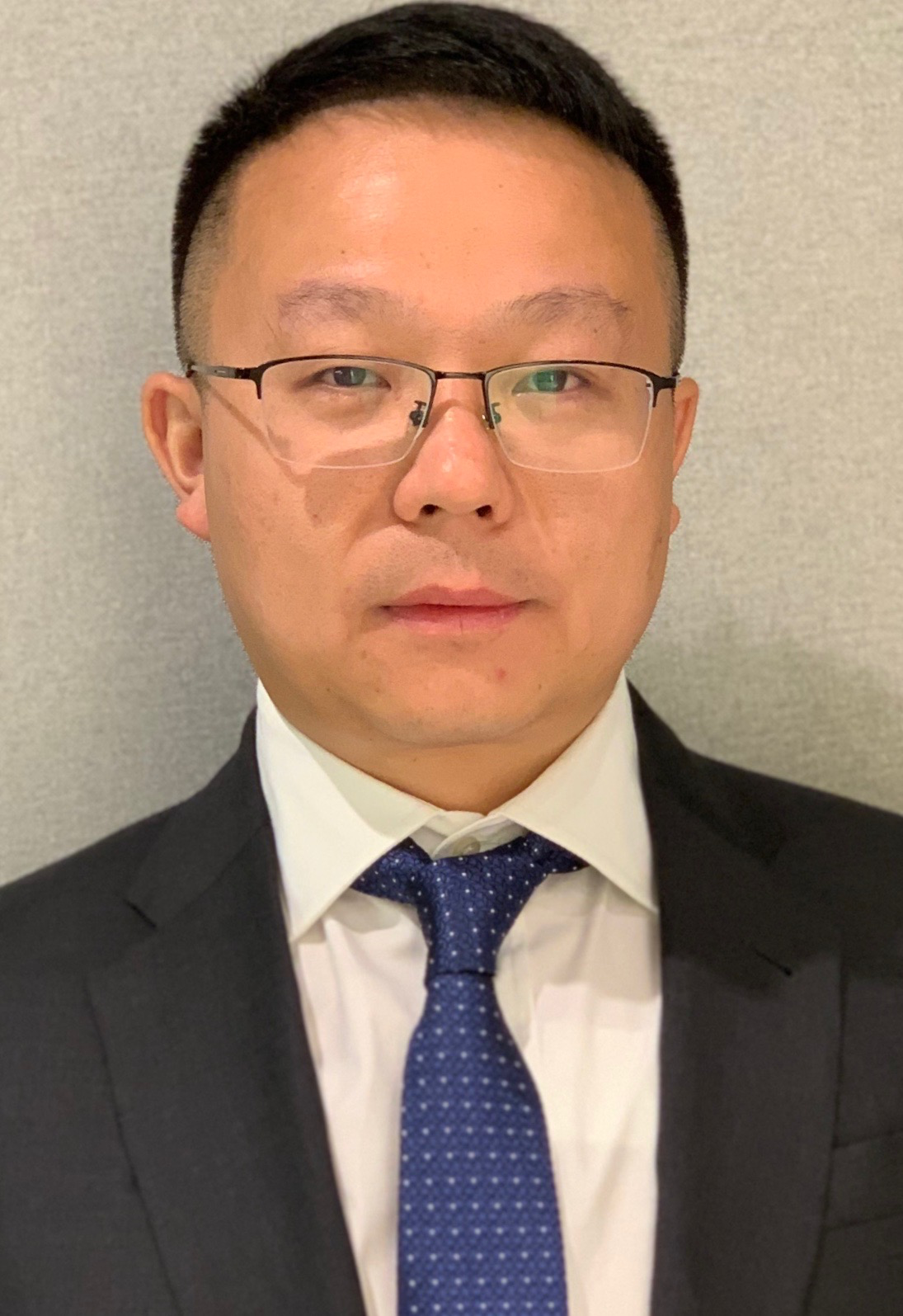}}]{Xiaochun Cao} is a Professor of School of Cyber Science and Technology, Shenzhen Campus of Sun Yat-sen University. He received the B.E. and M.E. degrees both in computer science from Beihang University (BUAA), China, and the Ph.D. degree in computer science from the University of Central Florida, USA, with his dissertation nominated for the university level Outstanding Dissertation Award. After graduation, he spent about three years at ObjectVideo Inc. as a Research Scientist. From 2008 to 2012, he was a professor at Tianjin University. Before joining SYSU, he was a professor at Institute of Information Engineering, Chinese Academy of Sciences. He has authored and coauthored over 200 journal and conference papers. In 2004 and 2010, he was the recipients of the Piero Zamperoni best student paper award at the International Conference on Pattern Recognition. He is on the editorial boards of IEEE Transactions on Image Processing and IEEE Transactions on Multimedia, and was on the editorial board of IEEE Transactions on Circuits and Systems for Video Technology.
\end{IEEEbiography}

\begin{IEEEbiography}
	[{\includegraphics[width=1in,height=1.25in,clip,keepaspectratio]{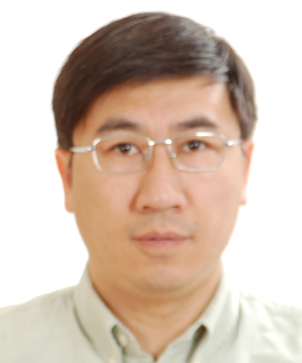}}]{Qingming Huang} is a chair professor in the University of Chinese Academy of Sciences and an adjunct research professor in the Institute of Computing Technology, Chinese Academy of Sciences. He graduated with a Bachelor degree in Computer Science in 1988 and Ph.D. degree in Computer Engineering in 1994, both from Harbin Institute of Technology, China. His research areas include multimedia computing, image processing, computer vision and pattern recognition. He has authored or coauthored more than 400 academic papers in prestigious international journals and top-level international conferences. He was the associate editor of IEEE Trans. on CSVT and Acta Automatica Sinica, and the reviewer of various international journals including IEEE Trans. on PAMI, IEEE Trans. on Image Processing, IEEE Trans. on Multimedia, etc. He is a Fellow of IEEE and has served as general chair, program chair, area chair and TPC member for various conferences, including ACM Multimedia, CVPR, ICCV, ICME, ICMR, PCM, BigMM, PSIVT, etc.
\end{IEEEbiography}


\clearpage
\appendices
\onecolumn

\section*{\textcolor{blue}{\Large{Contents}}}

\startcontents[sections]
\printcontents[sections]{l}{1}{\setcounter{tocdepth}{3}}
\newpage

\section{Detailed Explanation of the Sampling process}\label{app:detail}

In our study, we aim to ensure good overall performance across the entire meta distribution $\mathcal{E}$ of label distributions. Since our access is limited to training data with a long-tailed distribution, we employ the Logit Adjustment (LA) loss to tailor the model for specific test-label distributions using only training data. Our objectives are to 1) construct a meta-distribution $\mathcal{E}$ for test-label distributions, 2) assess the loss for different test-label distributions in $\mathcal{E}$, and 3) maintain minimal change in the mean and variance of the loss over $\mathcal{E}$. For 1), we model $\mathcal{E}$ as a Dirichlet mixture distribution, representing global and local variations through the randomness across and within its components, respectively. The Monte Carlo method is vital for implementing 2) and 3), detailed further in below.

We have to deal with both the training data distribution and a meta-distribution for test-label distribution, leading to a hierarchical process for the calculation:

\textbf{Layer 1}: For a given test-label distribution $\mathbb{P}_{te}$, we measure a model's adaptability to this new distribution through the expected LA loss over training data parametrized by $\mathbb{P}_{te}$: 

  $$\hat{\ell}_{\mathsf{LA}}(f,\mathbb{P}_{te})\approx  \frac{1}{N} \cdot \sum_{i=1}^N \ell_{\mathsf{LA}}(f_\theta(x_i),y_i; \mathbb{P}_{te}),$$   

  where $x_i,y_i$ represent the $i$-th instance's feature and label in the training data.

\textnormal{Layer 2}: Across various $\mathbb{P}_{te}$s, we calculate the mean and variance of $\ell_{\mathsf{LA}}(f,\mathbb{P}_{te})$ over meta-distribution $\mathcal{E}$, represented as 
  $$ \mathbb{E}_{ \mathbb{P}_{te} \sim \mathcal{E}}\left[ \hat{\ell}_\mathsf{LA}(f;\mathbb{P}_{te})\right],~ \mathbb{V}_{ \mathbb{P}_{te} \sim \mathcal{E}}\left[ \hat{\ell}_\mathsf{LA}(f;\mathbb{P}_{te})\right],$$
  respectively.

To minimize both the mean and variance of the loss over $\mathcal{E}$, we employ the following objective:

   $$ \mathbb{E}_{ \mathbb{P}_{te} \sim \mathcal{E}}\left[ \hat{\ell}_\mathsf{LA}(f;\mathbb{P}_{te})\right] +  \lambda \cdot   \mathbb{V }_{ \mathbb{P}_{te} \sim \mathcal{E}}\left[ \hat{\ell}_\mathsf{LA}(f;\mathbb{P}_{te})\right], $$
where $\lambda$ is a balancing coefficient.

Given the model $f$ as a neural network, we can't obtain the mean and variance in the closed-form. Instead, we have to approximate the mean and variance for layer 2 using the Monte Carlo method, sampling a finite number of test distributions $\mathbb{P}_{te}$ from $\mathcal{E}$ and then do the estimation.

Our training pipeline, based on the Monte Carlo method, includes:

- a) Sampling a set of test-label distributions from the Dirichlet mixture distribution $\{(\mathbb{P}_j,\xi_j)\}_{i=1}^M$, where $\mathbb{P}_j$ is the sampled test-label-distribution, and $\xi_j$ indicates its component.
- b) Evaluating the empirical LA loss for each $\mathbb{P}_j$ on the training data, adjusting the label distribution accordingly, then computing the mean and (semi-)variance of these losses.
- c) Training the model using backpropagation (BP).

Next, we outline each step of our process.

\subsection{Step a)}

In the main paper, we defined the meta distribution $\mathcal{E}$ as a Dirichlet mixture distribution:

$$
\mathbb{P}_{te}|\xi \sim \mathsf{Dir}\left({\alpha}^{(\xi)}\right),
$$

$$
\xi|\boldsymbol{p} \sim \mathsf{Discrete}\left(p_1,\cdots,p_K\right).
$$

To sample a test distribution $\mathbb{P}_{te}$, we first select a Dirichlet component $\xi$ from $K$ options based on the discrete probability $\left(p_1,\cdots,p_K\right)$. We then sample $\mathbb{P}_{te}$ from the $\xi$-th component $\mathsf{Dir}\left({\alpha}^{(\xi)}\right)$. This process repeats to generate a set of distributions $\{\mathbb{P}_j, \xi_j\}_{i=1}^M$.

\subsection{Step b)}

For each pair $(\mathbb{P}_{j}, \xi_j)$, we then calculate the loss. In our Mixture of Experts (MOE) strategy, the expert $\xi_j$ is assigned to its corresponding Dirichlet distribution component, and we train $f_\theta^{(\xi_j)}$ for this task. With this assignment,, we get the empirical LA loss average on the training data: $\hat{\ell}_\mathsf{LA}(f^{(\xi_j)};\mathbb{P}_j)$ for this specific test-label distribution. Repeating the calculate for each $(\mathbb{P}_{j}, \xi_j)$, we can then obtain the empirical mean and variance of $\hat{\ell}_\mathsf{LA}(f^{(\xi_j)};\mathbb{P}_j)$ as:

$$
  \hat{\mathbb{E}}(\hat{\ell}_{\mathsf{LA}}) = \frac{1}{M}\sum_{j=1}^M \hat{\ell}_\mathsf{LA}(f^{(\xi_j)};\mathbb{P}_j),
$$

$$
  \hat{\mathbb{V}}(\hat{\ell}_{\mathsf{LA}}) = \frac{1}{M}\sum_{j=1}^M \left(\hat{\ell}_\mathsf{LA}(f^{(\xi_j)};\mathbb{P}_j) -  \hat{\mathbb{E}}\left(\hat{\ell}_\mathsf{LA}(f^{(\xi_j)};\mathbb{P}_j)\right)\right)^2.
$$

Moreover, we find that the variance regularization tends to punish loss functions smaller than its mean. We use semi-variance regularization, which only penalizes loss deviations larger than the mean, as a surrogate to  avoid penalizing the smaller-than-mean losses:

$$
\hat{\mathbb{V}}_+(\hat{\ell}_{\mathsf{LA}}) = \frac{1}{M}\sum_{j=1}^M \left(\left(\hat{\ell}_\mathsf{LA}(f^{(\xi_j)};\mathbb{P}_j) -  \hat{\mathbb{E}}\left(\hat{\ell}_\mathsf{LA}(f^{(\xi_j)};\mathbb{P}_j)\right)\right)_+\right)^2.
$$

\subsection{Step c)}

We perform backpropagation (BP) based on the loss and train the MOE model.

Putting altogther, we summarize this procedure as the following algorithm.

\begin{algorithm}
    \caption{Training Algorithm}
    \label{alg:training}
    \begin{algorithmic}[1]
    \REQUIRE Batch size $bs$, Dirichlet components $\bm{\alpha}^{(1)}, \bm{\alpha}^{(2)}, \cdots, \bm{\alpha}^{(K)}$
    \ENSURE Trained models $f_{\theta}^{(1)}, f_{\theta}^{(2)}, \cdots, f_{\theta}^{(K)}$
    
    \FOR{$j \leftarrow 1$ \textbf{to} $M$}
        \STATE $\xi_{j} \leftarrow \text{RandomInteger}(1, K)$ with probability $(p_1,\cdots,p_K)$
        \STATE Sample $\mathbb{P}_{j}$ from Dirichlet Distribution $Dir(\bm{\alpha}^{(\xi_j)})$
        \STATE Add sampled pair $(\mathbb{P}_{j}, \xi_j)$ to the test-label-distribution set $\mathcal{P}$.
    \ENDFOR
    
    \WHILE{not converged}
        \STATE Sample training data Batch $B$
        \STATE Sample a subset $\mathcal{P}'$ from $\mathcal{P}$
        \FOR{each $(\mathbb{P}_j,\xi_j) \in \mathcal{P}'$}
            \STATE Calculate $\hat{\ell}_{LA_j} = \frac{1}{bs} \sum_{(\bm{x}, y) \in B_i} \hat{\ell}_{LA}(f_\theta^{(\xi_j)}(\bm{x}),y;\mathbb{P}_j)$
        \ENDFOR
        \STATE Calculate the mean $\hat{\mathbb{E}}(\ell_{\mathsf{LA}})$ and semi-variance $\hat{\mathbb{V}}_+(\ell_{\mathsf{LA}})$
        \STATE Calculate the objective function $L \leftarrow \hat{\mathbb{E}}(\ell_{\mathsf{LA}}) + \lambda \cdot \hat{\mathbb{V}}_+(\ell_{\mathsf{LA}})$
        \STATE Perform SGD with respect to $L$ to update the model
    \ENDWHILE
    
    \STATE \textbf{return} $f_{\theta}^{(1)}, f_{\theta}^{(2)}, \cdots, f_{\theta}^{(K)}$
    \end{algorithmic}
\end{algorithm}


  






\begin{textbo}

\section{Theoretical Challenges}    \label{app:theoretical_challenge}

The theoretical framework developed in this work addresses and resolves several fundamental challenges that are at least inadequately addressed by existing studies, particularly in the context of test-agnostic long-tail recognition and the fine-tuning of foundation models.
\begin{enumerate}
\item[a)]  \textbf{Hierarchical Generalization under a Meta-Distribution}

Classical generalization theory bounds the expected error over a single, fixed data distribution. Recent work on test-agnostic recognition (e.g., SADE) introduces multiple experts for \textbf{a few fixed} test distributions but lacks a stochastic model for the test-label distribution. In this way, it is impossible to construct a theoretical analysis of the generalization error when not only the data but also the test label distribution is changed over time.  By contrast, we model the unknown test distribution as being drawn from a meta-distribution $\mathcal{E}$, specifically a mixture of Dirichlet distributions. This introduces a hierarchical stochastic process: one samples a label distribution from $\mathcal{E}$, and then samples data conditioned on that label distribution. The core theoretical challenge is to derive a generalization bound that holds uniformly over this meta-distribution of tasks.

\item[b)] \textbf{Going Beyond Worst-case analysis in Uniform Convergence}

 Worst-case uniform convergence bounds over an entire hypothesis class $\mathcal{F}$ are loose for over-parameterized models. The results ignore the fact that training algorithms often find models in a good region of the space. Hence, the second challenge is to mathematically introduce such good regions into the definition of the hypothesis space and derive sharper bounds.

\item[c)] \textbf{Introduce Variance-Based Optimization and Generalization}

While variance-penalized objectives are well established in ML theory, their main goal is to achieve a sharp generalization error bound under a single, fixed test-label distribution. How to extend the bounds to our complicated scenario is the third challenge.

\item[d)] \textbf{Analyzing Parameter-Efficient Fine-Tuning (PEFT) in Foundation Models}

The PEFT framework updates only a small subset of the foundation model's parameters. But the traditional result shows that model complexity is proportional to the total number of parameters. In this work, we also aim to identify appropriate techniques and assumptions to mitigate this gap.

\end{enumerate}

To address these challenges, by moving beyond worst-case analysis, incorporating data-dependent priors, and developing new tools for PEFT, we establish a theoretical foundation that is both rigorous and relevant to practice.

\end{textbo}

\begin{textbo}

\section{Theoretical Novelty}       \label{app: theoretical novelty}

\subsection{Theoretical Novelty}

Our theoretical analysis introduces several foundational innovations that yield significantly sharper generalization bounds for test-agnostic long-tailed recognition. The major novelty for Thm.\ref{thm:gen} and Thm.\ref{thm:gen1} are as follows:

\begin{enumerate}
    \item[a)] \textbf{Hierarchical Decomposition of the Error Bound}: 
    
    The primary innovation in \textbf{Theorem \ref{thm:gen}} is the \textbf{hierarchical} concentration framework that decomposes the generalization error. Unlike standard analyses, we must account for stochasticity from two sources: the finite training dataset $\mathcal{S}$ and the finite sampling of label distributions $\mathcal{P}$ from the meta-distribution $\mathcal{E}$. Our proof structure explicitly decomposes the excess risk into: (i) the meta-distribution shift error $\mathfrak{Err}_{approx}$, (ii) the Monte Carlo sampling error $\mathfrak{Err}_{sto}$, and (iii) the data estimation error. This decomposition is pivotal, as it allows us to isolate and tightly control the errors stemming from our novel training paradigm.
    \item[b)] \textbf{Sharper Bound Based on the Induced Subspace}: 
    
    To avoid the loose bounds of worst-case analysis, we introduce a data-dependent refinement of the hypothesis space via the Induced Subclass (Def.\ref{def: induced subclass} in the original paper). This is a critical step. Instead of considering all possible functions in $\mathcal{F}$, we restrict our analysis to the subset $\mathcal{F}_{\mathcal{A}}(\rho, p)$ where the model's loss is, with high probability, below a threshold $\rho$. This is a reasonable prior, grounded in the empirical effectiveness of foundation models and our training objective, ensuring that most learned models reside in this well-behaved region. To make this theoretical construction work, we develop novel conditional concentration inequalities in Lem.\ref{lem:cov}-\ref{lem:varb}. The key contribution here is that traditional techniques like Hoeffding and empirical Bernstein bounds focus on the overall loss, while we consider conditioned expectation on the good event of low loss, effectively replacing the global bound $B$ with the much smaller data-dependent bound $\rho$ in the stochastic error terms. This leads to the $\mathfrak{Err}_{sto}$ term scaling as $O(\rho^{\mathcal{I}}/\sqrt{N} + \rho^{\mathcal{M}}/M)$, which is far tighter than the standard $O(B/\sqrt{N} + B/\sqrt{M})$ (please see the discussion in \textbf{c)}).
    
    \item[c)] \textbf{Sharper Bound Based on the Semi-variance Regularization}:
    
    Another novelty comes from the result induced by semi-variance regularization. Theoretically, we justify this by proving in Thm.\ref{thm:rho} and Thm.\ref{thm:multi} in the main paper that for a wide class of light- and heavy-tailed loss distributions (Exponential, Gamma, Pareto, and their mixtures), the semi-variance $\mathbb{V}_{+}$ is comparable to the full variance $\mathbb{V}$. This key insight allows us to use the empirical semi-variance $\hat{\mathbb{V}}_{+}$ in the regularization term $\mathfrak{Reg}$ of our bound. Minimizing this term during training directly tightens the generalization bound. Under the assumption of an exponentially decaying loss tail, this entire framework enables a stochastic error rate of $O(N^{-(1+2/\alpha)/2} + M^{-1-1/\alpha})$, a drastic improvement over the conventional $O(N^{-1/2} + M^{-1/2})$.
    \item[d)] \textbf{Extending the Results to Parameter-efficient Finetuning of Foundation Models}:

  Based on the generic result in Thm.\ref{thm:gen}, Thm.\ref{thm:gen1} takes a further step to construct a sharp generalization bound for PEFT-based fine-tuning. The non-trivial contribution is that the model's complexity is proportional only to the number of trainable parameters, not to the total number of parameters. This induces a sharp improvement of the bound because the former is way smaller. Technically, the novelty is two-fold.
  
  First, we perform a first-order Taylor expansion of the scoring function around a reference point $\Theta^{(i),*}$, which is chosen as the closest point in a set of local solutions. Moreover, we introduce the concept of $\delta$-compact parametrization (Def.\ref{def: delta compact parameterization}), which partitions the parameter space into Voronoi cells around local solutions. By choosing sufficiently small cells, when applying covering number-based analysis with a covering set of balls $\mathcal{B}_1, \cdots, \mathcal{B}_N$, $\delta$-compact parametrization ensures that functions within each $\mathcal{B}_i$ share the same reference point.
  As a result, this strategy effectively decouples (other parameters when calculating Lip. constants within each $\mathcal{B}_i$) the fixed, high-complexity backbone from the low-dimensional, trainable increments $\Delta \Theta^{(i)}$. Then, we only need to analyze the complexity of the low-dimensional manifold.

  Second, for the low-dimensional manifold where the trainable parameters reside, we leverage advanced results from geometric functional analysis, specifically tight covering number bounds for low-rank matrices (Lemma \ref{lem:rankr}) and Stiefel manifolds (Lemma \ref{lem:stef}), which depend only on the dimensions of the LoRA/Adapter components. Consequently, the complexity term $\nu$ in the bound scales only with the number of trainable parameters in LSF, not the whole backbone, yielding a tight and practical guarantee for large-scale foundation models.

\end{enumerate}

In summary, our theoretical contributions are built upon both novel conceptual frameworks (hierarchical error decomposition, induced subclasses) and the non-trivial application of sophisticated mathematical tools. We move beyond standard uniform convergence by incorporating data-dependent priors and problem-specific structure.

\subsection{Existing Results used in the Proof}

The theoretical derivations in this work are built upon a solid foundation of established mathematical tools, which are strategically extended and combined to address the novel challenges of our setting.

\begin{enumerate}
    \item[a)] \textbf{Covering Numbers and Uniform Convergence}
    
The use of covering numbers to construct uniform convergence bounds is a standard technique in statistical learning theory. In Thm.\ref{thm:gen}, we operate under the standard assumption that the covering number of our hypothesis class $\mathcal{F}$ scales polynomially. Specifically, we assume $\mathcal{N}_{\infty}(\mathcal{F}, \epsilon, n) \lesssim (r/\epsilon)^\nu$ for some $\nu, r > 0$, as defined in Sec. D.1. This is a conventional and widely adopted assumption \cite{anthony2009neural, barron1999risk, pollard2012convergence, DBLP:conf/iclr/LongS20, bach2024learning} that holds for many common function classes, including those parameterized by neural networks with bounded norms. The metric entropy, quantified by $\log \mathcal{N}_{\infty}$, then directly appears in the logarithmic terms of our generalization bound (e.g., $\log(\zeta_1), \log(\zeta_2)$), capturing the complexity of the model.

\item[b)]  \textbf{Extension of Basic Concentration Inequalities}

\textbf{Basic Tools}: The core tools we build upon are the classic Hoeffding's inequality (for bounded random variables) and the Empirical Bernstein Inequality (as in Maurer \& Pontil \cite{bern} ), which provides a data-dependent bound based on the empirical variance.

\textbf{Our contribution}: The primary innovation is found in Lem. \ref{lem:cov}-\ref{lem:varb}, where we generalize these classical results to hold uniformly over conditional expectations. This is the key to sharpen the bound for the induced subclass $\mathcal{F}_{\mathcal{A}}(\rho, p)$.

\item[c)] \textbf{Parameter Space Partition via Voronoi Diagram in Theorem \ref{thm:gen1}}

\textbf{Basic Tool}: The analysis of the LSF scheme in Thm.\ref{thm:gen1} requires controlling the complexity of the linearized hypothesis class. To achieve this, we introduce the Voronoi Diagram to induce Parameter Space Partition (Def.\ref{def: voronoi diagram}). 

\textbf{Usage}: This construction partitions the space of incremental parameters $\Delta \bm{\Theta}^{(i)}$ into cells $\mathcal{V}_{k}^{(i)}$ based on their proximity to a discrete set of local solutions in $\mathsf{Sol}$. The key utility of this partition is that within each Voronoi cell, the closest local solution $\Delta \bm{\Theta}^{(i),*}$ remains constant. This leads to a sharper bound.

\item[d)] \textbf{Variance-based Concentration}

\textbf{Existing Results}: For variance-penalized objectives, our results are built upon existing studies like Maurer \& Pontil \cite{bern}. This provides a bound of the form $\mathbb{E}[\ell] \lesssim \hat{\mathbb{E}}[\ell] + \sqrt{\hat{\mathbb{V}}[\ell] / N} + 1/N$.

\textbf{Our Contribution}:  We make two critical extensions.
    
    a) \textbf{Conditional Concentration} (Lemma \ref{lem:varb}): We integrate the variance-based bound with our novel conditional concentration framework. Lemma 8 is a variance-penalized analog that holds for the induced subclass, where the variance term is discounted by the conditioning on the good event $\mathcal{A}(\rho)$.
    
    b) \textbf{Hierarchical Application} (Proof of Thm.\ref{thm:gen}): We apply this principle to our hierarchical problem. The variance term $\hat{\mathbb{V}}_{+, \mathcal{E}}[\hat{\ell}_{\mathcal{S}, \mathcal{E}}]$ that appears in the $\mathfrak{Reg}$ term of Theorem 3 originates from applying a variance-penalized concentration argument to the outer layer (the Monte Carlo sampling over label distributions). This constructs a bound that is tight in both the number of data samples $N$ and the number of sampled label distributions $M$.

\item[e)] \textbf{Complexity Bound for Homogeneous Space}

To bound the complexity of the PEFT-based hypothesis classes in LSF, we rely on advanced results from geometric functional analysis.

\textbf{Homogeneous Spaces (\cite{metric})}: We apply the covering number bounds for homogeneous space from Szarek \cite{metric} to construct the covering number of the low-rank Stiefel mainfold.

\textbf{Lie Group Theory (\cite{geostat, lie1})}: The fundamental isomorphism $St(n, k) \cong \text{SO}(n)/\text{SO}(n-k)$ (Lem.12) is proven using basic Lie group theory from sources like Kirillov \cite{lie1} and Guigui et al. \cite{geostat}. This involves showing the transitivity of the $\text{SO}(n)$ action on the Stiefel manifold and characterizing the stabilizer subgroup. This geometrical perspective is crucial because the parameters of the low-rank factors in LoRA and Adapter inherently reside on or near such manifolds. By combining this with a low-rank matrix decomposition, Lem.10 provides the final covering number bound for the set of low-rank increments $\mathcal{M}(r, m, n)$, which depends only on the trainable parameters and is independent of the frozen backbone.
\end{enumerate}

\end{textbo}

\section{Proof for the Upper Bound of $\rho$} \label{app:rho}

\subsection{Proof for Thm.1}
\begin{proof}
1): For exponential distribution, we have:
\begin{align*}
    p_\lambda(\ell) = \lambda \cdot \exp(-\lambda \cdot \ell), ~~ \E_{\lambda}[\ell]  = \frac{1}{\lambda}, \vv_{\lambda} = \frac{1}{\lambda^2}.
\end{align*}

In this sense, we have: 
\begin{align*}
    \vv - \vpp &= \int_{0}^{1/\lambda} (\ell - \frac{1}{\lambda})^2  \cdot p_\lambda(\ell) \cdot d\ell\\ 
&\le \frac{1}{\lambda^2} \cdot \P\left[ \ell < \frac{1}{\lambda} \right]\\ 
&\le \frac{1-\exp(-1)}{\lambda^2}
\end{align*}
Hence:
\begin{align*}
 \frac{\vpp}{\vv} \ge 1- (1-\exp(-1)) = \exp(-1)
\end{align*}

2): For the gamma distribution, we have:
\begin{align*}
p_{\alpha,\beta}(\ell) = \frac{\beta^\alpha}{\Gamma(\alpha)} \cdot \ell^{\alpha-1}\cdot \exp(-\beta \cdot \ell), ~\E[\ell] = \frac{\alpha}{\beta}, ~\V[\ell] =  \frac{\alpha}{\beta^2}.
\end{align*}
Similarly, we have:

\begin{align*}
\vv- \vpp  &= \int_{0}^{\alpha/\beta} (\ell - \alpha/\beta)^2 \cdot  p_{\alpha,\beta}(\ell) \cdot d\ell \\ 
&\le \left(\frac{\alpha}{\beta}\right)^2 \cdot \frac{1}{\Gamma(\alpha)} \cdot \int_{0}^{\alpha/\beta} (\beta\cdot \ell )^{\alpha-1} \cdot \exp(-\beta\cdot \ell) \cdot d(\beta\cdot \ell) \\ 
&= \left(\frac{\alpha}{\beta}\right)^2 \cdot \frac{1 }{\Gamma(\alpha)} \cdot \int_{0}^{\alpha} (\ell )^{\alpha-1} \cdot \exp(- \ell) \cdot d( \ell)\\ 
& = \left(\frac{\alpha}{\beta}\right)^2  \cdot \frac{\Gamma^\uparrow(\alpha,\alpha)}{\Gamma(\alpha)}
\end{align*}
Then we have: 
\begin{align*}
    \frac{\vpp}{\vv} \ge 1- \alpha  \cdot \frac{\Gamma^\uparrow(\alpha,\alpha)}{\Gamma(\alpha)}
\end{align*}
3): For a Pareto distribution: we have 
\begin{align*}
    p_{\theta}(\ell) = \begin{cases}
        \frac{\theta \ell_m^\theta}{\ell^{\theta+1}}, &\ell \ge \ell_m,\\
        0, & \ell \le \ell_m
    \end{cases},~ 
    \E[\ell] = \begin{cases}
        \infty, & \theta \le 1,\\ 
        \frac{\theta}{\theta-1} \ell_m, &\theta > 1
    \end{cases},
    \V[\ell]  = \begin{cases}
        \infty, & \theta \le 2,\\ 
        \nicefrac{\theta}{(\theta-1)^2\cdot (\theta-2)}\cdot \ell_m^2, &\theta > 2
    \end{cases}
\end{align*}
In this sense, we have:
\begin{align*}
\vpp & = \int_{\ell_m}^{\frac{\theta}{\theta-1} \cdot \ell_m } \left( \ell - \frac{\theta}{\theta-1} \cdot \ell_m \right)^2 \cdot \frac{\theta \ell_m^\theta}{\ell^{\theta+1}} d\ell\\ 
& = \int_{\ell_m}^{\frac{\theta}{\theta-1} \cdot \ell_m } \frac{\theta \ell_m^\theta}{\ell^{\theta-1}} d\ell ~-~ 2\cdot \int_{\ell_m}^{\frac{\theta}{\theta-1} \cdot \ell_m } \frac{\theta \cdot \ell_m^{\theta+1}}{(\theta-1)\cdot\ell^{\theta-1}} d\ell ~+~ \left( \frac{\theta}{\theta-1} \cdot \ell_m \right)^2 \cdot  \int_{\ell_m}^{\frac{\theta}{\theta-1} \cdot \ell_m } \frac{\theta \ell_m^\theta}{\ell^{\theta+1}}  d\ell
\end{align*}

\begin{align*}
  &\int_{\ell_m}^{\frac{\theta}{\theta-1} \cdot \ell_m } \frac{\theta \ell_m^\theta}{\ell^{\theta-1}} d\ell ~ =  ~\frac{\theta}{\theta-2} \cdot \ell_m^2 \cdot \left[ 1- \left(   \frac{\theta}{\theta-1}\right)^{2-\theta} \right]\\
  &2\cdot \int_{\ell_m}^{\frac{\theta}{\theta-1} \cdot \ell_m } \frac{\theta \cdot \ell_m^{\theta+1}}{(\theta-1)\cdot\ell^{\theta-1}} \cdot d\ell ~ =  ~ 2 \cdot \left( \frac{\theta}{\theta -1} \cdot \ell_m \right)^2\cdot \left[ 1- \left( \frac{\theta}{\theta -1} \right)^{1-\theta} \right]\\ 
  &\left( \frac{\theta}{\theta-1} \cdot \ell_m \right)^2 \cdot  \int_{\ell_m}^{\frac{\theta}{\theta-1} \cdot \ell_m } \frac{\theta \ell_m^\theta}{\ell^{\theta+1}}  d\ell ~=~ \left( \frac{\theta}{\theta -1} \cdot \ell_m \right)^2\cdot \left[ 1- \left( \frac{\theta}{\theta -1} \right)^{\theta} \right]
    \end{align*}
Above all, we come to the conclusion:

\begin{align*}
    \frac{\vpp}{\vv}  = (\theta -1)^2\cdot\left[ 1- \left( \frac{\theta}{\theta -1} \right)^{2-\theta} \right] + \theta\cdot (\theta -2) \cdot \left[ 2\cdot \left( \frac{\theta}{\theta -1} \right)^{1-\theta} - \left( \frac{\theta-1}{\theta} \right)^{\theta} -1 \right]
    \end{align*}

\end{proof}

\subsection{Proof for Thm.2}

\begin{lem}
The function $f(x) = \left((x)_-\right)^2$ is a convex for $x \neq 0$.
\end{lem}
\begin{proof}
Denote $f_1(x) = x^2, f_2(x) = (x)_-$, then we have:
\begin{align*}
    f''(x) =  f''_2(x) \cdot f'_1(f_2(x)) + f''_1(f_2(x)) \cdot (f'_2(x))^2 = \begin{cases}
        0, & x> 0\\ 
        2, & x <0.
    \end{cases}
\end{align*} 
Obviously, we have $f''(x) \ge 0$. The proof is thus finished.

\end{proof}

\begin{proof}
We consider a mixture distribution of K-components with a p.d.f function 
\begin{align*}
    p_m(\ell) = \sum_{k=1}^K \omega_k \cdot p_k(\ell),
\end{align*}
where $p_k(\ell)$ is the p.d.f. for the $k$-th component, and $\omega_k$ is probability to observe the $k$-th component.
In this sense, we have:
\begin{align*}
\E[\ell] = \sum_{k=1}^K \omega_k \cdot \mu_k, ~ \V[\ell] \overset{(a)}{\ge} \sum_{k=1}^K \omega_k \cdot \sigma^2_k.
\end{align*}
where $\mu_k, \sigma^2_k$ are the corresponding means and variance given the component $k$. Here $(a)$ follows from the total law of variance.

In this sense, we can bound the semi-variance as:

\begin{align*}
    \vnn &= \int_{0}^{\infty} \left((\ell - \E[\ell])_-\right)^2  \cdot p_m(\ell) \cdot d\ell \\
& = \int_{0}^{\infty} \left((\ell - \sum_{k=1}^k \omega_k \cdot \mu_k)_-\right)^2 \cdot  p_m(\ell) \cdot d\ell\\ 
&\overset{(*)}{\le} \sum_{k=1}^k \omega_k \cdot \int_{0}^{\infty} \left((\ell - \mu_k)_-\right)^2  \cdot p_m(\ell) \cdot d\ell \\ 
& =  \sum_{i,j} \omega_i \cdot \omega_j \cdot \int_{0}^{\infty} \left((\ell - \mu_i)_-\right)^2  \cdot p_j(\ell) \cdot d\ell
\end{align*}
\begin{align*}
    1- \rho &~ {\le} \frac{1}{\vv} \cdot \left(\sum_{i,j} \omega_i \cdot \omega_j \cdot \int_{0}^{\infty} \left((\ell - \mu_i)_-\right)^2  \cdot p_j(\ell) \cdot d\ell  \right)\\ 
    &~\overset{(**)}{\le} \sum_{i,j} \left( \omega_i \cdot\omega_j \cdot \int_{0}^\infty\left((\ell - \mu_i)_-\right)^2  \cdot p_i(\ell) \cdot d\ell   \cdot \frac{1}{\sigma_i^2}\right) \\
    &~\le \sum_{i} \omega_i \cdot \frac{\vnn_i}{\vv_i} 
\end{align*}
$(*)$ is from the fact that $((\cdot)_-)^2$ is a convex function when $x\neq 0$, and that $(**)$ is from the assumption that:
\begin{align*}
    \frac{\vnn_{i,j}}{\vnn_{i}} \le \frac{\V[\ell]}{ \sigma_i^2},~ \frac{\V[\ell]}{ \sigma_i^2} > 1 
\end{align*}
which further implies that
\begin{align*}
   \frac{1}{\vv } \cdot {\int_{0}^\infty\left((\ell - \mu_i)_-\right)^2  \cdot p_j(\ell) \cdot d\ell} \le     \frac{1}{ \sigma_i^2 } \cdot {\int_{0}^\infty\left((\ell - \mu_i)_-\right)^2  \cdot p_i(\ell) \cdot d\ell}
\end{align*}
 It thus becomes clear that:
 \begin{align*}
    \rho &= 1- (1-\rho)\\ 
     &\overset{(***)}{\ge} 1- \sum_{i} \omega_i \cdot \frac{\vnn_i}{\vv_i}\\ 
     & = \sum_{i} \omega_i -  \sum_{i} \omega_i \cdot \frac{\vnn_i}{\vv_i} \\ 
     & = \sum_{i} \omega_i \cdot \left(1- \frac{\vnn_i}{\vv_i}\right)\\
     & = \sum_{i} \omega_i \cdot \left(1- (1-\rho_i)\right)\\
     & = \sum_{i=1}^K \omega_i \cdot \rho_i
 \end{align*}
Here $(***)$ is from the proven fact that $ 1- \rho ~\le~ \sum_{i} \omega_i \cdot \frac{\vnn_i}{\vv_i} $

\end{proof}

\section{Explanations for the Theoretical Results}\label{app:explan}
\textbf{Basic Notations}: In our approach, we use a natural training dataset consisting of images and their labels from the distribution $\mathcal{D}$, and a constructed set of potential test-label distributions from $\mathcal{E}$. We denote the training dataset as $\mathcal{S}$ and the constructed set as $\mathcal{P}$. Additionally, the training sample size is $N$ and the constructed set size is $M$.

Our primary theoretical findings relate to how well our method generalizes to new, unseen test-label distributions. According to standard learning theory, we measure generalization error by the difference between a) the expected error across the joint distribution of test-label and training data, and b) the empirical average across the sampled $\mathcal{S}$ and $\mathcal{P}$. This difference results from stochastic errors in sampling and estimation. Reflecting on the previous question, the Monte Carlo process incorporates a hierarchical sampling approach, also leading to a hierarchical structure of stochastic errors, a significant challenge in our theoretical analysis. This explain this below.

\subsection{Stochastic Errors}
\textbf{Inner Layer (training data)}: For a specific test label distribution $\mathbb{P}_i$ in $\mathcal{P}$, we define $\ell_{\mathcal{D},\mathcal{E},i}$ and $\ell_{\mathcal{S},\mathcal{E},i}$ as the expected loss on training distribution $\mathcal{D}$ and the empirical average loss on training data $\mathcal{S}$, respectively:
$$
\ell_{\mathcal{D},\mathcal{E},i} = \mathbb{E}_{(x,y) \sim \mathcal{D}}\left[ \ell_{\mathsf{LA}}(f^{(\xi_i)}(x),y;\mathbb{P}_i)   \right]
$$

$$
  \ell_{\mathcal{S},\mathcal{E},i} = \frac{1}{N}\sum_{j=1}^N\ell_{\mathsf{LA}}(f^{(\xi_i)}(x_j), y_j; \mathbb{P}_i)   
$$

These values estimate the error of the inner layer concerning the training data, specifically measuring the discrepancy when substituting expectation with empirical average for a given test-label distribution $\mathbb{P}_i$.

\textbf{Outer Layer (test label distributions):}  We also consider the expectation over the meta-distribution $\mathcal{E}$ to assess the outer Monte Carlo sampling error.

In practical training, we rely on Monte Carlo estimation results and finite training data to compute the empirical average:

$$
\hat{\mathbb{E}}_{\mathcal{\mathcal{E}}}\left[\ell_{\mathcal{S},\mathcal{E},i}\right]  =\frac{1}{NM} \sum_{i=1}^M \sum_{j=1}^N \ell_{\mathsf{LA}}(f^{(\xi_i)}(x_j),y_j;\mathbb{P}_i)
$$

But Theoretically, the genearlization error should be measured by the expected loss on the joint distribution of test label distribution $\mathcal{E}'$ (which could differ from $\mathcal{E}$ used in training) and training data $\mathcal{D}$, expressed as:  

$$
{\mathbb{E}}_{\mathbb{P} \sim \mathcal{\mathcal{E}'}}\left[\ell_{\mathcal{D},\mathcal{E}', i}\right]  = \mathbb{E}_{\mathbb{P} \sim \mathcal{E}'}\left[\mathbb{E}_{(x,y) \sim \mathcal{D}} \left[  \ell_{\mathsf{LA}}(f^{(\xi)}(x),y;\mathbb{P})\right]\right]
$$

Finally, we discuss the upper bound of the generalization error and briefly introduce the proof's key idea.

\subsection*{Error Decomposition}

Assume that our model is chosen from a hypothesis space $\mathcal{F}$ (for instance, CNNs within certain weight norm limits). The generalization ability of the entire hypothesis set is often measured by the worst-case performance gap:

$$
\Delta = \sup_{f \in \mathcal{F}} \left[{\mathbb{E}}_{\mathbb{P} \sim \mathcal{\mathcal{E}'}}\left[\ell_{\mathcal{D}, \mathcal{E}', i}\right] - \hat{\mathbb{E}}_{\mathcal{\mathcal{E}}}\left[\ell_{\mathcal{S},\mathcal{E},i}\right] \right].
$$

Analyzing this error directly is challenging due to its hierarchical nature. Nevertheless, we can further derive an upper bound by summing three types of error:
\begin{enumerate}
    \item[i)]  \textbf{Meta-distribution Approximation Error}: This error arises from approximating the ideal meta-distribution $\mathcal{E}'$ with $\mathcal{E}$, expressed as:
    $$
    \sup_{f \in \mathcal{F}} \left[{\mathbb{E}}_{\mathbb{P} \sim \mathcal{\mathcal{E}'}}\left[\ell_{\mathcal{D},\mathcal{E}',i}\right] - {\mathbb{E}}_{\mathbb{P} \sim \mathcal{\mathcal{E}}}\left[\ell_{\mathcal{D},\mathcal{E},i}\right] \right].
    $$
\item[ii)] \textbf{Label Dist}: This error stems from the approximation of the expected values over $\mathcal{E}$ with a Monte Carlo average, which can be shown as:
$$
\sup_{f \in \mathcal{F}} \left[{\mathbb{E}}_{\mathbb{P} \sim \mathcal{\mathcal{E}}}\left[\ell_{\mathcal{D},\mathcal{E},i}\right] - \frac{1}{M} \sum_{i=1}^M \ell_{\mathcal{D},\mathcal{E}, i} \right].
$$
\item[iii)] \textbf{Data Estimation Error}: This error is due to approximating the expectation over the training distribution $\mathcal{D}$ with the empirical average from the training data $S$:

$$
\sup_{f \in \mathcal{F}} \left[{\mathbb{E}}_{\mathbb{P} \sim \mathcal{\mathcal{E}}}\left[ \frac{1}{M} \sum_{i=1}^M \ell_{\mathcal{D},\mathcal{E}, i} - \frac{1}{M} \sum_{i=1}^M \ell_{\mathcal{S},\mathcal{E}, i} \right] \right].
$$
\end{enumerate}

\subsection{Key Idea of the Proof:}

From the error decomposition i)-iii) above, we outline an overall limit for the generalization gap $\Delta$.

- For i), we establish an upper bound using the loss function's boundedness, the probability simplex's volume, and the largest variation between two label distributions (measured by the inf-norm $\|\cdot\|_{\infty}$).

- For ii), the most complex part of this proof, we extend the Bernstein-type concentration bounds to upper bound the error with the empirical semi-variance from our objective function. This approach shows how the objective function directly enhances generalization.

- For iii), we translate this into the uniform convergence bound, where the covering number indicates the hypothesis class's complexity.

\section{{Proof for Overall Generalization Upper Bound}}\label{app:gen}
\subsection{Basic Definitions of The Hypothesis Class}\label{sec:def}
\textbf{The Hypothesis Class.} In this paper, we consider the  Generalization ability based on the employed architecture. This means that we only consider models chosen from the following hypothesis class $\mathcal{F}$:
\begin{align*}
 \mathcal{F}  = \left\{ \mathbb{R}^C \leftarrow \mathbb{R}^d:f^{(i)}_\theta(\cdot) = g^{(i)}  \circ \psi(\cdot), i\in [K]., ~g^{(i)}, \psi ~\text{are chosen from specific subclass of deep neural networks}.\right\}
\end{align*}
Note that since $f^{(i)}_\theta$ are logit functions for multi-class classification problems, they must be \textbf{vector-valued functions}. For the sake of simplicity, \textbf{we will use $f(x)$ under this context to denote the collection $\{f^{(i)}_\theta\}_{i \in [K]}$}. We use \underline{$f \in \mathcal{F}$} to express choosing one such collections out of $ \mathcal{F}$.

\textbf{The Norm of Hypotheses.} To measure the complexity of $\mathcal{F}$, we must define a norm on each hypothesis $f$. To do this, we adopt the overall infinity norm (over all K subbranches, all $C$ classes, and all input features $x \in \mathcal{X}$ ):
\begin{align*}
    \|f\|_\infty = \max_{i \in [K]} \max_{j \in \{1,2,\cdots, C\}} \sup_{\x \in \mathcal{X}} \left| f^{(i)}_{\theta,j} (\x) \right|.
\end{align*}
Here $f^{(i)}_{\theta,j} (\x)$ means the output score of the $i$-th branch and $j$-th (channel) for feature $\x$. It is easy to check that the new infinity norm is also a norm. 

\textbf{Measuring Complexity with Covering Number}. Given a functional class $\mathcal{F}$, we can use the covering number $\ninf$ to measure its corresponding complexity with the following definition.
\begin{align*}
 \ninf  = \sup_{\x \in \mathcal{X}^n} \mathcal{N}(\epsilon,\mathcal{F},||\cdot||_\infty)   
\end{align*}
where $\mathcal{N}(\epsilon,\mathcal{F},||\cdot||_\infty)$ is the smallest number of infinity norm open balls, denoted as $N_o$, such that there exist $N_o$ open balls that covers entire class $\mathcal{F}$. In other words,  $\exists~ \mathcal{B}_i \subseteq \mathcal{F}$, such that $\mathcal{F} \subseteq \bigcup_{i=1}^{N_o} \mathcal{B}_i$, where $\mathcal{B}_i = \{f: \|f-f_i\|_\infty \le \epsilon\}$.
\textbf{Measuring Complexity with Rademacher Complexity} Now we introduce another complexity measure based on the notion of width of a Set. 


In the proof, we will the following lemmas to as basic tools.
\subsection{Fundamental Inequalities }
\begin{lem}\label{lem:simp}
The volume of an $c$-dimensional  probability simplex is $c!$. In other words, we have:
\begin{align*}
  V_c =  \int\limits_{\sum_{i=1}^c x_i = 1} 1\cdot dx_1\cdot dx_2 \cdots dx_c = \frac{1}{c!}
\end{align*}
\end{lem}
\begin{proof}
    We proof it by induction.

\textbf{Base Case, c=1} Obviously, we have:
\begin{align*}
    V_ 1 = \int_0^1 dx = 1.
\end{align*}
\textbf{Induction} Supposes that $V_{i-1} = (i-1)!$, we have:
\begin{align*}
    V_i  &= \int\limits_{\sum_{j=1}^i x_j = 1} 1\cdot dx_1\cdot dx_2 \cdots dx_i \\ 
      & = \int_{0}^1 \left(\int\limits_{\sum_{j=1}^{i-1} x_j = 1-x_i} 1 \cdot {\color{blue}dx_1\cdot dx_2 \cdots dx_{i-1}}\right) dx_i\\ 
      & \overset{\color{org}u_j = x_j/(1-x_i), j=1,2,\cdots,i-1}{=} \int_{0}^1 (1-x_i)^{i-1} \cdot \left(\int\limits_{\sum_{j=1}^{i-1} u_j = 1} 1 \cdot {\color{org}du_1\cdot du_2 \cdots du_{i-1}}\right) dx_i\\
      &= V_{i-1} \cdot \int_{0}^1 (1-x_i)^{i-1} dx_i\\
      &= (1/i)\cdot V_{i-1}\\
      &= 1/i!
      \end{align*}
The proof is then completed by expanding the induction recursively.
\end{proof}

\begin{lem}
\label{lemma4}
When $g(\cdot)$ is Lipschitz continuous, the following holds:
\begin{equation}
    \Vert g(x)-g(\tilde{x})\Vert_{\infty} \le \sup \Vert\nabla_x g\Vert_p  \cdot \Vert x-\tilde x\Vert_q,
\end{equation}
where $\frac{1}{p}+\frac{1}{q}=1$.
\end{lem}

\begin{proof}
    \begin{equation}
    \begin{aligned}
        |g(x)-g(\tilde{x})| &=~ \left|\int_0^1 \left\langle\nabla g(\tau x + (1-\tau)\tilde{x}), x-\tilde{x} \right\rangle d\tau\right| \\
        & \le~ \sup_{ x \in \mathcal{X}} \big[ \Vert \nabla g \Vert_p \big] \cdot  \big\Vert x-\tilde{x}\big\Vert_q 
    \end{aligned}
    \end{equation}
    \end{proof}

\begin{lem}
        \label{lem:lip}
Let $\P_{tr}[y]$ be training label distribution  and $\P_{te}[y]$ a test label distribution. We denote their ratio as $q_i = \frac{\P_{te}[i]}{\P_{tr}[i]},~ i =1,2,\cdots,C$. Then we have the LA loss:
\begin{align*}
    \ell_{LA}\left(f^\xi_\theta(\x),j;\P_{j}\right)  =  \ell_{CE}\left(\mathsf{softmax}\left(f^\xi_{y,\theta}(\x)-  \log\left(q_y\right)\right)\right)
\end{align*}
is $2$-Lip. continuous w.r.t. the defined infinity norm.
    \end{lem}
\begin{proof}
According to Lem.\ref{lemma4}, if for any fixed $\xi$:
\begin{align}\label{eq:if}
    \sup_{(\x,y) \in \mathcal{Z}, f \in \mathcal{F}}\left\|\nabla_{f} \ell_{LA}\left(f^\xi_\theta(\x),j;\P_{tr}\right)  \right\|_1 \le 2
\end{align}
Then we have:
\begin{align*}
     \left|\ell_{LA}\left(f^\xi_\theta(\x),j;\P_{j}\right) - \ell_{LA}\left(\tilde{f}^\xi_\theta(\x),j;\P_{j}\right) \right| \le 2 \cdot \| f^\xi - \tilde{f}^\xi \|_\infty \le 2 \cdot \max_{i \in [K]}\| f^{(i)} - \tilde{f}^{(i)} \|_\infty = 2 \| f - \tilde{f} \|_\infty 
\end{align*}
Hence, we only need to proof \eqref{eq:if}. To see this,

\begin{align*}
    \left|\frac{\partial \ell_{LA}\left(f_\theta(\x),j;\P_{tr}\right)}{\partial f^{(j)}_\theta(\x)}\right| &= \left|\frac{\partial\left(\log\left[ \sum_{i} \exp(f^{(i)}(\x) - q_i) - (f_y - q_y) \right]\right)}{\partial f^{(j)}(\x)}\right| \\
    & = \left|\mathsf{softmax}\left(f^{(j)}_\theta(\x)-  \log\left(q_j\right)\right) - I[j = y]\right|.
\end{align*}

Since we have:
\begin{align*}
    \left\|\nabla_{f} \ell_{LA}\left(f_\theta(\x),y;\P_{tr}\right)  \right\|_1 &= \sum_{j} \left|\mathsf{softmax}\left(f^{(j)}_\theta(\x)-  \log\left(q_j\right)\right) - I[j = i]\right|\\ 
    &= 2\cdot \left(1-\mathsf{softmax}\left(f^{(y)}_\theta(\x)-  \log\left(q_y\right)\right)\right)\\ 
    &\le 2.
\end{align*}
The proof is completed since $\x,y,f$ are arbitrarily chosen.

\end{proof}

\begin{lem}
    \label{lem:lipz}
Given a meta-distribution $\mathcal{E}$ of label distributions, let    $\mathcal{E}$ be the Dirichlet mixture distribution of the components defined in the main paper. Let
\begin{align*}
    \mathcal{P}\sim \mathcal{E}^M,~  \mathcal{P}= \left\{\P_1,\cdots,\P_M \right\}
\end{align*}
For any function $\ell(f)$ with the property that  $\ell(f) = \ell(f^j)$, for the $j$ Dirichlet component, if for any fixed $j$, $\ell(f^j)$ is $L$-Lip. continuous w.r.t to the infinity norm, then $\E_{\mathcal{\mathcal{E}}}[\ell(f)],\Eh_{\mathcal{E}}[\ell(f)]$ is also  $L$-Lip. continuous w.r.t to the defined general infinity norm.
\end{lem}11
\begin{proof}
    We only prove the result for $\E_{\mathcal{\mathcal{E}}}[\ell(f)]$, since the result for $\Eh_{\mathcal{\mathcal{E}}}[\ell(f)]$ can be proven similarly.
 Since 
 \begin{align*}
    \E_{\mathcal{\mathcal{E}}}[\ell(f)]= \sum_{i \in [K]} p_i \cdot\E_{\mathcal{\mathcal{E}}|i}[\ell(f^i)],
 \end{align*}
 we have:
 \begin{align*}
    |\E_{\mathcal{\mathcal{E}}}[\ell(f)] - \E_{\mathcal{\mathcal{E}}}[\ell(\tilde{f})]| \le \sum_{i \in [K]} p_i \cdot |\E_{\mathcal{\mathcal{E}}|i}[\ell(f^i)] -  \E_{\mathcal{\mathcal{E}}|i}[\ell(\tilde{f}^i)]| \le \sum_{i \in [K]} L \cdot  p_i \cdot \|f^i - \tilde{f}^i\|_\infty \le  L \cdot \|f - \tilde{f}\|_\infty
 \end{align*}
 
\end{proof}

\subsection{\newsec{Basic (Uniform) Concentration Inequalities}}\label{App:D3}
\underline{\textbf{Please refer to Sec.\ref{sec:def} for the basic definitions of the hypothesis class $\mathcal{F}$, the covering number $\mathcal{N}_\infty$, and the infinity norm $\|\cdot\|_\infty$. }}

\begin{lem}\label{lem:cov}
    Let $\{\x_i,y_i\}_{i=1}^N$ be i.i.d. samples from a data distribution. The loss function is defined as $\ell(\x_i,y_i) = \ell(f(\x_i),y_i) \in [0,B]$ for any function $f \in \mathcal{F}$. Suppose $\ell$ is $L$-Lipschitz continuous with respect to the infinity norm $||f||_\infty = \max_{\x \in \mathcal{X}} |f(\x)|$. Then, with probability at least $1-\delta$, the following inequality holds uniformly for all $f \in \mathcal{F}_{\mathcal{A}}(\rho,p)$:
    \begin{align*}
        \left|\el - \ehl\right|  \lee   \left(\frac{p}{1-p}  +  \frac{\hat{p}}{1-\hat{p}}\right) \cdot B + \rho \cdot \sqrt{\frac{\log\left(2\mathcal{M}' /(\delta)\right)}{2N\cdot (1-\hat{p})}}
    \end{align*}
    where the condition $\mathcal{A}(\rho)  = \{(\x,y): \ell(f(\x),y) \le \rho\}$ satisfies $\P\left[ {A}^c(\rho) \right] \le p$. The term $\hat{p}$ denotes the uniform upper bound of the empirical frequency of $A^c(\rho)$ over the sampled training data for all $f \in \FA$, and $\mathcal{M}' = \ninff{1/2LN}{N}$.
\end{lem}

\begin{proof}
By the basic property of the covering number, we can construct a covering of $\mathcal{F}$ with open balls $\left\{\mathcal{B}_1,\cdots, \mathcal{B}_{\mathcal{M}'} \right\}$. Each ball $\mathcal{B}_j$ is centered at $f_j$, and defined as $\mathcal{B}_j = \{f \in \mathcal{F}: \|f -f_j\|_\infty \le \epsilon/2L \}$. 

To use the condition $\mathcal{A}(\rho)$, we first derive a uniform concentration bound of the conditional expectation over $\mathcal{A}$. Based on the covering number property, we have the following union bound:
\begin{align*}
    \P\left[ \supf{ \left|\ela - \ehla\right|} \ge t   \right] \le \sum_{j=1}^{\mathcal{M}'} \P\left[ \supB{ \left|\ela - \ehla\right|} \ge t \right]
\end{align*}

Fix any $j$. Then,
\begin{align*}
  \P\left[ \supB{ \left|\ela - \ehla\right|} \ge t + \epsilon  \right] \le & \P\left[ \supB{ \left|\ela - \E[\ell_j|\mathcal{A}] \right|} \ge \frac{\epsilon}{2}  \right] + 
  \P\left[ { \left|\E[\ell_j|\mathcal{A}(\rho)] - \Eh[\ell_j|\mathcal{A}(\rho)] \right|} \ge t  \right]  \\ 
&+\P\bigg[ \supB{ \left| \Eh[\ell|\mathcal{A}(\rho)] - \Eh[\ell_j|\mathcal{A}(\rho)] \right|} \ge \frac{\epsilon}{2}  \bigg] 
\end{align*}

Since $\ell$ is $L$-Lipschitz w.r.t. the infinity norm, we have:
\begin{align*}
    |\ell - \ell_j| \le L\cdot \infn{f-f_j} \le L\cdot \frac{\epsilon}{2L} \le \frac{\epsilon}{2}.
\end{align*}
Hence, both $\supB{ \left|\ela - \E[\ell_j|\mathcal{A}(\rho)] \right|} \ge \frac{\epsilon}{2}$ and $\supB{ \left|\Eh[\ell|\mathcal{A}(\rho)] - \Eh[\ell_j|\mathcal{A}(\rho)] \right|} \ge \frac{\epsilon}{2}$ occur with probability 0.

Moreover, by Hoeffding's inequality and the fact that all data points are sampled under $\mathcal{A}$, we obtain:
\begin{align*}
    \P\bigg[ { \left| \E[\ell_j|\mathcal{A}(\rho)] - \Eh[\ell_j|\mathcal{A}(\rho)] \right|} \ge t  \bigg] \le 2\exp\left( \frac{-2N \cdot (1-\hat{p}) \cdot t^2}{\rho^2} \right)
\end{align*}
Note that the empirical conditional expectation has only $N \cdot (1-\hat{p})$ effective samples, introducing a discount factor $(1-\hat{p})$ in the bound.

Combining all the arguments above, we get:
\begin{align*}
    \P\left[ \supf{ \left|\ela - \ehla \right|} \le t  \right] \ge 1 - \ninff{\epsilon/2L}{N} \cdot 2 \cdot \exp\left( \frac{-2N\cdot (1-\hat{p})\cdot t^2}{\rho^2} \right) .
\end{align*}

Furthermore, we have:
\begin{align*}
    |\el  - \ehl| \le |\el - \ela | + |\ela - \ehla| + |\ehla - \ehl|
\end{align*}

Also,
\begin{align*}
    |\el - \ela | \le \frac{p}{1-p}  \cdot \sup[\ell] \le  \frac{p\cdot B}{1-p} 
\end{align*}

and
\begin{align*}
    |\ehl - \ehla | \le \frac{\hat{p}}{1-\hat{p}} \cdot \sup[\ell] \le \frac{\hat{p}\cdot B}{1-\hat{p}}
\end{align*}

Then we conclude:
\begin{align*}
    \P\left[ \supf{ \left|\el - \ehl \right|} \le t  + \frac{p\cdot B}{1-p}  +  \frac{\hat{p}\cdot B}{1-\hat{p}} \right]  &\ge   \P\left[ \supf{ \left|\ela - \ehla \right|} \le t  \right] \\ 
    &\ge 1 - \ninff{\epsilon/2L}{N} \cdot 2 \cdot \exp\left( \frac{-2N\cdot (1-\hat{p})\cdot t^2}{\rho^2} \right) .
\end{align*}
The proof is completed by setting $\delta =\ninff{\epsilon/2L}{N} \cdot 2\left( \exp\left( \frac{-2N\cdot (1-\hat{p})\cdot t^2}{\rho^2} \right) \right)$ and $\epsilon = 1/N$.
\end{proof}

\begin{lem}\label{lem:covg} 
    Let $\{\x_i,y_i\}_{i=1}^N$ be i.i.d. samples from a data distribution. The loss function is defined as 
    $\ell(\x_i,y_i) = \ell(f(\x_i),y_i) \in [0,B]$ for a function $f \in \mathcal{F}$. 
    Suppose $\ell$ is $L$-Lipschitz continuous. Then, for any $L_g$-Lipschitz continuous function $g$ of $f$, the following inequality holds with probability at least $1-\delta$, uniformly for all $f \in \FA$:   
    \begin{align*}
        \left|g\bigg(\el\bigg) - g\left(\ehl\right)\right|  \lee  L_g  \cdot  \rho \cdot \sqrt{\frac{\log(\mathcal{M}'/\delta)}{N\cdot  (1-\hat{p})}} +  L_g \cdot  \left(\frac{p}{1-p}  +  \frac{\hat{p}}{1-\hat{p}}\right) \cdot B
    \end{align*}
    Here, $\mathcal{A}(\rho)  = \{(\x,y): \ell(f(\x),y) \le \rho\}$ with $\P\left[ {A}^c(\rho) \right] \le p$. The term $\hat{p}$ denotes the uniform upper bound of the empirical frequency of $A^c(\rho)$ on the training data for all $f \in \FA$, and $\mathcal{M}' = \ninff{1/2LN}{N}$.
\end{lem}

\begin{proof}
Since $g$ is $L_g$-Lipschitz, we have:
\begin{align*}
    \supf{\left|g\bigg(\el\bigg) - g\left(\ehl\right)\right|} \le  L_g \supf{\left|\el - \ehl\right|} 
\end{align*}
Therefore:
\begin{align*}
    \P\left[  \supf{\left|g\bigg(\el\bigg) - g\left(\ehl\right)\right|}  \ge t+ \epsilon  \right]  \le \P\left[ \supf{ \left|\el - \ehl\right|} \ge \frac{t+\epsilon}{L_g}  \right] 
\end{align*}
The result follows from Lem.\ref{lem:cov} by setting $\epsilon' =  \frac{\epsilon}{L_g},~ t' =  \frac{t}{L_g} $.
\end{proof}

\begin{lem}\label{lem:varb}
    Let $\{\x_i,y_i\}_{i=1}^N$ be i.i.d. samples from a data distribution. The loss function is defined as 
    $\ell(\x_i,y_i) = \ell(f(\x_i),y_i)$ for a function $f \in \mathcal{F}$. For sufficiently large $N$ and small enough $\delta$, with probability at least $1- \delta$ over the randomness of data sampling, we have:
    \begin{align*}
        \el \lee \ehl + \sqrt{\frac{\Vh[\ell] \cdot \log\left( \mathcal{M} /\delta \right)}{(1-\hat{p})^2 \cdot N}} + \frac{\rho \cdot \log\left( \mathcal{M} /\delta \right)}{(1-\hat{p})\cdot N} +  \left(\frac{p}{1-p}  +  \frac{\hat{p}}{1-\hat{p}}\right) \cdot B 
    \end{align*}
    where $\mathcal{A}(\rho)  = \{(\x,y): \ell(f(\x),y) \le \rho\}$ and $\P\left[ {A}^c(\rho) \right] \le p$. $\hat{p}$ is the uniform bound of the empirical frequency of $A^c(\rho)$ on the sampled training data for all $f \in \FA$, and $\mathcal{M} =  \ninff{1/N}{2N}$.
\end{lem}

\begin{proof}
We follow a similar strategy as in Lem.\ref{lem:cov}. First, we study the concentration of $\E[\ell|\mathcal{A}(\rho)]$. According to \cite{bern}, the following inequality holds with probability at least $1-\delta$: 
\begin{align*}
    \ela \lee \ehla + \sqrt{\frac{\Vh[\ell|\mathcal{A}(\rho)] \cdot \log\left( \mathcal{M} /\delta \right)}{(1-\hat{p})\cdot N}} + \frac{\rho \cdot \log\left( \mathcal{M} /\delta \right)}{(1-\hat{p})\cdot N} 
\end{align*}
Here, we use the fact that all samples are drawn under the condition $\mathcal{A}$, which results in $\rho$ replacing $B$ on the right-hand side.

Next, we consider the variance term $\Vh[\ell|\mathcal{A}(\rho)]$. By the law of total variance:
\begin{align*}
    \V[X]  &= \E_Y(\V(X|Y)) + \V_Y(\E(X|Y))  \\ 
           &\ge \E_Y(\V(X|Y))
\end{align*}
Choosing $Y$ as either $\mathcal{A}(\rho)$ or $\mathcal{A}^c(\rho)$ leads to:
\begin{align*}
    \Vh[\ell|\mathcal{A}(\rho)] \le \frac{\Vh[\ell]}{1-\hat{p}}, 
\end{align*}
where $\hat{p}$ is the empirical frequency of $\x \notin \mathcal{A}$ in the training set. Thus:
\begin{align*}
    \ela \lee \ehla + \sqrt{\frac{\Vh[\ell] \cdot \log\left( \mathcal{M} /\delta \right)}{(1-\hat{p})^2 \cdot N}} + \frac{\rho \cdot \log\left( \mathcal{M} /\delta \right)}{(1-\hat{p})\cdot N} 
\end{align*}

To complete the proof, we use the decomposition:
\begin{align*}
    \el - \ehl &\le |\el -\ela| + \ela -\ehla + |\ehl -\ehla| \\ 
               &\le \ela -\ehla + \left(\frac{p}{1-p}  +  \frac{\hat{p}}{1-\hat{p}}\right) \cdot B 
\end{align*}
\end{proof}

\begin{table}[t]
    \caption{\label{tab:note} Some Important Notations Used in the Proof.}
    \centering
     \begin{tabular}{ll}
    \toprule
    Notation & Description \\
    \midrule
    Basic Quantities\\
    \midrule
    $N$ & The number of data instances in the training data.\\ 
    $M$ & The number of sampled test distributions in the Monte Carlo process. \\
    $\tau$ & A specific test label distribution.\\
    $\mathcal{E}$ & The meta-distribution of label distributions.\\ 
    $\ls$ &  $\sum_{(\x_i y_i) \in \mathcal{S}} \la(f_\theta(\x_i),y_i;\P)$~The \underline{expected} loss when the  label distribution is fixed (for example $\P$). \\
    $\ld$ & The \underline{empirical} loss on training data $\mathcal{S}$ when the label distribution is fixed. \\
    $\ell_{\mathcal{S},\mathcal{E},i}$ &  $\Eh_{\mathcal{S}}\left[ \la(f^{(\xi_i)}(\x),y;P_i)   \right]$ The empirical risk on training data $\mathcal{S}$ for $\mathcal{P}_i \in \mathcal{P}$ \\ 
    $\ell_{\mathcal{D},\mathcal{E},i}$ &  $\E_{\mathcal{D}}\left[ \la(f^{(\xi_i)}(\x),y;P_i)   \right]$ The empirical risk on training data distribution $\mathcal{D}$ for $\mathcal{P}_i \in \mathcal{P}$ \\
    $\bar{\ell}_{\mathcal{P}}$ & $\frac{1}{M}  \sum_{m=1}^M \la(f^{(\xi_m)}(\x),y;P_m)$ The empirical average over $\mathcal{P}$ given a fixed sample pair $(\x,y)$.\\ 
    \midrule
   & {\color{lightseagreen}Estimations based on the true meta-distribution $\mathcal{E}$} \\
   \midrule
$\els$ & The expected $\ls$ over the \underline{meta-distribution} of  label distributions \\
$\eld$ &The expected $\ld$ over the\underline{ meta-distribution} of label distributions \\ 
 $\vls$ & The variance of  $\ls$ over the \underline{meta-distribution} of  label distributions     \\
 $\vld$ &    The variance $\ld$ over the\underline{ meta-distribution} of label distributions  \\
 $\vpld$ &    The semi-variance $\ld$ over the\underline{ meta-distribution} of label distributions  \\
 \midrule
 & {\color{ballblue} Estimations based on Empirical meta-distribution $\mathcal{P}$} \\
 \midrule
 $\ehls$ & The empirical average of $\ls$ over the sampled label distributions in the Monte Carlo process. \\
$\ehld$ &The empirical average of $\ld$ oover the sampled label distributions in the Monte Carlo process. \\ 
 $\vhls$ & The empirical variance of  $\ls$ over the sampled label distributions in the Monte Carlo process.     \\
 $\vhld$ &    The empirical variance $\ld$ over the sampled label distributions in the Monte Carlo process.  \\
 $\vphls$ & The empirical semi-variance of  $\ls$ over the sampled label distributions in the Monte Carlo process.     \\
 $\vphld$ &    The empirical semi-variance $\ld$ over the sampled label distributions in the Monte Carlo process.  \\
 \midrule
 & Complexity Measures \\
 \midrule
 $\ninf$ & The covering number for a hypothesis class $\mathcal{F}$, with radius of the covering open ball chosen as $\epsilon$.    \\
 \bottomrule
 \end{tabular}
    \end{table}

\subsection{\newsec{Proof of the Main Result}}\label{App:D4}
\subsubsection{Notations in the Proof}\label{subsub}

\textbf{The Hierarchy of Stochastic Error.} Recall that we use a stratified sampling process for the Dirichlet mixture distribution. For each component $j$ of the mixture, we sample $M$ label distributions, denoted as
\begin{align*}
    \mathcal{P}\sim \mathcal{E}^M,~  \mathcal{P}= \left\{\P_1,\cdots,\P_M \right\}
\end{align*}
This yields an empirical sample of the test label distributions. We denote the corresponding overall estimation as $\ehld, \vhld$. We also have a fixed training dataset 
\begin{align*}
    \mathcal{S} \sim \mathcal{D}^N,~  \mathcal{S}= \{\x_i,y_i\}_{i=1}^N,
\end{align*}
which provides an empirical approximation of the training data distribution. The corresponding empirical (or population) estimations are denoted by subscripts $\mathcal{S}$ (or $\mathcal{D}$, respectively). To estimate the excess risk, we must account for the hierarchical stochastic error introduced by both label distribution sampling and data sampling. We now present the following error decomposition scheme.

In this proof, we use two groups of intermediate empirical estimations.

The first group involves quantities based on a fixed test label distribution $\P_i \in \mathcal{P}$. Specifically, we define $\ell_{\mathcal{D},\mathcal{E},i},\ell_{\mathcal{S},\mathcal{E},i}$ as the expected loss under distribution $\mathcal{D}$ (or its empirical counterpart over $\mathcal{S}$):
\begin{align}\label{eq:quant1}
    &\ell_{\mathcal{D},\mathcal{E},i} = \E_{\mathcal{D}}\left[ \la(f^{(\xi_i)}(\x),y;P_i) \right] \notag \\ 
    &\ell_{\mathcal{S},\mathcal{E},i} = \Eh_{\mathcal{S}}\left[ \la(f^{(\xi_i)}(\x),y;P_i) \right]
\end{align}

The second group of quantities aggregates the loss over all sampled label distributions, treating the average loss as a fixed loss function and analyzing both its population and empirical means. Specifically, define:
\begin{align*}
    \bar{\ell}_{\mathcal{P}} = \frac{1}{M} \sum_{m=1}^M \la(f^{(\xi_m)}(\x),y;P_m),
\end{align*}
Then we rewrite $\ehld, \ehls$ as:
\begin{align}\label{eq:quant2}
    &\ehld = \Ec_{\mathcal{D}}\left[ \bar{\ell}_{\mathcal{P}} \right] \notag \\ 
    &\ehls = \Ehc_{\mathcal{S}}\left[ \bar{\ell}_{\mathcal{P}} \right]
\end{align}

\subsubsection{The Formal Restate and Proof}

\begin{thm}[\textbf{Restate of Thm.\ref{thm:gen}}] 
    Let $\mathcal{E}$ be the true meta-distribution and $\mathcal{P}$ the empirical distribution of label distributions sampled via Monte Carlo. Let the training data $\mathcal{S}$ be sampled i.i.d. from $\mathcal{D}$. Assume that $\ninff{\epsilon}{M} \le \left(\frac{r}{\epsilon}\right)^\nu$, $\ell(\cdot) \in [0, B]$, and $\vhld \ee \vphld$ as defined in Tab.~\ref{tab:note}. Then, for any possible meta-distribution $\mathcal{E}'$ over the probability simplex $\mathbb{S}^{c-1}$, the following bound holds uniformly for all $f \in \FA$, with probability at least $1 - \delta$ over the randomness of $\mathcal{S}, \mathcal{P}$:
    \begin{align*}
        \eld \lee \ehls + \mathfrak{Reg} + \mathfrak{Err}_{approx} + \mathfrak{Err}_{sto} + \Delta^\mathcal{I} + \Delta^\mathcal{M}
    \end{align*}
    where $\hat{p}^\mathcal{I}$ and $\hat{p}^\mathcal{M}$ are uniform upper bounds on the empirical frequency of $\left(A^{\mathcal{I}}\right)^c(\rho)$ and $\left(A^{\mathcal{M}}\right)^c(\rho)$, respectively, over the training data, for all $f \in \FA$.
    
    \begin{flalign*}
      & \mathfrak{Err}_{approx} = \frac{B \cdot \|\mathcal{E} -\mathcal{E}'\|_{\infty} }{C!}, ~~~C_M =  2M +2,, 
      \\[2pt] 
      & \zeta_1 = \left(C_MBM/\delta\right)^{(1/\nu)} \cdot r, ~~~\zeta_2 = \left(C_MN/\delta\right)^{(1/\nu)} \cdot r,\\[2pt] 
    &\Delta^\mathcal{I} = \frac{p^{\mathcal{I}}B^2}{1-p^{\mathcal{I}}}  +  \frac{\hat{p}^\mathcal{I}B^2}{1-\hat{p}^\mathcal{I}}, ~~\Delta^\mathcal{M} = \frac{p^{\mathcal{M}}\cdot B}{1-p^{\mathcal{M}}}  +  \frac{\hat{p}^\mathcal{M}\cdot B}{1-\hat{p}^\mathcal{M}},  \\[6pt] 
    & \mathfrak{Err}_{sto} =  \rm \cdot \frac{  \nu \cdot \log\left( \zeta_1 \right)}{(1-\hat{p}^\mathcal{M})\cdot M} + \ri \cdot \sqrt{\frac{ \nu \cdot \log( \zeta_2)}{(1-\hat{p}^\mathcal{I})\cdot N}}\\ 
    &\mathfrak{Reg}  = \sqrt{\frac{ \nu \cdot \vphls \cdot \log\left(\zeta_1\right)}{(1-\hat{p}^\mathcal{M})^2 \cdot M}}
    \end{flalign*}
\end{thm}

\begin{proof} {\color{white} this is the beginning} \\
    {\color{white} this is the beginning} \\
    \textbf{NOTE}: \underline{Please refer to Tab.\ref{tab:note} and Sec.\ref{subsub} for all notations not explained here.}

\textbf{The Error Decomposition}. The overall uniform excess risk can be split into three parts:
\begin{align*}
    \supf{\eldt - \ehls} \le \underbrace{\supf {\eldt - \eld}}_{(1)} + \underbrace{\supf {\eld - \ehld}}_{(2)} + \underbrace{\supf{\ehld - \ehls}}_{(3)}
\end{align*}
Term (1) reflects the meta-distribution shift from the known distribution $\tau$ to an unknown distribution $\tau'$; term (2) accounts for the randomness in Monte Carlo sampling; term (3) measures the error from estimating the population mean loss using training data, assuming fixed label distributions.

\underline{\textbf{Since bounding (1) and (3) is relatively straightforward, the proof proceeds in the order: (1) $\rightarrow$ (3) $\rightarrow$ (2).}}

\textbf{The Bound for (1).} Since $\ell \in [0,B]$, we have:
\begin{align*}
    \supf {\eldt - \eld} &\le B \cdot \int_{\mathbb{S}_{c-1}} |\mathcal{E}'(\mathbb{P}) - \mathcal{E}(\mathbb{P})| d\mathbb{P}\\ 
    &\overset{(a)}{\le }\frac{B \cdot \|\mathcal{E}'(\mathbb{P}) - \mathcal{E}(\mathbb{P})\|_{\infty} }{C!}
\end{align*}
Here, $\mathbb{S}_{c-1}$ is the probability simplex in $\mathbb{R}^c$, and $\mathbb{P}$ is a $c$-dimensional probability vector sampled from either $\mathcal{E}$ or $\mathcal{E}'$. The step (a) follows directly from Lem.\ref{lem:simp}.

\textbf{The Bound for (3).} From \eqref{eq:quant2}, we obtain:
\begin{align*}
    \supf{\ehld - \ehls} = \supf{\Ec_{\mathcal{D}}\left[ \bar{\ell}_{\mathcal{P}} \right] - \Ehc_{\mathcal{S}}\left[ \bar{\ell}_{\mathcal{P}} \right] }
\end{align*}
This allows us to apply concentration inequalities to bound the deviation of $\bar{\ell}_{\mathcal{P}}$. 
By Lem.\ref{lem:cov}, with probability at least $1 - \delta$, and conditioning on $\mathcal{A}^{\mathcal{I}}(\rho)$ with $\rho = \rho^\mathcal{I}$ and $p = p^\mathcal{I}$, we have for all $f \in \FA$:
\begin{align*}
    \ehld - \ehls \lee  \rho^\mathcal{I} \cdot \sqrt{\frac{\log( \mf \cdot \mathcal{N}_1/\delta)}{(1-\hat{p}^\mathcal{I}) \cdot N}} +  \left(\frac{p^\mathcal{I}}{1 - p^\mathcal{I}} + \frac{\hat{p}^\mathcal{I}}{1 - \hat{p}^\mathcal{I}} \right) \cdot B,
\end{align*}
where $\mathcal{N}_1 = \ninff{1/4N}{N}$ and we use the Lipschitz constant in Lem.\ref{lem:lip}-\ref{lem:lipz} and standard expectation properties.

\textbf{The Bound for (2).} Applying Lem.\ref{lem:varb} to the Monte Carlo process, we obtain, with probability at least $1 - \delta$ and conditioning on $\mathcal{A}^{\mathcal{M}}(\rho)$ with $\rho = \rho^\mathcal{M}$ and $p = p^\mathcal{M}$, the following bound for all $f \in \FA$:
\begin{align}\label{eq:start1}
    \eld &\lee \ehld + \sqrt{\frac{\vhld \cdot \log\left( \mf \cdot \mathcal{M} /\delta \right)}{(1 - \hat{p}^\mathcal{M})^2 \cdot M}} + \frac{\rho \cdot \log\left( \mf \cdot \mathcal{M} /\delta \right)}{M} \\
    &\quad + \left(\frac{p^\mathcal{M}}{1 - p^\mathcal{M}} + \frac{\hat{p}^\mathcal{M}}{1 - \hat{p}^\mathcal{M}} \right) \cdot B \\  
    &\ee \ehld + \sqrt{\frac{\vphld \cdot \log\left( \mf \cdot \mathcal{M} /\delta \right)}{(1 - \hat{p}^\mathcal{M})^2 \cdot M}} + \frac{\rho \cdot \log\left( \mf \cdot \mathcal{M} /\delta \right)}{M} + \left(\frac{p^\mathcal{M}}{1 - p^\mathcal{M}} + \frac{\hat{p}^\mathcal{M}}{1 - \hat{p}^\mathcal{M}} \right) \cdot B
\end{align}
with $\mathcal{M} = \ninff{1/4M}{2M}$.

We now upper bound $\vphld$ using $\vphls$:
\begin{align*}
    \supf{\vphld - \vphls} \le \sum_{i=1}^M \frac{1}{M} \cdot \Big( 
    \underbrace{\supf{\left|\left( \left( \ell_{\mathcal{D},\mathcal{E},i} - \ehld \right)_+ \right)^2 - \left( \left( \ell_{\mathcal{D},\mathcal{E},i} - \ehls \right)_+ \right)^2 \right|}}_{\colt{(a)}} \\
    + \underbrace{\supf{\left| \left( \left( \ell_{\mathcal{D},\mathcal{E},i} - \ehls \right)_+ \right)^2 - \left( \left( \ell_{\mathcal{S},\mathcal{E},i} - \ehls \right)_+ \right)^2 \right|}}_{\colt{(b)}}
    \Big)
\end{align*}

For \colt{(a)}, we have with probability at least $1 - \delta$:
\begin{align*}
    &\supf{\left|\left( \left( \ell_{\mathcal{D},\mathcal{E},i} - \ehld \right)_+ \right)^2 - \left( \left( \ell_{\mathcal{D},\mathcal{E},i} - \ehls \right)_+ \right)^2 \right|} \\
    &\le 4B \cdot \supf{\left| \ehld - \ehls \right|} \\
    &= 4B \cdot \supf{\left| \Ec_{\mathcal{D}}\left[ \bar{\ell}_{\mathcal{P}} \right] - \Ehc_{\mathcal{S}}\left[ \bar{\ell}_{\mathcal{P}} \right] \right|} \\
    &\lee B \cdot \rho \cdot \sqrt{\frac{\log(\mf \cdot \mathcal{N}_1 / \delta)}{(1 - \hat{p}^\mathcal{I}) \cdot N}} + \left( \frac{p^\mathcal{I}}{1 - p^\mathcal{I}} + \frac{\hat{p}^\mathcal{I}}{1 - \hat{p}^\mathcal{I}} \right) \cdot B^2
\end{align*}

Similarly, for \colt{(b)}:
\begin{align*}
    &\supf{\left| \left( \left( \ell_{\mathcal{D},\mathcal{E},i} - \ehls \right)_+ \right)^2 - \left( \left( \ell_{\mathcal{S},\mathcal{E},i} - \ehls \right)_+ \right)^2 \right|} \\
    &\le 4B \cdot \supf{\left| \ell_{\mathcal{D},\mathcal{E},i} - \ell_{\mathcal{S},\mathcal{E},i} \right|} \\
    &\lee B \cdot \rho \cdot \sqrt{\frac{\log(\mf \cdot \mathcal{N}_1 / \delta)}{(1 - \hat{p}^\mathcal{I}) \cdot N}} + \left( \frac{p^\mathcal{I}}{1 - p^\mathcal{I}} + \frac{\hat{p}^\mathcal{I}}{1 - \hat{p}^\mathcal{I}} \right) \cdot B^2
\end{align*}

Combining \colt{(a)} and \colt{(b)} using the union bound, we have with probability at least $1 - 2M \cdot \delta$:
\begin{align}\label{eq:star2}
    \supf{\vphld - \vphls} \lee B \cdot \rho \cdot \sqrt{\frac{\log(\mf \cdot \mathcal{N}_1 / \delta)}{(1 - \hat{p}^\mathcal{I}) \cdot N}} + \left( \frac{p^\mathcal{I}}{1 - p^\mathcal{I}} + \frac{\hat{p}^\mathcal{I}}{1 - \hat{p}^\mathcal{I}} \right) \cdot B^2
\end{align}

Finally, by combining Eq.\eqref{eq:start1} and Eq.\eqref{eq:star2}, we conclude that with probability at least $1 - (2M + 1) \cdot \delta$:
\begin{align*}
    \eld \lee \ehld + \sqrt{\frac{\vphls \cdot \log\left( \mf \cdot \mathcal{M} / \delta \right)}{(1 - \hat{p}^\mathcal{M})^2 \cdot M}} + \frac{\rho \cdot \log\left( \mf \cdot \mathcal{M} / \delta \right)}{(1 - \hat{p}^\mathcal{M}) \cdot M} + \Delta^\mathcal{I} + \Delta^\mathcal{M}.
\end{align*}
Here, we ignore lower-order terms involving both $N^{-1/2}$ and $M^{-1/2}$.

\textbf{Final Step.} By combining the bounds for (1), (2), and (3), we obtain the following result, which holds with probability at least $1 - \delta$ uniformly for all $f \in \FA$:
\begin{align*}
    \eld \lee&~ \ehls + \frac{B \cdot \|\mathcal{E} - \mathcal{E}'\|_{\infty}}{C!} + \sqrt{\frac{\vphls \cdot \log\left( \mf \cdot \mathcal{M} / \delta \right)}{(1 - \hat{p}^\mathcal{M})^2 \cdot M}} + \frac{\rho^\mathcal{M} \cdot \log\left( \mf \cdot \mathcal{M} / \delta \right)}{(1 - \hat{p}^\mathcal{M}) \cdot M} \\
    &+ \rho^\mathcal{I} \cdot \sqrt{\frac{\log( \mf \cdot \mathcal{N}_1 / \delta)}{(1 - \hat{p}^\mathcal{I}) \cdot N}} + \Delta^\mathcal{I} + \Delta^\mathcal{M},
\end{align*}
This result follows from the assumption that
\begin{align*}
    \ninff{\epsilon}{M} \le \left( \frac{r}{\epsilon} \right)^\nu.
\end{align*}
\end{proof}

\begin{col}[\textbf{Tigher Bound with Light-tail Loss Decay}]
  Based on Thm.3 in our paper, suppose that the loess satisfies an asymptotic exponential-tail condition: there exists $\alpha_0 > 0$, such that for all $\tau \ge \tau_0 := \min(N^{-1 / \alpha_0}, M^{-1 / \alpha_0})$, 
  \begin{align*}
  &\mathbb{P}_x \left[ \mathbb{E}_{\mathcal{E}}[\ell(f(x), y, P_i)] \ge \tau \right] \lesssim \exp(-\lambda_1 \cdot \tau)\\
  &\mathbb{P}_{\mathcal{E}}\left[ \mathbb{E}_{\mathcal{D}}[\ell(f(x), y, P_i)] \ge \tau \right] \lesssim \exp(-\lambda_2 \cdot \tau)\\ 
  &\hat{\mathbb{P}}_x \left[ \mathbb{E}_{\mathcal{E}}[\ell(f(x), y, P_i)] \ge \tau \right] \lesssim \exp(-\hat{\lambda}_1 \cdot \tau)\\
  &\hat{\mathbb{P}}_{\mathcal{E}}\left[ \mathbb{E}_{\mathcal{D}}[\ell(f(x), y, P_i)] \ge \tau \right] \lesssim \exp(-\hat{\lambda}_2 \cdot \tau)
\end{align*}
  and $\lambda_1,\hat{\lambda}_1 \asymp \log(N) \cdot N^{1/\alpha_0}, \lambda_2, \hat{\lambda}_2 \asymp \log(M) \cdot M^{1/\alpha_0}$, then we can pick:
  \begin{align*}
    &\rho^\mathcal{M} \asymp   M^{-1/\alpha_0}, \rho^\mathcal{I} \asymp  N^{-1/\alpha_0}, p^\mathcal{I} \asymp \frac{1}{N}, p^\mathcal{M} \asymp \frac{1}{M} \\ 
    &\hat\rho^\mathcal{M} \asymp   M^{-1/\alpha_0}, \hat\rho^\mathcal{I} \asymp  N^{-1/\alpha_0}, \hat p^\mathcal{I} \asymp \frac{1}{N}, \hat p^\mathcal{M} \asymp \frac{1}{M}
  \end{align*}
  such that $\Delta^{\mathcal{I}} \lesssim \frac{1}{N}, \Delta^{\mathcal{M}} \lesssim \frac{1}{M}$, and
  \begin{align*}
   \mathfrak{Err}_{sto} =  &    \frac{\nu \cdot \log\left( \zeta_1 \right)}{(1-\hat{p}^\mathcal{M})\cdot {M^{{1+1/\alpha_0}}}} + \sqrt{\frac{ \nu \cdot \log( \zeta_2)}{(1-\hat{p}^\mathcal{I})\cdot {N^{{1+2/\alpha_0}}}}}
  \end{align*}
\end{col}
\begin{proof}
    For the sake of simplicity, we only prove the case for $\rho^{\mathcal{I}}$ and $p^\mathcal{I}$, the remaining cases follow similarly. 
   
    Based on the assumption, for all $\tau \ge \tau_0$: 
    \begin{align*}
        \mathbb{P}_x[ \ell(f(x),y,p_i) \ge \tau ] &\lesssim \exp(-\lambda_1 \cdot \tau)\\ 
                                          & \asymp  \exp(-\log(N) \cdot N^{1/\alpha_0} \cdot \tau)
    \end{align*}
    
    Then, by selecting $\rho^\mathcal{I} \asymp N^{-1/\alpha}$ and substituting $\tau = \rho^\mathcal{I}$, we have:
    \begin{align*}
        \mathbb{P}_x[\ell \ge \rho^\mathcal{I}] \lesssim \exp(-\log (N) \cdot N^{-1 / \alpha_0} \cdot N^{-1 / \alpha}) = \frac{1}{N}
    \end{align*}
    Thus,
    \begin{align*}
        p^\mathcal{I} \asymp \frac{1}{N}
    \end{align*}
    which verifies that the predefined $\rho^{\mathcal{I}}$ and $p^{\mathcal{I}}$ are compatible under the assumed tail behavior. The same reasoning applies to the remaining quantities.
\end{proof}


\section{\newsec{Proof of the PEFT Framework}}\label{App:E}

\subsection{Definitions and Preliminaries}

\begin{defi}[Stiefel Manifold]
The Stiefel manifold \( St(N+K,K) \) is the set of all ordered orthonormal $K$-tuples in \( \mathbb{R}^{N+K} \):
\[
St(N+K,K) := \left\{ \bm{V} = (v_1, \ldots, v_K) \in \mathbb{R}^{(N+K) \times K} \mid \bm{V}^\top\bm{V} = \bm{I}_K \right\}.
\]
\end{defi}

\begin{defi}[Special Orthogonal Group Action]
The special orthogonal group \( \mathrm{SO}(N) \) is defined as:
\begin{align*}
    \mathrm{SO}(N) = \left\{ \bm{X} \in \mathbb{R}^{N \times N} \mid \bm{X}\bm{X}^\top = \bm{X}^\top\bm{X} = \bm{I}_N,~ \det(\bm{X}) = 1 \right\}
\end{align*}
\end{defi}

\subsection{Result on the Tight Covering Number}

\begin{lem}\label{lem:stef}
Given the restricted Stiefel manifold $St(n,k)$, the covering number satisfies the following upper bound:
\begin{align*}
    \log\left( \mathcal{N}(St(n,k), ||\cdot||_{op}, \epsilon) \right) \lee \left(nk - \frac{(k+1)k}{2} \right) \cdot \log\left( \frac{1}{\epsilon} \right)
\end{align*}
\end{lem}

\begin{proof}
The proof follows from verifying the conditions of Thm. 8 in \cite{metric}.
By setting $N+K = n$ and $K = k$, Lem.\ref{lem:geo} implies that $St(n,k) \cong SO(n)/SO(n-k)$. Moreover, by Lem.\ref{lem:operator}, we have $||\mathcal{P}_{\mathfrak{x}}|| = 1$. Applying Thm. 8 in \cite{metric} then completes the proof.
\end{proof}

\begin{lem}\label{lem:rankr}
Let $\mathcal{M}(r,m,n)$ denote the set of all $m \times n$ real matrices $\bm{A}$ with rank at most $r$ and Frobenius norm bounded by $R$, i.e., $||\bm{A}||_F \le R$. Then its metric entropy satisfies the following upper bound:
\begin{align*}
    \log\left( \mathcal{N}(\mathcal{M}(r,m,n), \epsilon, ||\cdot||_F ) \right) \le \left( r(m+n) - r^2 \right) \cdot \log\left( \frac{(2R+1)^2}{\epsilon} \right)
\end{align*}
\end{lem}

\begin{proof}

We aim to derive an $\epsilon$-covering number \( N(\epsilon, \mathcal{M}(r,m,n), \|\cdot\|) \) under the Frobenius norm \( \|\cdot\|_F \).

\noindent \textbf{a) Singular Decomposition.} Any matrix \( A \in \mathcal{M}(r,m,n) \) can be decomposed using singular value decomposition (SVD):
\[
A = U \Sigma V^\top
\]
where:
\begin{itemize}
    \item \( U \in \mathbb{R}^{m \times r} \) has orthonormal columns,
    \item \( V \in \mathbb{R}^{n \times r} \) has orthonormal columns,
    \item \( \Sigma = \text{diag}(\sigma_1, \ldots, \sigma_r) \) with \( \sigma_1 \geq \cdots \geq \sigma_r \geq 0 \).
\end{itemize}

\noindent \textbf{b) Parameterization and Constraints.}
We consider the following classes for \( U, \Sigma, V \):
\begin{enumerate}
    \item The diagonal matrix \( \Sigma \) belongs to the class:
    \begin{align*}
        \mathcal{M}_1 := \left\{ \Sigma = \text{diag}(\sigma) \mid \sigma \in \mathbb{R}^r, \|\sigma\|_F \leq R \right\}.
    \end{align*}
    \item The matrices \( U, V \) belong to the Stiefel manifolds: \( U \in \text{St}(m,r), V \in \text{St}(n,r) \).
\end{enumerate}

\noindent \textbf{c) Covering Construction.}
We will construct separate covers for each component:
\begin{enumerate}
    \item For the singular values \( \Sigma \), we use a grid with radius \( \epsilon/(2R+1) \). This set can be covered with \( \mathcal{N}(\mathcal{M}_1, ||\cdot||_{\infty}, \epsilon/(2R+1)) \) open balls in the \( \ell_{\infty} \)-norm. We have the bound \( \mathcal{N}(\mathcal{M}_1, ||\cdot||_{\infty}, \epsilon) \leq r \cdot \log(R/\epsilon) \).
    \item For \( U \in \text{St}(m,r) \), with a radius of \( \epsilon/(2R+1) \), we can cover this set with \( \mathcal{N}( \text{St}(m,r), ||\cdot||_{op}, \epsilon) \) open balls. The bound is \( \log \left( \mathcal{N}( \text{St}(m,r), ||\cdot||_{op}, \epsilon) \right) \leq (mr - (r+1)r/2) \cdot \log \left( \frac{1}{\epsilon} \right) \) (Lem.\ref{lem:stef}).
    \item For \( V \in \text{St}(n,r) \), with a radius of \( \epsilon/(2R+1) \), we can cover this set with \( \mathcal{N}( \text{St}(n,r), ||\cdot||_{op}, \epsilon) \) open balls. The bound is \( \log \left( \mathcal{N}( \text{St}(n,r), ||\cdot||_{op}, \epsilon) \right) \leq (nr - (r+1)r/2) \cdot \log \left( \frac{1}{\epsilon} \right) \) (Lem.\ref{lem:stef}).
\end{enumerate}

\noindent \textbf{d) Error Decomposition.} We now use an error decomposition scheme to reduce the covering number counting problem for \( \mathcal{M}(r,m,n) \) to three simpler classes.  
For any matrix \( A = U \Sigma V^\top \), let \( \hat{A} = \hat{U} \hat{\Sigma} \hat{V}^\top \) be its closest approximation in the three series of open balls constructed above. The error decomposes as:
\begin{align*}
\|A - \hat{A}\|_F &= \|U \Sigma V^\top - \hat{U} \hat{\Sigma} \hat{V}^\top\|_F \\
&\leq \|U \Sigma V^\top - \hat{U} \Sigma V^\top\|_F + \|\hat{U} \Sigma V^\top - \hat{U} \hat{\Sigma} V^\top\|_F + \|\hat{U} \hat{\Sigma} V^\top - \hat{U} \hat{\Sigma} \hat{V}^\top\|_F \\
&= \|(U - \hat{U}) \Sigma V^\top\|_F + \|\hat{U} (\Sigma - \hat{\Sigma}) V^\top\|_F + \|\hat{U} \hat{\Sigma} (V - \hat{V})^\top\|_F \\
&\leq \|U - \hat{U}\|_{op} \|\Sigma\|_F + \|\Sigma - \hat{\Sigma}\|_{op} + \|\hat{\Sigma}\|_F \|V - \hat{V}\|_{op} \\
&= \|U - \hat{U}\|_{op} \|A\|_F + \|\Sigma - \hat{\Sigma}\|_{op} + \|\hat{A}\|_F \|V - \hat{V}\|_{op} \\
&\leq \frac{\epsilon}{2R+1} \cdot R + \frac{\epsilon}{2R+1} + R \cdot \frac{\epsilon}{2R+1} = \epsilon
\end{align*}
Using the basic property of the covering number, the total \( \epsilon \)-covering number becomes the product of individual \( \epsilon/(2R+1) \)-covers of the three classes. Taking the logarithm and using the results from part (c), we obtain:
\begin{align*}
\log\left( \mathcal{N}(\mathcal{M}(r,m,n), ||\cdot||_F, \epsilon) \right) &\leq \log(\mathcal{N}(\mathcal{M}_1, ||\cdot||_{\infty}, \tilde{\epsilon})) +  \log\left(\mathcal{N}(\text{St}(m,r), ||\cdot||_{op}, \tilde{\epsilon})\right) + \log\left(\mathcal{N}(\text{St}(n,r), ||\cdot||_{op}, \tilde{\epsilon})\right) \\ 
&\lee \left( r(m+n) - r^2 \right) \cdot \log\left( \frac{(2R+1)^2}{\epsilon} \right),
\end{align*}
where \( \tilde{\epsilon} = \epsilon/(2R+1) \).
\end{proof}

\subsection{Proof of the Main Result}
We define the following hypothesis classes, which are consistent across all tasks:
\begin{enumerate}
    \item The parameter class of the $l$-th layer of the transformer in LoRA-LSF that involves parameter updates:
    \begin{align*}
        {\mathsf{LoRA}}_l = \left\{ \Delta \bm{\theta}_l = \bm{A}_l \bm{B}_l^\top \mid \bm{A} = \sum_{j} w_{A,j,l} \cdot \bm{A}_{(j),l}, \quad \bm{B} = \sum_{j} w_{B,j,l} \cdot \bm{B}_{(j),l}, \|\Delta \bm{\theta}_l\|_F \le R \right\}
    \end{align*}
    where $*_l$ denotes the corresponding parameter for the $l$-th layer.
    \item The overall parameter class:
    \begin{align*}
        {\mathsf{LoRA}} = \bigotimes_{l=1}^{n_L} {\mathsf{LoRA}}_l
    \end{align*}
    \item The parameter class of the $l$-th layer of the transformer in Adapter-LSF that involves parameter updates:
    \begin{align*}
        \mathsf{Adapter}_l = \left\{ \Delta \bm{\theta}_l = \left\{ \bm{W}_{\downarrow,l}, \bm{W}_{\uparrow,l} \right\} \mid \bm{W}_{\downarrow,l} = \sum_{j} w_{\downarrow, j,l} \cdot \bm{W}_{\downarrow,(j),l}, \quad \bm{W}_{\uparrow,l} = \sum_{j} w_{\uparrow, j,l} \cdot \bm{W}_{\uparrow,(j),l}, \|\Delta \bm{\theta}_l\|_F \le R \right\}
    \end{align*}
    \item The overall parameter class:
    \begin{align*}
        {\mathsf{Adapter}} = \bigotimes_{l=1}^{n_L} {\mathsf{Adapter}}_l
    \end{align*}
    \item The hypothesis class for the score function for LoRA-LSF:
    \begin{align*}
        \mathcal{F}^{\mathsf{LoRA}} = \left\{ f^{(i)}_{\bm{\Theta}^{i}}(\cdot), i=1,2,\cdots, N_K \mid f^{(i)} \text{ are linearized scores as in Eq.\ref{eq:lin}}, \text{ the overall parameter set } \dif{\thi{i}} \in {\mathsf{LoRA}}, \forall i \right\}
    \end{align*}
    \item The hypothesis class for the score function for Adapter-LSF:
    \begin{align*}
        \mathcal{F}^{\mathsf{Adapter}} = \left\{ f^{(i)}_{\bm{\Theta}^{i}}(\cdot), i=1,2,\cdots, N_K \mid f^{(i)} \text{ are linearized scores as in Eq.\ref{eq:lin}}, \text{ the overall parameter set } \dif{\thi{i}} \in {\mathsf{Adapter}}, \forall i \right\}
    \end{align*}
\end{enumerate}

\begin{clm}\label{clm:pre}
    Given the linearized hypothesis class \( \mathcal{F}^{\mathsf{LoRA}} \), if the parameter class is \( \epsilon / L_f \)-compact, then its covering number satisfies:
    \begin{align*}
        \mathcal{N}(\mathcal{F}^{\mathsf{LoRA}}, \epsilon, \|\cdot\|_\infty) \le \mathcal{N}({\mathsf{LoRA}}, \epsilon / L_f, \|\cdot\|_2)^{n_K}.
    \end{align*}
\end{clm}

\begin{proof}
    For any \( i \in [n_K] \), and for any pair of functions \( f, \tilde{f} \in \mathcal{F}^{\mathsf{LoRA}} \), such that the parameter set is $\epsilon/L_f$-compact, we have the following inequality based on Eq. (\ref{eq:lin}) and Assumption 1:
    \begin{align}\label{eq:compact}
        |f^{(i)}(\x) - \tilde{f}^{(i)}(\x)| \le L_f \cdot \|\dif{\thi{i}} - \dif{\widetilde{\thi{i}}}\| \le L_f \cdot \max_{i=1,\cdots, n_K}\|\dif{\thi{i}} - \dif{\widetilde{\thi{i}}}\| \bm{:=} L_f \cdot \|\dif{\thi{i}} - \dif{\widetilde{\thi{i}}}\|_{2,\infty}
    \end{align}

Based on the definition of \( \mathcal{F}^{\mathsf{LoRA}} \), we know that the incremental parameters must come from \( \mathsf{LoRA}^{n_K} \), since for each expert we need to choose a parameter from \( \mathsf{LoRA} \), leading to a product space. Since the parameter set is \( \epsilon / L_f \)-compact, the arguments in (\ref{eq:compact}) hold. In this case, choosing a radius of \( \epsilon / L_f \) is valid. From (\ref{eq:compact}), we know that a \( \epsilon / L_f \)-cover of \( \mathsf{LoRA}^{n_K} \) must induce a \( \epsilon \)-cover of \( \mathcal{F}^{\mathsf{LoRA}} \), which implies:
\begin{align*}
    \mathcal{N}(\mathcal{F}^{\mathsf{LoRA}}, \epsilon, \|\cdot\|_\infty) \le \mathcal{N}({\mathsf{LoRA}}^{n_K}, \epsilon / L_f, \|\cdot\|_{2,\infty})
\end{align*}

Next, we bound the result using the covering number of a single copy of the incremental parameter from \( \mathsf{LoRA} \). Specifically, let \( \mathcal{K} = \left\{ \theta_1, \cdots, \theta_{\mathcal{N}} \right\} \) be a covering of \( \mathsf{LoRA} \) with radius \( \epsilon / L_f \). For an element \( \dif{\Th} = \left( \dif{\thi{1}}, \cdots, \dif{\thi{n_K}} \right) \in \mathsf{LoRA}^{n_K} \), we must cover each \( \dif{\thi{i}} \) with elements from \( \mathcal{K} \) (taking the maximum of the norm requires covering the incremental parameters for each task). In this sense, \( \mathcal{K}^{n_K} \) induces a cover of \( \mathsf{LoRA}^{n_K} \) with the desired norm. Therefore, we have:
\begin{align*}
    \mathcal{N}({\mathsf{LoRA}}^{n_K}, \epsilon / L_f, \|\cdot\|_{2,\infty}) \le \mathcal{N}({\mathsf{LoRA}}, \epsilon / L_f, \|\cdot\|_2)^{n_K}
\end{align*}

\end{proof}

\begin{clm}\label{clm:q}
    Given the linearized hypothesis class \( \mathcal{F}^{\mathsf{LoRA}} \), if the parameter class is \( \epsilon / L_f \)-compact, then its covering number satisfies:
\begin{align*}
    \mathcal{N}(\mathcal{F}^{\mathsf{LoRA}}, \epsilon, \|\cdot\|_\infty) \le n_L \cdot (Kr \cdot (m+n) - K^2 r^2) \cdot \log\left(L_f \cdot n_L \cdot R / \epsilon\right) 
\end{align*}
\end{clm}
\begin{proof}
According to Claim \ref{clm:pre}, we have:
\begin{align*}
    \mathcal{N}(\mathcal{F}^{\mathsf{LoRA}}, \epsilon, \|\cdot\|_\infty) \le \mathcal{N}({\mathsf{LoRA}}, \epsilon / L_f, \|\cdot\|_2)^{n_K}.
\end{align*}

By definition, we have:
\begin{align*}
    {\mathsf{LoRA}} = \bigotimes_{l=1}^{n_L} {\mathsf{LoRA}}_l \subseteq \mathcal{M}(Kr, m, n)^{n_L}
\end{align*}
Using the property of the covering number for a product space and Lem.\ref{lem:space}, we get:
\begin{align*}
    \mathcal{N}(\mathcal{F}^{\mathsf{LoRA}}, \epsilon, \|\cdot\|_\infty) \le \prod_{l=1}^{n_L} \left( \mathcal{N}({\mathsf{LoRA}}_l, \|\cdot\|_F, \epsilon / (L_f \cdot n_L)) \right)^{n_K} \le \mathcal{N}(\mathcal{M}(Kr, m, n), \epsilon / (L_f \cdot n_L), \|\cdot\|_F)^{n_L \cdot n_K}
\end{align*}
Combining all results, we obtain:
\begin{align*}
    \log\left( \mathcal{N}(\mathcal{F}^{\mathsf{LoRA}}, \epsilon, \|\cdot\|_\infty) \right) \le n_L \cdot n_K \cdot (Kr \cdot (m+n) - K^2 r^2) \cdot \log\left( L_f \cdot n_L \cdot R / \epsilon \right)
\end{align*}
where the last inequality follows from Lem.\ref{lem:rankr}.
\end{proof}

\begin{clm}\label{clm:qq}
    Given the linearized hypothesis class \( \mathcal{F}^{\mathsf{Adapter}} \), where \( f \) is the linearized logit, its covering number satisfies:
    \begin{align*}
        \mathcal{N}(\mathcal{F}^{\mathsf{Adapter}}, \epsilon, \|\cdot\|_\infty) \le 2 \cdot n_L \cdot (r \cdot (m+n) - r^2) \cdot \log\left( L_f \cdot n_L \cdot R / \epsilon \right)
    \end{align*}
\end{clm}
\begin{proof}
    The proof follows similarly to the previous claim and is omitted.
\end{proof}

\begin{proof}[\textbf{Proof of the result in Thm.\ref{thm:gen1}}]
    Since the loss function is 1-Lipschitz, for a score function \( f^{(1)}, \cdots, f^{(n_K)} \) and the corresponding loss \( \ell \), we denote its linearized approximation in Eq.\ref{eq:lin} as \( f'^{(1)}, \cdots, f'^{(n_K)} \) with the corresponding loss \( \ell' \). Using the residue \( \epsilon \) from the Taylor expansion in Eq.\ref{eq:lin}, we have:
    \begin{align}\label{eq:resq}
        \eld \le \eldp + \mathbb{E}_{\mathcal{E}, \mathcal{D}}\left[ \max_{\theta \in \mathsf{LinSeg}(\Delta \theta^{(i)}, \Delta \theta^{(i), \star})} \left\| \bm{H}_{\theta}(\x) \right\|_{op} \cdot \|\Delta \theta^{(i)} - \Delta \theta^{(i), \star}\|_F^2 \right]
    \end{align}
    Similarly, we apply the same principle to the empirical loss terms, leading to the following inequality:
    \begin{align}\label{eq:resqq}
        \ehlsp \le \ehls + \hat{\mathbb{E}}_{\mathcal{E}, \mathcal{D}}\left[ \max_{\theta \in \mathsf{LinSeg}(\Delta \theta^{(i)}, \Delta \theta^{(i), \star})} \left\| \bm{H}_{\theta}(\x) \right\|_{op} \cdot \|\Delta \theta^{(i)} - \Delta \theta^{(i), \star}\|_F^2 \right]
    \end{align}
    The proof then follows by applying Thm.\ref{thm:gen} and Clms.\ref{clm:q} and \ref{clm:qq} to bound the gap \( \eldp - \ehlsp \), and using Eq.\ref{eq:resq}-\ref{eq:resqq} to bound the residual error of the linear approximation.
\end{proof}

\subsection{Additional Tools for the Proof}

\begin{lem}\label{lem:space}
For LoRA-LSF, the hypothesis class satisfies:
\begin{align*}
    \mathsf{LoRA}^{(i)}_l \subseteq \mathcal{M}(Kr,m,n)
\end{align*}
For Adapter-LSF, the hypothesis class satisfies:
\begin{align*}
    \mathsf{Adapter}^{(i)}_l  \subseteq \mathcal{M}(r,m,n) \otimes \mathcal{M}(r,m,n)
\end{align*}
\end{lem}
\begin{proof} {\color{white}dsaad} \\ 
\textbf{LoRA-LSF.} Since \( \bm{B}^{(i)} \in \mathbb{R}^{n \times r} \), we have \( \mathrm{rank}(\bm{B}^{(i)}) \le r \). Therefore:
\begin{align*}
    \mathrm{rank}\left( w^i_{A,j,l} \cdot \bm{A}_{(j),l} \cdot \bm{B}^{(i),l} \right) \le r.
\end{align*}
Using the fact that \( \mathrm{rank}(A + B) \le \mathrm{rank}(A) + \mathrm{rank}(B) \), we get:
\begin{align*}
    \mathrm{rank} \left( \bm{A}^{(i)}_l \bm{B}^{(i)^\top}_l \right) \le Kr.
\end{align*}
Thus, \( \bm{A}^{(i)}_l \bm{B}^{(i)^\top}_l \in \mathcal{M}(Kr,m,n) \).

\noindent \textbf{Adapter-LSF.} The proof follows similarly by bounding the rank of \( \bm{W}^{(i)}_{\downarrow,l} \) and \( \bm{W}^{(i)}_{\uparrow,l} \), and is thus omitted.

\end{proof}

\noindent \textbf{Geometrical Properties.} We now show a fundamental property of the Stiefel manifold.
\begin{lem}\label{lem:geo}
    \( St(N+K,K) \cong \mathrm{SO}(N+K)/\mathrm{SO}(N) \). In other words, the Stiefel manifold \( St(N+K,K) \) is diffeomorphic to the homogeneous space \( \mathrm{SO}(N+K)/\mathrm{SO}(N) \).
\end{lem}

\subsubsection{Proof}

\begin{clm}[\textbf{Transitivity of the Group Action}] The \( \mathrm{SO}(N+K) \)-action on \( St(N+K,K) \) is transitive.
\end{clm}
\begin{proof}
    For any two \( K \)-frames \( \bm{V} = (v_i)_{i=1}^K \) and \( \bm{W} = (w_i)_{i=1}^K \in St(N+K,K) \), extend them to full orthonormal bases of \( \mathbb{R}^{N+K} \), resulting in \( \bm{\tilde{V}}, \bm{\tilde{W}} \), such that:
\begin{align*}
    \bm{\tilde{V}} = (\bm{V}; v_{K+1}, \dots, v_{N+K}) \in O(N+K), \quad \bm{\tilde{W}} = (\bm{W}; w_{K+1}, \dots, w_{N+K}) \in O(N+K).
\end{align*}
Then, there exists a basis transfer mapping \( \bm{A} \) such that:
\begin{align*}
    \bm{A} \bm{\tilde{V}} = \bm{\tilde{W}}.
\end{align*}
It follows that:
\begin{align*}
    \bm{A} = \bm{\tilde{W}} \bm{\tilde{V}}^{-1}.
\end{align*}
Since \( \bm{A} \in O(N+K) \), we now need to show that for every such \( \bm{A} \), there exists a \( \bm{A}' \in SO(N+K) \) that also realizes the basis transfer between the original \( K \)-frames. Specifically, construct a diagonal matrix:
\begin{align*}
    \bm{R} = \text{diag}(1, \dots, -1^{\det(\bm{A})}).
\end{align*}
We can then construct a mapping \( \bm{A}' \in SO(N+K) \) such that:
\begin{align*}
    \bm{A}' = \bm{\tilde{W}} \bm{R} \bm{\tilde{V}}^{-1}.
\end{align*}
Thus, we have:
\begin{align*}
    \bm{A}' \bm{\tilde{V}} &= \bm{A}'(\bm{V}; \bm{v}_{K+1}, \dots, \bm{v}_{N+K}) \\
    &= \bm{\tilde{W}} \bm{R} = (\bm{W}; \bm{w}_{K+1}, \dots, (-1)^{\det(\bm{A})} \bm{w}_{N+K}).
\end{align*}
Performing blockwise calculations, we confirm:
\begin{align*}
    \bm{A}' \bm{V} = \bm{W}.
\end{align*}
Since the choice of \( \bm{V} \) and \( \bm{W} \) is arbitrary, we conclude that for all \( \bm{V}, \bm{W} \in St(N+K,K) \), there exists \( \bm{A}' \in SO(N+K) \) such that \( \bm{A}' \bm{V} = \bm{W} \). Thus, we have proved the claim that for all \( \bm{V} \in St(N+K,K) \), \( \text{Orbit}(\bm{V}) = St(N+K,K) \).
\end{proof}

\begin{clm}
    For the \( K \)-frame \( \bm{E}_K = (e_1, \dots, e_K) \), the stabilizer subgroup, defined as:
    \[
    \mathrm{Stab}(\bm{E}_K) = \left\{ \bm{B} \in \mathrm{SO}(N+K) \mid \bm{B} \bm{E}_K = \bm{E}_K, \ 1 \leq i \leq K \right\},
    \]
    satisfies:
    \[
    \mathrm{Stab}(\bm{E}_K) \cong \mathrm{SO}(N).
    \]
\end{clm}

\begin{proof}
First, we rewrite the mapping \( \bm{B} \in \mathrm{SO}(N+K) \) in block form:
\begin{align*}
    \bm{B} = \begin{pmatrix} 
        \bm{B}_1 & \bm{B}_2 \\ 
        \bm{B}_3 & \bm{B}_4
    \end{pmatrix}.
\end{align*}
Since \( \bm{B} \bm{E}_K = \bm{E}_K \), we have:
\begin{align*}
    \bm{B}_1 = \bm{I}_K, \quad \bm{B}_3 = 0.
\end{align*}
Moreover, since \( \bm{B}^\top \bm{B} = \bm{B} \bm{B}^\top = \bm{I}_{N+K} \), we also have:
\begin{align*}
    \bm{B}_2^\top \bm{B}_2 = 0, \quad \bm{B}_4^\top \bm{B}_4 = \bm{B}_4 \bm{B}_4^\top = \bm{I}_{N}.
\end{align*}
Thus, \( \bm{B} \) must satisfy:
\begin{align*}
    \bm{B} = \begin{pmatrix} 
        \bm{I}_K & \bm{0} \\ 
        \bm{0} & \bm{B}_4
    \end{pmatrix},
\end{align*} 
and therefore:
\begin{align*}
    \mathrm{Stab}(\bm{E}_K) = \left\{ 
        \bm{B} = \begin{pmatrix} 
            \bm{I}_K & \bm{0} \\ 
            \bm{0} & \bm{B}_4
        \end{pmatrix}, \, \bm{B}_4 \in SO(N) \right\}.
\end{align*}
It is straightforward to verify that the mapping:
\begin{align*}
    \phi: \begin{pmatrix} 
        \bm{I}_K & \bm{0} \\ 
        \bm{0} & \bm{B}_4
    \end{pmatrix}  \mapsto  \bm{B}_4  
\end{align*}
is a smooth group homomorphism. Thus, the claim is proven by definition.
\end{proof}

\begin{clm}
    The Stiefel manifold \( St(N+K,K) \) is diffeomorphic to the homogeneous space \( \mathrm{SO}(N+K)/\mathrm{SO}(N) \).
\end{clm}

\begin{proof}
   This follows directly from the Orbit-Stabilizer Theorem of the Lie Group (see Thm. 4.4.3 in \cite{geostat} and Chap. 2 in \cite{lie1}):
   \begin{align*}
    St(N+K,K) = \mathrm{Orbit}(\bm{E}_K) \cong \mathrm{SO}(N+K) / \mathrm{Stab}(\bm{E}_K) \cong \mathrm{SO}(N+K) / \mathrm{SO}(N).
   \end{align*}
\end{proof}

\noindent \textbf{Bound of the Operator Norm.} In this proof, we use the Lie algebra of \( SO(N) \), which is defined as:
\begin{align*}
    \mathfrak{so}(N) = \left\{ \bm{X} \in \mathbb{R}^{N \times N} \mid \bm{X} + \bm{X}^\top = \bm{0} \right\}.
\end{align*}
In other words, it is the subspace of skew-symmetric matrices.

Let \( \mathfrak{g} = \mathfrak{so}(N+K) \) be the Lie algebra of \( SO(N+K) \), \( \mathfrak{h} = \mathfrak{so}(N) \) be the Lie algebra of \( SO(N) \), and \( \mathfrak{x} \) be the orthogonal complement of \( \mathfrak{h} \) in \( \mathfrak{g} \). Now, we derive the operator norm of the projection onto \( \mathfrak{x} \) with respect to the Frobenius norm, denoted as \( \mathcal{P}_{\mathfrak{x}} \).

\begin{lem}\label{lem:operator}
Given the definitions above, we have:
\begin{align*}
    \mathcal{P}_{\mathfrak{x}} = 1.
\end{align*}
\end{lem}

\begin{proof}

First, for any \( \bm{X} \in \mathfrak{h} \), we can embed it in \( \mathfrak{g} \) as:
\begin{align*}
    \begin{pmatrix}
        \bm{X} & \bm{0} \\ 
        \bm{0} & \bm{0}
    \end{pmatrix}
\end{align*}
This is an element of \( \mathfrak{g} \). It is easy to see that the complement, \( \mathfrak{x} \), consists of matrices of the form:
\begin{align*}
    \mathfrak{x} = \left\{ \begin{pmatrix}
        \bm{0} & \bm{B} \\ 
        -\bm{B}^\top & \bm{D}
    \end{pmatrix}, \bm{B} \in \mathbb{R}^{N \times K}, \bm{D} \in \mathfrak{so}(K) \right\}
\end{align*}
Furthermore, any \( \bm{A} \in \mathfrak{g} \) must have the form:
\begin{align*}
    \mathfrak{g} = \left\{ \begin{pmatrix}
        \bm{A} & \bm{B} \\ 
        -\bm{B}^\top & \bm{D}
    \end{pmatrix}, \bm{A} \in \mathfrak{so}(N), \bm{D} \in \mathfrak{so}(K), \bm{B} \in \mathbb{R}^{(N-K) \times K} \right\}
\end{align*}
This shows that if \( \bm{X} \in \mathfrak{g} \) has the form:
\begin{align*}
    \begin{pmatrix}
        \bm{A} & \bm{B} \\ 
        -\bm{B}^\top & \bm{D}
    \end{pmatrix},
\end{align*}
then 
\begin{align*}
    \mathcal{P}_{\mathfrak{x}}(\bm{X}) = \begin{pmatrix}
        \bm{0} & \bm{B} \\ 
        -\bm{B}^\top & \bm{D}
    \end{pmatrix}
\end{align*}
Since the operator \( \mathcal{P}_{\mathfrak{x}}(\bm{X}) \) is linear and satisfies \( \mathcal{P}^2_{\mathfrak{x}}(\bm{X}) = \mathcal{P}_{\mathfrak{x}}(\bm{X}) \), we immediately have \( \|\mathcal{P}_{\mathfrak{x}}(\bm{X})\| = 1 \).

\end{proof}

\begin{textbo}

\section{Validity of the Covering Number Assumption}    \label{app: validity of covering number}

The covering number assumption
\begin{equation}    \label{eq: covering number assumption}
    \mathcal N_{\infty}(\mathcal F,\varepsilon,M)\lesssim \Bigl(\frac{r}{\varepsilon}\Bigr)^{\nu},
    \qquad 0<\varepsilon\le r,
\end{equation}
\textbf{is not restrictive—it generally holds for bounded high-dimensional spaces under the $\ell_\infty$ metric}.

Below, we explain why this is the case. First, we show that this assumption \textbf{originates from a general result on covering a high-dimensional $\ell_\infty$ ball}, which provides a universal upper bound of the above form. Then, we demonstrate that this result indeed holds for \textbf{a wide range of common model families}, including single-layer linear models, fully connected DNNs, CNNs, and LoRA/LoRA-LSF. Hence, \textbf{the covering number assumption is theoretically grounded and broadly valid in practical scenarios}.



\subsection{General High-Dimensional Ball Covering}

\begin{lem}[grid cover of an $\ell_\infty$-ball \cite{DBLP:conf/iclr/LongS20}] \label{lem: grid cover of ball}
    Let $B^{\nu}_{\infty}(r)=\{\theta\in\mathbb R^{\nu}:\Vert\theta\Vert_{\infty}\le r\}$ be a ball of radius $r$ in  $\nu$ dimensions. For any $0 < \varepsilon\le r$, the covering number satisfies: 
    $$
    \mathcal N\bigl(B^{\nu}_{\infty}(r),\varepsilon; \Vert\cdot\Vert_{\infty})
    \lesssim \bigl(\frac{3r}{\varepsilon}\bigr)^{\nu}.
    $$
\end{lem}



\begin{lem}[Lipschitz push-forward \cite{DBLP:conf/iclr/LongS20}]  \label{lem: lipschitz push-forward}
    Let $(\Theta,\Vert\cdot\Vert_{\infty})$ be a parameter space, $\mathcal{F}=\{f_{\theta}:\theta\in\Theta\}$ a function class on $\mathcal X$, and assume
    $$
    \sup_{x\in\mathcal X}\Vert f_{\theta}(x)-f_{\theta'}(x)\Vert_{\infty}
    \le L\,\Vert\theta-\theta'\Vert_{\infty}
    \quad \text{for all }\theta,\theta'\in\Theta .
    $$
    Then, we have: 
    $$
    \mathcal N_{\infty}(\mathcal{F},\varepsilon,M)\le \mathcal N(\Theta,\varepsilon/L;\Vert\cdot\Vert_{\infty})
    $$
\end{lem}


\begin{col}
If $\Theta\subseteq B_{\infty}^{\nu}(R)$ and the mapping $\theta\mapsto f_\theta$ is $L$-Lipschitz as above, then by Lemma. \ref{lem: grid cover of ball} and Lemma. \ref{lem: lipschitz push-forward},
$$
\mathcal N_{\infty}(\mathcal{F},\varepsilon,M)\le
\Bigl(\frac{3RL}{\varepsilon}\Bigr)^{\nu}
\lesssim\Bigl(\frac{r}{\varepsilon}\Bigr)^{\nu},
$$
where $r:=3RL$. 
\end{col}

Thus, the assumption directly follows from covering an $\ell_\infty$ ball and a Lipschitz parameter-to-function map. Next, we show that this Lipschitz condition holds for \textbf{many common model families}, so the covering number assumption applies to them. Therefore, the covering number assumption is reasonable and broadly valid across practical cases.

\subsection{Applications}

Throughout, we assume the input set $\mathcal X\subseteq\{x:\Vert x\Vert_{\infty}\le M\}$. For any matrix $W \in \mathbb{R}^{m \times n}$, define:

\begin{itemize}
    \item $\Vert W \Vert_{\max}=\max_{i,j}\Vert W_{ij}\Vert$ (entrywise $\ell_\infty$),
    \item $\Vert W \Vert_{\infty}=\max_{i}\sum_{j}\Vert W_{ij}\Vert$ (operator norm for $\ell_\infty\to\ell_\infty$)
\end{itemize}
Note that $\Vert W \Vert_{\max} \le \Vert W \Vert_\infty$ and $\Vert W \Vert_{\infty}\le n \Vert W \Vert_{\max}$.

\subsubsection{Single-Layer Linear Models} 

Let $f_\theta(x)=Wx + b$ be a single-layer linear model with $W\in\mathbb R^{m\times n}$ and $b \in \mathbb{R}^n$. The parameter set is denoted as $\Theta = \{\theta = (W, b) \ \big| \ \Vert \theta \Vert_\infty := \max(\Vert W \Vert_{\max}, \Vert b \Vert_\infty) \le R \}$ with $\nu = mn + n$ dimensions.

For any $\theta = (W, b), \, \theta' = (W', b')$ and $x \in \mathcal{X}$, we have:
\begin{equation}
    \begin{aligned}
        \Vert f_\theta(x)-f_{\theta'}(x)\Vert_{\infty}
        &=\Vert (W-W')x + (b - b') \Vert_{\infty} \\
        &\le \Vert W-W'\Vert_{\infty}\Vert x\Vert_{\infty} + \Vert b - b' \Vert_\infty \\
        &\le 2n\Vert W-W'\Vert_{\max}\,M  + \Vert b - b' \Vert_\infty \\
        &\le (2nM + 1)\Vert \theta - \theta' \Vert_\infty
    \end{aligned}    
\end{equation}

Thus, the mapping $\theta \mapsto f_\theta$ is $L=(2nM + 1)$-Lipschitz under the parameter $\ell_\infty$ norm.
By Lemma. \ref{lem: grid cover of ball} and Lemma. \ref{lem: lipschitz push-forward}, 
$$
\mathcal N_{\infty}(\mathcal{F}_{\text{linear}},\varepsilon,M)
\le \Bigl(\frac{3R(2nM+1)}{\varepsilon}\Bigr)^{mn + n}
\lesssim \Bigl(\frac{r}{\varepsilon}\Bigr)^{\nu}
\quad(\nu=mn + n,\ r= 6RnM + 3R).
$$


\subsubsection{DNNs}

We consider DNNs with complicated conv layers. Specifically, in the following proposition, we consider a set of DNNs with sparsity constraints in each layer. Note that we introduce sparsity since most real-world NNs can be regarded as a specific kind of sparse DNN. Moreover, the result also covers fully connected DNNs by setting $s_l = d_{l-1}d_l$. 

\begin{textbo}    
\begin{restatable}{prop}{propCompSparse}\label{prop:compSparse}
Let $\mathcal{F}$ be the function class of an $L$-layer neural network with:
\begin{enumerate}
    \item No biases
    \item  1-Lipchitz activation functions (e.g., ReLU)
    \item Layer widths $d_0, d_1, \dots, d_L$ where $d_0$ is the input dimension
    \item Weight matrices $W^l \in \mathbb{R}^{d_l \times d_{l-1}}$ for $l = 1, \dots, L$
    \item Sparsity: each weight matrix $W^l$ has a fixed sparsity pattern with at most $s_l = \rho_l d_l d_{l-1}$ non-zero elements, where $\rho_l \ll 1$
    \item Logits output: $f(\bm{x}) \in \mathbb{R}^{N_C}, f \in \mathcal{F}$ output logits for $N_C$ classes
\end{enumerate}
Moreover, assume the parameter space $\Theta$ satisfies:
\begin{enumerate}
    \item Each layer has an initial weight matrix $W_0^l$ with $\|W_0^l\|_{\text{op}} \leq 1$
    \item The hypothesis space is constrained by $\|W^l - W_0^l\|_{\text{op}} \leq \beta$ for each layer $l$
    \item The input space is bounded: $\{x : \|x\|_2 \leq B_X\}$
\end{enumerate}
Then, for any $\epsilon > 0$, the covering number of $\mathcal{F}$ under the supremum norm satisfies:
\[
\log N(\epsilon, \mathcal{F}, \|\cdot\|_{\infty}) \lesssim \left( \sum_{l=1}^L s_l \right) \cdot \log\left(\rho\right),
\]
where
$\|f(\bm{x})\|_{\infty} = \sup_{\bm{x} \in \mathcal{X}} \max_{i \in [N_C]} |f_i (\bm{x})|$

$$\rho =   \frac{2\beta B_X \sqrt{L} (1+\beta)^{L-1} \max_l \sqrt{\min(d_l, d_{l-1})}}{\epsilon}.$$

\end{restatable}
\end{textbo}

The proof is delayed to Sec.\ref{subsec: proof of covering number}

The bound shows that the covering number depends on the total number of non-zero parameters $\sum s_l$, and the logarithmic term involves the network depth and width through $\max_l \sqrt{\min(d_l, d_{l-1})}$. When sparsity is high ($\rho_l \ll 1$), the covering number is significantly reduced compared to the dense case.

\subsubsection{CNNs}



Long et al. \cite{DBLP:conf/iclr/LongS20} have demonstrated that the covering number of CNN models also follows the formulation in Equation \eqref{eq: covering number assumption}. For more details, please refer to \cite{DBLP:conf/iclr/LongS20}.

\subsubsection{LoRA and LoRA-LSF}

For the hypothesis class of LoRA-LSF, denoted as $\mathcal{F}_{\text{LoRA}}$, as \textbf{proved in Claim 2 of our paper}, we have
$$
\mathcal N_{\infty}(\mathcal{F}_{\text{LoRA}}, \varepsilon, M) \lesssim \bigl(\frac{R}{\varepsilon}\bigr)^{n_L \cdot n_K \cdot (Kr \cdot (m + n) - K^2r^2)}
$$

Thus, both LoRA and LoRA-LSF also satisfy the covering number assumption.



\subsection{Conclusion}

In summary, we establish several key results.
Lemma~\ref{lem: grid cover of ball} shows that the covering number assumption follows from a general geometric fact about covering an $(\ell_\infty)$ ball in $\mathbb R^{\nu}$.
Next, Lemma~\ref{lem: lipschitz push-forward} demonstrates that if a model’s \textbf{output} changes \textbf{Lipschitz-continuously} with respect to its \textbf{parameters}, then the function-class covering number is bounded by the covering number of the parameter ball, with the radius scaled by $L$.

We further prove that for \textbf{many common model classes}, including single-layer linear models, DNNs, CNNs, and LoRA/LoRA-LSF, the parameter-to-function mapping is Lipschitz on bounded sets.
Therefore, in all these cases, the covering number assumption holds:
$$
\mathcal N_{\infty}(\mathcal F,\varepsilon,M)\le \bigl(r/\varepsilon\bigr)^{\nu}\ 
$$
 
This analysis clarifies both \textbf{why the covering number assumption is reasonable} and \textbf{how it broadly applies} to these common model families.

\subsection{Proofs}     \label{subsec: proof of covering number}

To prove the Prop.\ref{prop:compSparse}, we need the following lemma about the Lipchitz constant for a NN defined above.
\begin{textbo}
    \begin{lem}\label{lem:Lip}
    For any neural network defined in Prop.\ref{prop:compSparse}, its Lipchitz constant $L_\theta$ can be upper bounded as follows:
    \begin{align*}
        L_\theta \leq B_X \cdot \sqrt{L} \cdot (1 + \beta)^{L-1}.
    \end{align*}
\end{lem}
\end{textbo}

\begin{proof}[\textbf{Proof of Lem.\ref{lem:Lip}}]{\color{white}xxxx}\\

The network has $ L $ layers, 1-Lipschitz activation functions (e.g., ReLU), and layer widths $ d_0, d_1, \dots, d_L $. The parameter vector $ \theta $ contains all weight matrices. We aim to determine a Lipschitz constant $ L_\theta $ such that for any two parameter sets $ \theta $ and $ \theta' $:



$$
\| f_\theta(x) - f_{\theta'}(x) \|_{\infty} \le
 L_\theta \| \theta - \theta' \|_2,
$$
where $i$ denotes the class channel, $ \mathcal{X} = \{ x : \| x \|_2 \leq B_X \} $.

Since $\| f_\theta(x) - f_{\theta'}(x) \|_{\infty} \le \sup_{x \in \mathcal{X}} \| f_\theta(x) - f_{\theta'}(x) \|_{2}$, it suffices to find $L_\theta$ satisfying:
$$
\sup_{x \in \mathcal{X}} \| f_\theta(x) - f_{\theta'}(x) \|_{2} \leq L_\theta \| \theta - \theta' \|_2,
$$


\noindent  \textbf{Step 1: Network Definition and Layer-wise Decomposition}

Let $ a_0 = x $ be the input, and for each layer $ l = 1, \dots, L $:
$$
z_l = W^l a_{l-1}, \quad a_l = \sigma_l(z_l),
$$
where $ \sigma_l $ is 1-Lipschitz. The final output is $ f_\theta(x) = a_L $.

Consider two parameter sets $ \theta = (W^1, \dots, W^L) $ and $ \theta' = (V^1, \dots, V^L) $. Define the intermediate activations:

$$ a_l = \sigma_l(W^l a_{l-1}),~~ a_l' = \sigma_l(V^l a_{l-1}'). $$




\noindent Let $ \Delta_l = \sup_{x \in \mathcal{X}} \| a_l - a_l' \|_2$. We will bound $ \Delta_L $ in terms of $ \| \theta - \theta' \|_2 $.

\noindent  \textbf{Step 2: Recursive Bound on Activation Differences}

For layer $ l $, we have:
$$
\| a_l - a_l' \|_2 = \| \sigma_l(W^l a_{l-1}) - \sigma_l(V^l a_{l-1}') \|_2 \leq \| W^l a_{l-1} - V^l a_{l-1}' \|_2,
$$
since $ \sigma_l $ is 1-Lipschitz. Then:
$$
\| W^l a_{l-1} - V^l a_{l-1}' \|_2 \leq \| W^l a_{l-1} - W^l a_{l-1}' \|_2 + \| W^l a_{l-1}' - V^l a_{l-1}' \|_2.
$$
The first term is bounded by $ \| W^l \|_{\text{op}} \cdot \| a_{l-1} - a_{l-1}' \|_2 $, and the second term by $ \| W^l - V^l \|_{\text{op}} \cdot \| a_{l-1}' \|_2 $. Thus:
$$
\| a_l - a_l' \|_2 \leq \| W^l \|_{\text{op}} \cdot \| a_{l-1} - a_{l-1}' \|_2 + \| W^l - V^l \|_{\text{op}} \cdot \| a_{l-1}' \|_2.
$$
Taking supremum over $ x $, we get:
$$
\Delta_l \leq \| W^l \|_{\text{op}} \cdot \Delta_{l-1} + \| W^l - V^l \|_{\text{op}} \cdot A_{l-1},
$$
where $ A_{l-1} = \sup_{x \in \mathcal{X}} \| a_{l-1}' \|_2 $.

\noindent  \textbf{Step 3: Bounding the Norm of Activations}

We now bound $ A_l = \sup_{x \in \mathcal{X}} \| a_l \|_2 $. Since $ \sigma_l $ is 1-Lipschitz and $ \sigma_l(0) = 0 $ (for ReLU), we have $ \| a_l \|_2 \leq \| z_l \|_2 $. Then:
$$
\| z_l \|_2 = \| W^l a_{l-1} \|_2 \leq \| W^l \|_{\text{op}} \| a_{l-1} \|_2.
$$
Thus, $ A_l \leq \| W^l \|_{\text{op}} A_{l-1} $. With $ A_0 = B_X $, we get:
$$
A_l \leq B_X \prod_{j=1}^l \| W^j \|_{\text{op}}.
$$
From the constraints, $ \| W^j \|_{\text{op}} \leq \| W_0^j \|_{\text{op}} + \beta \leq 1 + \beta $, so:
$$
A_l \leq B_X (1 + \beta)^l.
$$
Similarly, for the network with parameters $ \theta' $, $ A_l' \leq B_X (1 + \beta)^l $.

\noindent  \textbf{Step 4: Solving the Recursion for $ \Delta_L $}

We have:
$$
\Delta_l \leq (1 + \beta) \Delta_{l-1} + \| W^l - V^l \|_{\text{op}} B_X (1 + \beta)^{l-1}.
$$
Unrolling the recursion:
$$
\Delta_L \leq B_X \sum_{l=1}^L \| W^l - V^l \|_{\text{op}} (1 + \beta)^{L-1}.
$$

\noindent \textbf{Step 5: Relating to Parameter Difference}

The parameter difference is:
$$
\| \theta - \theta' \|_2^2 = \sum_{l=1}^L \| W^l - V^l \|_F^2.
$$
We have $ \| W^l - V^l \|_{\text{op}} \leq \| W^l - V^l \|_F $, so:
$$
\sum_{l=1}^L \| W^l - V^l \|_{\text{op}} \leq \sqrt{L} \sqrt{ \sum_{l=1}^L \| W^l - V^l \|_F^2 } = \sqrt{L} \| \theta - \theta' \|_2.
$$
Thus:
$$
\Delta_L \leq B_X (1 + \beta)^{L-1} \sqrt{L} \| \theta - \theta' \|_2.
$$
Therefore, the Lipschitz constant $ L_\theta $ satisfies:

$$
L_\theta \leq B_X \cdot \sqrt{L} \cdot (1 + \beta)^{L-1}.
$$

\end{proof}

\begin{proof}[\textbf{Proof of Prop.\ref{prop:compSparse}}]{\color{white}xxxx}\\
    
\noindent \textbf{ Step 1: Parameter Space Decomposition}

Note that the hypothesis space $\mathcal{F}$ consists of neural networks with different sparsity patterns. For each layer $l$, there are $\binom{d_l d_{l-1}}{s_l}$ possible sparsity patterns (choices of which  $s_l$ elements are non-zero).

Let $\mathcal{S}_l$ be the set of all possible support sets (positions of non-zero elements) for layer $l$, with $|\mathcal{S}_l| = \binom{d_l d_{l-1}}{s_l}$.

The overall parameter space can be decomposed as:
\[
\Theta = \bigcup_{S_1 \in \mathcal{S}_1, \dots, S_L \in \mathcal{S}_L} \Theta_{S_1, \dots, S_L}
\]
where $\Theta_{S_1, \dots, S_L}$ is the set of parameters with fixed sparsity pattern $(S_1, \dots, S_L)$.

\noindent \textbf{Step 2: Covering Each Fixed Sparsity Pattern}

For a fixed sparsity pattern $(S_1, \dots, S_L)$, the parameter space $\Theta_{S_1, \dots, S_L}$ has dimension $\sum_{l=1}^L s_l$. 

From the spectral norm constraint $\|W^l - W_0^l\|_{\text{op}} \leq \beta$ and the fact that 
$$\|A\|_F \leq \mathsf{rank}(A) \cdot \|A\|_{\text{op}} \le  \sqrt{\min(d_l, d_{l-1})} \|A\|_{\text{op}}$$
we have:
\[
\|W^l - W_0^l\|_F \leq \beta \sqrt{\min(d_l, d_{l-1})}.
\]
Since only the $s_l$ non-zero elements contribute to the Frobenius norm, this implies that the vector of non-zero parameters $\theta^l$ lies in a ball of radius $R_l = \beta \sqrt{\min(d_l, d_{l-1})}$ in $\mathbb{R}^{s_l}$.

Therefore, for a fixed sparsity pattern, the parameter space is contained in a product of balls:
\[
\Theta_{S_1, \dots, S_L} \subseteq \prod_{l=1}^L B_{R_l}(0) \subset \mathbb{R}^{\sum_{l=1}^L s_l}
\]
where $B_{R_l}(0)$ is the ball of radius $R_l$ in $\mathbb{R}^{s_l}$.

\noindent  \textbf{Step 3: Lipchitz Constant (Same as Before)}

From the previous derivation, we have:
\[
\|f_\theta - f_{\theta'}\|_{\infty} \leq L_\theta \|\theta - \theta'\|_2
\]
where $L_\theta = B_X \sqrt{L} (1+\beta)^{L-1}$.

\noindent \textbf{Step 4: Covering Number for Fixed Sparsity Pattern}

For a fixed sparsity pattern, the covering number satisfies:
\[
N(\epsilon, \mathcal{F}_{S_1, \dots, S_L}, \|\cdot\|_{\infty}) \leq N\left( \frac{\epsilon}{L_\theta}, \Theta_{S_1, \dots, S_L}, \|\cdot\|_2 \right)
\]
where $\mathcal{F}_{S_1, \dots, S_L}$ is the function class with fixed sparsity pattern.

Since $\Theta_{S_1, \dots, S_L} \subseteq B_R(0) \subset \mathbb{R}^{\sum s_l}$ where $R = \sqrt{\sum_{l=1}^L R_l^2} \leq \sqrt{L} \max_l R_l$, we have:
\[
N\left( \delta, \Theta_{S_1, \dots, S_L}, \|\cdot\|_2 \right) \leq \left(1 + \frac{2R}{\delta}\right)^{\sum s_l}
\]
Setting $\delta = \epsilon / L_\theta$:
\[
N(\epsilon, \mathcal{F}_{S_1, \dots, S_L}, \|\cdot\|_{\infty}) \leq \left(1 + \frac{2R L_\theta}{\epsilon}\right)^{\sum s_l}
\]
Substituting $R \leq \sqrt{L} \max_l \beta \sqrt{\min(d_l, d_{l-1})}$ and $L_\theta = B_X \sqrt{L} (1+\beta)^{L-1}$:
\[
N(\epsilon, \mathcal{F}_{S_1, \dots, S_L}, \|\cdot\|_{\infty}) \leq \left(1 + \frac{2\beta B_X \sqrt{L} (1+\beta)^{L-1} \max_l \sqrt{\min(d_l, d_{l-1})}}{\epsilon}\right)^{\sum s_l}
\]

\noindent \textbf{Step 5: Union Bound Over All Sparsity Patterns}

Since there are $\prod_{l=1}^L \binom{d_l d_{l-1}}{s_l}$ total sparsity patterns, we have:
\[
N(\epsilon, \mathcal{F}, \|\cdot\|_{\infty}) \leq \left[ \prod_{l=1}^L \binom{d_l d_{l-1}}{s_l} \right] \cdot \left(1 + \frac{2\beta B_X \sqrt{L} (1+\beta)^{L-1} \max_l \sqrt{\min(d_l, d_{l-1})}}{\epsilon}\right)^{\sum s_l}
\]

$$\rho =   \frac{2\beta B_X \sqrt{L} (1+\beta)^{L-1} \max_l \sqrt{\min(d_l, d_{l-1})}}{\epsilon}$$

Taking logarithms:
\[
\log N(\epsilon, \mathcal{F}, \|\cdot\|_\infty) \leq \sum_{l=1}^L \log \binom{d_l d_{l-1}}{s_l} + \left( \sum_{l=1}^L s_l \right) \cdot \log\left(1 +\rho \right)
\]

Using the bound $\binom{n}{k} \leq \left(\frac{en}{k}\right)^k$, we get:
\begin{align*}
    \log N(\epsilon, \mathcal{F}, \|\cdot\|_\infty) &\leq \sum_{l=1}^L  \log\left(\frac{e d_l d_{l-1}}{s_l}\right) + \left( \sum_{l=1}^L s_l \right) \cdot \log\left(1 + \rho\right) \\ 
    &\lesssim \left( \sum_{l=1}^L s_l \right) \cdot \log\left(\rho\right)
\end{align*}

This completes the proof.

\end{proof}

\end{textbo}

\begin{textbo}

\section{Discussion on the Dependency of the Generalization Bound on $C$}   \label{app: discussion on dependcy of class}

In this paper, our goal is to sharpen the bound w.r.t the number of samples $M, N$. We do not improve the bound over the number of classes $C$, since typically for large-scale datasets $C \ll \min\{M, N\}$. 

In this paper, we employ covering number based technique to derive the bound. The dependency on $C$ is implicitly contained in the metric entropy, \textit{i.e.}, the logarithm of the covering number. In this sense, we now provide the bound for metric entropy in the new version and analyze the dependency over $C$.


\subsection{Covering Number Derivation and Dependency on $C$ in Mixture-of-Experts Models }

Note that we use mixture-of-experts (MoE) models to address test-agnostic long-tail recognition. Specifically, the entire network can be viewed as a composition of a shared backbone $G$ and $K$ expert-specific modules $F_i$, where each expert $F_i$ is responsible for a specific meta-distribution. Formally, each expert outputs its logits through the composition $f_i \circ g$. 

This structure naturally induces a \textbf{product space} of expert function classes ($F = F_1 \times \cdots \times F_K$), combined with a shared network for latent embedding ($G$). Consequently, the overall function class $F \circ G$ involves both \textbf{composition} (due to hierarchical structure) and \textbf{product} (due to multiple experts). This dual structure leads to a \textbf{more complex but expressive} form of the resulting covering number bound:


\begin{textbo}

\begin{prop}[Covering Number Bound for Composite MLP Space with Spectral Norm Constraints]\label{thm:multi}
Let $ G $ be an $ L_1 $-layer MLP which satisfies the assumptions in Prop.\ref{prop:compSparse}. Specifically, the layer widths are defined as  $k_1,k_2\cdots,k_{L_1}$, respectively. For the $\ell-th$ layer, there are only $g_l$ nonzero elements in the weight matrix.

Let $ F = F_1 \times \cdots \times F_K $ be a product space for multi-expert neural networks, where each $ F_i $ is an $ L_2 $-layer MLP. Each $F_i$ satisfies all the assumptions in  Prop.\ref{prop:compSparse}.
We assume that all $F_i$s share the same layer widths denoted by  $ d_1, \dots, d_{L_2} $. Moreover, for each $i$, there are at most $s_l$ non-zero elements for the weight matrix of the $\ell$-th layer. 

Then, for any $ \epsilon > 0 $, the covering number of $ F \circ G $ under the supremum norm satisfies:
$$
\log N(\epsilon, F \circ G, \| \cdot \|_{\infty, M}) \lesssim  \left( \sum_{l=1}^{L_1} g_l \right) \cdot \log\left(\rho_G\right) + K\cdot \left( \sum_{l=1}^{L_2} s_l \right) \cdot \log\left(\rho_F\right)
$$
where:
$$
\forall f \in F, ~~~\|f\|_{\infty,M} = \max_{i \in [K]}\max_{j \in [N_C]} \sup_{x \in \mathcal{X}} | f_{\theta,j}^{(i)}(x)|,
$$
$f_{\theta,j}^{(i)}(x)$ is $j$-th class logit for the $i$-th expert, and
\begin{align*}
    &\rho_F =  \frac{4\beta B_X^2 \sqrt{L_1 \cdot L_2} (1 + \beta)^{L_1 + L_2 - 2} \max_l \sqrt{\min(d_l, d_{l-1})}}{\epsilon} \\ 
    &\rho_G = \frac{ 4 \sqrt{k_{L_1}} \beta B_X^3 L_1 \sqrt{L_2} (1 + \beta)^{2L_1 + L_2 - 3} \max_l \sqrt{\min(k_l, k_{l-1})}}{\epsilon}
\end{align*}
\end{prop}
\end{textbo}

The proof is delayed to the \textbf{Proofs} below.

From Proposition. \ref{thm:multi}, the dependency on the \textbf{number of classes $C$ arises} naturally through the \textbf{last layer} of each expert network. The only function of $C$ is $s_{L_2}$. In particular, the sparsity term $s_{L_2}$ in the bound corresponds to the number of non-zero elements in the weight matrix of the final layer, whose output dimension is $C$.
Therefore, $s_{L_2}$ is approximately proportional to $C$, leading to a \textbf{linear dependency} of the overall covering number and hence the generalization bound on $C$.

However, since $s_{L_2}$ contributes only a single term in the total sparsity sum $\sum_{l=1}^{L_2} s_l$, the impact of $C$ is relatively small when earlier layers dominate the parameter count. Moreover, in practice, the number of sampled test distributions $M$ and the training sample size $N$ are typically much larger than $C$ (e.g., in ImageNet-LT, $C=1000$, while $M$ and $N$ can be tens of thousands).
Therefore, while $C$ appears linearly in theory, it has a minor impact on the overall bound. Our primary contribution lies in achieving tight bounds with respect to $M$ and $N$.



\subsection{Avoiding Dependency on $C$ in Section 6.3 (Fine-Tuning Foundation Models)}

In Section 6.3 of the main paper, we analyze the generalization bounds of DirMixE with Latent Skill Fine-tuning (LSF). Our formulation effectively decouples the bound from $C$ by leveraging parameter-efficient fine-tuning (PEFT) and Taylor expansion.

\noindent \textbf{Key Steps:}

\begin{itemize}

\item \textbf{Local Linearized Assumption based on Taylor Expansion}: We approximate the logit function $f^{(i)}_{\Theta^{(i)}}(x)$ by its first-order Taylor expansion around a reference point $\Theta^{(i),*}$ (Eqs. 12 and 13):
\begin{align}
   f^{(i)}_{\thi{i}}(\x) = & f^{(i)}_{\this{i}}(\x) + \langle \nabla_{\this{i}}f^{(i)}_{\this{i}}(\x),  \Delta \thi{i} - \Delta\this{i} \rangle +  r(i,\x) 
  \end{align}

Here 
    \begin{equation}
        r(i,\x) = \sup_{\theta \in \mathsf{LinSeg}}   (\Delta\thi{i} - \Delta\this{i})^\top \bm{H}_{\theta}(\x) (\Delta\thi{i} - \Delta\this{i}),    
    \end{equation}
    
    where $\mathsf{LinSeg}$ denotes the set of all linear segments between $\Delta\thi{i}$ and $\Delta\this{i}$, while $\bm{H}_{\theta}(\x)$ is the corresponding Hessian matrix. 

where $\Delta\Theta^{(i)}$ is the incremental parameter for task $i$, and $\mathcal{R}(i,x)$ captures higher-order residuals (Hessian information).  In this paper, we employ the $\delta$-compact assumption to choose reference points dynamically to ensure the approximation precision and consistency with each $\epsilon$-ball. \textbf{Please see the response 2-1 for more details about the assumption.} Based on this linearized logit function, for two distinct functions $f,\tilde{f}$, we have:
$$
|f^{(i)}_{\Theta^{(i)}}(x) - \tilde{f}^{(i)}_{\Theta^{(i)}}(x)| \le L \cdot \|\Delta\Theta^{(i)} - {\Delta}\tilde{\Theta}^{(i)} \|.
$$
Then, according to Lem.\ref{lem: lipschitz push-forward} in the response, we know that the covering number for the overall model class can be upper bounded by the covering number of the class for $\Delta\Theta^{(i)}$, i.e., the trainable parameters. 

In PEFT, the pre-trained foundation model (e.g., CLIP) is kept frozen, and only a small number of low-rank incremental parameters (e.g., via LoRA or Adapter) are updated. Crucially, the output layer (classifier) remains fixed, either frozen or initialized by the pre-trained text encoder, thus its complexity does not affect the overall trainable parameters space.

\item \textbf{Covering Number for Incremental Parameters}: The hypothesis class for incremental parameters is defined based on low-rank structures. For example, in LoRA-LSF, we have $\Delta\Theta = A B^\top$ with $A \in \mathbb{R}^{m \times r}$, $B \in \mathbb{R}^{n \times r}$, with rank $r \ll \min(m,n)$. The resulting covering number depends only on $m, n, r, K$ and the number of layers $n_L$, but not on $C$. Specifically, as shown in Appendix E.3, we have:
$$
\log \mathcal{N}_\infty(\mathcal{F}^{\text{LoRA}}, \epsilon, n) \lesssim n_L \cdot n_K \cdot (K r \cdot (m + n - K r)) \log\left(\frac{L_f n_L R}{\epsilon}\right),
$$
where $\mathcal{F}^{\text{LoRA}}$ denotes the function class of networks with LoRA modules, $n_K$ is the number of sampled meta-distributions, and $L_f$ is the Lipschitz constant with respect to the parameters.

This covering number bound is clearly independent of $C$. A similar conclusion holds for Adapter-LSF, whose bound scales as $2 n_L n_K (m + n - r) \log(L_f n_L R / \epsilon)$, also free from $C$-dependence.

\end{itemize}

Consequently, in Theorem 4, the generalization bound for LSF does not explicitly depend on $C$, as the covering number and Lipschitz constant are derived solely from the incremental parameters. This is a key innovation of our analysis for foundation models.

\subsection{Proofs}

\begin{proof}[\textbf{Proof of Prop.\ref{thm:multi}}] {\color{white}xxxxxx}\\
From Lem.\ref{lem:all}, we have:
\begin{align}\label{eq:all}
    N(\epsilon, F \circ G, \|\cdot\|_{\infty, M}) \leq \left[ \prod_{i=1}^K N\left(\frac{\epsilon}{2}, F_i, \|\cdot\|_{\infty} \right) \right] \cdot N\left(\frac{\epsilon}{2 L_F}, G, \|\cdot\|_{\infty} \right),
\end{align}

For each of the covering numbers on the right-hand side, we can directly employ Prop.\ref{prop:compSparse} to derive the bound. For each $F_i$, the input becomes the output of the functions in $G$. Hence, we need to bound the norm of the outputs in G. Since $x \in \mathcal{X}$ is bounded, we only need to upper bound the Lip. Constant of $G$ as $L_G$ and then get $B_G \le L_G \cdot B_X$. The upper bound of the Lip. constant of $G$ follows Lem.\ref{lem:Lip}. This results in the following Lip. bound for $G$:
 $$ L_G \leq B_X \sqrt{L_1} (1 + \beta)^{L_1-1}. $$

In this sense, we have: 

$$ B_G \leq L_G B_X \leq B_X^2 \sqrt{L_1} (1 + \beta)^{L_1-1}. $$

This means that the Lip. for $F_i$ becomes:

$$
L_F \le B_G \sqrt{L_2} (1 + \beta)^{L_2-1}\leq B_X^2 \sqrt{L_1 \cdot L_2} (1 + \beta)^{L_1 + L_2 - 2}.
$$

From Prop.\ref{prop:compSparse}, we have:

$$\log\left(N\left(\frac{\epsilon}{2}, F_i, \|\cdot\|_\infty\right)\right) \lesssim \left( \sum_{l=1}^{L_2} s_l \right) \cdot \log\left(\rho_F\right), $$
where
$$
\rho_F =  \frac{4\beta B_X^2 \sqrt{L_1 \cdot L_2} (1 + \beta)^{L_1 + L_2 - 2} \max_l \sqrt{\min(d_l, d_{l-1})}}{\epsilon}
$$

Moreover, combining the upper bound of $L_F, L_G$ and the proof of Prop.\ref{prop:compSparse}, we reach an upper bound for the metric entropy for $G$:

$$
 \log\left(N\left(\frac{\epsilon}{2 L_F \sqrt{k_{L_1}}}, G, \|\cdot\|_\infty\right)\right) \lesssim \left( \sum_{l=1}^{L_1} g_l \right) \cdot \log\left(\rho_G\right), 
$$

where

\begin{align*}
    \rho_G  &= \frac{2\beta B_X \sqrt{L_1} (1+\beta)^{L_1-1} \max_l \sqrt{\min(k_l, k_{l-1})}}{\epsilon/(L_F \sqrt{k_{L_1}})} \\
    & = \frac{ 4 \sqrt{k_{L_1}}\beta B_X^3 L_1 \sqrt{L_2} (1 + \beta)^{2L_1 + L_2 - 3} \max_l \sqrt{\min(k_l, k_{l-1})}}{\epsilon}
\end{align*}

Combining the upper bounds with Eq.\ref{eq:all}, we complete the proof. 

\end{proof}

\begin{textbo}    
\begin{lem}\label{lem:comp}
Consider two hypothesis spaces $\mathcal{H}, \mathcal{G}$, and their composition:
$$\mathcal{H} \circ \mathcal{G}  = \left\{\text{The set of}~ N_C ~\text{class logits }f: f =  h\circ g, h \in \mathcal{H},~ g \in \mathcal{G},f_i~~ \text{is i-th class channel output}  \right\}.$$
We have the following upper bound for the covering number w.r.t the infinity norm:
$$
    N(\epsilon, \mathcal{H} \circ \mathcal{G}, \|\cdot\|_{\infty}) \leq N\left(\frac{\epsilon}{2}, \mathcal{H}, \|\cdot\|_{\infty}\right) \times N\left(\frac{\epsilon}{2 L_H \sqrt{d}}, \mathcal{G}, \|\cdot\|_{\infty}\right).
$$
where
$L_H$ is the uniform Lipschitz constant for the elements in $\mathcal{H}$ and $d$ is the output dimension of $g \in \mathcal{G}$
\end{lem}
\end{textbo}

\begin{proof}[\textbf{Proof of Lem.\ref{lem:comp}}]

The following proof is based on constructing an $\epsilon$-cover for $H \circ G$.

\noindent \textbf{ Step 1 Covering Precision Selection}

Because $H$ is Lipschitz continuous, we adjust the covering precision for $G$ and $H$. Let:
    \begin{align*}
        \delta = \frac{\epsilon}{2 L_H}, \eta = \frac{\epsilon}{2}
    \end{align*}

\noindent \textbf{ Step 2 Constructing the Covering Set}
\begin{enumerate}
    \item Let $\mathcal{G}_{\text{cover}}$ be a $\delta$-cover set of $G$ under $\|\cdot\|_{\infty}$, with size $N_G = N(\delta, G, \|\cdot\|_{\infty})$.
    \item  Let $\mathcal{H}_{\text{cover}}$ be an $\eta$-cover set of $H$ under $\|\cdot\|_{\infty}$, with size $N_H = N(\eta, H, \|\cdot\|_{\infty})$.
    \item  Define the composite covering set $\mathcal{C} = \{ h' \circ g' : h' \in \mathcal{H}_{\text{cover}}, g' \in \mathcal{G}_{\text{cover}} \}$. Clearly, $\mathcal{C} \subseteq H \circ G$ and $|\mathcal{C}| \leq N_H \times N_G$.
\end{enumerate}
  
\noindent \textbf{ Step 3 Verifying $\mathcal{C}$ is an $\epsilon$-cover}
\begin{enumerate}
    \item Take any $f = h \circ g \in \mathcal{H} \circ \mathcal{G}$, where $h \in \mathcal{H}$, $g \in \mathcal{G}$.
    \item  Because $\mathcal{G}_{\text{cover}}$ is a $\delta$-cover of $G$, there exists $g' \in \mathcal{G}_{\text{cover}}$ such that $\|g - g'\|_{\infty} \leq \delta$.
    \item  Because $\mathcal{H}_{\text{cover}}$ is an $\eta$-cover of $H$, there exists $h' \in \mathcal{H}_{\text{cover}}$ such that $\|h - h'\|_{\infty} \leq \eta$.
    \item Consider $f' = h' \circ g' \in \mathcal{C}$. For any $x \in X$:
        $$
        \linf{f(x) - f'(x)} = \linf{h(g(x)) - h'(g'(x))}.
        $$
        Decompose this expression:
        $$
        \linf{h(g(x)) - h'(g'(x))} \leq \linf{h(g(x)) - h(g'(x))} + \linf{h(g'(x)) - h'(g'(x))}.
        $$
    \item Analyze the first term: Because elements of $\mathcal{H}$ are Lipschitz continuous,
       \begin{align*}
              \linf{h(g(x)) - h(g'(x))} &\leq L_H \cdot \sqrt{d} \cdot \linf{g(x) - g'(x)} \leq L_H \sqrt{d}
             \|g - g'\|_{\infty} \\ 
             &\leq L_H \cdot \sqrt{d} \cdot \delta = L_H \cdot \sqrt{d} \cdot \frac{\epsilon}{2 L_H \sqrt{d}} = \frac{\epsilon}{2}.
       \end{align*}

    \item Analyze the second term:
        $$
        \linf{h(g'(x)) - h'(g'(x))} \leq \|h - h'\|_{\infty} \leq \eta = \frac{\epsilon}{2}.
        $$
    \item  Therefore,
        $$
        \linf{f(x) - f'(x)} \leq \frac{\epsilon}{2} + \frac{\epsilon}{2} = \epsilon.
        $$
        Since this holds for all $x \in X$, we have $\|f - f'\|_{\infty} \leq \epsilon$.
\end{enumerate}

\noindent \textbf{Conclusion}
$\mathcal{C}$ is an $\epsilon$-cover for $\mathcal{H} \circ \mathcal{G}$, therefore:
    $$
    N(\epsilon, \mathcal{H} \circ \mathcal{G}, \|\cdot\|_{\infty}) \leq N\left(\frac{\epsilon}{2}, \mathcal{H}, \|\cdot\|_{\infty}\right) \times N\left(\frac{\epsilon}{2 L_H \sqrt{d}}, \mathcal{G}, \|\cdot\|_{\infty}\right).
    $$

\end{proof}

\begin{textbo}
    \begin{lem}\label{lem:product}
Consider a function space $F = F_1 \times F_2 \times \cdots \times F_K$ as a product space for K-expert neural networks, where  $F_1, F_2, \cdots, F_K$ are the corresponding experts.  For any $f = (f^{(1)}, f^{(2)}, \cdots, f^{(K)}) \in F$, its norm is defined as:
$$
\|f\|_{\infty, M} = \max_{i=1,\cdots,K} \|f^{(i)}\|_\infty.
$$
Note that each $f^{(i)}$ is an expert NN model which outputs logits for $N_C$ classes. 

Then we have the following upper bound for the covering number:
$$
N(\epsilon, F, \|\cdot\|_{\infty, M}) \leq \prod_{i=1}^K N(\epsilon, F_i, \|\cdot\|_{\infty}).
$$
\end{lem}
\end{textbo}

\begin{proof}[\textbf{Proof of Lem.\ref{lem:product}}]{\color{white}xxxxxx} \\
    
\noindent \textbf{Step 1 Define the Covering Set:}
\begin{enumerate}
    \item For each $i = 1, 2, \cdots, K$, let $\mathcal{C}_i$ be an $\epsilon$-cover of $F_i$ under the norm $\|\cdot\|_\infty$, meaning for any $f^{(i)} \in F_i$, there exists $c_i \in \mathcal{C}_i$ such that $\|f^{(i)} - c_i\|_\infty \leq \epsilon$. The covering number size is $N_i = N(\epsilon, F_i, \|\cdot\|_\infty)$.
    \item 
Define the product covering set $\mathcal{C} = \mathcal{C}_1 \times \mathcal{C}_2 \times \cdots \times \mathcal{C}_K$, i.e., each element $c \in \mathcal{C}$ is a $K$-tuple $c = (c_1, c_2, \cdots, c_K)$ where $c_i \in \mathcal{C}_i$. Clearly, $\mathcal{C} \subseteq F$, and $|\mathcal{C}| = \prod_{i=1}^K |\mathcal{C}_i| = \prod_{i=1}^K N(\epsilon, F_i, \|\cdot\|_\infty)$.

\end{enumerate}
    
\noindent \textbf{Step 2 Verify $\mathcal{C}$ is an $\epsilon$-cover for $F$}

\begin{enumerate}
    \item  Take any $f = (f^{(1)}, f^{(2)}, \cdots, f^{(K)}) \in F$.
    \item  For each $i$, since $\mathcal{C}_i$ is an $\epsilon$-cover of $F_i$, there exists $c_i \in \mathcal{C}_i$ such that $\|f^{(i)} - c_i\|_\infty \leq \epsilon$.
    \item Let $c = (c_1, c_2, \cdots, c_K) \in \mathcal{C}$. Then for each $i$, $\|f^{(i)} - c_i\|_\infty \leq \epsilon$.
    \item Therefore, according to the definition of the norm on $F$:
        $$
        \|f - c\|_{\infty, M} = \max_{i=1,\cdots,K} \|f^{(i)} - c_i\|_\infty \leq \epsilon.
        $$
    \item  This shows that for any $f \in F$, there exists $c \in \mathcal{C}$ such that $\|f - c\| \leq \epsilon$, so $\mathcal{C}$ is an $\epsilon$-cover for $F$.
\end{enumerate}

\noindent \textbf{Step 3 Upper Bound on Covering Number}

Since $\mathcal{C}$ is an $\epsilon$-cover for $F$ and $|\mathcal{C}| = \prod_{i=1}^K N(\epsilon, F_i, \|\cdot\|_\infty)$, it follows that:
        $$
        N(\epsilon, F, \|\cdot\|_{\infty, M}) \leq \prod_{i=1}^K N(\epsilon, F_i, \|\cdot\|_\infty).
        $$
According to the proof of Lem.\ref{lem:comp}, we have:
$$
N(\epsilon, F_i, \|\cdot\|_\infty) \le N(\epsilon, F_i, \|\cdot\|_{\infty}).
$$
Then the proof is completed. 
\end{proof}

\begin{textbo}    
\begin{lem}\label{lem:all}
We consider the composite space $F \circ G$, where $F$ is the product space  $ F_1 \times F_2 \times \cdots \times F_K$ in Lem.\ref{lem:product}. The space $G$ is another function class. The composite space $F \circ G$ is defined as $F \circ G = \{ f \circ g : f \in F, g \in G \}$, where $f \circ g(x) = f(g(x))$. 

Then we have the following inequality:
$$
N(\epsilon, F \circ G, \|\cdot\|_{\infty, M}) \leq \left[ \prod_{i=1}^K N\left(\frac{\epsilon}{2}, F_i, \|\cdot\|_\infty  \right) \right] \cdot N\left(\frac{\epsilon}{2 L_F}, G, \|\cdot\|_{\infty} \right),
$$
where elements in $\mathcal{F}_1,\cdots, \mathcal{F}_K$ are all $L_F$-Lip w.r.t. $\ell_2$ norm.
\end{lem}
\end{textbo}

\begin{proof}[Proof of Lem.\ref{lem:all}]
The proof follows Lem.\ref{lem:comp} and Lem.\ref{lem:product} directly.
\end{proof}

\end{textbo}

\begin{textbo}

\section{Validity of the Exponential-tail Loss Decay Assumption}        \label{app: validity of exponential-tail loss}

Below, we validate the assumption in Col.\ref{col: exponential loss decay}. Specifically, we validate the assumption for \textbf{1) per-sample losses},
$$
\mathbb{P}_x[\mathbb{E}_{\mathcal{E}}[\ell(f(x),y,P_i)] \ge \tau] \lesssim \exp(-\lambda_1 \tau), \quad \text{for} \ \tau > \tau_0
$$
and for \textbf{2) per-distribution loss}:
$$
\mathbb{P}_{\mathcal{E}}[\mathbb{E}_{\mathcal{D}}[\ell(f(x),y,P_i)] \ge \tau] \lesssim \exp(-\lambda_2 \tau), \quad \text{for} \ \tau > \tau_0
$$
For the right-hand side cumulative distribution function (CDF) $\mathbb{P}(\tau) = \exp(-\lambda \tau), \ \tau > \tau_0$, its corresponding probability density function (PDF) follows the form:
$$
    p(\tau) = \lambda e^{-\lambda \tau}, \text{for} \ \tau > \tau_0
$$
Accordingly, we use the cumulative distribution function (CDF) and its associated PDF to fit the empirical loss distributions and assess the validity of the assumption.

To verify this assumption in practice, we conduct experiments on \textbf{CIFAR-100-LT and ImageNet-LT}. From the results, we observe that both distributions demonstrate approximately exponential decay in their upper tails. This provides strong empirical support for the assumptions underlying Corollary 1 and suggests that the exponential-tail characterization is reasonable in practice. We will present the detailed empirical validation below.

\noindent \textbf{1) Validation of the per-sample tail assumption.}


    


\end{textbo}

\begin{figure*}[h]
\centering
\subfigure[]{
    \includegraphics[width=0.3\textwidth]{figs/exponential_tail_loss_decay/cifar_sample_loss_histogram_with_exponential_fit.png}
}
\subfigure[]{
    \includegraphics[width=0.3\textwidth]{figs/exponential_tail_loss_decay/imagenet_sample_loss_histogram_with_exponential_fit.png}
}
\subfigure[]{
    \includegraphics[width=0.3\textwidth]{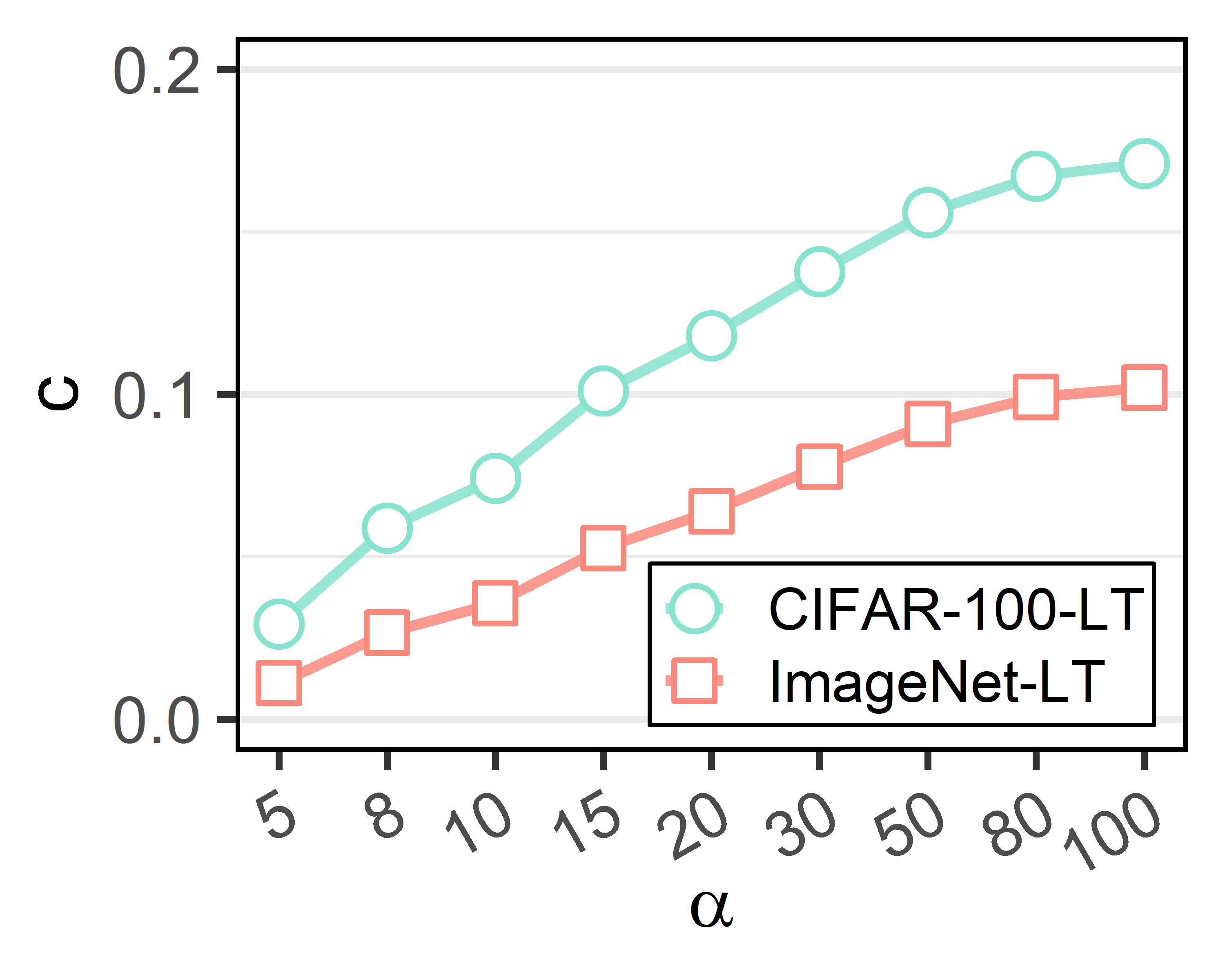}
}
\caption{
\textbf{Empirical validation of per-sample loss exponential-tail behavior.}
\textbf{(a)} Loss histogram and exponential PDF fit on CIFAR-100-LT.  
\textbf{(b)} Loss histogram and exponential PDF fit on ImageNet-LT.  
\textbf{(c)} Variation of the coefficient $c$ with respect to $\alpha$ for CIFAR-100-LT and ImageNet-LT.
}

\label{fig: sample_light_tail}

\end{figure*}

\begin{textbo}

For each training example $x$, we compute its averaged loss
$$
\bar{\ell}(x)=\frac{1}{M}\sum_{i=1}^{M}\ell\!\left(f(x),y,P_i\right),
$$
and use the resulting empirical distribution $\{\bar{\ell}(x):x\in S\}$ to evaluate tail behavior.

To test the exponential-tail hypothesis, we examine the survival function
$$
F(\tau) = \mathbb{P}[\bar{\ell}(x) \ge \tau]
$$
and perform \textbf{log-linear regression} by fitting $\log F(\tau)$ against $\tau$. Under an exponential tail, we expect
$$
\log F(\tau) = - \lambda_1 \tau + \text{const}
$$
The fitted model achieves a \textbf{high $R^2$ value of $0.95$ on CIFAR-100-LT and $0.94$ on ImageNet-LT}, indicating a strong linear relationship after the log transform and suggesting that the upper tail of the empirical loss distribution closely matches exponential decay. We further visualize the empirical histogram together with the fitted exponential PDF curve ($\lambda_1 e^{-\lambda_1 \tau}$) for CIFAR-100-LT and ImageNet-LT datasets in Fig. \ref{fig: sample_light_tail}-(a) and Fig. \ref{fig: sample_light_tail}-(b). The figures demonstrate that the observed tail trend aligns well with the exponential form. These results provide compelling evidence supporting the per-sample exponential-tail assumption.

The estimated decay parameter is $\lambda_1 = 0.758$ for CIFAR-100-LT and $\lambda_1 = 0.581$ for ImageNet-LT. Our theory suggests that $\lambda_1$ should scale as $\lambda_1 \asymp \log(N)\cdot N^{1/\alpha}$, meaning that there exists a constant $c>0$ such that
$$
\lambda_1 = c \cdot \log(N)\cdot N^{1/\alpha}.
$$
Using the empirical estimate of $\lambda_1$, we compute the implied coefficient $c$ and plot its variation with respect to $\alpha$ in Fig. \ref{fig: sample_light_tail}-(c). We observe that the estimated coefficient $c$ consistently remains above approximately $0.05$ when $\alpha > 10$ for CIFAR-100-LT and ImageNet-LT, consistent with the theoretical constraint that $\lambda_1$ and $\log(N) \cdot N^{1/\alpha}$ \textbf{are of the same order}.

\noindent \textbf{2) Validation of the per-distribution tail assumption.}

For each sampled label distribution $P_i$, we average the losses over the dataset:
$$
\tilde{\ell}(P_i)=\frac{1}{N}\sum_{(x, y)\in S}\ell\left(f(x),y,P_i\right).
$$
We then analyze the empirical distribution $\{\tilde{\ell}(P_i):i=1,\dots, M\}$ to evaluate the tail behavior for per-distribution losses.

Similar to the per-sample analysis, we investigate whether the distribution of $\tilde{\ell}(P_i)$ exhibits exponential decay in its upper tail. Specifically, we validate this by applying a similar log–linear regression procedure, fitting $\log \mathbb{P}[\tilde{\ell}(P_i)\ge \tau]$ against $\tau$. 

\end{textbo}

\begin{figure*}[h]
\centering
\subfigure[]{
    \includegraphics[width=0.3\textwidth]{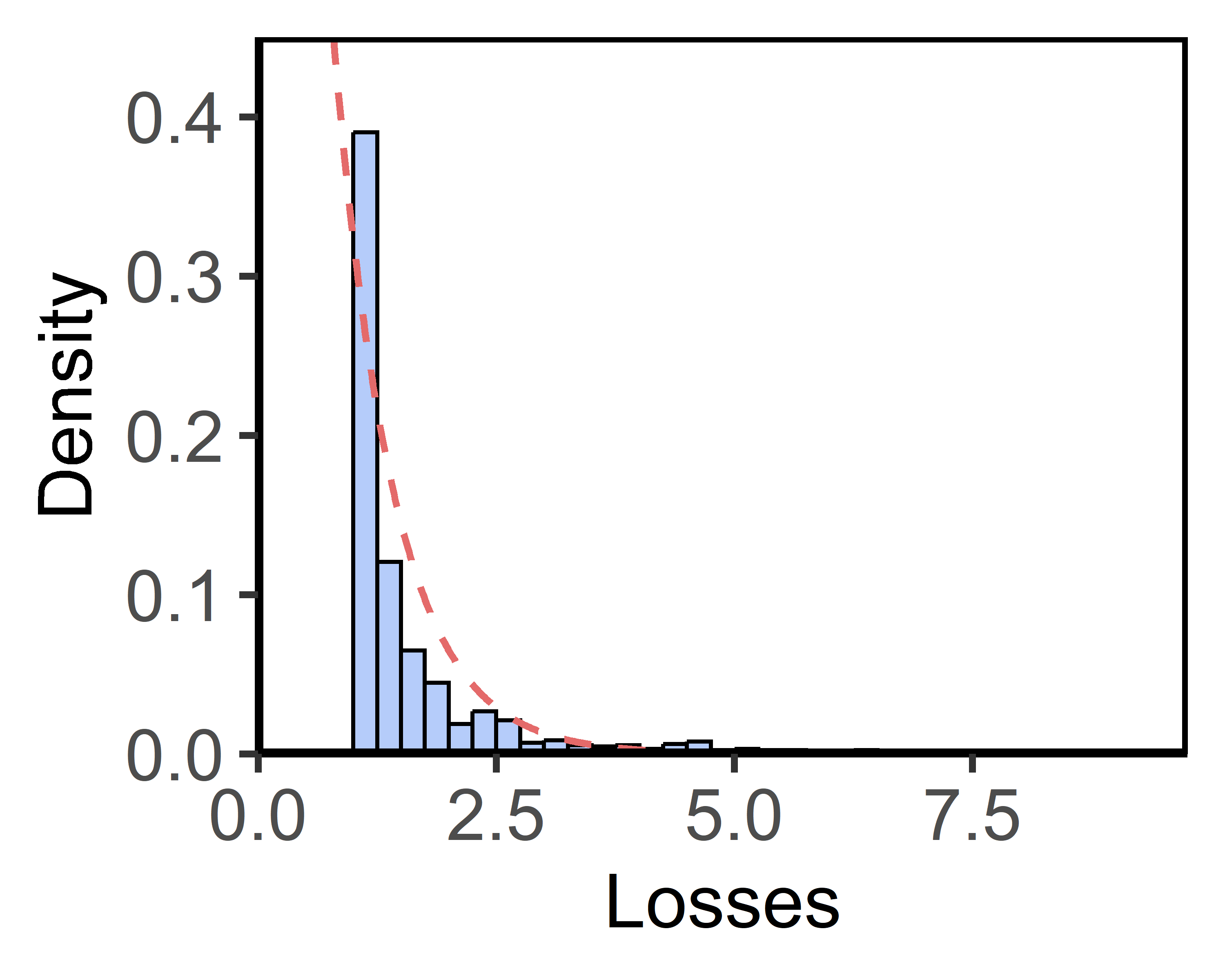}
}
\subfigure[]{
    \includegraphics[width=0.3\textwidth]{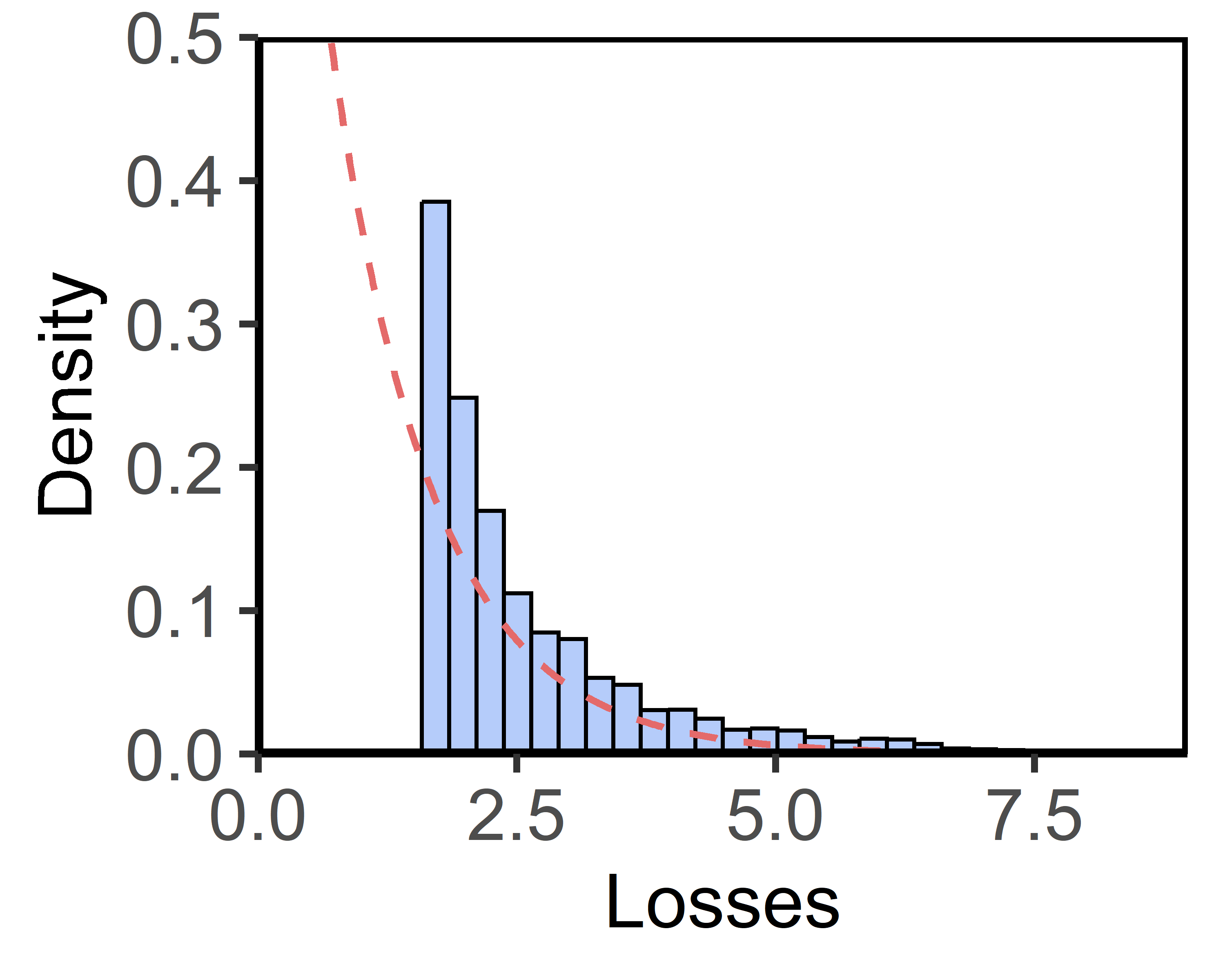}
}
\subfigure[]{
    \includegraphics[width=0.3\textwidth]{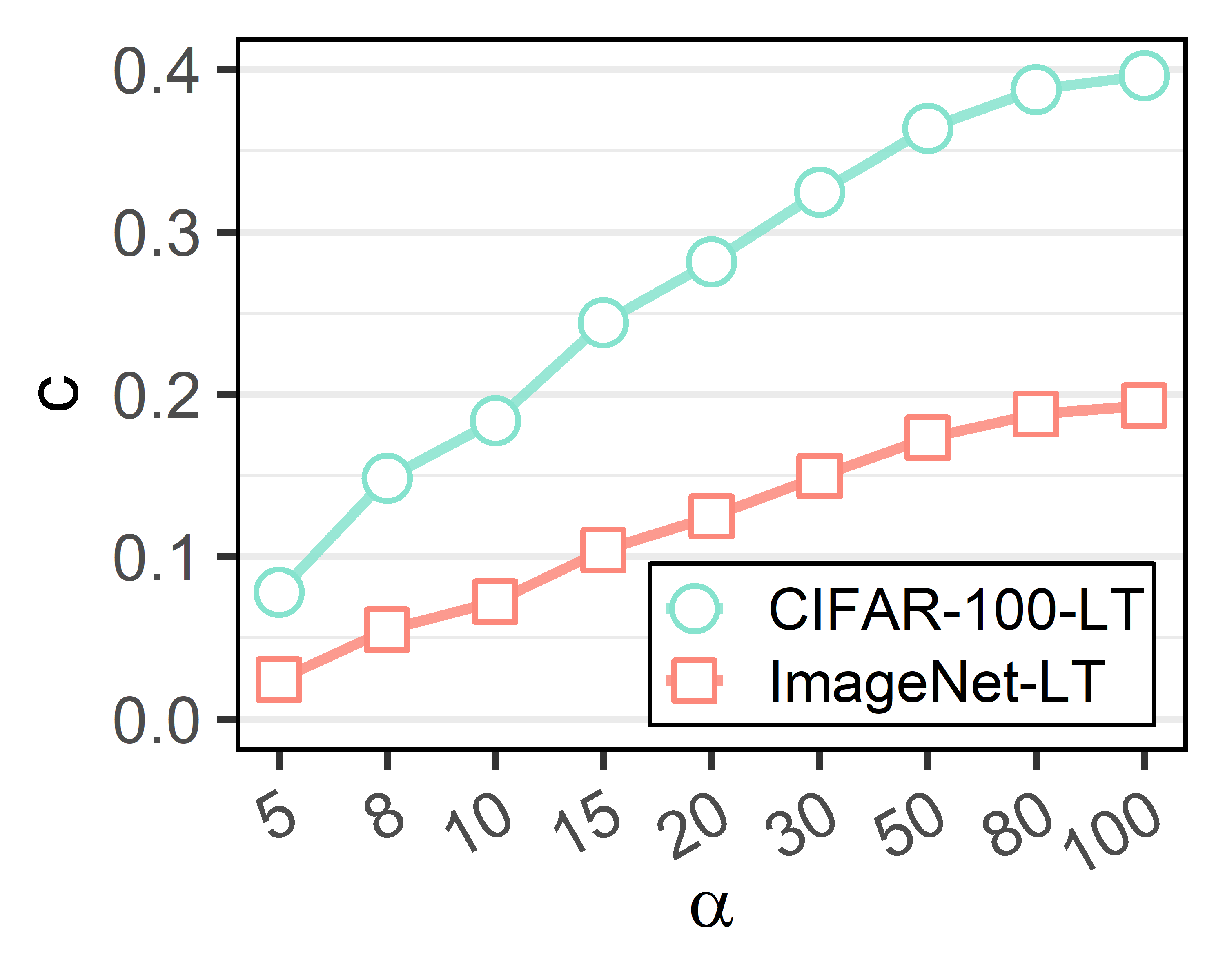}
}
\caption{
\textbf{Empirical validation of per-distribution loss exponential-tail behavior.}
\textbf{(a)} Loss histogram and exponential PDF fit on CIFAR-100-LT.  
\textbf{(b)} Loss histogram and exponential PDF fit on ImageNet-LT.  
\textbf{(c)} Variation of the coefficient $c$ with respect to $\alpha$ for CIFAR-100-LT and ImageNet-LT.
}

\label{fig: distribution_light_tail}

\end{figure*}

\begin{textbo}

The fitted model achieves a \textbf{$R^2$ value of $0.70$ on CIFAR-100-LT and $0.95$ on ImageNet-LT}. The visualization of the empirical histogram, along with the fitted exponential PDF curve, is shown in Fig. \ref{fig: distribution_light_tail}-(a) and Fig. \ref{fig: distribution_light_tail}-(b). Notably, compared with the per-sample case, the exponential fit near the lower-loss region is less precise. This behavior is expected: per-distribution losses are averaged across many samples and therefore rarely approach zero—in fact, the minimum value of $\tilde{\ell}(P_i)$ is strictly positive in our experiments because some samples are inevitably misclassified under each distribution. However, this deviation at small loss values is \textbf{not problematic} for our theory. The exponential-tail assumption is only required to hold for \textbf{sufficiently large $\tau$}. In the high-loss regime, the fitted exponential curve aligns well with the empirical decay as shown in the results. Thus, the behavior of the upper tail remains consistent with our assumption, validating the applicability of the assumption in the per-distribution setting.

The estimated decay parameter is $\lambda_2 = 1.599$ for CIFAR-100-LT and $\lambda_2 = 1.021$ for ImageNet-LT. To further support the same-order scaling relationship, we repeat the above procedure and compute the implied coefficient 
$c$ for varying values of $\alpha$ in the per-distribution setting, as shown in Fig.~\ref{fig: distribution_light_tail}-(c). In particular, when $\alpha > 8$, the estimated coefficient $c$ consistently remains above approximately $0.05$ across both CIFAR-100-LT and ImageNet-LT. This empirical finding again aligns with the theoretical constraint that $\lambda_2$ and $\log(M) \cdot M^{1/\alpha}$ are of the same order.

\end{textbo}

\begin{textbo}

\section{Validity of the Lipschitz Continuous Assumption}

Thank you for your question! We would like to clarify that the loss function considered in our analysis is actually the composition of the cross-entropy loss and the softmax function, i.e., $\ell = \ell_{\mathrm{CE}} \circ \mathsf{softmax}$, where the remaining part corresponds to the logits $f(x)$. Under the composition $\ell$ is Lip. Continuous. The earlier version of the manuscript misstated this point, and we will revise Assumption 1 accordingly to make this explicit. The detailed adjustment is described below.

We show below that the  $\ell_{\text{CE}} \circ \mathsf{softmax}$ is \textbf{Lipschitz continuous} with respect to the logits $z$ under the infinity norm, which addresses the concern about the \textbf{validity of the assumption} raised in the question.

Let $z \in \mathbb{R}^n$ be the logits, $y$ the one-hot label, and $p = \mathrm{softmax}(z)$. The standard cross-entropy loss is
$$
L(z) = -\sum_{i=1}^n y_i \log p_i,
\qquad p_i = \frac{e^{z_i}}{\sum_j e^{z_j}}.
$$
Its gradient with respect to $z$ is:
$$
\nabla_z L = p - y.
$$
The $\ell_1$ norm of the gradient, which corresponds to the Lipschitz constant under $\|\cdot\|_\infty$, is
$$
\|\nabla_z L\|_1 = \sum_j |p_j - y_j| = 2(1 - p_k),
$$
where $k$ is the index of the true class ($y_k = 1$). Since $p_k \in (0,1]$, the supremum of $\|\nabla_z L\|_1$ is $2$ as $p_k \to 0$. Therefore,
$$
|L(z_1) - L(z_2)| \le 2 \|z_1 - z_2\|_\infty,
$$
which shows that the cross-entropy loss is \textbf{$2$-Lipschitz} with respect to the logits.

In summary, the composed loss is Lipschitz continuous under the infinity norm as a function of the logits, and its Lipschitz constant is bounded by $2$. Hence, Assumption 1.1 is both mild and valid.

\end{textbo}

\begin{textbo}

\section{More Discussion About Definition \ref{def: induced subclass}-\ref{def: delta compact parameterization}}

\subsection{Def.\ref{def: induced subclass}: Avoiding Uniform Bounds and Focusing on "Good Regions"}

In generalization analysis, traditional uniform bounds (e.g., those based on VC-dimension or Rademacher complexity) can be loose because they consider the entire hypothesis space, including poorly performing models. However, in practice, we are only interested in well-trained models that exhibit good generalization behavior. Definition 1 is designed to narrow down the hypothesis space to these "good regions" where the model performs well consistently.

There are two subsets to capture "good regions":
  
$ A^I $ represents the set where, for a given sample, the loss is small with high probability across different test distributions. 

$ A^M $ represents the set where, for a fixed test distribution, the overall expected risk is small with high probability. This captures the model's average performance on the target distribution.

The intuition here is that a well-trained model should have a high probability of lying within these two sets. By intersecting the original hypothesis space with the high probability sets, we restrict our analysis to these models correspondingly. This approach avoids the pessimism of uniform bounds and allows for tighter generalization guarantees in practical scenarios where models are already well-optimized.

\subsection{Def.\ref{def: voronoi diagram} \& \ref{def: delta compact parameterization}: Linear Approximation and Covering Number Bound}

Def.\ref{def: voronoi diagram} introduces the Voronoi diagram, a well-known mathematical tool \cite{de2008computational,okabe2009spatial, DBLP:journals/csur/Aurenhammer91,edelsbrunner2010computational, erwig2000graph}. But Def.\ref{def: delta compact parameterization} is new and specific to our problem. 

Definitions \ref{def: voronoi diagram} and \ref{def: delta compact parameterization} together enable a linear approximation of the model's logits around local minima, which is crucial for deriving covering-number-based bounds. Specifically, Eq. \eqref{eq:lin} approximate the logits by expanding around the nearest local minimum. Here we repeat it for clarity. 
  \begin{align}\label{eq:lin}
   f^{(i)}_{\thi{i}}(\x) = & f^{(i)}_{\this{i}}(\x) + \langle \nabla_{\this{i}}f^{(i)}_{\this{i}}(\x),  \Delta \thi{i} - \Delta\this{i} \rangle +  r(i,\x) 
  \end{align}
Here 
    \begin{equation}
        r(i,\x) = \sup_{\theta \in \mathsf{LinSeg}}   (\Delta\thi{i} - \Delta\this{i})^\top \bm{H}_{\theta}(\x) (\Delta\thi{i} - \Delta\this{i}),    
    \end{equation}
    
where $\mathsf{LinSeg}$ denotes the set of all linear segments between $\Delta\thi{i}$ and $\Delta\this{i}$, while $\bm{H}_{\theta}(\x)$ is the corresponding Hessian matrix. 

The theoretical arguments are based on covering sets for the parameter space, each of which is a small open ball. To integrate the simplicity of linear approximation, we need to ensure that the same local minimum $\Delta\Theta^\star$ is chosen within a specific ball of the covering set. 

Specifically, we can rule out the effect of residue $r(i,\x)$ through a Lip. based inequality on the loss function (see Proof of Thm.\ref{thm:gen1}). So, for two parameters $\Delta \Theta$ and  $\Delta \tilde{\Theta}$, if their local solutions are consistent, we have the bound:
\begin{align*}
    &|f^{(i)}_{\this{i}}(\x) + \langle \nabla_{\this{i}}f^{(i)}_{\this{i}}(\x),  \Delta \thi{i} - \Delta\this{i} \rangle - f^{(i)}_{\thist{i}}(\x) + \langle \nabla_{\thist{i}}f^{(i)}_{\thist{i}}(\x),  \Delta \thit{i} - \Delta\thist{i} \rangle| \\
    &\le  L \cdot \|\Delta \thi{i}  - \Delta \thit{i}\|_F
\end{align*}
Then the covering number of the linearized logits could be upper-bounded by that of the incremental parameter space. 

Def.\ref{def: voronoi diagram} defines a Voronoi diagram, which partitions the parameter space into three cells $\mathcal{V}_1, \mathcal{V}_2, \mathcal{V}_3$. The nearest local minimum is the same for all parameters within a cell. In other words, within a cell, the model's behavior can be characterized by a fixed local minimum, allowing us to use linear approximations uniformly. Based 
on this definition, to ensure $\this{i}$ is the nearest local solution within the small covering ball, we only need to ensure that the ball does not cross the border of the cells. This is precisely the meaning of Def.\ref{def: delta compact parameterization}.

According to Fig.\ref{fig:voi}, Def.\ref{def: delta compact parameterization} states that all the elements should be at least $\delta$ away from the boundary cross cells.  By choosing $\delta$ as the diameter of the balls in the covering set. We can then ensure that each ball in the covering set does not cross the boundary.

In practice, the $\delta$ parametrization assumes that, before parameter-efficient fine-tuning (PEFT), the pre-trained model ensures the initial parameters are already close to a specific local minimum. For example, in  Fig.\ref{fig:voi}, the parameters are assumed to be located in $\mathcal{H}_1 \cup  \mathcal{H}_2 \cup \mathcal{H}_3$. This is a natural and mild assumption because PEFT starts from a pre-trained model and performs only small updates (e.g., via adapters or low-rank adjustments). The margin $\delta$ is chosen as $O(1/N)$, which is very small. Hence, the assumption is very easy to achieve. Moreover, we don't need all of the local solutions; we only need to choose a subset. So, by properly choosing these local solutions, we can easily meet the assumption.

\end{textbo}

\section{Additional Experiment Settings}\label{sec:settings}

\subsection{Construction of Training meta-distribution $\mathcal{E}$}

Consider the training class distribution represented by $P_1, P_2, ..., P_C$, where $C$ denotes the number of classes. Without loss of generality, \textbf{we assume that the classes are sorted in a sense}: $P_1 \ge P_2 \ge ... \ge P_C$. Otherwise, we can change the numbering and obtain the same result.  

The training meta-distribution $\mathcal{E}$ is a Dirichlet mixture distribution as Sec.\ref{sec:meta}. It consists of three Dirichlet distribution components: the forward component $\alpha^{(f)}$, the uniform component $\alpha^{(u)}$, and the backward component $\alpha^{(b)}$. 

The forward component parameter $\alpha^{(f)}$ is set element-wisely: 
\begin{align*}
    \alpha^{(f)}_i = {S \cdot P_i}, \ i = 1, 2, \ldots, C
\end{align*}
where $S$ is a predefined normalization factor to control the variance of the component. {\textbf{Here we set it as $S = 10000$ for all datasets.}}  Since the mean of the Dirichlet distribution is exactly $\alpha^{(f)}$, this component represents local variations concentrated around the long-tail distribution aligned with the training data.

The backward component $\alpha^{(b)}$ is set element-wisely as:
\begin{align*}
    \alpha^{(b)}_i = {S \cdot P_{(C - i)}}
\end{align*}
where $S$ is a predefined normalization factor to control the variance of the component. {\textbf{Here we set it as $S = 10000$ for all datasets.}} By connecting $\alpha^{(b)}_i$ with $P_{(C - i)}$, the long-tail distribution is reversed, where the head classes become tail classes, and \textit{vice versa}. Since the mean of the Dirichlet distribution is exactly $\alpha^{(b)}$, this component represents local variations concentrated around the \textbf{inverse} long-tail distribution aligned with the training data.

The uniform component $\alpha^{(u)}$ is set element-wisely as:
\begin{align*}
    \alpha^{(u)}_i = \frac{S}{C}, \ i = 1, 2, \ldots, C 
\end{align*}
where $S$ is a predefined normalization factor to control the variance of the component. {\textbf{Here we set it as $S = 10000$ for all datasets.}} By connecting $\alpha^{(b)}_i$ with $1/C$, it recovers a uniform distribution.  Since the mean of the Dirichlet distribution is exactly $\alpha^{(u)}$, this component represents local variations concentrated around the uniform distribution.

The corresponding p.d.f are expressed as: $\dir^{(f)}, \dir^{(u)}, \dir^{(b)}$. Moreover, the three components are mixed with a uniform distribution. Above all, the p.d.f for the mixture distribution becomes:

\begin{align*}
    \mathcal{E} = \frac{1}{3} \cdot \dir^{(f)} + \frac{1}{3} \cdot \dir^{(u)} + \frac{1}{3} \cdot \dir^{(b)}
\end{align*}

\subsection{Sampling Procedure of $\mathbb{P}_{te}$ for the Monte Carlo Approximation}

Initially, we employ a Monte Carlo method to sample a set of $\mathbb{P}_{te}$ from the training meta-distribution $\mathcal{E}$ and obtain the generated data
 $\mathcal{P} = \{\P_j, \xi_j\}_{j=1}^M$ for Monte Carlo approximation.

We implement the sampling process by repeating the following two-step procedure for $M$ times. 
\begin{enumerate}
    \item Randomly sampling a Dirichlet distribution component $\alpha$ from $\{\alpha^{(f)}, \alpha^{(u)}, \alpha^{(b)}\}$ with equal probability.
    \item Sampling a distribution $\mathbb{P}_{te}$ from the Dirichlet distribution component $\alpha$.
\end{enumerate}

\subsection{Mini-Batch Construction}

Denote the number of mini-batches as $B$. We sample $60 \cdot B$ label distributions to construct the dataset $\mathcal{P}$. We randomly sample 60 label distributions for each mini-batch without replacement in each epoch to get an unbiased estimation.

\subsection{Construction of Testing Datasets}

We employ two settings, SADE's Setting and our setting, to construct our test data. The details are discussed as follows:

\subsubsection{SADE's setting}
The test data in this setting directly follows SADE ~\cite{DBLP:conf/nips/ZhangHHF22}. The only difference here is that we include more imbalance ratios, denoted as $\rho$, in our experiments. For CIFAR 100-LT and CIFAR 10-LT, $\rho \in \{2, 5, 10, 25, 50, 100\}$. For ImageNet-LT, $\rho \in \{2, 5, 10, 25, 50\}$. For iNaturalist, $\rho \in \{2, 3\}$.

\subsubsection{Ours setting}

According to our meta-distribution, we employ three kinds of Dirichlet components to generate test distributions: the forward Dirichlet distribution $\alpha^{(f_{test})}$, the uniform Dirichlet distribution $\alpha^{(u_{test})}$ and the backward Dirichlet distribution $\alpha^{(b_{test})}$.

From each component of the meta-distribution, we sample three specific distributions as the test class distribution, resulting in a total of\textbf{ 9 test class distributions}.

The Dirichlet distributions are chosen based on a predefined imbalance ratio $\rho$. The greater the $\rho$, the more challenging the test distribution. \textbf{For CIFAR 100-LT and CIFAR 10-LT, we set $\rho$ to be 100. For ImageNet-LT, we set $\rho = 50$. For iNaturalist, we set $\rho = 3$.} 

For the forward Dirichlet distribution $\alpha^{(f_{test})}$, we set:
\begin{align*}
    \alpha^{(f_{test})}_i = \rho^{-(i - 1) / (C - 1)}, \ i = 1, 2, \ldots, C
\end{align*}
For the backward Dirichlet distribution $\alpha^{(b_{test})}$, we set:
\begin{align*}
    \alpha^{(b_{test})}_i = \rho^{-(C - i) / (C - 1)}, \ i = 1, 2, \ldots, C
\end{align*}
For the uniform Dirichlet distribution $\alpha^{(u_{test})}$, we set:
\begin{align*}
    \alpha^{(u_{test})}_i = 1 / C, \ i = 1, 2, \ldots, C
\end{align*}
Subsequent steps remain consistent across the three components. For simplicity, we denote each Dirichlet distribution as $\alpha$.

To enhance the simulation of the randomness in the test distribution, we introduce a perturbation to $\alpha$, allowing up to 5\% variations. This involves adjusting the $i$-th element $\alpha_i$:
\begin{align*}
    \alpha_i = \alpha_i + \alpha_i * \epsilon, \ \epsilon ~\sim~ U(-0.05, 0.05)    
\end{align*}
Following this, we normalize $\alpha$ to ensure the sum of its components equals $S$. The role of $S$ is identical to its function in constructing the Training meta-distribution. We  set $S=100$ for CIFAR 100-LT, $S=1000$ for CIFAR 10-LT,  $S=10000$ for ImageNet-LT and $S=100000$ for iNaturalist.

\section{Additional Experiments}\label{app:exp}

\subsection{Overall Performance for DirMixE on Ours Setting}

The overall performance on the CIFAR-10-LT dataset for DirMixE on Ours Setting
is shown in Tab.\ref{tab:Cifar10_Our}.

\begin{table*}[htbp]
  \centering
\caption{Performance Comparison on CIFAR-10-LT When Training ResNet Models (\textbf{Ours Setting)}}
  \renewcommand\arraystretch{1.0}
    \begin{tabular}{lcccccccccc}
    \toprule
    \multirow{1.5}[4]{*}{\textbf{Method}} & \multicolumn{3}{c}{\textbf{Forward-LT}} & \multicolumn{3}{c}{\textbf{Uniform}} & \multicolumn{3}{c}{\textbf{Backward-LT}} & \multirow{1.5}[4]{*}{\textbf{Mean}} \\
          \cmidrule(lr){2-4}\cmidrule(lr){5-7}\cmidrule(lr){8-10}     & \textbf{1} & \textbf{2} & \textbf{3} & \textbf{1} & \textbf{2} & \textbf{3} & \textbf{1} & \textbf{2} & \textbf{3} &  \\
\toprule
    LDAM  & \cellcolor[rgb]{ .812,  .886,  .953} 89.66  & \cellcolor[rgb]{ .784,  .871,  .945} \underline{90.50} & \cellcolor[rgb]{ .827,  .894,  .957} 90.30  & \cellcolor[rgb]{ 1,  1,  1} 74.27  & \cellcolor[rgb]{ 1,  1,  1} 74.39  & \cellcolor[rgb]{ 1,  1,  1} 74.77  & \cellcolor[rgb]{ 1,  1,  1} 60.65  & \cellcolor[rgb]{ 1,  1,  1} 59.89  & \cellcolor[rgb]{ 1,  1,  1} 60.42  & \cellcolor[rgb]{ 1,  1,  1} 74.98$_{\color{blue}(\pm 12.19)}$\\
    LA    & \cellcolor[rgb]{ .808,  .886,  .953} 89.74  & \cellcolor[rgb]{ .839,  .902,  .961} 88.98  & \cellcolor[rgb]{ .82,  .89,  .957} 90.51  & \cellcolor[rgb]{ .867,  .918,  .965} 79.02  & \cellcolor[rgb]{ .863,  .918,  .965} 79.07  & \cellcolor[rgb]{ .871,  .922,  .969} 79.27  & \cellcolor[rgb]{ .89,  .933,  .973} 72.90  & \cellcolor[rgb]{ .898,  .937,  .976} 71.67  & \cellcolor[rgb]{ .894,  .937,  .973} 72.51  & \cellcolor[rgb]{ .89,  .933,  .973} 80.41$_{\color{blue}(\pm \phantom{0}7.17)}$\\
    VS    & \cellcolor[rgb]{ 1,  1,  1} 85.12  & \cellcolor[rgb]{ 1,  1,  1} 84.71  & \cellcolor[rgb]{ 1,  1,  1} 85.24  & \cellcolor[rgb]{ .816,  .89,  .953} 80.69  & \cellcolor[rgb]{ .82,  .89,  .953} 80.50  & \cellcolor[rgb]{ .827,  .894,  .957} 80.85  & \cellcolor[rgb]{ .808,  .886,  .953} 82.10  & \cellcolor[rgb]{ .812,  .886,  .953} 81.21  & \cellcolor[rgb]{ .82,  .89,  .953} 80.95  & \cellcolor[rgb]{ .851,  .91,  .965} 82.37$_{\color{blue}(\pm \phantom{0}1.92)}$\\
    LADE  & \cellcolor[rgb]{ .91,  .945,  .976} 87.31  & \cellcolor[rgb]{ .91,  .945,  .98} 87.11  & \cellcolor[rgb]{ .914,  .949,  .98} 87.79  & \cellcolor[rgb]{ .847,  .906,  .961} 79.69  & \cellcolor[rgb]{ .843,  .906,  .961} 79.72  & \cellcolor[rgb]{ .843,  .906,  .961} 80.20  & \cellcolor[rgb]{ .855,  .914,  .965} 76.79  & \cellcolor[rgb]{ .859,  .918,  .965} 75.77  & \cellcolor[rgb]{ .859,  .914,  .965} 76.47  & \cellcolor[rgb]{ .875,  .925,  .969} 81.21$_{\color{blue}(\pm \phantom{0}4.62)}$\\
    DDC   & \cellcolor[rgb]{ .922,  .953,  .98} 87.07  & \cellcolor[rgb]{ .914,  .949,  .98} 87.07  & \cellcolor[rgb]{ .933,  .961,  .984} 87.24  & \cellcolor[rgb]{ .788,  .875,  .949} 81.62  & \cellcolor[rgb]{ .788,  .875,  .949} 81.44  & \cellcolor[rgb]{ .792,  .875,  .949} 82.03  & \cellcolor[rgb]{ .827,  .894,  .957} 80.15  & \cellcolor[rgb]{ .831,  .898,  .957} 79.04  & \cellcolor[rgb]{ .831,  .898,  .957} 79.47  & \cellcolor[rgb]{ .843,  .906,  .961} 82.79$_{\color{blue}(\pm \phantom{0}3.20)}$\\
    RIDE  & \cellcolor[rgb]{ .945,  .969,  .988} 86.47  & \cellcolor[rgb]{ .969,  .98,  .992} 85.63  & \cellcolor[rgb]{ .929,  .957,  .984} 87.33  & \cellcolor[rgb]{ .78,  .867,  .945} 81.92  & \cellcolor[rgb]{ .776,  .867,  .945} 81.89  & \cellcolor[rgb]{ .792,  .875,  .949} 81.98  & \cellcolor[rgb]{ .816,  .89,  .953} 81.34  & \cellcolor[rgb]{ .812,  .886,  .953} 81.25  & \cellcolor[rgb]{ .816,  .89,  .953} 81.19  & \cellcolor[rgb]{ .831,  .898,  .957} 83.22$_{\color{blue}(\pm \phantom{0}2.35)}$\\
    SADE  & \cellcolor[rgb]{ .788,  .871,  .945} 90.26  & \cellcolor[rgb]{ .804,  .882,  .949} 89.94  & \cellcolor[rgb]{ .8,  .878,  .949} 91.05  & \cellcolor[rgb]{ .745,  .847,  .937} \underline{83.14} & \cellcolor[rgb]{ .753,  .851,  .937} \underline{82.71} & \cellcolor[rgb]{ .753,  .851,  .937} \underline{83.38} & \cellcolor[rgb]{ .749,  .847,  .937} \underline{88.89} & \cellcolor[rgb]{ .749,  .847,  .937} \underline{88.29} & \cellcolor[rgb]{ .741,  .843,  .933} \textbf{89.48} & \cellcolor[rgb]{ .745,  .847,  .937} \underline{87.46}$_{\color{blue}(\pm \phantom{0}3.19)}$\\
    BalPoE & \cellcolor[rgb]{ .741,  .843,  .933} \textbf{91.30} & \cellcolor[rgb]{ .741,  .843,  .933} \textbf{91.54} & \cellcolor[rgb]{ .741,  .843,  .933} \textbf{92.72} & \cellcolor[rgb]{ .792,  .875,  .949} 81.58  & \cellcolor[rgb]{ .78,  .867,  .945} 81.78  & \cellcolor[rgb]{ .796,  .878,  .949} 81.89  & \cellcolor[rgb]{ .835,  .902,  .961} 78.97  & \cellcolor[rgb]{ .847,  .906,  .961} 77.28  & \cellcolor[rgb]{ .847,  .906,  .961} 77.87  & \cellcolor[rgb]{ .82,  .89,  .953} 83.88$_{\color{blue}(\pm \phantom{0}5.86)}$\\
    \toprule
    \textbf{DirMixE} & \cellcolor[rgb]{ .776,  .867,  .945} \underline{90.46} & \cellcolor[rgb]{ .804,  .882,  .953} 89.90  & \cellcolor[rgb]{ .792,  .875,  .949} \underline{91.30} & \cellcolor[rgb]{ .741,  .843,  .933} \textbf{83.24} & \cellcolor[rgb]{ .741,  .843,  .933} \textbf{82.98} & \cellcolor[rgb]{ .741,  .843,  .933} \textbf{83.71} & \cellcolor[rgb]{ .741,  .843,  .933} \textbf{89.39} & \cellcolor[rgb]{ .741,  .843,  .933} \textbf{88.78} & \cellcolor[rgb]{ .753,  .851,  .937} \underline{88.40} & \cellcolor[rgb]{ .741,  .843,  .933} \textbf{87.57}$_{\color{blue}(\pm \phantom{0}3.12)}$\\
    \bottomrule
    \end{tabular}%
  \label{tab:Cifar10_Our}%
\end{table*}%

\subsection{Overall Performance for DirMixE on SADE Setting}

The overall performance on the CIFAR-10-LT, CIFAR-100-LT, ImageNet-LT and iNaturalist datasets for DirMixE on Ours Setting
is shown in Tab.\ref{tab:Cifar10_Sade}-Tab.\ref{tab:inat-sade}.

\begin{table*}[h]       
    \centering
    \caption{Performance Comparison on CIFAR-10 When Training ResNet Models (\textbf{SADE's Setting)}}
    \renewcommand\arraystretch{1.0}
    \resizebox{\linewidth}{!}{
      \begin{tabular}{lcccccccccccccc}
      \toprule
      \multirow{1.5}[4]{*}{\textbf{Method}} & \multicolumn{6}{c}{\textbf{Forward-LT}}       & \textbf{Uni.} & \multicolumn{6}{c}{\textbf{Backward-LT}}      & \multirow{1.5}[4]{*}{\textbf{Mean}} \\
            \cmidrule(lr){2-7} \cmidrule(lr){8-8} \cmidrule(lr){9-14}      & \textbf{100} & \textbf{50} & \textbf{25} & \textbf{10} & \textbf{5} & \textbf{2} & \textbf{1} & \textbf{2} & \textbf{5} & \textbf{10} & \textbf{25} & \textbf{50} & \textbf{100} &  \\
    \toprule
      LDAM  & \cellcolor[rgb]{ .796,  .875,  .949} 90.56  & \cellcolor[rgb]{ .82,  .89,  .957} 89.12  & \cellcolor[rgb]{ .855,  .914,  .965} 87.18  & \cellcolor[rgb]{ .918,  .949,  .98} 84.28  & \cellcolor[rgb]{ 1,  1,  1} 81.58  & \cellcolor[rgb]{ 1,  1,  1} 77.60  & \cellcolor[rgb]{ 1,  1,  1} 74.49  & \cellcolor[rgb]{ 1,  1,  1} 71.41  & \cellcolor[rgb]{ 1,  1,  1} 68.06  & \cellcolor[rgb]{ 1,  1,  1} 65.35  & \cellcolor[rgb]{ 1,  1,  1} 62.40  & \cellcolor[rgb]{ 1,  1,  1} 60.86  & \cellcolor[rgb]{ 1,  1,  1} 60.09  & \cellcolor[rgb]{ 1,  1,  1} 74.84$_{\color{blue}(\pm 10.60)}$\\
      LA    & \cellcolor[rgb]{ .82,  .89,  .957} 89.79  & \cellcolor[rgb]{ .839,  .906,  .961} 88.55  & \cellcolor[rgb]{ .855,  .914,  .965} 87.18  & \cellcolor[rgb]{ .859,  .914,  .965} 85.50  & \cellcolor[rgb]{ .878,  .925,  .969} 83.70  & \cellcolor[rgb]{ .851,  .91,  .965} 81.30  & \cellcolor[rgb]{ .863,  .918,  .965} 79.17  & \cellcolor[rgb]{ .878,  .925,  .969} 77.59  & \cellcolor[rgb]{ .878,  .929,  .969} 76.13  & \cellcolor[rgb]{ .882,  .929,  .973} 74.93  & \cellcolor[rgb]{ .882,  .929,  .973} 73.52  & \cellcolor[rgb]{ .882,  .929,  .973} 73.38  & \cellcolor[rgb]{ .89,  .933,  .973} 72.76  & \cellcolor[rgb]{ .882,  .929,  .973} 80.27$_{\color{blue}(\pm \phantom{0}5.89)}$\\
      VS    & \cellcolor[rgb]{ 1,  1,  1} 84.26  & \cellcolor[rgb]{ 1,  1,  1} 83.97  & \cellcolor[rgb]{ 1,  1,  1} 83.49  & \cellcolor[rgb]{ 1,  1,  1} 82.52  & \cellcolor[rgb]{ .996,  1,  1} 81.66  & \cellcolor[rgb]{ .871,  .922,  .969} 80.78  & \cellcolor[rgb]{ .824,  .894,  .957} 80.59  & \cellcolor[rgb]{ .824,  .894,  .957} 80.11  & \cellcolor[rgb]{ .812,  .886,  .953} 80.73  & \cellcolor[rgb]{ .812,  .886,  .953} 80.66  & \cellcolor[rgb]{ .812,  .886,  .953} 80.30  & \cellcolor[rgb]{ .812,  .886,  .953} 80.75  & \cellcolor[rgb]{ .812,  .886,  .953} 81.23  & \cellcolor[rgb]{ .855,  .914,  .965} 81.62$_{\color{blue}(\pm \phantom{0}1.39)}$\\
      LADE  & \cellcolor[rgb]{ .898,  .937,  .976} 87.45  & \cellcolor[rgb]{ .906,  .941,  .976} 86.73  & \cellcolor[rgb]{ .925,  .957,  .98} 85.41  & \cellcolor[rgb]{ .937,  .961,  .984} 83.91  & \cellcolor[rgb]{ .929,  .957,  .984} 82.82  & \cellcolor[rgb]{ .859,  .914,  .965} 81.14  & \cellcolor[rgb]{ .847,  .906,  .961} 79.77  & \cellcolor[rgb]{ .847,  .91,  .961} 79.02  & \cellcolor[rgb]{ .847,  .91,  .961} 78.21  & \cellcolor[rgb]{ .851,  .91,  .961} 77.64  & \cellcolor[rgb]{ .847,  .91,  .961} 76.74  & \cellcolor[rgb]{ .851,  .91,  .965} 76.74  & \cellcolor[rgb]{ .855,  .914,  .965} 76.55  & \cellcolor[rgb]{ .871,  .922,  .969} 80.93$_{\color{blue}(\pm \phantom{0}3.78)}$\\
      DDC   & \cellcolor[rgb]{ .922,  .953,  .98} 86.72  & \cellcolor[rgb]{ .922,  .953,  .98} 86.23  & \cellcolor[rgb]{ .929,  .957,  .984} 85.35  & \cellcolor[rgb]{ .918,  .949,  .98} 84.30  & \cellcolor[rgb]{ .882,  .929,  .973} 83.63  & \cellcolor[rgb]{ .8,  .882,  .949} 82.49  & \cellcolor[rgb]{ .792,  .875,  .949} 81.64  & \cellcolor[rgb]{ .8,  .878,  .949} 81.31  & \cellcolor[rgb]{ .812,  .886,  .953} 80.61  & \cellcolor[rgb]{ .82,  .89,  .953} 80.12  & \cellcolor[rgb]{ .82,  .894,  .957} 79.31  & \cellcolor[rgb]{ .824,  .894,  .957} 79.32  & \cellcolor[rgb]{ .831,  .898,  .957} 79.38  & \cellcolor[rgb]{ .839,  .902,  .961} 82.34$_{\color{blue}(\pm \phantom{0}2.56)}$\\
      RIDE  & \cellcolor[rgb]{ .91,  .949,  .98} 87.01  & \cellcolor[rgb]{ .922,  .953,  .98} 86.26  & \cellcolor[rgb]{ .929,  .957,  .984} 85.32  & \cellcolor[rgb]{ .902,  .941,  .976} 84.60  & \cellcolor[rgb]{ .863,  .918,  .965} 83.94  & \cellcolor[rgb]{ .8,  .878,  .949} 82.57  & \cellcolor[rgb]{ .784,  .871,  .945} 81.80  & \cellcolor[rgb]{ .792,  .875,  .949} 81.74  & \cellcolor[rgb]{ .796,  .878,  .949} 81.64  & \cellcolor[rgb]{ .8,  .882,  .949} 81.44  & \cellcolor[rgb]{ .804,  .882,  .953} \underline{80.92} & \cellcolor[rgb]{ .812,  .886,  .953} 80.89  & \cellcolor[rgb]{ .812,  .886,  .953} 81.28  & \cellcolor[rgb]{ .824,  .894,  .957} 83.03$_{\color{blue}(\pm \phantom{0}2.05)}$\\
      SADE  & \cellcolor[rgb]{ .808,  .886,  .953} 90.15  & \cellcolor[rgb]{ .82,  .89,  .953} 89.19  & \cellcolor[rgb]{ .824,  .894,  .957} 87.95  & \cellcolor[rgb]{ .824,  .894,  .957} 86.24  & \cellcolor[rgb]{ .812,  .886,  .953} 84.83  & \cellcolor[rgb]{ .761,  .855,  .941} 83.51  & \cellcolor[rgb]{ .749,  .847,  .937} \underline{83.10} & \cellcolor[rgb]{ .757,  .851,  .937} \underline{83.55} & \cellcolor[rgb]{ .757,  .851,  .937} \underline{84.37} & \cellcolor[rgb]{ .757,  .855,  .937} \underline{85.01} & \cellcolor[rgb]{ .741,  .843,  .933} \textbf{86.62} & \cellcolor[rgb]{ .741,  .843,  .933} \textbf{87.91} & \cellcolor[rgb]{ .741,  .843,  .933} \textbf{89.10} & \cellcolor[rgb]{ .753,  .851,  .937} \underline{86.27}$_{\color{blue}(\pm \phantom{0}2.31)}$\\
      BalPoE & \cellcolor[rgb]{ .741,  .843,  .933} \textbf{92.13} & \cellcolor[rgb]{ .741,  .843,  .933} \textbf{91.31} & \cellcolor[rgb]{ .741,  .843,  .933} \textbf{89.97} & \cellcolor[rgb]{ .741,  .843,  .933} \textbf{87.93} & \cellcolor[rgb]{ .741,  .843,  .933} \textbf{86.01} & \cellcolor[rgb]{ .741,  .843,  .933} \textbf{83.92} & \cellcolor[rgb]{ .788,  .875,  .949} 81.70  & \cellcolor[rgb]{ .816,  .89,  .953} 80.57  & \cellcolor[rgb]{ .824,  .894,  .957} 79.94  & \cellcolor[rgb]{ .82,  .89,  .953} 80.14  & \cellcolor[rgb]{ .835,  .902,  .957} 78.10  & \cellcolor[rgb]{ .839,  .902,  .961} 77.89  & \cellcolor[rgb]{ .843,  .906,  .961} 77.80  & \cellcolor[rgb]{ .812,  .886,  .953} 83.65$_{\color{blue}(\pm \phantom{0}5.05)}$\\
      \toprule
      \textbf{DirMixE} & \cellcolor[rgb]{ .784,  .871,  .945} \underline{90.92} & \cellcolor[rgb]{ .784,  .871,  .945} \underline{90.16} & \cellcolor[rgb]{ .78,  .867,  .945} \underline{89.04} & \cellcolor[rgb]{ .784,  .871,  .945} \underline{87.10} & \cellcolor[rgb]{ .753,  .851,  .937} \underline{85.83} & \cellcolor[rgb]{ .753,  .851,  .937} \underline{83.66} & \cellcolor[rgb]{ .741,  .843,  .933} \textbf{83.26} & \cellcolor[rgb]{ .741,  .843,  .933} \textbf{84.16} & \cellcolor[rgb]{ .741,  .843,  .933} \textbf{85.16} & \cellcolor[rgb]{ .741,  .843,  .933} \textbf{86.17} & \cellcolor[rgb]{ .741,  .843,  .933} \textbf{86.62} & \cellcolor[rgb]{ .745,  .847,  .937} \underline{87.55} & \cellcolor[rgb]{ .749,  .851,  .937} \underline{88.30} & \cellcolor[rgb]{ .741,  .843,  .933} \textbf{86.76}$_{\color{blue}(\pm \phantom{0}2.31)}$\\
      \bottomrule
      \end{tabular}%
      }
    \label{tab:Cifar10_Sade}%
    \end{table*}%

    \begin{table*}[h!]
    \centering
    \caption{Performance Comparison on CIFAR-100-LT When Training ResNet Models (\textbf{SADE's Setting)}}
    \renewcommand\arraystretch{1.0}
    \resizebox{\linewidth}{!}{
      \begin{tabular}{lcccccccccccccc}
      \toprule
      \multirow{1.5}[4]{*}{\textbf{Method}} & \multicolumn{6}{c}{\textbf{Forward-LT}}       & \textbf{Uni.} & \multicolumn{6}{c}{\textbf{Backward-LT}}      & \multirow{1.5}[4]{*}{\textbf{Mean}} \\
            \cmidrule(lr){2-7} \cmidrule(lr){8-8} \cmidrule(lr){9-14}      & \textbf{100} & \textbf{50} & \textbf{25} & \textbf{10} & \textbf{5} & \textbf{2} & \textbf{1} & \textbf{2} & \textbf{5} & \textbf{10} & \textbf{25} & \textbf{50} & \textbf{100} &  \\
      \toprule
      LDAM  & \cellcolor[rgb]{ .847,  .918,  .796} 66.03  & \cellcolor[rgb]{ .871,  .929,  .827} 63.03  & \cellcolor[rgb]{ .89,  .941,  .851} 59.98  & \cellcolor[rgb]{ .918,  .957,  .89} 55.37  & \cellcolor[rgb]{ .961,  .98,  .949} 51.01  & \cellcolor[rgb]{ 1,  1,  1} 45.02  & \cellcolor[rgb]{ 1,  1,  1} 39.99  & \cellcolor[rgb]{ 1,  1,  1} 35.04  & \cellcolor[rgb]{ 1,  1,  1} 28.75  & \cellcolor[rgb]{ 1,  1,  1} 24.33  & \cellcolor[rgb]{ 1,  1,  1} 19.49  & \cellcolor[rgb]{ 1,  1,  1} 16.49  & \cellcolor[rgb]{ 1,  1,  1} 13.66  & \cellcolor[rgb]{ 1,  1,  1} 39.86$_{\color{blue}(\pm 17.66)}$\\
      LA    & \cellcolor[rgb]{ .961,  .98,  .949} 60.68  & \cellcolor[rgb]{ .961,  .98,  .949} 58.89  & \cellcolor[rgb]{ .965,  .98,  .953} 56.90  & \cellcolor[rgb]{ .953,  .976,  .937} 54.08  & \cellcolor[rgb]{ .941,  .969,  .922} 51.74  & \cellcolor[rgb]{ .902,  .949,  .871} 48.45  & \cellcolor[rgb]{ .875,  .933,  .835} 45.40  & \cellcolor[rgb]{ .859,  .925,  .816} 42.30  & \cellcolor[rgb]{ .855,  .922,  .808} 39.44  & \cellcolor[rgb]{ .863,  .925,  .816} 37.23  & \cellcolor[rgb]{ .871,  .929,  .831} 33.69  & \cellcolor[rgb]{ .878,  .933,  .839} 31.98  & \cellcolor[rgb]{ .882,  .937,  .843} 30.22  & \cellcolor[rgb]{ .902,  .949,  .871} 45.46$_{\color{blue}(\pm 10.12)}$\\
      VS    & \cellcolor[rgb]{ 1,  1,  1} 58.80  & \cellcolor[rgb]{ 1,  1,  1} 57.16  & \cellcolor[rgb]{ 1,  1,  1} 55.38  & \cellcolor[rgb]{ .996,  1,  .996} 52.55  & \cellcolor[rgb]{ 1,  1,  1} 49.66  & \cellcolor[rgb]{ .98,  .992,  .973} 45.74  & \cellcolor[rgb]{ .933,  .965,  .914} 42.84  & \cellcolor[rgb]{ .914,  .953,  .886} 39.46  & \cellcolor[rgb]{ .918,  .957,  .89} 35.04  & \cellcolor[rgb]{ .914,  .953,  .886} 32.33  & \cellcolor[rgb]{ .918,  .957,  .894} 28.52  & \cellcolor[rgb]{ .922,  .957,  .898} 26.43  & \cellcolor[rgb]{ .922,  .957,  .894} 24.78  & \cellcolor[rgb]{ .961,  .98,  .949} 42.21$_{\color{blue}(\pm 11.58)}$\\
      LADE  & \cellcolor[rgb]{ .992,  .996,  .988} 59.22  & \cellcolor[rgb]{ .988,  .992,  .98} 57.84  & \cellcolor[rgb]{ .996,  1,  .996} 55.55  & \cellcolor[rgb]{ 1,  1,  1} 52.40  & \cellcolor[rgb]{ 1,  1,  1} 49.70  & \cellcolor[rgb]{ .949,  .973,  .933} 46.82  & \cellcolor[rgb]{ .898,  .945,  .863} 44.47  & \cellcolor[rgb]{ .875,  .933,  .835} 41.53  & \cellcolor[rgb]{ .871,  .929,  .827} 38.46  & \cellcolor[rgb]{ .871,  .929,  .827} 36.46  & \cellcolor[rgb]{ .875,  .933,  .831} 33.49  & \cellcolor[rgb]{ .878,  .933,  .839} 31.98  & \cellcolor[rgb]{ .878,  .937,  .839} 30.41  & \cellcolor[rgb]{ .922,  .957,  .894} 44.49$_{\color{blue}(\pm \phantom{0}9.60)}$\\
      DDC   & \cellcolor[rgb]{ .988,  .996,  .984} 59.36  & \cellcolor[rgb]{ .976,  .988,  .969} 58.33  & \cellcolor[rgb]{ .973,  .984,  .961} 56.60  & \cellcolor[rgb]{ .957,  .976,  .945} 53.95  & \cellcolor[rgb]{ .949,  .973,  .933} 51.48  & \cellcolor[rgb]{ .886,  .937,  .847} 49.01  & \cellcolor[rgb]{ .839,  .914,  .788} 46.98  & \cellcolor[rgb]{ .824,  .906,  .769} 44.01  & \cellcolor[rgb]{ .831,  .91,  .78} 41.06  & \cellcolor[rgb]{ .839,  .914,  .788} 39.22  & \cellcolor[rgb]{ .847,  .918,  .8} 36.20  & \cellcolor[rgb]{ .859,  .922,  .812} 34.55  & \cellcolor[rgb]{ .859,  .925,  .816} 33.13  & \cellcolor[rgb]{ .886,  .937,  .851} 46.45$_{\color{blue}(\pm \phantom{0}8.82)}$\\
      RIDE  & \cellcolor[rgb]{ .878,  .933,  .839} 64.57  & \cellcolor[rgb]{ .871,  .929,  .831} 62.91  & \cellcolor[rgb]{ .863,  .925,  .82} 61.06  & \cellcolor[rgb]{ .839,  .914,  .788} 58.02  & \cellcolor[rgb]{ .835,  .91,  .784} 55.33  & \cellcolor[rgb]{ .808,  .894,  .745} 51.67  & \cellcolor[rgb]{ .804,  .894,  .741} 48.40  & \cellcolor[rgb]{ .812,  .898,  .753} 44.66  & \cellcolor[rgb]{ .843,  .914,  .792} 40.43  & \cellcolor[rgb]{ .859,  .922,  .812} 37.54  & \cellcolor[rgb]{ .867,  .929,  .824} 34.10  & \cellcolor[rgb]{ .875,  .933,  .835} 32.46  & \cellcolor[rgb]{ .878,  .937,  .839} 30.41  & \cellcolor[rgb]{ .863,  .925,  .816} 47.81$_{\color{blue}(\pm 11.62)}$\\
      SADE  & \cellcolor[rgb]{ .808,  .898,  .749} 67.81  & \cellcolor[rgb]{ .816,  .898,  .753} 65.45  & \cellcolor[rgb]{ .82,  .902,  .765} 62.75  & \cellcolor[rgb]{ .82,  .902,  .765} 58.69  & \cellcolor[rgb]{ .816,  .898,  .757} 56.04  & \cellcolor[rgb]{ .8,  .89,  .737} 51.91  & \cellcolor[rgb]{ .776,  .878,  .706} \textbf{49.53} & \cellcolor[rgb]{ .788,  .886,  .722} \underline{45.90} & \cellcolor[rgb]{ .792,  .886,  .725} \underline{44.04} & \cellcolor[rgb]{ .796,  .89,  .729} \underline{43.34} & \cellcolor[rgb]{ .792,  .886,  .725} \underline{42.32} & \cellcolor[rgb]{ .792,  .89,  .729} \underline{42.48} & \cellcolor[rgb]{ .788,  .886,  .722} \underline{42.75} & \cellcolor[rgb]{ .792,  .886,  .725} \underline{51.77}$_{\color{blue}(\pm \phantom{0}9.02)}$\\
      BalPoE & \cellcolor[rgb]{ .776,  .878,  .706} \textbf{69.22} & \cellcolor[rgb]{ .776,  .878,  .706} \textbf{67.02} & \cellcolor[rgb]{ .776,  .878,  .706} \textbf{64.45} & \cellcolor[rgb]{ .776,  .878,  .706} \textbf{60.19} & \cellcolor[rgb]{ .776,  .878,  .706} \textbf{57.26} & \cellcolor[rgb]{ .796,  .89,  .729} \underline{52.05} & \cellcolor[rgb]{ .8,  .89,  .733} \underline{48.66} & \cellcolor[rgb]{ .816,  .902,  .757} 44.54  & \cellcolor[rgb]{ .839,  .914,  .788} 40.57  & \cellcolor[rgb]{ .863,  .925,  .82} 37.13  & \cellcolor[rgb]{ .875,  .933,  .835} 33.25  & \cellcolor[rgb]{ .882,  .937,  .843} 31.58  & \cellcolor[rgb]{ .886,  .941,  .851} 29.24  & \cellcolor[rgb]{ .843,  .918,  .792} 48.86$_{\color{blue}(\pm 13.44)}$\\
          \toprule
      \textbf{DirMixE} & \cellcolor[rgb]{ .796,  .89,  .733} \underline{68.32} & \cellcolor[rgb]{ .796,  .89,  .733} \underline{66.21} & \cellcolor[rgb]{ .812,  .898,  .753} \underline{63.09} & \cellcolor[rgb]{ .8,  .89,  .733} \underline{59.49} & \cellcolor[rgb]{ .804,  .894,  .741} \underline{56.35} & \cellcolor[rgb]{ .776,  .878,  .706} \textbf{52.62} & \cellcolor[rgb]{ .804,  .894,  .745} 48.38  & \cellcolor[rgb]{ .776,  .878,  .706} \textbf{46.40} & \cellcolor[rgb]{ .776,  .878,  .706} \textbf{45.05} & \cellcolor[rgb]{ .776,  .878,  .706} \textbf{44.79} & \cellcolor[rgb]{ .776,  .878,  .706} \textbf{43.71} & \cellcolor[rgb]{ .776,  .878,  .706} \textbf{44.41} & \cellcolor[rgb]{ .776,  .878,  .706} \textbf{44.25} & \cellcolor[rgb]{ .776,  .878,  .706} \textbf{52.54}$_{\color{blue}(\pm \phantom{0}8.74)}$\\
      \bottomrule
      \end{tabular}%
      }
    \label{tab:Cifar100_Sade}%
  \end{table*}%

  \begin{table*}[h!]
  \centering
  \caption{Performance Comparison on ImageNet-LT When Training ResNet Models (\textbf{SADE's Setting)}}
  \renewcommand\arraystretch{1.0}
  \resizebox{\linewidth}{!}{
    \begin{tabular}{lcccccccccccc}
    \toprule
    \multirow{1.5}[4]{*}{\textbf{Method}} & \multicolumn{5}{c}{\textbf{Forward-LT}} & \textbf{Uni.} & \multicolumn{5}{c}{\textbf{Backward-LT}} & \multirow{1.5}[4]{*}{\textbf{Mean}} \\
    \cmidrule(lr){2-6}\cmidrule(lr){7-7}\cmidrule(lr){8-12}          & \textbf{50} & \textbf{25} & \textbf{10} & \textbf{5} & \textbf{2} & \textbf{1} & \textbf{2} & \textbf{5} & \textbf{10} & \textbf{25} & \textbf{50} &  \\
    \toprule
    LDAM  & \cellcolor[rgb]{ .992,  .929,  .89} 63.18  & \cellcolor[rgb]{ .992,  .941,  .91} 61.31  & \cellcolor[rgb]{ 1,  .976,  .965} 58.10  & \cellcolor[rgb]{ 1,  1,  .996} 55.29  & \cellcolor[rgb]{ 1,  1,  1} 51.28  & \cellcolor[rgb]{ 1,  1,  1} 47.92  & \cellcolor[rgb]{ 1,  1,  1} 44.63  & \cellcolor[rgb]{ 1,  1,  1} 40.19  & \cellcolor[rgb]{ 1,  1,  1} 37.11  & \cellcolor[rgb]{ 1,  1,  1} 33.80  & \cellcolor[rgb]{ 1,  1,  1} 31.17  & \cellcolor[rgb]{ 1,  1,  1} 47.63$_{\color{blue}(\pm 10.65)}$\\
    LA    & \cellcolor[rgb]{ 1,  .98,  .969} 60.57  & \cellcolor[rgb]{ .996,  .973,  .957} 59.83  & \cellcolor[rgb]{ 1,  .988,  .98} 57.63  & \cellcolor[rgb]{ 1,  .992,  .984} 55.64  & \cellcolor[rgb]{ .996,  .969,  .949} 52.92  & \cellcolor[rgb]{ .996,  .953,  .922} 50.66  & \cellcolor[rgb]{ .992,  .941,  .91} 48.39  & \cellcolor[rgb]{ .992,  .933,  .894} 45.43  & \cellcolor[rgb]{ .992,  .929,  .89} 43.35  & \cellcolor[rgb]{ .992,  .933,  .894} 41.08  & \cellcolor[rgb]{ .992,  .933,  .898} 39.07  & \cellcolor[rgb]{ .996,  .957,  .933} 50.42$_{\color{blue}(\pm \phantom{0}7.21)}$\\
    VS    & \cellcolor[rgb]{ 1,  .988,  .976} 60.27  & \cellcolor[rgb]{ 1,  .984,  .973} 59.33  & \cellcolor[rgb]{ 1,  .984,  .976} 57.77  & \cellcolor[rgb]{ 1,  .98,  .969} 56.05  & \cellcolor[rgb]{ .996,  .949,  .918} 53.76  & \cellcolor[rgb]{ .992,  .922,  .878} 52.20  & \cellcolor[rgb]{ .988,  .91,  .859} 50.42  & \cellcolor[rgb]{ .988,  .894,  .835} 48.36  & \cellcolor[rgb]{ .988,  .894,  .831} 46.64  & \cellcolor[rgb]{ .988,  .898,  .835} 44.97  & \cellcolor[rgb]{ .988,  .898,  .835} 43.55  & \cellcolor[rgb]{ .992,  .929,  .89} 52.12$_{\color{blue}(\pm \phantom{0}5.56)}$\\
    LADE  & \cellcolor[rgb]{ .992,  .918,  .867} 63.92  & \cellcolor[rgb]{ .992,  .922,  .875} 62.41  & \cellcolor[rgb]{ .992,  .929,  .89} 60.17  & \cellcolor[rgb]{ .992,  .937,  .898} 57.91  & \cellcolor[rgb]{ .992,  .925,  .878} 54.98  & \cellcolor[rgb]{ .988,  .914,  .863} 52.71  & \cellcolor[rgb]{ .988,  .914,  .863} 50.29  & \cellcolor[rgb]{ .988,  .914,  .859} 47.10  & \cellcolor[rgb]{ .988,  .91,  .859} 45.19  & \cellcolor[rgb]{ .992,  .918,  .867} 42.89  & \cellcolor[rgb]{ .992,  .922,  .878} 40.55  & \cellcolor[rgb]{ .992,  .925,  .878} 52.56$_{\color{blue}(\pm \phantom{0}7.66)}$\\
    DDC   & \cellcolor[rgb]{ 1,  1,  1} 59.49  & \cellcolor[rgb]{ 1,  1,  1} 58.37  & \cellcolor[rgb]{ 1,  1,  1} 57.02  & \cellcolor[rgb]{ 1,  1,  1} 55.16  & \cellcolor[rgb]{ .996,  .961,  .941} 53.15  & \cellcolor[rgb]{ .992,  .941,  .902} 51.32  & \cellcolor[rgb]{ .992,  .925,  .882} 49.45  & \cellcolor[rgb]{ .988,  .91,  .855} 47.39  & \cellcolor[rgb]{ .988,  .906,  .851} 45.53  & \cellcolor[rgb]{ .988,  .906,  .847} 44.10  & \cellcolor[rgb]{ .988,  .906,  .851} 42.46  & \cellcolor[rgb]{ .996,  .945,  .914} 51.22$_{\color{blue}(\pm \phantom{0}5.64)}$\\
    RIDE  & \cellcolor[rgb]{ .988,  .898,  .835} 64.98  & \cellcolor[rgb]{ .988,  .894,  .831} 63.73  & \cellcolor[rgb]{ .984,  .886,  .816} 62.20  & \cellcolor[rgb]{ .984,  .882,  .816} 60.15  & \cellcolor[rgb]{ .984,  .878,  .808} 57.09  & \cellcolor[rgb]{ .984,  .871,  .796} 54.98  & \cellcolor[rgb]{ .984,  .878,  .804} 52.53  & \cellcolor[rgb]{ .984,  .878,  .808} 49.59  & \cellcolor[rgb]{ .984,  .882,  .812} 47.75  & \cellcolor[rgb]{ .988,  .898,  .835} 44.86  & \cellcolor[rgb]{ .988,  .902,  .847} 42.88  & \cellcolor[rgb]{ .988,  .89,  .827} 54.61$_{\color{blue}(\pm \phantom{0}7.35)}$\\
    SADE  & \cellcolor[rgb]{ .976,  .8,  .686} \underline{69.93} & \cellcolor[rgb]{ .976,  .804,  .69} 68.19  & \cellcolor[rgb]{ .973,  .796,  .678} \textbf{66.02} & \cellcolor[rgb]{ .976,  .804,  .69} \underline{63.53} & \cellcolor[rgb]{ .973,  .796,  .678} \textbf{60.94} & \cellcolor[rgb]{ .973,  .796,  .678} \textbf{59.00} & \cellcolor[rgb]{ .973,  .796,  .678} \textbf{57.53} & \cellcolor[rgb]{ .973,  .796,  .678} \textbf{55.91} & \cellcolor[rgb]{ .976,  .804,  .69} \underline{54.60} & \cellcolor[rgb]{ .976,  .816,  .706} \underline{53.58} & \cellcolor[rgb]{ .976,  .816,  .71} \underline{53.15} & \cellcolor[rgb]{ .976,  .804,  .686} \underline{60.22}$_{\color{blue}(\pm \phantom{0}5.69)}$\\
    BalPoE & \cellcolor[rgb]{ .976,  .808,  .694} 69.66  & \cellcolor[rgb]{ .976,  .8,  .686} \underline{68.28} & \cellcolor[rgb]{ .976,  .8,  .686} 65.87  & \cellcolor[rgb]{ .973,  .796,  .678} \textbf{63.77} & \cellcolor[rgb]{ .976,  .8,  .682} \underline{60.91} & \cellcolor[rgb]{ .976,  .8,  .682} \underline{58.95} & \cellcolor[rgb]{ .976,  .808,  .694} 57.00  & \cellcolor[rgb]{ .976,  .808,  .694} 55.22  & \cellcolor[rgb]{ .976,  .812,  .702} 53.85  & \cellcolor[rgb]{ .976,  .824,  .722} 52.59  & \cellcolor[rgb]{ .976,  .827,  .725} 51.88  & \cellcolor[rgb]{ .976,  .808,  .698} 59.82$_{\color{blue}(\pm \phantom{0}6.05)}$\\
    \toprule
    \textbf{DirMixE}  & \cellcolor[rgb]{ .973,  .796,  .678} \textbf{70.09} & \cellcolor[rgb]{ .973,  .796,  .678} \textbf{68.46} & \cellcolor[rgb]{ .976,  .8,  .682} \underline{65.93} & \cellcolor[rgb]{ .976,  .812,  .702} 63.22  & \cellcolor[rgb]{ .976,  .808,  .694} 60.50  & \cellcolor[rgb]{ .976,  .804,  .69} 58.61  & \cellcolor[rgb]{ .976,  .804,  .686} \underline{57.27} & \cellcolor[rgb]{ .976,  .808,  .694} \underline{55.27} & \cellcolor[rgb]{ .973,  .796,  .678} \textbf{55.04} & \cellcolor[rgb]{ .973,  .796,  .678} \textbf{55.38} & \cellcolor[rgb]{ .973,  .796,  .678} \textbf{55.33} & \cellcolor[rgb]{ .973,  .796,  .678} \textbf{60.46}$_{\color{blue}(\pm \phantom{0}5.36)}$\\
    \bottomrule
    \end{tabular}%
    }
  \label{tab:ImageNet_Sade}%
  \end{table*}%

  \begin{table}[h!]
    \centering
    \caption{Performance Comparison on iNaturalist When Training ResNet Models (\textbf{SADE's Setting)}}
      \begin{tabular}{lcccccc}
      \toprule
      \multirow{1.5}[4]{*}{\textbf{Method}} & \multicolumn{2}{c}{\textbf{Forward-LT}} & \textbf{Uni.} & \multicolumn{2}{c}{\textbf{Backward-LT}} & \multirow{1.5}[4]{*}{\textbf{Mean}} \\
      \cmidrule(lr){2-3}\cmidrule(lr){4-4}\cmidrule(lr){5-6}          & \textbf{3} & \textbf{2} & \textbf{1} & \textbf{2} & \textbf{3} &  \\
      \midrule
      LDAM  & \cellcolor[rgb]{ .953,  .976,  .937}65.95  & \cellcolor[rgb]{ .961,  .98,  .949}66.21  & \cellcolor[rgb]{ .965,  .98,  .953}66.66  & \cellcolor[rgb]{ .965,  .98,  .953}67.45  & \cellcolor[rgb]{ .969,  .984,  .957}67.23  & \cellcolor[rgb]{ .961,  .98,  .949}66.70$_{\color{blue}(\pm 0.52)}$  \\
      LA    & \cellcolor[rgb]{ .945,  .973,  .929}66.17  & \cellcolor[rgb]{ .957,  .98,  .945}66.31  & \cellcolor[rgb]{ .976,  .988,  .969}66.28  & \cellcolor[rgb]{ .996,  1,  .996}66.27  & \cellcolor[rgb]{ .996,  1,  .996}66.26  & \cellcolor[rgb]{ .973,  .988,  .965}66.26$_{\color{blue}(\pm 0.04)}$  \\
      VS    & 64.06  & 64.71  & 65.34  & 66.10  & 66.05  & 65.25$_{\color{blue}(\pm 0.72)}$  \\
      LADE  & \cellcolor[rgb]{ .859,  .922,  .812}69.53  & \cellcolor[rgb]{ .867,  .925,  .82}69.80  & \cellcolor[rgb]{ .871,  .929,  .827}70.06  & \cellcolor[rgb]{ .878,  .937,  .843}70.41  & \cellcolor[rgb]{ .878,  .937,  .843}70.36  & \cellcolor[rgb]{ .867,  .929,  .827}70.03$_{\color{blue}(\pm 0.30)}$  \\
      DDC   & \cellcolor[rgb]{ .996,  1,  .992}64.31  & \cellcolor[rgb]{ .996,  1,  .996}64.92  & \cellcolor[rgb]{ .984,  .992,  .98}65.93  & \cellcolor[rgb]{ .98,  .992,  .973}66.86  & \cellcolor[rgb]{ .976,  .988,  .969}66.90  & \cellcolor[rgb]{ .988,  .992,  .984}65.78$_{\color{blue}(\pm 0.94)}$  \\
      RIDE  & \cellcolor[rgb]{ .816,  .902,  .757}71.10  & \cellcolor[rgb]{ .82,  .902,  .761}71.52  & \cellcolor[rgb]{ .827,  .906,  .773}71.59  & \cellcolor[rgb]{ .831,  .91,  .776}72.17  & \cellcolor[rgb]{ .843,  .914,  .792}71.69  & \cellcolor[rgb]{ .824,  .906,  .769}71.61$_{\color{blue}(\pm 0.31)}$  \\
      SADE  & \cellcolor[rgb]{ .792,  .89,  .729}71.95  & \cellcolor[rgb]{ .788,  .886,  .722}72.58  & \cellcolor[rgb]{ .796,  .89,  .729}72.70  & \cellcolor[rgb]{ .8,  .894,  .737}73.16  & \cellcolor[rgb]{ .796,  .89,  .733}73.27  & \cellcolor[rgb]{ .792,  .886,  .725}72.73$_{\color{blue}(\pm 0.43)}$  \\
      BalPoE & \cellcolor[rgb]{ .788,  .886,  .718}\underline{72.23}  & \cellcolor[rgb]{ .776,  .878,  .706}\textbf{73.00} & \cellcolor[rgb]{ .776,  .878,  .706}\textbf{73.31} & \cellcolor[rgb]{ .776,  .878,  .706}\textbf{73.99} & \cellcolor[rgb]{ .788,  .886,  .722}\underline{73.55}  & \cellcolor[rgb]{ .78,  .882,  .71}\underline{73.22}$_{\color{blue}(\pm 0.54)}$  \\
      \midrule
      \textbf{DirMixE} & \cellcolor[rgb]{ .776,  .878,  .706}\textbf{72.53} & \cellcolor[rgb]{ .78,  .882,  .714}\underline{72.88}  & \cellcolor[rgb]{ .78,  .882,  .71}\underline{73.21}  & \cellcolor[rgb]{ .788,  .886,  .722}\underline{73.66}  & \cellcolor[rgb]{ .776,  .878,  .706}\textbf{73.94} & \cellcolor[rgb]{ .776,  .878,  .706}\textbf{73.24}$_{\color{blue}(\pm 0.47)}$ \\
      \bottomrule
      \end{tabular}%
    \label{tab:inat-sade}%
  \end{table}%

\subsection{\newsec{Overall Performance for DirMixE-LSF on Ours Setting}}   \label{exp:ours_setting}

\newcont{

The overall performance on the CIFAR-10-LT and CIFAR-100-LT dataset for LoRA and AdaptFormer on Ours Setting
is shown in Tab.\ref{tab:performance_comparison_cifar10_lora}-Tab.\ref{tab:performance_comparison_cifar100_adaptformer}, respectively.

}

\begin{table*}[h!]
  \centering
  \caption{Performance Comparison on CIFAR-10 When Fine-tuning with LoRA (\textbf{Ours Setting})}
  \small
      \begin{tabular}{lcccccccccc}
          \toprule
          \multirow{1.5}[4]{*}{\textbf{Method}} & \multicolumn{3}{c}{\textbf{Forward-LT}} & \multicolumn{3}{c}{\textbf{Uniform}} & \multicolumn{3}{c}{\textbf{Backward-LT}} & \multirow{1.5}[4]{*}{\textbf{Mean}} \\
          \cmidrule(lr){2-4}\cmidrule(lr){5-7}\cmidrule(lr){8-10}     & \textbf{1} & \textbf{2} & \textbf{3} & \textbf{1} & \textbf{2} & \textbf{3} & \textbf{1} & \textbf{2} & \textbf{3} &  \\
          \midrule
          LIFT-LoRA & 97.33 & 96.89 & 97.45 & \cellcolor[rgb]{ .741,  .843,  .933}\textbf{96.54} & \cellcolor[rgb]{ .741,  .843,  .933}\textbf{96.43} & \cellcolor[rgb]{ .741,  .843,  .933}\textbf{96.56} & 97.21 & 97.05 & 97.00 & 96.94$_{\color{blue}(\pm 0.34)}$ \\
          SADE-LoRA & \cellcolor[rgb]{ .812,  .886,  .953}98.12 & \cellcolor[rgb]{ .788,  .871,  .945}98.08 & \cellcolor[rgb]{ .835,  .902,  .961}98.12 & 96.07 & 96.10 & 95.83 & \cellcolor[rgb]{ .996,  .996,  1}97.25 & \cellcolor[rgb]{ .843,  .906,  .961}98.12 & \cellcolor[rgb]{ .816,  .89,  .953}98.04 & \cellcolor[rgb]{ .898,  .937,  .976}97.30$_{\color{blue}(\pm 0.96)}$ \\
          DirMixE-LoRA & \cellcolor[rgb]{ .8,  .878,  .949}\underline{98.16} & \cellcolor[rgb]{ .757,  .855,  .937}\underline{98.24} & \cellcolor[rgb]{ .784,  .871,  .945}\underline{98.33} & \cellcolor[rgb]{ .875,  .925,  .969}\underline{96.30} & \cellcolor[rgb]{ .922,  .953,  .98}96.20 & \cellcolor[rgb]{ .945,  .969,  .988}95.99 & \cellcolor[rgb]{ .804,  .882,  .949}\underline{98.36} & \cellcolor[rgb]{ .808,  .886,  .953}\underline{98.36} & \cellcolor[rgb]{ .784,  .871,  .945}\underline{98.20} & \cellcolor[rgb]{ .82,  .89,  .953}\underline{97.57}$_{\color{blue}(\pm 1.00)}$ \\
          \midrule
          \textbf{DirMixE-LoRA-LSF} & \cellcolor[rgb]{ .741,  .843,  .933}\textbf{98.40} & \cellcolor[rgb]{ .741,  .843,  .933}\textbf{98.32} & \cellcolor[rgb]{ .741,  .843,  .933}\textbf{98.49} & \cellcolor[rgb]{ .741,  .843,  .933}\textbf{96.54} & \cellcolor[rgb]{ .753,  .851,  .937}\underline{96.42} & \cellcolor[rgb]{ .804,  .882,  .949}\underline{96.39} & \cellcolor[rgb]{ .741,  .843,  .933}\textbf{98.70} & \cellcolor[rgb]{ .741,  .843,  .933}\textbf{98.81} & \cellcolor[rgb]{ .741,  .843,  .933}\textbf{98.44} & \cellcolor[rgb]{ .741,  .843,  .933}\textbf{97.83}$_{\color{blue}(\pm 0.99)}$ \\
          \bottomrule
      \end{tabular}%
  \label{tab:performance_comparison_cifar10_lora}%
\end{table*}%

\begin{table*}[h!]
    \centering
    \caption{Performance Comparison on CIFAR-10 When Fine-tuning with AdaptFormer (\textbf{Ours Setting})}
    \small
        \begin{tabular}{lcccccccccc}
        \toprule
        \multirow{1.5}[4]{*}{\textbf{Method}} & \multicolumn{3}{c}{\textbf{Forward-LT}} & \multicolumn{3}{c}{\textbf{Uniform}} & \multicolumn{3}{c}{\textbf{Backward-LT}} & \multirow{1.5}[4]{*}{\textbf{Mean}} \\
            \cmidrule(lr){2-4}\cmidrule(lr){5-7}\cmidrule(lr){8-10}     & \textbf{1} & \textbf{2} & \textbf{3} & \textbf{1} & \textbf{2} & \textbf{3} & \textbf{1} & \textbf{2} & \textbf{3} &  \\
        \midrule
        LIFT-AF & 97.41 & 97.05 & 97.62 & 96.20 & 96.17 & 96.31 & 96.83 & 96.52 & 96.72 & 96.76$_{\color{blue}(\pm 0.49)}$ \\
        SADE-AF & \cellcolor[rgb]{ .773,  .863,  .941}98.08 & \cellcolor[rgb]{ .753,  .851,  .937}\underline{98.24} & \cellcolor[rgb]{ .769,  .859,  .941}98.41 & \cellcolor[rgb]{ .741,  .843,  .933}\textbf{96.60} & \cellcolor[rgb]{ .741,  .843,  .933}\textbf{96.47} & \cellcolor[rgb]{ .741,  .843,  .933}\underline{96.53} & \cellcolor[rgb]{ .851,  .91,  .965}97.90 & \cellcolor[rgb]{ .757,  .855,  .937}98.57 & \cellcolor[rgb]{ .804,  .882,  .949}98.20 & \cellcolor[rgb]{ .773,  .863,  .941}97.67$_{\color{blue}(\pm 0.82)}$ \\
        DirMixE-AF & \cellcolor[rgb]{ .757,  .855,  .937}\underline{98.12} & \cellcolor[rgb]{ .753,  .851,  .937}\underline{98.24} & \cellcolor[rgb]{ .757,  .851,  .937}\underline{98.45} & \cellcolor[rgb]{ .769,  .859,  .941}\underline{96.56} & \cellcolor[rgb]{ .741,  .843,  .933}\textbf{96.47} & \cellcolor[rgb]{ .8,  .882,  .949}\textbf{96.48} & \cellcolor[rgb]{ .741,  .843,  .933}\textbf{98.66} & \cellcolor[rgb]{ .753,  .851,  .937}\underline{98.61} & \cellcolor[rgb]{ .741,  .843,  .933}\textbf{98.64} & \cellcolor[rgb]{ .765,  .859,  .941}\underline{97.70}$_{\color{blue}(\pm 0.94)}$ \\
        \midrule
        \textbf{DirMixE-AF-LSF} & \cellcolor[rgb]{ .741,  .843,  .933}\textbf{98.16} & \cellcolor[rgb]{ .741,  .843,  .933}\textbf{98.28} & \cellcolor[rgb]{ .741,  .843,  .933}\textbf{98.49} & \cellcolor[rgb]{ .835,  .902,  .957}96.46 & \cellcolor[rgb]{ .847,  .91,  .961}\underline{96.35} & \cellcolor[rgb]{ .859,  .918,  .965}96.43 & \cellcolor[rgb]{ .749,  .847,  .937}\underline{98.63} & \cellcolor[rgb]{ .741,  .843,  .933}\textbf{98.69} & \cellcolor[rgb]{ .753,  .851,  .937}\underline{98.56} & \cellcolor[rgb]{ .741,  .843,  .933}\textbf{97.78}$_{\color{blue}(\pm 0.98)}$ \\
        \bottomrule
        \end{tabular}%
    \label{tab:performance_comparison_cifar10_adaptformer}%
\end{table*}%

\begin{table*}[h!]
    \centering
    \caption{Performance Comparison on CIFAR-100 When Fine-tuning with LoRA (\textbf{Ours Setting)}}
    \small
        \begin{tabular}{lcccccccccc}
        \toprule
        \multirow{1.5}[4]{*}{\textbf{Method}} & \multicolumn{3}{c}{\textbf{Forward-LT}} & \multicolumn{3}{c}{\textbf{Uniform}} & \multicolumn{3}{c}{\textbf{Backward-LT}} & \multirow{1.5}[4]{*}{\textbf{Mean}} \\
    \cmidrule(lr){2-4}\cmidrule(lr){5-7}\cmidrule(lr){8-10}            & \textbf{1} & \textbf{2} & \textbf{3} & \textbf{1} & \textbf{2} & \textbf{3} & \textbf{1} & \textbf{2} & \textbf{3} &  \\
        \midrule
        LIFT-LoRA & 84.32 & 81.19 & 82.76 & 79.49 & 78.49 & \cellcolor[rgb]{ .851,  .922,  .804}\underline{80.89} & 76.44 & 71.45 & 72.38 & 78.60$_{\color{blue}(\pm 4.18)}$ \\
        SADE-LoRA & \cellcolor[rgb]{ .776,  .878,  .706}\textbf{89.17} & \cellcolor[rgb]{ .776,  .878,  .706}\textbf{87.08} & \cellcolor[rgb]{ .808,  .894,  .745}\underline{87.92} & \cellcolor[rgb]{ .827,  .906,  .773}\underline{80.77} & \cellcolor[rgb]{ .776,  .878,  .706}\textbf{80.49} & 78.51 & \cellcolor[rgb]{ .839,  .914,  .788}83.61 & \cellcolor[rgb]{ .824,  .906,  .769}80.77 & \cellcolor[rgb]{ .808,  .894,  .745}82.64 & \cellcolor[rgb]{ .808,  .894,  .745}83.44$_{\color{blue}(\pm 3.56)}$ \\
        DirMixE-LoRA & \cellcolor[rgb]{ .776,  .878,  .706}\textbf{89.17} & \cellcolor[rgb]{ .792,  .89,  .729}\underline{86.68} & \cellcolor[rgb]{ .804,  .894,  .741}\textbf{87.97} & \cellcolor[rgb]{ .82,  .902,  .765}\textbf{80.83} & \cellcolor[rgb]{ .792,  .886,  .725}\underline{80.38} & \cellcolor[rgb]{ .976,  .988,  .969}78.89 & \cellcolor[rgb]{ .808,  .894,  .745}\underline{85.01} & \cellcolor[rgb]{ .808,  .894,  .745}\underline{81.79} & \cellcolor[rgb]{ .788,  .886,  .722}\textbf{83.58} & \cellcolor[rgb]{ .792,  .886,  .725}\underline{83.81}$_{\color{blue}(\pm 3.41)}$ \\
        \midrule
        \textbf{DirMixE-LoRA-LSF} & \cellcolor[rgb]{ .792,  .886,  .725}\underline{88.90} & \cellcolor[rgb]{ .796,  .89,  .729}86.61 & \cellcolor[rgb]{ .804,  .894,  .741}\textbf{87.97} & \cellcolor[rgb]{ .863,  .925,  .816}80.53 & \cellcolor[rgb]{ .937,  .965,  .914}79.08 & \cellcolor[rgb]{ .776,  .878,  .706}\textbf{82.04} & \cellcolor[rgb]{ .776,  .878,  .706}\textbf{86.22} & \cellcolor[rgb]{ .784,  .882,  .714}\textbf{82.99} & \cellcolor[rgb]{ .8,  .89,  .733}\underline{83.05} & \cellcolor[rgb]{ .776,  .878,  .706}\textbf{84.15}$_{\color{blue}(\pm 3.22)}$ \\
        \bottomrule
        \end{tabular}%
    \label{tab:performance_comparison_cifar100_lora}%
  \end{table*}%

  \begin{table*}[h!]
  \centering
      \caption{Performance Comparison on CIFAR-100 When Fine-tuning with AdaptFormer (\textbf{Ours Setting)}}
      \small
          \begin{tabular}{lcccccccccc}
          \toprule
          \multirow{1.5}[4]{*}{\textbf{Method}} & \multicolumn{3}{c}{\textbf{Forward-LT}} & \multicolumn{3}{c}{\textbf{Uniform}} & \multicolumn{3}{c}{\textbf{Backward-LT}} & \multirow{1.5}[4]{*}{\textbf{Mean}} \\
      \cmidrule(lr){2-4}\cmidrule(lr){5-7}\cmidrule(lr){8-10}            & \textbf{1} & \textbf{2} & \textbf{3} & \textbf{1} & \textbf{2} & \textbf{3} & \textbf{1} & \textbf{2} & \textbf{3} &  \\
          \midrule
          LIFT-AF & 85.27 & 83.47 & 84.64 & \cellcolor[rgb]{ .776,  .878,  .706}\textbf{80.59} & \cellcolor[rgb]{ .776,  .878,  .706}\textbf{80.05} & \cellcolor[rgb]{ .776,  .878,  .706}\textbf{82.43} & 78.03 & 73.93 & 76.26 & 80.52$_{\color{blue}(\pm 3.65)}$ \\
          SADE-AF & \cellcolor[rgb]{ .8,  .894,  .737}88.90 & \cellcolor[rgb]{ .788,  .886,  .722}\underline{85.88} & \cellcolor[rgb]{ .804,  .894,  .741}\underline{87.81} & 78.94 & \cellcolor[rgb]{ .882,  .937,  .847}\underline{78.76} & 78.89 & \cellcolor[rgb]{ .863,  .925,  .82}82.03 & \cellcolor[rgb]{ .898,  .945,  .863}77.26 & \cellcolor[rgb]{ .776,  .878,  .706}\textbf{81.07} & \cellcolor[rgb]{ .855,  .922,  .812}82.17$_{\color{blue}(\pm 4.07)}$ \\
          DirMixE-AF & \cellcolor[rgb]{ .776,  .878,  .706}\textbf{89.29} & \cellcolor[rgb]{ .776,  .878,  .706}\textbf{86.01} & \cellcolor[rgb]{ .776,  .878,  .706}\textbf{88.24} & \cellcolor[rgb]{ .976,  .988,  .969}79.12 & \cellcolor[rgb]{ .992,  .996,  .992}77.41 & \cellcolor[rgb]{ .945,  .969,  .925}79.82 & \cellcolor[rgb]{ .812,  .898,  .749}\underline{83.61} & \cellcolor[rgb]{ .843,  .914,  .792}\underline{78.97} & \cellcolor[rgb]{ .792,  .89,  .725}\underline{80.75} & \cellcolor[rgb]{ .82,  .902,  .765}\underline{82.58}$_{\color{blue}(\pm 4.12)}$ \\
          \midrule
          \textbf{DirMixE-AF-LSF} & \cellcolor[rgb]{ .784,  .882,  .718}\underline{89.17} & \cellcolor[rgb]{ .824,  .906,  .769}85.48 & \cellcolor[rgb]{ .804,  .894,  .741}\underline{87.81} & \cellcolor[rgb]{ .863,  .925,  .816}\underline{79.98} & 77.30 & \cellcolor[rgb]{ .831,  .91,  .78}\underline{81.58} & \cellcolor[rgb]{ .776,  .878,  .706}\textbf{84.54} & \cellcolor[rgb]{ .776,  .878,  .706}\textbf{81.03} & \cellcolor[rgb]{ .796,  .89,  .733}80.65 & \cellcolor[rgb]{ .776,  .878,  .706}\textbf{83.06}$_{\color{blue}(\pm 3.69)}$ \\
          \bottomrule
          \end{tabular}%
      \label{tab:performance_comparison_cifar100_adaptformer}%
  \end{table*}%

\subsection{\newsec{Overall Performance for DirMixE-LSF on SADE's Setting}}   \label{exp:sade_setting}

\newcont{

The overall performance on the CIFAR-10-LT, CIFAR 100-LT, ImageNet-LT and iNaturalist-LT datasets for LoRA and AdaptFormer on SADE's Setting are shown in Tab.\ref{tab:performance_comparison_cifar10_lora_sade}-Tab.\ref{tab:performance_comparison_inaturalist_adaptformer_sade}, respectively.

}

\begin{table*}[h!]
  \centering
  \caption{Performance Comparison on CIFAR-10 When Fine-tuning with LoRA (\textbf{SADE's Setting})}
  \small
  \resizebox{\linewidth}{!}{
      \begin{tabular}{lcccccccccccccc}
          \toprule
          \multirow{1.5}[4]{*}{\textbf{Method}} & \multicolumn{6}{c}{\textbf{Forward-LT}}       & \textbf{Uni.} & \multicolumn{6}{c}{\textbf{Backward-LT}}      & \multirow{1.5}[4]{*}{\textbf{Mean}} \\
          \cmidrule(lr){2-7} \cmidrule(lr){8-8} \cmidrule(lr){9-14}      & \textbf{100} & \textbf{50} & \textbf{25} & \textbf{10} & \textbf{5} & \textbf{2} & \textbf{1} & \textbf{2} & \textbf{5} & \textbf{10} & \textbf{25} & \textbf{50} & \textbf{100} &  \\
          \midrule
          LIFT-LoRA & 97.22 & 97.25 & 97.06 & 96.69 & 96.50 & 96.51 & 96.50 & 96.49 & 96.63 & \cellcolor[rgb]{ .816,  .89,  .953}96.96 & 97.09 & 97.25 & \underline{97.30} & 96.88$_{\color{blue}(\pm 0.31)}$ \\
          SADE-LoRA & \cellcolor[rgb]{ .827,  .898,  .957}97.98 & \cellcolor[rgb]{ .859,  .914,  .965}97.78 & \cellcolor[rgb]{ .859,  .914,  .965}97.52 & \cellcolor[rgb]{ .82,  .89,  .957}97.21 & \cellcolor[rgb]{ .898,  .937,  .976}96.73 & \cellcolor[rgb]{ .812,  .886,  .953}\underline{96.87} & \cellcolor[rgb]{ .741,  .843,  .933}\textbf{96.93} & \cellcolor[rgb]{ .776,  .867,  .945}\underline{96.77} & \cellcolor[rgb]{ .976,  .984,  .996}96.67 & 96.13 & \cellcolor[rgb]{ .961,  .976,  .992}97.21 & \cellcolor[rgb]{ .812,  .886,  .953}98.14 & \cellcolor[rgb]{ .765,  .859,  .941}\textbf{98.71} & \cellcolor[rgb]{ .859,  .918,  .965}97.28$_{\color{blue}(\pm 0.66)}$ \\
          DirMixE-LoRA & \cellcolor[rgb]{ .765,  .859,  .941}\underline{98.26} & \cellcolor[rgb]{ .741,  .843,  .933}\textbf{98.21} & \cellcolor[rgb]{ .741,  .843,  .933}\textbf{97.89} & \cellcolor[rgb]{ .804,  .882,  .949}\underline{97.26} & \cellcolor[rgb]{ .863,  .918,  .965}\underline{96.81} & \cellcolor[rgb]{ .824,  .894,  .957}96.85 & \cellcolor[rgb]{ .847,  .906,  .961}\underline{96.76} & \cellcolor[rgb]{ .776,  .867,  .945}\underline{96.77} & \cellcolor[rgb]{ .961,  .976,  .992}\underline{96.69} & \cellcolor[rgb]{ .761,  .855,  .941}\underline{97.21} & \cellcolor[rgb]{ .824,  .894,  .957}\underline{97.62} & \cellcolor[rgb]{ .788,  .871,  .945}\underline{98.25} & \cellcolor[rgb]{ .765,  .859,  .941}\textbf{98.71} & \cellcolor[rgb]{ .788,  .875,  .949}\underline{97.48}$_{\color{blue}(\pm 0.66)}$ \\
          \midrule
          \textbf{DirMixE-LoRA-LSF} & \cellcolor[rgb]{ .741,  .843,  .933}\textbf{98.35} & \cellcolor[rgb]{ .761,  .855,  .941}\underline{98.14} & \cellcolor[rgb]{ .741,  .843,  .933}\textbf{97.89} & \cellcolor[rgb]{ .741,  .843,  .933}\textbf{97.43} & \cellcolor[rgb]{ .741,  .843,  .933}\textbf{97.07} & \cellcolor[rgb]{ .741,  .843,  .933}\textbf{97.00} & \cellcolor[rgb]{ .741,  .843,  .933}\textbf{96.93} & \cellcolor[rgb]{ .741,  .843,  .933}\textbf{96.81} & \cellcolor[rgb]{ .741,  .843,  .933}\textbf{97.01} & \cellcolor[rgb]{ .741,  .843,  .933}\textbf{97.28} & \cellcolor[rgb]{ .741,  .843,  .933}\textbf{97.86} & \cellcolor[rgb]{ .741,  .843,  .933}\textbf{98.46} & \cellcolor[rgb]{ .765,  .859,  .941}\textbf{98.71} & \cellcolor[rgb]{ .741,  .843,  .933}\textbf{97.61}$_{\color{blue}(\pm 0.61)}$ \\
          \bottomrule
      \end{tabular}%
  }
  \label{tab:performance_comparison_cifar10_lora_sade}%
\end{table*}%

\begin{table*}[h!]
  \centering
  \caption{Performance Comparison on CIFAR-10 When Fine-tuning with AdaptFormer (\textbf{SADE's Setting})}
  \small
  \resizebox{\linewidth}{!}{
      \begin{tabular}{lcccccccccccccc}
      \toprule
      \multirow{1.5}[4]{*}{\textbf{Method}} & \multicolumn{6}{c}{\textbf{Forward-LT}}       & \textbf{Uni.} & \multicolumn{6}{c}{\textbf{Backward-LT}}      & \multirow{1.5}[4]{*}{\textbf{Mean}} \\
          \cmidrule(lr){2-7} \cmidrule(lr){8-8} \cmidrule(lr){9-14}      & \textbf{100} & \textbf{50} & \textbf{25} & \textbf{10} & \textbf{5} & \textbf{2} & \textbf{1} & \textbf{2} & \textbf{5} & \textbf{10} & \textbf{25} & \textbf{50} & \textbf{100} &  \\
      \midrule
      LIFT-AF & 97.38 & 97.35 & 97.09 & 96.60 & \cellcolor[rgb]{ .976,  .988,  .996}96.36 & 96.15 & 96.23 & 96.08 & \cellcolor[rgb]{ .741,  .843,  .933}\textbf{96.20} & \cellcolor[rgb]{ .827,  .898,  .957}\underline{96.45} & 96.59 & 96.89 & 96.89 & 96.64$_{\color{blue}(\pm 0.50)}$ \\
      SADE-AF & \cellcolor[rgb]{ .741,  .843,  .933}\textbf{98.10} & \cellcolor[rgb]{ .741,  .843,  .933}\textbf{97.85} & \cellcolor[rgb]{ .741,  .843,  .933}\textbf{97.52} & \cellcolor[rgb]{ .741,  .843,  .933}\underline{96.96} & \cellcolor[rgb]{ .741,  .843,  .933}\textbf{96.54} & \cellcolor[rgb]{ .78,  .867,  .945}96.34 & \cellcolor[rgb]{ .741,  .843,  .933}\textbf{96.51} & \cellcolor[rgb]{ .741,  .843,  .933}\textbf{96.35} & 95.43 & 95.71 & \cellcolor[rgb]{ .98,  .988,  .996}96.66 & \cellcolor[rgb]{ .859,  .914,  .965}97.60 & \cellcolor[rgb]{ .824,  .894,  .957}98.10 & \cellcolor[rgb]{ .859,  .914,  .965}96.90$_{\color{blue}(\pm 0.80)}$ \\
      DirMixE-AF & \cellcolor[rgb]{ .784,  .871,  .945}\underline{97.98} & \cellcolor[rgb]{ .741,  .843,  .933}\textbf{97.85} & \cellcolor[rgb]{ .796,  .878,  .949}\underline{97.43} & \cellcolor[rgb]{ .878,  .929,  .969}96.77 & \cellcolor[rgb]{ .875,  .925,  .969}\underline{96.44} & \cellcolor[rgb]{ .769,  .859,  .941}\underline{96.35} & \cellcolor[rgb]{ .863,  .918,  .965}96.38 & \cellcolor[rgb]{ .914,  .949,  .98}96.17 & \cellcolor[rgb]{ .863,  .918,  .965}95.85 & \cellcolor[rgb]{ .839,  .906,  .961}96.40 & \cellcolor[rgb]{ .835,  .902,  .961}\underline{97.15} & \cellcolor[rgb]{ .773,  .863,  .941}\underline{98.03} & \cellcolor[rgb]{ .741,  .843,  .933}\textbf{98.63} & \cellcolor[rgb]{ .788,  .871,  .945}\underline{97.03}$_{\color{blue}(\pm 0.72)}$ \\
      \midrule
      \textbf{DirMixE-AF-LSF} & \cellcolor[rgb]{ .827,  .898,  .957}97.86 & \cellcolor[rgb]{ .757,  .855,  .941}\underline{97.82} & \cellcolor[rgb]{ .851,  .91,  .965}97.34 & \cellcolor[rgb]{ .827,  .898,  .957}\textbf{96.84} & 96.34 & \cellcolor[rgb]{ .741,  .843,  .933}\textbf{96.37} & \cellcolor[rgb]{ .808,  .886,  .953}\underline{96.44} & \cellcolor[rgb]{ .792,  .875,  .949}\underline{96.30} & \cellcolor[rgb]{ .757,  .855,  .937}\underline{96.16} & \cellcolor[rgb]{ .741,  .843,  .933}\textbf{96.82} & \cellcolor[rgb]{ .741,  .843,  .933}\textbf{97.46} & \cellcolor[rgb]{ .741,  .843,  .933}\textbf{98.18} & \cellcolor[rgb]{ .757,  .851,  .937}\underline{98.55} & \cellcolor[rgb]{ .741,  .843,  .933}\textbf{97.11}$_{\color{blue}(\pm 0.64)}$ \\
      \bottomrule
      \end{tabular}%
  }
  \label{tab:performance_comparison_cifar10_adaptformer_sade}%
\end{table*}%

\begin{table*}[h!]
  \centering
  \caption{Performance Comparison on CIFAR-100 When Fine-tuning with LoRA (\textbf{SADE's Setting})}
  \small
  \resizebox{\linewidth}{!}{
      \begin{tabular}{lcccccccccccccc}
      \toprule
      \multirow{1.5}[4]{*}{\textbf{Method}} & \multicolumn{6}{c}{\textbf{Forward-LT}}       & \textbf{Uni.} & \multicolumn{6}{c}{\textbf{Backward-LT}}      & \multirow{1.5}[4]{*}{\textbf{Mean}} \\
  \cmidrule(lr){2-7}\cmidrule(lr){8-8}\cmidrule(lr){9-14}          & \textbf{100} & \textbf{50} & \textbf{25} & \textbf{10} & \textbf{5} & \textbf{2} & \textbf{1} & \textbf{2} & \textbf{5} & \textbf{10} & \textbf{25} & \textbf{50} & \textbf{100} &  \\
      \midrule
      LIFT-LoRA & 83.72 & 83.43 & 82.81 & 82.28 & 81.51 & 80.73 & 79.98 & \cellcolor[rgb]{ .863,  .925,  .82}78.96 & \cellcolor[rgb]{ .902,  .949,  .871}78.02 & 77.45 & 75.71 & 75.22 & 74.66 & 79.58$_{\color{blue}(\pm 3.04)}$ \\
      SADE-LoRA & \cellcolor[rgb]{ .784,  .886,  .718}\textbf{89.07} & \cellcolor[rgb]{ .792,  .886,  .725}\textbf{87.93} & \cellcolor[rgb]{ .788,  .886,  .722}\textbf{86.84} & \cellcolor[rgb]{ .792,  .89,  .729}\underline{85.14} & \cellcolor[rgb]{ .776,  .878,  .706}\textbf{84.29} & \cellcolor[rgb]{ .812,  .898,  .753}\textbf{82.33} & \cellcolor[rgb]{ .918,  .957,  .894}80.64 & 75.43 & 76.01 & \cellcolor[rgb]{ .922,  .957,  .898}78.87 & \cellcolor[rgb]{ .847,  .918,  .796}80.65 & \cellcolor[rgb]{ .824,  .906,  .769}82.38 & \cellcolor[rgb]{ .816,  .898,  .757}83.48 & \cellcolor[rgb]{ .847,  .918,  .8}82.54$_{\color{blue}(\pm 4.06)}$ \\
      DirMixE-LoRA & \cellcolor[rgb]{ .796,  .89,  .729}\underline{88.83} & \cellcolor[rgb]{ .804,  .894,  .741}87.69 & \cellcolor[rgb]{ .8,  .894,  .737}86.60 & \cellcolor[rgb]{ .788,  .886,  .722}\textbf{85.22} & \cellcolor[rgb]{ .808,  .898,  .745}\underline{83.92} & \cellcolor[rgb]{ .831,  .91,  .776}82.17 & \cellcolor[rgb]{ .882,  .937,  .847}\underline{80.92} & \cellcolor[rgb]{ .855,  .922,  .808}\underline{79.15} & \cellcolor[rgb]{ .882,  .937,  .847}\underline{78.37} & \cellcolor[rgb]{ .843,  .918,  .792}\underline{80.26} & \cellcolor[rgb]{ .82,  .902,  .761}\underline{81.50} & \cellcolor[rgb]{ .808,  .898,  .745}\underline{83.02} & \cellcolor[rgb]{ .808,  .898,  .749}\underline{83.72} & \cellcolor[rgb]{ .816,  .902,  .757}\underline{83.18}$_{\color{blue}(\pm 3.11)}$ \\
      \midrule
      \textbf{DirMixE-LoRA-LSF} & \cellcolor[rgb]{ .804,  .894,  .741}88.60 & \cellcolor[rgb]{ .796,  .89,  .733}\underline{87.81} & \cellcolor[rgb]{ .796,  .89,  .729}\underline{86.71} & \cellcolor[rgb]{ .816,  .898,  .753}84.86 & \cellcolor[rgb]{ .824,  .906,  .769}83.72 & \cellcolor[rgb]{ .816,  .898,  .753}\underline{82.31} & \cellcolor[rgb]{ .831,  .91,  .776}\textbf{81.34} & \cellcolor[rgb]{ .776,  .878,  .706}\textbf{81.09} & \cellcolor[rgb]{ .776,  .878,  .706}\textbf{80.49} & \cellcolor[rgb]{ .776,  .878,  .706}\textbf{81.42} & \cellcolor[rgb]{ .776,  .878,  .706}\textbf{82.78} & \cellcolor[rgb]{ .776,  .878,  .706}\textbf{84.19} & \cellcolor[rgb]{ .776,  .878,  .706}\textbf{85.17} & \cellcolor[rgb]{ .776,  .878,  .706}\textbf{83.88}$_{\color{blue}(\pm 2.54)}$ \\
      \bottomrule
      \end{tabular}%
  }
  \label{tab:performance_comparison_cifar100_lora_sade}%
\end{table*}%

\begin{table*}[h!]
  \centering
  \caption{Performance Comparison on CIFAR-100 When Fine-tuning with AdaptFormer (\textbf{SADE's Setting})}
  \small
  \resizebox{\linewidth}{!}{
      \begin{tabular}{lcccccccccccccc}
      \toprule
      \multirow{1.5}[4]{*}{\textbf{Method}} & \multicolumn{6}{c}{\textbf{Forward-LT}}       & \textbf{Uni.} & \multicolumn{6}{c}{\textbf{Backward-LT}}      & \multirow{1.5}[4]{*}{\textbf{Mean}} \\
  \cmidrule(lr){2-7}\cmidrule(lr){8-8}\cmidrule(lr){9-14}          & \textbf{100} & \textbf{50} & \textbf{25} & \textbf{10} & \textbf{5} & \textbf{2} & \textbf{1} & \textbf{2} & \textbf{5} & \textbf{10} & \textbf{25} & \textbf{50} & \textbf{100} &  \\
      \midrule
      LIFT-AF & 84.89 & 84.75 & 84.20 & 83.46 & \cellcolor[rgb]{ .882,  .937,  .843}82.77 & \cellcolor[rgb]{ .776,  .878,  .706}\textbf{82.18} & \cellcolor[rgb]{ .776,  .878,  .706}\textbf{81.58} & \cellcolor[rgb]{ .776,  .878,  .706}\textbf{80.83} & \cellcolor[rgb]{ .776,  .878,  .706}\textbf{79.64} & \cellcolor[rgb]{ .894,  .945,  .859}78.97 & 78.11 & 77.63 & 77.19 & 81.25$_{\color{blue}(\pm 2.63)}$ \\
      SADE-AF & \cellcolor[rgb]{ .776,  .878,  .706}\textbf{88.46} & \cellcolor[rgb]{ .776,  .878,  .706}\textbf{87.37} & \cellcolor[rgb]{ .776,  .878,  .706}\textbf{86.23} & \cellcolor[rgb]{ .776,  .878,  .706}\textbf{84.70} & \cellcolor[rgb]{ .776,  .878,  .706}\textbf{83.19} & \cellcolor[rgb]{ .824,  .906,  .769}\underline{81.96} & \cellcolor[rgb]{ .973,  .988,  .965}80.05 & 76.25 & 75.49 & 77.94 & \cellcolor[rgb]{ .796,  .89,  .729}\underline{81.17} & \cellcolor[rgb]{ .808,  .898,  .749}\underline{82.17} & \cellcolor[rgb]{ .847,  .918,  .796}81.87 & \cellcolor[rgb]{ .91,  .953,  .878}81.87$_{\color{blue}(\pm 3.86)}$ \\
      DirMixE-AF & \cellcolor[rgb]{ .8,  .89,  .733}88.13 & \cellcolor[rgb]{ .804,  .894,  .741}87.09 & \cellcolor[rgb]{ .843,  .914,  .792}85.66 & \cellcolor[rgb]{ .925,  .961,  .902}83.88 & 82.28 & 81.09 & 79.83 & \cellcolor[rgb]{ .847,  .918,  .8}\underline{79.39} & \cellcolor[rgb]{ .8,  .89,  .737}\underline{79.24} & \cellcolor[rgb]{ .784,  .886,  .718}\underline{80.01} & \cellcolor[rgb]{ .843,  .918,  .796}80.45 & \cellcolor[rgb]{ .827,  .91,  .776}81.70 & \cellcolor[rgb]{ .831,  .91,  .776}\underline{82.36} & \cellcolor[rgb]{ .831,  .91,  .776}\underline{82.39}$_{\color{blue}(\pm 2.84)}$ \\
      \midrule
      \textbf{DirMixE-AF-LSF} & \cellcolor[rgb]{ .788,  .886,  .718}\underline{88.32} & \cellcolor[rgb]{ .784,  .882,  .718}\underline{87.29} & \cellcolor[rgb]{ .812,  .898,  .753}\underline{85.93} & \cellcolor[rgb]{ .89,  .941,  .855}\underline{84.08} & \cellcolor[rgb]{ .929,  .961,  .906}\underline{82.58} & \cellcolor[rgb]{ .949,  .973,  .933}81.34 & \cellcolor[rgb]{ .902,  .949,  .875}\underline{80.60} & \cellcolor[rgb]{ .878,  .933,  .839}78.80 & \cellcolor[rgb]{ .851,  .922,  .804}78.30 & \cellcolor[rgb]{ .776,  .878,  .706}\textbf{80.08} & \cellcolor[rgb]{ .776,  .878,  .706}\textbf{81.43} & \cellcolor[rgb]{ .776,  .878,  .706}\textbf{82.90} & \cellcolor[rgb]{ .776,  .878,  .706}\textbf{83.90} & \cellcolor[rgb]{ .776,  .878,  .706}\textbf{82.73}$_{\color{blue}(\pm 2.98)}$ \\
      \bottomrule
      \end{tabular}%
  }
  \label{tab:performance_comparison_cifar100_adaptformer_sade}%
\end{table*}%

\begin{table*}[h!]
  \centering
  \caption{Performance Comparison on ImageNet-LT When Fine-tuning with LoRA (\textbf{SADE's Setting})}
  \small
      \begin{tabular}{lcccccccccccc}
      \toprule
      \multirow{1.5}[4]{*}{\textbf{Method}} & \multicolumn{5}{c}{\textbf{Forward-LT}} & \textbf{Uni.} & \multicolumn{5}{c}{\textbf{Backward-LT}} & \multirow{1.5}[4]{*}{\textbf{Mean}} \\
  \cmidrule(lr){2-6}\cmidrule(lr){7-7}\cmidrule(lr){8-12}          & \textbf{50} & \textbf{25} & \textbf{10} & \textbf{5} & \textbf{2} & \textbf{1} & \textbf{2} & \textbf{5} & \textbf{10} & \textbf{25} & \textbf{50} &  \\
      \midrule
      LIFT-LoRA & 78.69 & 77.95 & 77.52 & \cellcolor[rgb]{ .996,  .957,  .933}76.86 & \cellcolor[rgb]{ .98,  .827,  .729}\underline{76.28} & \cellcolor[rgb]{ .973,  .796,  .678}\textbf{75.64} & \cellcolor[rgb]{ .984,  .875,  .8}\underline{75.01} & \cellcolor[rgb]{ .984,  .875,  .804}74.09 & \cellcolor[rgb]{ 1,  .984,  .973}73.52 & 72.92 & 71.64 & 75.47$_{\color{blue}(\pm 2.14)}$ \\
      SADE-LoRA & \cellcolor[rgb]{ .976,  .816,  .706}\underline{82.58} & \cellcolor[rgb]{ .976,  .804,  .694}\underline{81.11} & \cellcolor[rgb]{ .973,  .796,  .678}\underline{79.01} & \cellcolor[rgb]{ .973,  .796,  .678}\textbf{77.53} & \cellcolor[rgb]{ .973,  .796,  .678}\textbf{76.46} & \cellcolor[rgb]{ .984,  .871,  .792}75.50 & 74.57 & 72.45 & 73.31 & \cellcolor[rgb]{ .988,  .906,  .851}74.70 & \cellcolor[rgb]{ .98,  .843,  .753}75.71 & \cellcolor[rgb]{ .984,  .878,  .808}76.63$_{\color{blue}(\pm 3.02)}$ \\
      DirMixE-LoRA & \cellcolor[rgb]{ .976,  .816,  .71}82.52 & \cellcolor[rgb]{ .976,  .824,  .722}80.83 & \cellcolor[rgb]{ .984,  .863,  .78}78.54 & 76.68 & 75.29 & 75.24 & \cellcolor[rgb]{ .984,  .878,  .808}74.99 & \cellcolor[rgb]{ .98,  .851,  .765}\underline{74.39} & \cellcolor[rgb]{ .98,  .851,  .769}\underline{75.08} & \cellcolor[rgb]{ .98,  .831,  .737}\underline{76.09} & \cellcolor[rgb]{ .976,  .816,  .706}\underline{76.46} & \cellcolor[rgb]{ .98,  .851,  .761}\underline{76.92}$_{\color{blue}(\pm 2.51)}$ \\
      \midrule
      \textbf{DirMixE-LoRA-LSF} & \cellcolor[rgb]{ .973,  .796,  .678}\textbf{82.91} & \cellcolor[rgb]{ .973,  .796,  .678}\textbf{81.23} & \cellcolor[rgb]{ .976,  .8,  .682}\textbf{79.00} & \cellcolor[rgb]{ .984,  .871,  .792}\underline{77.23} & \cellcolor[rgb]{ .988,  .91,  .859}75.81 & \cellcolor[rgb]{ .976,  .808,  .698}\underline{75.62} & \cellcolor[rgb]{ .973,  .796,  .678}\textbf{75.27} & \cellcolor[rgb]{ .973,  .796,  .678}\textbf{75.09} & \cellcolor[rgb]{ .973,  .796,  .678}\textbf{75.73} & \cellcolor[rgb]{ .973,  .796,  .678}\textbf{76.75} & \cellcolor[rgb]{ .973,  .796,  .678}\textbf{76.89} & \cellcolor[rgb]{ .973,  .796,  .678}\textbf{77.41}$_{\color{blue}(\pm 2.46)}$ \\
      \bottomrule
      \end{tabular}%
  \label{tab:performance_comparison_imagenet_lora_sade}%
\end{table*}%

\begin{table*}[h!]
  \centering
  \caption{Performance Comparison on ImageNet-LT When Fine-tuning with AdaptFormer (\textbf{SADE's Setting})}
  \small
      \begin{tabular}{lcccccccccccc}
      \toprule
      \multirow{1.5}[4]{*}{\textbf{Method}} & \multicolumn{5}{c}{\textbf{Forward-LT}} & \textbf{Uni.} & \multicolumn{5}{c}{\textbf{Backward-LT}} & \multirow{1.5}[4]{*}{\textbf{Mean}} \\
  \cmidrule(lr){2-6}\cmidrule(lr){7-7}\cmidrule(lr){8-12}          & \textbf{50} & \textbf{25} & \textbf{10} & \textbf{5} & \textbf{2} & \textbf{1} & \textbf{2} & \textbf{5} & \textbf{10} & \textbf{25} & \textbf{50} &  \\
      \midrule
      LIFT-AF & 80.42 & 79.67 & 79.07 & \cellcolor[rgb]{ .996,  .965,  .941}\underline{78.54} & \cellcolor[rgb]{ .973,  .796,  .678}\textbf{77.82} & \cellcolor[rgb]{ .973,  .796,  .678}\textbf{77.28} & \cellcolor[rgb]{ .973,  .796,  .678}\textbf{76.65} & \cellcolor[rgb]{ .984,  .871,  .796}75.86 & 75.31 & 74.72 & 73.55 & 77.17$_{\color{blue}(\pm 1.81)}$ \\
      SADE-AF & \cellcolor[rgb]{ .976,  .812,  .698}83.64 & \cellcolor[rgb]{ .973,  .796,  .678}\textbf{82.29} & \cellcolor[rgb]{ .973,  .796,  .678}\textbf{80.31} & \cellcolor[rgb]{ .973,  .796,  .678}\textbf{79.14} & \cellcolor[rgb]{ .98,  .835,  .737}\underline{77.70} & 76.79 & 75.93 & 75.43 & \cellcolor[rgb]{ .996,  .945,  .914}75.64 & \cellcolor[rgb]{ .98,  .851,  .765}76.66 & \cellcolor[rgb]{ .976,  .808,  .694}77.11 & \cellcolor[rgb]{ .976,  .82,  .714}78.24$_{\color{blue}(\pm 2.75)}$ \\
      DirMixE-AF & \cellcolor[rgb]{ .976,  .808,  .694}\underline{83.70} & \cellcolor[rgb]{ .976,  .812,  .702}\underline{82.11} & \cellcolor[rgb]{ .98,  .851,  .765}\underline{79.99} & \cellcolor[rgb]{ 1,  .984,  .973}78.47 & 77.16 & \cellcolor[rgb]{ 1,  .98,  .969}76.84 & \cellcolor[rgb]{ .984,  .871,  .792}76.40 & \cellcolor[rgb]{ .976,  .824,  .718}\underline{76.02} & \cellcolor[rgb]{ .98,  .831,  .729}\underline{76.34} & \cellcolor[rgb]{ .976,  .812,  .706}\underline{77.16} & \cellcolor[rgb]{ .976,  .8,  .686}\underline{77.22} & \cellcolor[rgb]{ .976,  .808,  .698}\underline{78.31}$_{\color{blue}(\pm 2.53)}$ \\
      \midrule
      \textbf{DirMixE-AF-LSF} & \cellcolor[rgb]{ .973,  .796,  .678}\textbf{83.84} & \cellcolor[rgb]{ .976,  .82,  .714}82.02 & \cellcolor[rgb]{ .98,  .855,  .769}79.97 & 78.40 & \cellcolor[rgb]{ 1,  .984,  .973}77.22 & \cellcolor[rgb]{ .992,  .929,  .89}\underline{76.96} & \cellcolor[rgb]{ .984,  .859,  .78}\underline{76.43} & \cellcolor[rgb]{ .973,  .796,  .678}\textbf{76.10} & \cellcolor[rgb]{ .973,  .796,  .678}\textbf{76.53} & \cellcolor[rgb]{ .973,  .796,  .678}\textbf{77.36} & \cellcolor[rgb]{ .973,  .796,  .678}\textbf{77.28} & \cellcolor[rgb]{ .973,  .796,  .678}\textbf{78.37}$_{\color{blue}(\pm 2.50)}$ \\
      \bottomrule
      \end{tabular}%
  \label{tab:performance_comparison_imagenet_adaptformer_sade}%
\end{table*}%

\begin{table}[h!]
  \centering
  \caption{Performance Comparison on iNaturalist When Fine-tuning with AdaptFormer (\textbf{SADE's Setting})}
  \small
      \begin{tabular}{lcccccc}
      \toprule
      \multirow{1.5}[4]{*}{\textbf{Method}} & \multicolumn{2}{c}{\textbf{Forward-LT}} & \textbf{Uni.} & \multicolumn{2}{c}{\textbf{Backward-LT}} & \multirow{1.5}[4]{*}{\textbf{Mean}} \\
  \cmidrule(lr){2-3}\cmidrule(lr){4-4}\cmidrule(lr){5-6}          & \textbf{3} & \textbf{2} & \textbf{1} & \textbf{2} & \textbf{3} &  \\
      \midrule
      LIFT-AF & \cellcolor[rgb]{ .922,  .953,  .98}76.84 & \cellcolor[rgb]{ .988,  .992,  1}77.28 & \cellcolor[rgb]{ .965,  .976,  .992}\underline{77.86} & \cellcolor[rgb]{ .882,  .929,  .973}78.23 & 78.16 & \cellcolor[rgb]{ .98,  .988,  .996}77.67$_{\color{blue}(\pm 0.53)}$ \\
      SADE-AF & \cellcolor[rgb]{ .859,  .918,  .965}\underline{77.03} & \cellcolor[rgb]{ .957,  .976,  .992}\underline{77.37} & 77.77 & 77.70 & \cellcolor[rgb]{ .988,  .992,  1}78.20 & 77.61$_{\color{blue}(\pm 0.39)}$ \\
      DirMixE-AF & 76.59 & 77.24 & \cellcolor[rgb]{ .969,  .98,  .992}77.85 & \cellcolor[rgb]{ .824,  .894,  .957}\underline{78.49} & \cellcolor[rgb]{ .933,  .961,  .984}\underline{78.37} & \cellcolor[rgb]{ .965,  .98,  .992}\underline{77.71}$_{\color{blue}(\pm 0.71)}$ \\
      \midrule
      \textbf{DirMixE-AF-LSF} & \cellcolor[rgb]{ .741,  .843,  .933}\textbf{77.39} & \cellcolor[rgb]{ .741,  .843,  .933}\textbf{77.99} & \cellcolor[rgb]{ .741,  .843,  .933}\textbf{78.37} & \cellcolor[rgb]{ .741,  .843,  .933}\textbf{78.85} & \cellcolor[rgb]{ .741,  .843,  .933}\textbf{78.95} & \cellcolor[rgb]{ .741,  .843,  .933}\textbf{78.31}$_{\color{blue}(\pm 0.58)}$ \\
      \bottomrule
      \end{tabular}%
  \label{tab:performance_comparison_inaturalist_adaptformer_sade}%
\end{table}%

  \subsection{Experts Assignment}   \label{exp:weight}
  In this part, we validate the ability of the test-time self-supervised aggregation by visualizing the weight assignments for different label distributions in Fig.\ref{fig:weight10}-\ref{fig:weightI}. The forward and backward experts always tend to have a significant weight for their corresponding distributions. Uniform distributions tend to utilize all three experts. This is because tail and head classes are equally crucial for uniform distribution.

    
    

  \begin{figure*}[h!]  
    \centering
    \subfigure[Our's Setting]{
    
      \includegraphics[width=0.44\columnwidth]{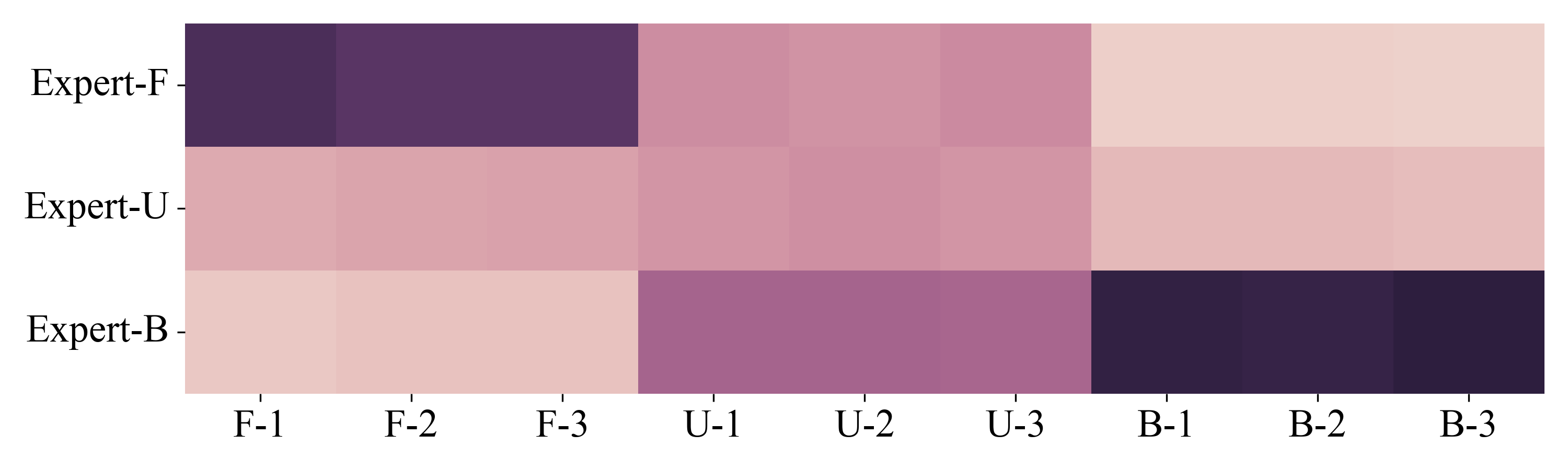} 
    }
    \subfigure[SADE's Setting]{
      \includegraphics[width=0.44\columnwidth]{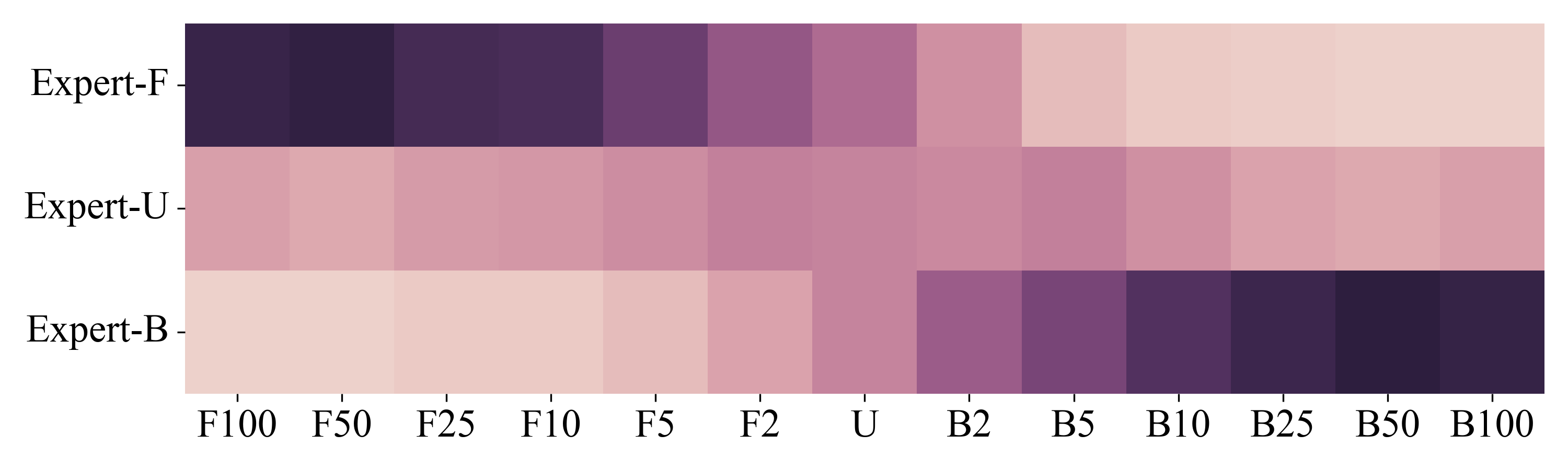} 
    }
    
    \caption{\label{fig:weight10} \textbf{Weight Assignment in the Self-supervised Aggregations on CIFAR-10.} \textbf{F,U,B} represent the forward, uniform and backward distributions. For the x-axis, \textbf{F-1,F-2,F-3} in \textbf{(a)} denote the three observed label distributions in the test data respectively. \textbf{F-2,F-5,$\cdots$,F-100} represents the corresponding imbalance ratio of the forward distribution under SADE's setting. The case for \textbf{the suffices of U and B} are similar. \textbf{Expert-F,U,B} represents the experts assigned to the forward, backward, uniform Dirichlet distribution, respectively.}
  \end{figure*}

  \begin{figure*}[h!]  
    \centering
    \subfigure[Our's Setting]{
    
      \includegraphics[width=0.44\columnwidth]{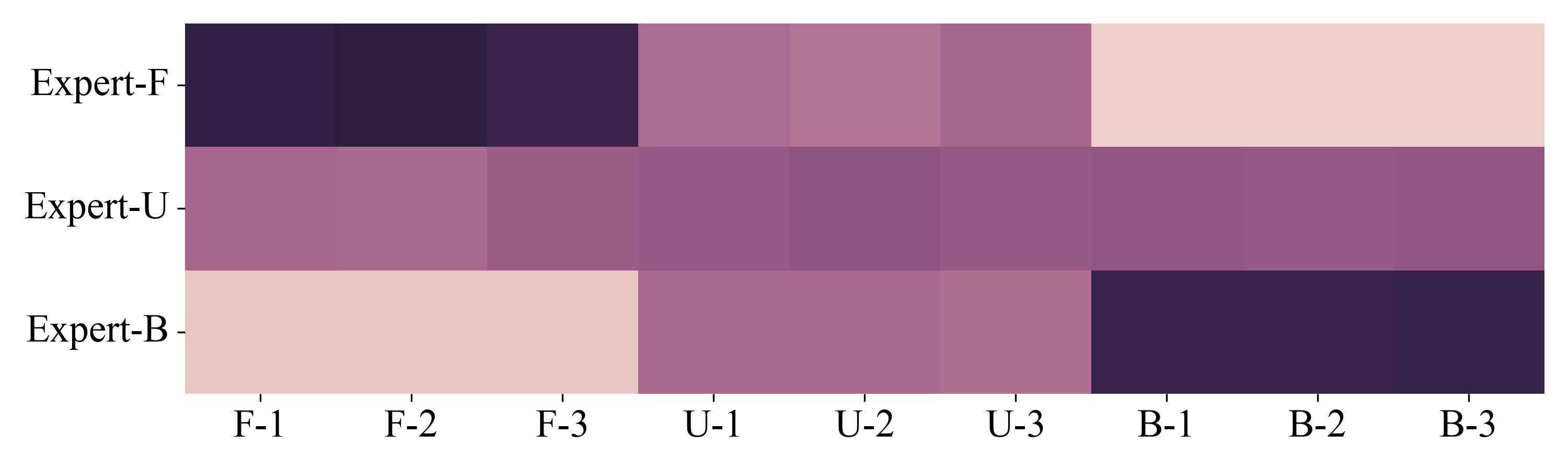} 
    }
    \subfigure[SADE's Setting]{
      \includegraphics[width=0.44\columnwidth]{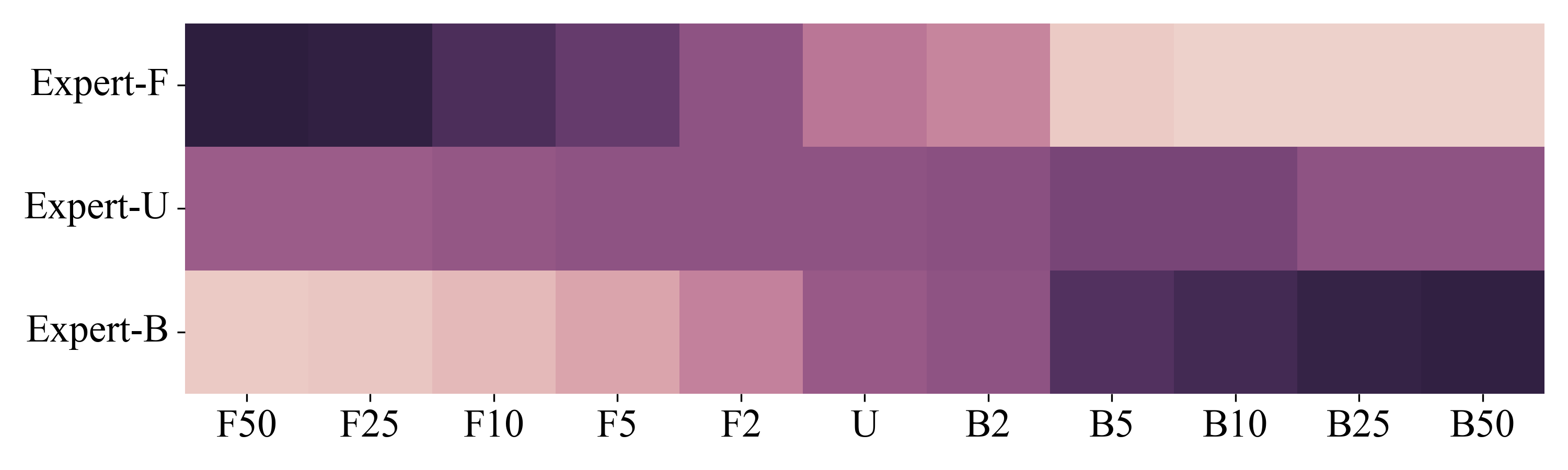} 
    }
    
    \caption{\label{fig:weightI} \textbf{Weight Assignment in the Self-supervised Aggregations on ImageNet.} \textbf{F,U,B} represent the forward, uniform and backward distributions. For the x-axis, \textbf{F-1,F-2,F-3} in \textbf{(a)} denote the three observed label distributions in the test data respectively. \textbf{F-2,F-5,$\cdots$,F-100} represents the corresponding imbalance ratio of the forward distribution under SADE's setting. The case for \textbf{the suffices of U and B} are similar. \textbf{Expert-F,U,B} represents the experts assigned to the forward, backward, uniform Dirichlet distribution, respectively.}
  \end{figure*}

  \subsection{Fine-grained Performance}  \label{exp:shot_acc}
  In addition to the overall accuracy of test datasets,  we also examine the effectiveness of DirMixE across many-shot, medium-shot, and few-shot classes. For CIFAR 100-LT and ImageNet-LT, classes are divided into many-shot ($> 100$), medium-shot ($20 \sim 100$), and few-shot ($< 20$) categories. In the case of CIFAR 10-LT, classes are split into many-shot ($> 1000$), medium-shot ($200 \sim 1000$), and few-shot ($< 200$) categories. All these statistics are based on the training label distribution. The results are illustrated in heat maps in Fig.\ref{fig:fine-cifar10} for CIFAR-10.

  \begin{figure}[h!]  
    \centering
    \subfigure[Our's Setting]{
    
      \includegraphics[width=0.88\columnwidth]{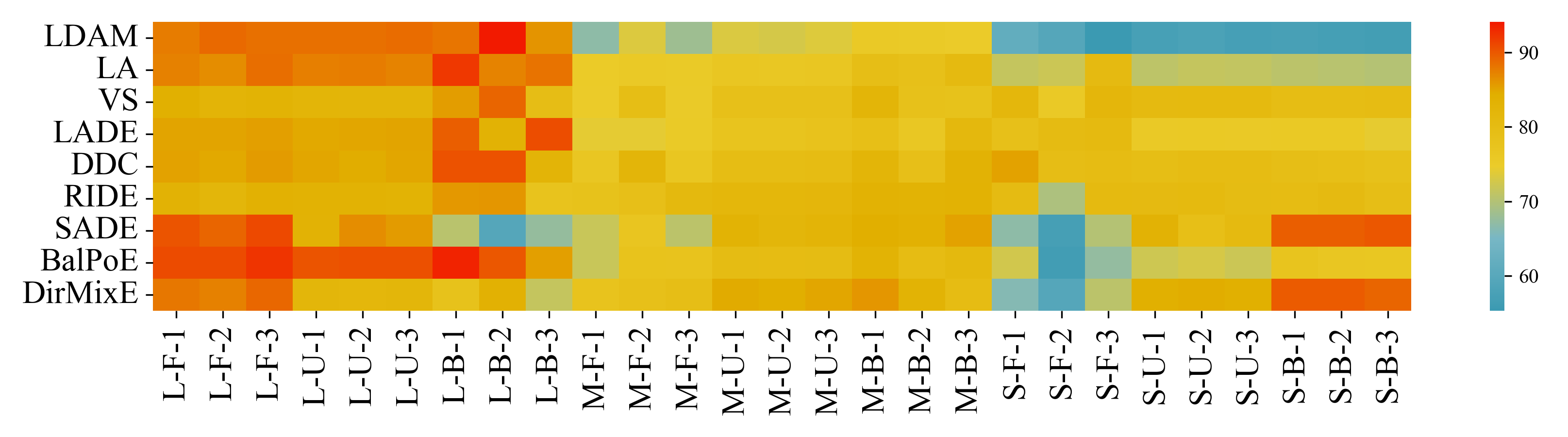} 
    }
    \subfigure[SADE's Setting]{
      \includegraphics[width=0.88\columnwidth]{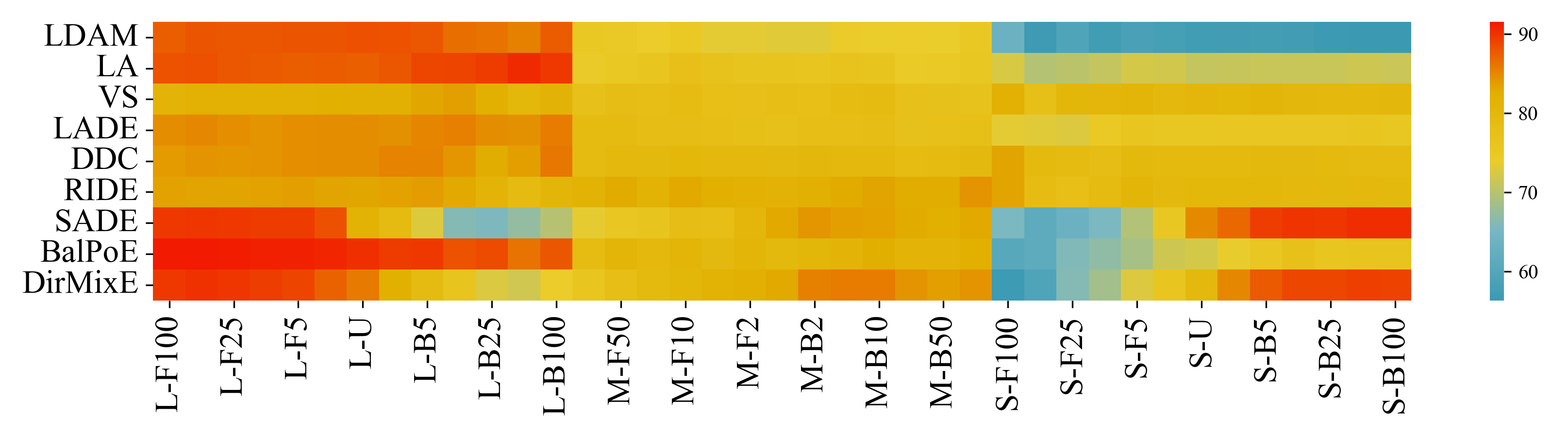} 
    }
    
    \caption{\label{fig:fine-cifar10} \textbf{Fine-grained Performance of DirMixE on CIFAR-10-LT} \textbf{F,U,B} represent the forward, uniform and backward distributions. For the x-axis, \textbf{F-1,F-2,F-3} in \textbf{(a)} denote the three observed label distributions in the test data respectively. \textbf{F2,F5,$\cdots$,F100} in \textbf{(b)} represents the corresponding imbalance ratio of the forward distribution under SADE's setting. The case for \textbf{the suffixes of U and B} are similar. \textbf{L, M, S} denote the many-shot, medium-shot and few-shot classes respectively. For example, \textbf{L-B100} indicates the performance of medium-shot classes on the backward distribution with an imbalance ratio of 100.}
  \end{figure}

  \subsection{Comparison with Stronger Baselines}\label{exp:strong}
  In this subsection, we further compare our method with two stronger basline: GPACO \cite{GPACO}, and BCL \cite{BCL}. The results are shown in Tab.\ref{tab:strong-100-sade}-\ref{tab:strong-IMG-ours}.

  \begin{table*}[h!]
    \centering
    \caption{Performance Comparison on CIFAR-100-LT (SADE's Setting)}
    \renewcommand\arraystretch{1.0}
    \resizebox{\linewidth}{!}{
      \begin{tabular}{lcccccccccccccc}
      \toprule
      \multicolumn{1}{r}{\multirow{2}[3]{*}{Method}} & \multicolumn{6}{c}{Foward-LT}                 & Uni.  & \multicolumn{6}{c}{Backward-LT}               & \multirow{2}[3]{*}{Mean} \\
  \cmidrule(lr){2-7}\cmidrule(lr){8-8}\cmidrule(lr){9-14}          & 100   & 50    & 25    & 10    & 5     & 2     & 1     & 2     & 5     & 10    & 25    & 50    & 100   &  \\
      \midrule
      GPaCo-ResNet32 & 64.01 & 62.59 & 61.06 & \cellcolor[rgb]{ .961,  .98,  .945}58.49 & \cellcolor[rgb]{ .776,  .878,  .706}56.55 & \cellcolor[rgb]{ .776,  .878,  .706}53.71 & \cellcolor[rgb]{ .776,  .878,  .706}51.53 & \cellcolor[rgb]{ .776,  .878,  .706}48.8 & \cellcolor[rgb]{ .776,  .878,  .706}45.36 & \cellcolor[rgb]{ .843,  .918,  .796}43.19 & \cellcolor[rgb]{ .878,  .933,  .839}40.02 & \cellcolor[rgb]{ .89,  .941,  .855}38.86 & \cellcolor[rgb]{ .902,  .949,  .871}37.07 & \cellcolor[rgb]{ .882,  .937,  .843}50.86$_{\color{blue}(\pm \phantom{0}9.02)}$ \\
      BCL-ResNet32 & \cellcolor[rgb]{ .933,  .965,  .91}65.37 & \cellcolor[rgb]{ .941,  .969,  .922}63.56 & \cellcolor[rgb]{ .961,  .98,  .949}61.43 & 58.26 & 56.00    & 52.43 & \cellcolor[rgb]{ .898,  .945,  .867}49.82 & 46.43 & 42.19 & 39.37 & 35.35 & 33.43 & 31.30  & 48.84$_{\color{blue}(\pm 11.33)}$ \\
      Ours-ResNet32 & \cellcolor[rgb]{ .776,  .878,  .706}68.32 & \cellcolor[rgb]{ .776,  .878,  .706}66.21 & \cellcolor[rgb]{ .776,  .878,  .706}63.09 & \cellcolor[rgb]{ .776,  .878,  .706}59.49 & \cellcolor[rgb]{ .859,  .925,  .816}56.35 & \cellcolor[rgb]{ .792,  .886,  .725}53.63 & 48.38 & 46.40  & \cellcolor[rgb]{ .8,  .894,  .737}45.05 & \cellcolor[rgb]{ .776,  .878,  .706}44.79 & \cellcolor[rgb]{ .776,  .878,  .706}43.71 & \cellcolor[rgb]{ .776,  .878,  .706}44.41 & \cellcolor[rgb]{ .776,  .878,  .706}44.25 & \cellcolor[rgb]{ .776,  .878,  .706}52.62$_{\color{blue}(\pm \phantom{0}8.75)}$ \\
      \bottomrule
      \end{tabular}%
    }
    \label{tab:strong-100-sade}%
  \end{table*}%

  \begin{table*}[h!]
    \centering
    \caption{Performance Comparison on CIFAR-100-LT (Ours Setting)}
    \renewcommand\arraystretch{1.0}
    \small
      \begin{tabular}{lcccccccccc}
      \toprule
      \multicolumn{1}{r}{\multirow{2}[4]{*}{Method}} & \multicolumn{3}{c}{Foward-LT} & \multicolumn{3}{c}{Uniform} & \multicolumn{3}{c}{Backward-LT} & \multirow{2}[4]{*}{Mean} \\
  \cmidrule(lr){2-4}\cmidrule(lr){5-7}\cmidrule(lr){8-10}          & 1     & 2     & 3     & 1     & 2     & 3     & 1     & 2     & 3     &  \\
      \midrule
      GPaCo-ResNet32 & 62.89  & 60.31  & 64.50  & \cellcolor[rgb]{ .776,  .878,  .706}50.18  & \cellcolor[rgb]{ .843,  .914,  .792}48.86  & \cellcolor[rgb]{ .776,  .878,  .706}49.58  & \cellcolor[rgb]{ .918,  .957,  .89}37.15  & \cellcolor[rgb]{ .91,  .953,  .882}37.09  & \cellcolor[rgb]{ .961,  .98,  .945}34.94  & \cellcolor[rgb]{ .941,  .969,  .922}49.50$_{\color{blue}(\pm 10.75)}$  \\
      BCL-ResNet32 & \cellcolor[rgb]{ .906,  .949,  .878}64.56  & \cellcolor[rgb]{ .91,  .953,  .878}62.85  & \cellcolor[rgb]{ .914,  .953,  .882}66.49  & 47.37  & 47.51  & \cellcolor[rgb]{ .831,  .91,  .78}48.27  & 32.59  & 30.17  & 32.74  & 48.06$_{\color{blue}(\pm 13.44)}$  \\
      Ours-ResNet32 & \cellcolor[rgb]{ .776,  .878,  .706}66.85  & \cellcolor[rgb]{ .776,  .878,  .706}66.40  & \cellcolor[rgb]{ .776,  .878,  .706}69.44  & \cellcolor[rgb]{ .953,  .976,  .937}47.99  & \cellcolor[rgb]{ .776,  .878,  .706}49.41  & 44.21  & \cellcolor[rgb]{ .776,  .878,  .706}44.41  & \cellcolor[rgb]{ .776,  .878,  .706}47.01  & \cellcolor[rgb]{ .776,  .878,  .706}44.35  & \cellcolor[rgb]{ .776,  .878,  .706}53.34$_{\color{blue}(\pm 10.22)}$  \\
      \bottomrule
      \end{tabular}%
    \label{tab:strong-100-ours}%
  \end{table*}%

  \begin{table*}[h!]
    \centering
    \caption{Performance Comparison on ImageNet-LT (SADE's Setting)}
    \renewcommand\arraystretch{1.0}
    \small
      \begin{tabular}{lcccccccccccc}
      \toprule
      \multicolumn{1}{r}{\multirow{2}[4]{*}{Method}} & \multicolumn{5}{c}{Foward-TL}         & Uni.  & \multicolumn{5}{c}{Backward-LT}       & \multirow{2}[4]{*}{Mean} \\
  \cmidrule(lr){2-6}\cmidrule(lr){7-7}\cmidrule(lr){8-12}          & 50    & 25    & 10    & 5     & 2     & 1     & 2     & 5     & 10    & 25    & 50    &  \\
      \midrule
      GPaCo-ResNeXt50 & 65.68 & 64.67 & 63.91 & \cellcolor[rgb]{ .996,  1,  1}62.31 & \cellcolor[rgb]{ .804,  .882,  .949}60.27 & \cellcolor[rgb]{ .741,  .843,  .933}58.69 & \cellcolor[rgb]{ .788,  .871,  .945}56.86 & \cellcolor[rgb]{ .812,  .886,  .953}54.3 & \cellcolor[rgb]{ .855,  .914,  .965}52.56 & \cellcolor[rgb]{ .898,  .937,  .976}50.49 & \cellcolor[rgb]{ .91,  .945,  .98}48.37 & \cellcolor[rgb]{ .918,  .949,  .98}58.01$_{\color{blue}(\pm 5.69)}$ \\
      BCL-ResNeXt50 & \cellcolor[rgb]{ .898,  .941,  .976}67.44 & \cellcolor[rgb]{ .882,  .929,  .973}66.41 & \cellcolor[rgb]{ .949,  .969,  .988}64.33 & 62.29 & 59.49 & 57.25 & 54.83 & 51.63 & 49.39 & 47.2  & 44.6  & 56.81$_{\color{blue}(\pm 7.55)}$ \\
      Ours-ResNeXt50 & \cellcolor[rgb]{ .741,  .843,  .933}70.09 & \cellcolor[rgb]{ .741,  .843,  .933}68.48 & \cellcolor[rgb]{ .741,  .843,  .933}65.93 & \cellcolor[rgb]{ .741,  .843,  .933}63.22 & \cellcolor[rgb]{ .741,  .843,  .933}60.5 & \cellcolor[rgb]{ .757,  .855,  .937}58.61 & \cellcolor[rgb]{ .741,  .843,  .933}57.27 & \cellcolor[rgb]{ .741,  .843,  .933}55.27 & \cellcolor[rgb]{ .741,  .843,  .933}55.04 & \cellcolor[rgb]{ .741,  .843,  .933}55.38 & \cellcolor[rgb]{ .741,  .843,  .933}55.33 & \cellcolor[rgb]{ .741,  .843,  .933}60.47$_{\color{blue}(\pm 5.37)}$ \\
      \bottomrule
      \end{tabular}%
    \label{tab:strong-IMG-sade}%
  \end{table*}%

  \begin{table*}[h!]
    \centering
    \caption{Performance Comparison on ImageNet-LT (Ours Setting)}
    \renewcommand\arraystretch{1.0}
    \small
      \begin{tabular}{lcccccccccc}
      \toprule
      \multicolumn{1}{r}{\multirow{2}[4]{*}{Method}} & \multicolumn{3}{c}{Foward-LT} & \multicolumn{3}{c}{Uniform} & \multicolumn{3}{c}{Backward-LT} & \multirow{2}[4]{*}{Mean} \\
  \cmidrule(lr){2-4}\cmidrule(lr){5-7}\cmidrule(lr){8-10}          & 1     & 2     & 3     & 1     & 2     & 3     & 1     & 2     & 3     &  \\
      \midrule
      GPaCo-ResNeXt50 & 65.74  & 64.69  & 64.92  & \cellcolor[rgb]{ .741,  .843,  .933}59.52  & \cellcolor[rgb]{ .745,  .847,  .937}58.83  & \cellcolor[rgb]{ .741,  .843,  .933}59.06  & \cellcolor[rgb]{ .898,  .941,  .976}49.43  & \cellcolor[rgb]{ .898,  .941,  .976}50.04  & \cellcolor[rgb]{ .902,  .941,  .976}49.08  & \cellcolor[rgb]{ .933,  .961,  .984}57.92$_{\color{blue}(\pm 6.44)}$  \\
      BCL-ResNeXt50 & \cellcolor[rgb]{ .941,  .965,  .988}66.75  & \cellcolor[rgb]{ .918,  .953,  .98}66.69  & \cellcolor[rgb]{ .933,  .961,  .984}66.31  & 57.41  & 57.52  & \cellcolor[rgb]{ .898,  .941,  .976}58.43  & 45.40  & 46.71  & 44.64  & 56.65$_{\color{blue}(\pm 8.63)}$  \\
      Ours-ResNext50 & \cellcolor[rgb]{ .741,  .843,  .933}70.13  & \cellcolor[rgb]{ .741,  .843,  .933}70.88  & \cellcolor[rgb]{ .741,  .843,  .933}70.29  & \cellcolor[rgb]{ .882,  .929,  .973}58.38  & \cellcolor[rgb]{ .741,  .843,  .933}58.85  & 58.02  & \cellcolor[rgb]{ .741,  .843,  .933}55.59  & \cellcolor[rgb]{ .741,  .843,  .933}55.09  & \cellcolor[rgb]{ .741,  .843,  .933}56.25  & \cellcolor[rgb]{ .741,  .843,  .933}61.50$_{\color{blue}(\pm 6.43)}$  \\
      \bottomrule
      \end{tabular}%
    \label{tab:strong-IMG-ours}%
  \end{table*}%

  \subsection{Ablation Study on the Effect of Semi-Variance}
  We show the performance with and without employing the semi-variance regularization term in Tab.\ref{tab:semi-cifar-100-ours}-Tab.\ref{tab:semi:cifar-10:sade}.The results consistently show that the proposed regularization induces a better average performance on all the datasets. 

  \begin{table}[h!]
    \centering
    \caption{Ablation of Semi-Variance on CIFAR-100 \textbf{(Ours Setting)}}
    \begin{tabular}{ccccccccccc}
        \toprule
        \multicolumn{1}{c}{\multirow{2}[4]{*}{\textbf{Semi-Var}}} & \multicolumn{3}{c}{\textbf{Uniform}} & \multicolumn{3}{c}{\textbf{Forward-LT}} & \multicolumn{3}{c}{\textbf{Backward-LT}} & \multirow{2}[4]{*}{\textbf{Mean}} \\
    \cmidrule{2-10}    \multicolumn{1}{c}{} & \textbf{1} & \textbf{2} & \textbf{3} & \textbf{1} & \textbf{2} & \textbf{3} & \textbf{1} & \textbf{2} & \multicolumn{1}{c}{\textbf{3}} &  \\
        \toprule
        w/o   & 47.31  & 48.76  & \cellcolor[rgb]{ .776,  .878,  .706}44.67  & 66.41  & 65.80  & 69.33  & 43.20  & 45.47  & 42.89  & 52.65  \\
        \midrule
        w     & \cellcolor[rgb]{ .776,  .878,  .706}47.99  & \cellcolor[rgb]{ .776,  .878,  .706}49.41  & 44.21  & \cellcolor[rgb]{ .776,  .878,  .706}66.85  & \cellcolor[rgb]{ .776,  .878,  .706}66.40  & \cellcolor[rgb]{ .776,  .878,  .706}69.44  & \cellcolor[rgb]{ .776,  .878,  .706}44.41  & \cellcolor[rgb]{ .776,  .878,  .706}47.01  & \cellcolor[rgb]{ .776,  .878,  .706}44.35  & \cellcolor[rgb]{ .776,  .878,  .706}53.34  \\
        \bottomrule
    \end{tabular}
    \label{tab:semi-cifar-100-ours}%
  \end{table}%

  \begin{table}[h!]
    \centering
    \caption{Ablation of Semi-Variance on CIFAR-100 \textbf{(SADE Setting)}}
    \renewcommand\arraystretch{1.0}
  \resizebox{\linewidth}{!}{
      \begin{tabular}{ccccccccccccccc}
      \toprule
      \multirow{2}[4]{*}{\textbf{Semi-Var}} & \multicolumn{6}{c}{\textbf{Forward-LT}}       & \textbf{Uni.} & \multicolumn{6}{c}{\textbf{Backward-LT}}      & \multirow{2}[4]{*}{\textbf{Mean}} \\
  \cmidrule{2-14}          & \textbf{100} & \textbf{50} & \textbf{25} & \textbf{10} & \textbf{5} & \textbf{2} & \textbf{1} & \textbf{2} & \textbf{5} & \textbf{10} & \textbf{25} & \textbf{50} & \textbf{100} &  \\
      \toprule
      w/o   & 68.18  & 65.73  & 62.69  & 58.85  & 55.60  & 52.11  & \cellcolor[rgb]{ .776,  .878,  .706}48.85  & \cellcolor[rgb]{ .776,  .878,  .706}46.89  & \cellcolor[rgb]{ .776,  .878,  .706}45.13  & 43.60  & 42.96  & 43.00  & 43.13  & 52.06  \\
      \toprule
      w     & \cellcolor[rgb]{ .776,  .878,  .706}68.32  & \cellcolor[rgb]{ .776,  .878,  .706}66.21  & \cellcolor[rgb]{ .776,  .878,  .706}63.09  & \cellcolor[rgb]{ .776,  .878,  .706}59.49  & \cellcolor[rgb]{ .776,  .878,  .706}56.35  & \cellcolor[rgb]{ .776,  .878,  .706}52.62  & 48.38  & 46.40  & 45.05  & \cellcolor[rgb]{ .776,  .878,  .706}44.79  & \cellcolor[rgb]{ .776,  .878,  .706}43.71  & \cellcolor[rgb]{ .776,  .878,  .706}44.41  & \cellcolor[rgb]{ .776,  .878,  .706}44.25  & \cellcolor[rgb]{ .776,  .878,  .706}52.54  \\
      \bottomrule
      \end{tabular}%
  }
    \label{tab:semi:cifar-100:sade}%
  \end{table}%
  
  \begin{table}[h!]
    \centering
    \caption{Ablation of Semi-Variance on CIFAR-10 \textbf{(Ours Setting)}}
      \begin{tabular}{ccccccccccc}
      \toprule
      \multirow{2}[4]{*}{\textbf{Semi-Var}} & \multicolumn{3}{c}{\textbf{Uniform}} & \multicolumn{3}{c}{\textbf{Forward-LT}} & \multicolumn{3}{c}{\textbf{Backward-LT}} & \multirow{2}[4]{*}{\textbf{Mean}} \\
  \cmidrule{2-10}          & \textbf{1} & \textbf{2} & \textbf{3} & \textbf{1} & \textbf{2} & \textbf{3} & \textbf{1} & \textbf{2} & \textbf{3} &  \\
      \midrule
      w/o   & 82.87  & \cellcolor[rgb]{ .741,  .843,  .933}82.99  & 83.36  & 89.86  & 89.58  & 90.80  & \cellcolor[rgb]{ .741,  .843,  .933}90.04  & \cellcolor[rgb]{ .741,  .843,  .933}89.15  & \cellcolor[rgb]{ .741,  .843,  .933}88.96  & 87.51  \\
      w     & \cellcolor[rgb]{ .741,  .843,  .933}83.24  & 82.98  & \cellcolor[rgb]{ .741,  .843,  .933}83.71  & \cellcolor[rgb]{ .741,  .843,  .933}90.46  & \cellcolor[rgb]{ .741,  .843,  .933}89.90  & \cellcolor[rgb]{ .741,  .843,  .933}91.30  & 89.39  & 88.78  & 88.40  & \cellcolor[rgb]{ .741,  .843,  .933}87.57  \\
      \bottomrule
      \end{tabular}%
    \label{tab:semi:CIFAR-10:ours}%
  \end{table}%

  \begin{table}[h!]
    \centering
    \caption{Ablation of Semi-Variance on CIFAR-10 \textbf{(SADE Setting)}}
    \renewcommand\arraystretch{1.0}
    \resizebox{\linewidth}{!}{
      \begin{tabular}{ccccccccccccccc}
      \toprule
      \multirow{2}[4]{*}{\textbf{Semi-Var}} & \multicolumn{6}{c}{\textbf{Forward-LT}}      & \textbf{Uni.} & \multicolumn{6}{c}{\textbf{Backward-LT}}      & \multirow{2}[4]{*}{\textbf{Mean}} \\
  \cmidrule{2-14}          & \textbf{100} & \textbf{50} & \textbf{25} & \textbf{10} & \textbf{5} & \textbf{2} & \textbf{1} & \textbf{2} & \textbf{5} & \textbf{10} & \textbf{25} & \textbf{50} & \textbf{100} &  \\
  \cmidrule{1-15}    w/o   & 90.72  & 90.02  & 88.91  & 87.07  & 85.65  & \cellcolor[rgb]{ .741,  .843,  .933}84.01  & \cellcolor[rgb]{ .741,  .843,  .933}83.36  & 83.80  & 84.53  & 84.82  & 85.48  & 86.37  & 86.84  & 86.28  \\
      w     & \cellcolor[rgb]{ .741,  .843,  .933}90.92  & \cellcolor[rgb]{ .741,  .843,  .933}90.16  & \cellcolor[rgb]{ .741,  .843,  .933}89.04  & \cellcolor[rgb]{ .741,  .843,  .933}87.10  & \cellcolor[rgb]{ .741,  .843,  .933}85.83  & 83.66  & 83.26  & \cellcolor[rgb]{ .741,  .843,  .933}84.16  & \cellcolor[rgb]{ .741,  .843,  .933}85.16  & \cellcolor[rgb]{ .741,  .843,  .933}86.17  & \cellcolor[rgb]{ .741,  .843,  .933}86.62  & \cellcolor[rgb]{ .741,  .843,  .933}87.55  & \cellcolor[rgb]{ .741,  .843,  .933}88.30  & \cellcolor[rgb]{ .741,  .843,  .933}86.76  \\
      \bottomrule
      \end{tabular}%
    }
    \label{tab:semi:cifar-10:sade}%
  \end{table}%

  \subsection{\newsec{Ablation Study on the Effect of Initialization Regime for LoRA-LSF}}\label{app:init}

  \newcont{

  We show the performance with and without employing the initialization regime for LoRA-LSF in Tab.\ref{tab:LoRA-init-cifar10}-Tab.\ref{tab:LoRA-init-IMG-sade}. The results consistently show that the proposed initialization regime induces a better average performance on all the datasets.

  }

  \begin{table}[h!]
    \centering
    \caption{Ablation of Initialization Regime for LoRA-LSF on CIFAR-10 \textbf{(Ours Setting)}}
      \begin{tabular}{ccccccccccc}
      \toprule
      \multirow{2}[4]{*}{\textbf{Initialization Regime}} & \multicolumn{3}{c}{\textbf{Forward-LT}} & \multicolumn{3}{c}{\textbf{Uniform}} & \multicolumn{3}{c}{\textbf{Backward-LT}} & \multirow{2}[4]{*}{\textbf{Mean}} \\
  \cmidrule{2-10}          & \textbf{1} & \textbf{2} & \textbf{3} & \textbf{1} & \textbf{2} & \textbf{3} & \textbf{1} & \textbf{2} & \textbf{3} &  \\
      \midrule
      w/o   & 98.16 & 98.24 & 98.41 & 96.53 & 96.40 & \cellcolor[rgb]{ .741,  .843,  .933}96.44 & 98.32 & 98.65 & 98.36 & 97.72 \\
      \midrule
      w     & \cellcolor[rgb]{ .741,  .843,  .933}98.40 & \cellcolor[rgb]{ .741,  .843,  .933}98.32 & \cellcolor[rgb]{ .741,  .843,  .933}98.49 & \cellcolor[rgb]{ .741,  .843,  .933}96.54 & \cellcolor[rgb]{ .741,  .843,  .933}96.42 & 96.39 & \cellcolor[rgb]{ .741,  .843,  .933}98.70 & \cellcolor[rgb]{ .741,  .843,  .933}98.81 & \cellcolor[rgb]{ .741,  .843,  .933}98.44 & \cellcolor[rgb]{ .741,  .843,  .933}97.83 \\
      \bottomrule
      \end{tabular}%
    \label{tab:LoRA-init-cifar10}%
  \end{table}%

  \begin{table}[h!]
    \centering
    \caption{Ablation of Initialization Regime for LoRA-LSF on CIFAR-10 \textbf{(SADE's Setting)}}
      \begin{tabular}{ccccccccccccccc}
      \toprule
      \multirow{2}[4]{*}{\textbf{Initialization Regime}} & \multicolumn{6}{c}{\textbf{Forward-LT}}       & \textbf{Uni.} & \multicolumn{6}{c}{\textbf{Backward-LT}}      & \multirow{2}[4]{*}{\textbf{Mean}} \\
  \cmidrule{2-14}          & \textbf{100} & \textbf{50} & \textbf{25} & \textbf{10} & \textbf{5} & \textbf{2} & \textbf{1} & \textbf{2} & \textbf{5} & \textbf{10} & \textbf{25} & \textbf{50} & \textbf{100} &  \\
      \midrule
      w/o   & 98.18 & 98.03 & 97.86 & \cellcolor[rgb]{ .741,  .843,  .933}97.43 & 96.93 & 96.89 & 96.76 & \cellcolor[rgb]{ .741,  .843,  .933}96.81 & 96.91 & \cellcolor[rgb]{ .741,  .843,  .933}97.28 & 97.80 & 98.39 & \cellcolor[rgb]{ .741,  .843,  .933}98.83 & 97.55 \\
      \midrule
      w     & \cellcolor[rgb]{ .741,  .843,  .933}98.35 & \cellcolor[rgb]{ .741,  .843,  .933}98.14 & \cellcolor[rgb]{ .741,  .843,  .933}97.89 & \cellcolor[rgb]{ .741,  .843,  .933}97.43 & \cellcolor[rgb]{ .741,  .843,  .933}97.07 & \cellcolor[rgb]{ .741,  .843,  .933}97.00 & \cellcolor[rgb]{ .741,  .843,  .933}96.93 & \cellcolor[rgb]{ .741,  .843,  .933}96.81 & \cellcolor[rgb]{ .741,  .843,  .933}97.01 & \cellcolor[rgb]{ .741,  .843,  .933}97.28 & \cellcolor[rgb]{ .741,  .843,  .933}97.86 & \cellcolor[rgb]{ .741,  .843,  .933}98.46 & 98.71 & \cellcolor[rgb]{ .741,  .843,  .933}97.61 \\
      \bottomrule
      \end{tabular}%
    \label{tab:LoRA-init-cifar10-sade}%
  \end{table}%

  \begin{table}[h!]
    \centering
    \caption{Ablation of Initialization Regime for LoRA-LSF on CIFAR-100 \textbf{(Ours Setting)}}
      \begin{tabular}{ccccccccccc}
      \toprule
      \multirow{2}[4]{*}{\textbf{Initialization Regime}} & \multicolumn{3}{c}{\textbf{Forward-LT}} & \multicolumn{3}{c}{\textbf{Uniform}} & \multicolumn{3}{c}{\textbf{Backward-LT}} & \multirow{2}[4]{*}{\textbf{Mean}} \\
  \cmidrule{2-10}          & \textbf{1} & \textbf{2} & \textbf{3} & \textbf{1} & \textbf{2} & \textbf{3} & \textbf{1} & \textbf{2} & \textbf{3} &  \\
      \midrule
      w/o   & 88.28 & \cellcolor[rgb]{ .776,  .878,  .706}86.68 & \cellcolor[rgb]{ .776,  .878,  .706}88.67 & \cellcolor[rgb]{ .776,  .878,  .706}81.14 & \cellcolor[rgb]{ .776,  .878,  .706}79.84 & 78.97 & 85.29 & \cellcolor[rgb]{ .776,  .878,  .706}83.25 & \cellcolor[rgb]{ .776,  .878,  .706}84.10 & 84.02 \\
      \midrule
      w     & \cellcolor[rgb]{ .776,  .878,  .706}88.90 & 86.61 & 87.97 & 80.53 & 79.08 & \cellcolor[rgb]{ .776,  .878,  .706}82.04 & \cellcolor[rgb]{ .776,  .878,  .706}86.22 & 82.99 & 83.05 & \cellcolor[rgb]{ .776,  .878,  .706}84.15 \\
      \bottomrule
      \end{tabular}%
    \label{tab:LoRA-init-cifar100}%
  \end{table}%

  \begin{table}[h!]
    \centering
    \caption{Ablation of Initialization Regime for LoRA-LSF on CIFAR-100 \textbf{(SADE's Setting)}}
      \begin{tabular}{ccccccccccccccc}
      \toprule
      \multirow{2}[4]{*}{\textbf{Initialization Regime}} & \multicolumn{6}{c}{\textbf{Forward-LT}}       & \textbf{Uni.} & \multicolumn{6}{c}{\textbf{Backward-LT}}      & \multirow{2}[4]{*}{\textbf{Mean}} \\
  \cmidrule{2-14}          & \textbf{100} & \textbf{50} & \textbf{25} & \textbf{10} & \textbf{5} & \textbf{2} & \textbf{1} & \textbf{2} & \textbf{5} & \textbf{10} & \textbf{25} & \textbf{50} & \textbf{100} &  \\
      \midrule
      w/o   & \cellcolor[rgb]{ .776,  .878,  .706}89.25 & \cellcolor[rgb]{ .776,  .878,  .706}88.21 & \cellcolor[rgb]{ .776,  .878,  .706}87.04 & \cellcolor[rgb]{ .776,  .878,  .706}85.35 & \cellcolor[rgb]{ .776,  .878,  .706}84.27 & \cellcolor[rgb]{ .776,  .878,  .706}82.61 & \cellcolor[rgb]{ .776,  .878,  .706}81.75 & 80.20 & 78.95 & 80.34 & 81.33 & 82.90 & 83.86 & 83.54 \\
      \midrule
      w     & 88.60 & 87.81 & 86.71 & 84.86 & 83.72 & 82.31 & 81.34 & \cellcolor[rgb]{ .776,  .878,  .706}81.09 & \cellcolor[rgb]{ .776,  .878,  .706}80.49 & \cellcolor[rgb]{ .776,  .878,  .706}81.42 & \cellcolor[rgb]{ .776,  .878,  .706}82.78 & \cellcolor[rgb]{ .776,  .878,  .706}84.19 & \cellcolor[rgb]{ .776,  .878,  .706}85.17 & \cellcolor[rgb]{ .776,  .878,  .706}83.88 \\
      \bottomrule
      \end{tabular}%
    \label{tab:LoRA-init-cifar100-sade}%
  \end{table}%

  \begin{table}[h!]
    \centering
    \caption{Ablation of Initialization Regime for LoRA-LSF on ImageNet \textbf{(Ours Setting)}}
      \begin{tabular}{ccccccccccc}
      \toprule
      \multirow{2}[4]{*}{\textbf{Initialization Regime}} & \multicolumn{3}{c}{\textbf{Forward-LT}} & \multicolumn{3}{c}{\textbf{Uniform}} & \multicolumn{3}{c}{\textbf{Backward-LT}} & \multirow{2}[4]{*}{\textbf{Mean}} \\
  \cmidrule{2-10}          & \textbf{1} & \textbf{2} & \textbf{3} & \textbf{1} & \textbf{2} & \textbf{3} & \textbf{1} & \textbf{2} & \textbf{3} &  \\
      \midrule
      w/o   & 82.32 & \cellcolor[rgb]{ .973,  .796,  .678}82.20 & 82.37 & 74.54 & 75.29 & 74.99 & 76.84 & 76.32 & 76.65 & 77.95 \\
      \midrule
      w     & \cellcolor[rgb]{ .973,  .796,  .678}82.87 & 82.03 & \cellcolor[rgb]{ .973,  .796,  .678}82.51 & \cellcolor[rgb]{ .973,  .796,  .678}74.96 & \cellcolor[rgb]{ .973,  .796,  .678}75.66 & \cellcolor[rgb]{ .973,  .796,  .678}75.08 & \cellcolor[rgb]{ .973,  .796,  .678}77.36 & \cellcolor[rgb]{ .973,  .796,  .678}76.61 & \cellcolor[rgb]{ .973,  .796,  .678}77.06 & \cellcolor[rgb]{ .973,  .796,  .678}78.24 \\
      \bottomrule
      \end{tabular}%
    \label{tab:LoRA-init-IMG}%
  \end{table}%

  \begin{table}[h!]
    \centering
    \caption{Ablation of Initialization Regime for LoRA-LSF on ImageNet \textbf{(SADE's Setting)}}
      \begin{tabular}{ccccccccccccc}
      \toprule
      \multirow{2}[4]{*}{\textbf{Initialization Regime}} & \multicolumn{5}{c}{\textbf{Forward-LT}} & \textbf{Uni.} & \multicolumn{5}{c}{\textbf{Backward-LT}} & \multirow{2}[4]{*}{\textbf{Mean}} \\
  \cmidrule{2-12}          & \textbf{50} & \textbf{25} & \textbf{10} & \textbf{5} & \textbf{2} & \textbf{1} & \textbf{2} & \textbf{5} & \textbf{10} & \textbf{25} & \textbf{50} &  \\
      \midrule
      w/o   & 82.71 & 80.99 & 78.78 & 76.98 & 75.69 & 75.45 & 75.16 & 74.84 & 75.26 & 76.23 & 76.24 & 77.12 \\
      \midrule
      w     & \cellcolor[rgb]{ .973,  .796,  .678}82.91 & \cellcolor[rgb]{ .973,  .796,  .678}81.23 & \cellcolor[rgb]{ .973,  .796,  .678}79.00 & \cellcolor[rgb]{ .973,  .796,  .678}77.23 & \cellcolor[rgb]{ .973,  .796,  .678}75.81 & \cellcolor[rgb]{ .973,  .796,  .678}75.62 & \cellcolor[rgb]{ .973,  .796,  .678}75.27 & \cellcolor[rgb]{ .973,  .796,  .678}75.09 & \cellcolor[rgb]{ .973,  .796,  .678}75.73 & \cellcolor[rgb]{ .973,  .796,  .678}76.75 & \cellcolor[rgb]{ .973,  .796,  .678}76.89 & \cellcolor[rgb]{ .973,  .796,  .678}77.41 \\
      \bottomrule
      \end{tabular}%
    \label{tab:LoRA-init-IMG-sade}%
  \end{table}%

  \subsection{Ablation Study on the Effect of the Normalization Factor $S$}

  As shown in Tab.\ref{tab:S-100-SADE}-Tab.\ref{tab:IMG-Ours}, our investigation into $S$, the sum of the Dirichlet distribution hyperparameter $\alpha$, shows performance improves with higher $S$ values before a slight decline. A larger $S$ reduces the randomness in the sampled test-label distributions. When $S$ is too small, local variation overshadows global trends, leading to misassigned experts and performance loss. Conversely, excessively large $S$ values minimize local variation, making the meta-distribution collapse to fix distributions. Thus, we find a moderate value is optimal.

  \begin{table*}[h!]
    \centering
    \caption{Ablation of $S$ on CIFAR-100-LT (SADE's Setting)}
    \renewcommand\arraystretch{1.0}
    \resizebox{\linewidth}{!}{
      \begin{tabular}{ccccccccccccccc}
      \toprule
      \multirow{2}[4]{*}{$S$} & \multicolumn{6}{c}{Foward-LT}                 & Uni.  & \multicolumn{6}{c}{Backward-LT}               & \multirow{2}[4]{*}{Mean} \\
  \cmidrule(lr){2-7}\cmidrule(lr){8-8}\cmidrule(lr){9-14}          & 100   & 50    & 25    & 10    & 5     & 2     & 1     & 2     & 5     & 10    & 25    & 50    & 100   &  \\
      \midrule
      100   & 55.33  & 52.29  & 49.22  & 44.07  & 39.31  & 33.31  & 28.27  & 23.53  & 17.60  & 14.50  & 9.64  & 7.32  & 5.40  & 29.21$_{\color{blue}(\pm 17.03)}$  \\
      500   & \cellcolor[rgb]{ .973,  .984,  .996}56.92  & \cellcolor[rgb]{ .961,  .976,  .992}54.51  & \cellcolor[rgb]{ .957,  .973,  .988}51.89  & \cellcolor[rgb]{ .937,  .965,  .984}47.99  & \cellcolor[rgb]{ .925,  .957,  .984}44.34  & \cellcolor[rgb]{ .922,  .953,  .98}39.17  & \cellcolor[rgb]{ .933,  .961,  .984}34.18  & \cellcolor[rgb]{ .937,  .965,  .984}29.31  & \cellcolor[rgb]{ .925,  .957,  .98}25.93  & \cellcolor[rgb]{ .953,  .973,  .988}20.51  & \cellcolor[rgb]{ .949,  .969,  .988}16.85  & \cellcolor[rgb]{ .949,  .969,  .988}15.12  & \cellcolor[rgb]{ .91,  .945,  .976}19.43  & \cellcolor[rgb]{ .937,  .961,  .984}35.09$_{\color{blue}(\pm 14.44)}$  \\
      1000  & \cellcolor[rgb]{ .741,  .843,  .933}69.73  & \cellcolor[rgb]{ .741,  .843,  .933}66.81  & \cellcolor[rgb]{ .741,  .843,  .933}64.17  & \cellcolor[rgb]{ .753,  .851,  .937}59.26  & \cellcolor[rgb]{ .757,  .855,  .937}55.64  & \cellcolor[rgb]{ .78,  .867,  .945}49.87  & \cellcolor[rgb]{ .8,  .878,  .949}45.47  & \cellcolor[rgb]{ .8,  .882,  .949}41.68  & \cellcolor[rgb]{ .792,  .875,  .949}40.39  & \cellcolor[rgb]{ .804,  .882,  .949}37.90  & \cellcolor[rgb]{ .8,  .878,  .949}36.33  & \cellcolor[rgb]{ .82,  .894,  .957}33.27  & \cellcolor[rgb]{ .827,  .894,  .957}31.72  & \cellcolor[rgb]{ .788,  .871,  .945}48.63$_{\color{blue}(\pm 12.70)}$  \\
      5000  & \cellcolor[rgb]{ .765,  .859,  .941}68.51  & \cellcolor[rgb]{ .757,  .855,  .941}65.93  & \cellcolor[rgb]{ .753,  .851,  .937}63.53  & \cellcolor[rgb]{ .749,  .847,  .937}59.52  & \cellcolor[rgb]{ .749,  .851,  .937}56.06  & \cellcolor[rgb]{ .753,  .851,  .937}51.90  & \cellcolor[rgb]{ .741,  .843,  .933}50.37  & \cellcolor[rgb]{ .741,  .843,  .933}46.99  & \cellcolor[rgb]{ .741,  .843,  .933}45.52  & \cellcolor[rgb]{ .741,  .843,  .933}45.05  & \cellcolor[rgb]{ .741,  .843,  .933}43.84  & \cellcolor[rgb]{ .749,  .847,  .937}43.72  & \cellcolor[rgb]{ .749,  .847,  .937}43.64  & \cellcolor[rgb]{ .741,  .843,  .933}52.66$_{\color{blue}(\pm \phantom{0}8.73)}$  \\
      10000 & \cellcolor[rgb]{ .769,  .859,  .941}68.32  & \cellcolor[rgb]{ .753,  .851,  .937}66.21  & \cellcolor[rgb]{ .761,  .855,  .941}63.09  & \cellcolor[rgb]{ .749,  .851,  .937}59.49  & \cellcolor[rgb]{ .745,  .847,  .937}56.35  & \cellcolor[rgb]{ .741,  .843,  .933}52.62  & \cellcolor[rgb]{ .765,  .859,  .941}48.38  & \cellcolor[rgb]{ .749,  .851,  .937}46.40  & \cellcolor[rgb]{ .749,  .847,  .937}45.05  & \cellcolor[rgb]{ .745,  .847,  .937}44.79  & \cellcolor[rgb]{ .745,  .847,  .937}43.71  & \cellcolor[rgb]{ .741,  .843,  .933}44.41  & \cellcolor[rgb]{ .741,  .843,  .933}44.25  & \cellcolor[rgb]{ .745,  .847,  .937}52.54$_{\color{blue}(\pm \phantom{0}8.74)}$  \\
      50000 & \cellcolor[rgb]{ .765,  .859,  .941}68.51  & \cellcolor[rgb]{ .753,  .851,  .937}66.21  & \cellcolor[rgb]{ .757,  .851,  .937}63.46  & \cellcolor[rgb]{ .749,  .847,  .937}59.57  & \cellcolor[rgb]{ .753,  .851,  .937}55.86  & \cellcolor[rgb]{ .749,  .847,  .937}52.14  & \cellcolor[rgb]{ .757,  .851,  .937}49.35  & \cellcolor[rgb]{ .745,  .847,  .937}46.71  & \cellcolor[rgb]{ .753,  .851,  .937}44.44  & \cellcolor[rgb]{ .757,  .855,  .937}43.42  & \cellcolor[rgb]{ .757,  .851,  .937}42.25  & \cellcolor[rgb]{ .757,  .851,  .937}42.60  & \cellcolor[rgb]{ .757,  .855,  .937}42.14  & \cellcolor[rgb]{ .749,  .851,  .937}52.05$_{\color{blue}(\pm \phantom{0}9.30)}$  \\
      100000 & \cellcolor[rgb]{ .78,  .871,  .945}67.57  & \cellcolor[rgb]{ .773,  .863,  .941}65.12  & \cellcolor[rgb]{ .769,  .863,  .941}62.65  & \cellcolor[rgb]{ .749,  .847,  .937}59.65  & \cellcolor[rgb]{ .741,  .843,  .933}56.53  & \cellcolor[rgb]{ .749,  .847,  .937}52.14  & \cellcolor[rgb]{ .753,  .851,  .937}49.42  & \cellcolor[rgb]{ .753,  .851,  .937}46.27  & \cellcolor[rgb]{ .749,  .851,  .937}44.77  & \cellcolor[rgb]{ .749,  .847,  .937}44.32  & \cellcolor[rgb]{ .749,  .851,  .937}42.83  & \cellcolor[rgb]{ .753,  .851,  .937}42.92  & \cellcolor[rgb]{ .753,  .851,  .937}42.61  & \cellcolor[rgb]{ .749,  .851,  .937}52.06$_{\color{blue}(\pm \phantom{0}8.82)}$  \\
      500000 & \cellcolor[rgb]{ .776,  .867,  .945}67.81  & \cellcolor[rgb]{ .757,  .855,  .937}66.05  & \cellcolor[rgb]{ .757,  .851,  .937}63.46  & \cellcolor[rgb]{ .741,  .843,  .933}59.91  & \cellcolor[rgb]{ .741,  .843,  .933}56.53  & \cellcolor[rgb]{ .749,  .847,  .937}52.14  & \cellcolor[rgb]{ .757,  .855,  .937}49.12  & \cellcolor[rgb]{ .761,  .855,  .941}45.48  & \cellcolor[rgb]{ .753,  .851,  .937}44.28  & \cellcolor[rgb]{ .753,  .851,  .937}43.83  & \cellcolor[rgb]{ .749,  .851,  .937}42.93  & \cellcolor[rgb]{ .753,  .851,  .937}43.28  & \cellcolor[rgb]{ .745,  .847,  .937}43.83  & \cellcolor[rgb]{ .749,  .847,  .937}52.20$_{\color{blue}(\pm \phantom{0}9.04)}$  \\
      1000000 & \cellcolor[rgb]{ .784,  .871,  .945}67.43  & \cellcolor[rgb]{ .765,  .859,  .941}65.57  & \cellcolor[rgb]{ .773,  .863,  .941}62.52  & \cellcolor[rgb]{ .761,  .855,  .941}58.75  & \cellcolor[rgb]{ .757,  .855,  .937}55.66  & \cellcolor[rgb]{ .753,  .851,  .937}51.76  & \cellcolor[rgb]{ .769,  .863,  .941}48.03  & \cellcolor[rgb]{ .765,  .859,  .941}45.13  & \cellcolor[rgb]{ .769,  .859,  .941}42.98  & \cellcolor[rgb]{ .765,  .859,  .941}42.31  & \cellcolor[rgb]{ .765,  .859,  .941}41.14  & \cellcolor[rgb]{ .765,  .859,  .941}41.27  & \cellcolor[rgb]{ .761,  .855,  .941}41.34  & \cellcolor[rgb]{ .761,  .855,  .941}51.07$_{\color{blue}(\pm \phantom{0}9.47)}$  \\
      \bottomrule
      \end{tabular}%
    }
    \label{tab:S-100-SADE}%
  \end{table*}%

  \begin{table*}[h!]
    \centering
    \caption{Ablation of $S$ on CIFAR-100-LT (Ours Setting)}
    \renewcommand\arraystretch{1.0}
    \small
      \begin{tabular}{ccccccccccc}
      \toprule
      \multirow{2}[4]{*}{$S$} & \multicolumn{3}{c}{Foward-LT} & \multicolumn{3}{c}{Uniform} & \multicolumn{3}{c}{Backward} & \multirow{2}[4]{*}{Mean} \\
  \cmidrule(lr){2-4}\cmidrule(lr){5-7}\cmidrule(lr){8-10}          & 1     & 2     & 3     & 1     & 2     & 3     & 1     & 2     & 3     &  \\
      \midrule
      100   & 54.41  & 53.35  & 58.11  & 29.24  & 26.00  & 28.70  & 9.12  & 5.81  & 9.41  & 30.46$_{\color{blue}(\pm 19.42)}$  \\
      500   & \cellcolor[rgb]{ .98,  .988,  .996}55.47  & \cellcolor[rgb]{ .98,  .988,  .996}54.42  & \cellcolor[rgb]{ .976,  .984,  .996}59.24  & \cellcolor[rgb]{ .976,  .984,  .996}31.50  & \cellcolor[rgb]{ .906,  .945,  .976}34.43  & \cellcolor[rgb]{ .914,  .949,  .98}35.69  & \cellcolor[rgb]{ .855,  .914,  .965}28.77  & \cellcolor[rgb]{ .929,  .957,  .984}16.75  & \cellcolor[rgb]{ .894,  .937,  .973}23.64  & \cellcolor[rgb]{ .918,  .949,  .98}37.77$_{\color{blue}(\pm 14.25)}$  \\
      1000  & \cellcolor[rgb]{ .741,  .843,  .933}67.47  & \cellcolor[rgb]{ .741,  .843,  .933}66.47  & \cellcolor[rgb]{ .741,  .843,  .933}69.23  & \cellcolor[rgb]{ .808,  .886,  .953}45.18  & \cellcolor[rgb]{ .812,  .886,  .953}42.86  & \cellcolor[rgb]{ .757,  .855,  .937}48.20  & \cellcolor[rgb]{ .839,  .902,  .961}31.01  & \cellcolor[rgb]{ .863,  .918,  .965}26.50  & \cellcolor[rgb]{ .824,  .894,  .957}32.53  & \cellcolor[rgb]{ .804,  .882,  .949}47.72$_{\color{blue}(\pm 15.62)}$  \\
      5000  & \cellcolor[rgb]{ .757,  .855,  .937}66.80  & \cellcolor[rgb]{ .745,  .847,  .937}66.40  & \cellcolor[rgb]{ .749,  .847,  .937}68.96  & \cellcolor[rgb]{ .765,  .859,  .941}48.78  & \cellcolor[rgb]{ .753,  .851,  .937}47.89  & \cellcolor[rgb]{ .741,  .843,  .933}49.35  & \cellcolor[rgb]{ .753,  .851,  .937}42.64  & \cellcolor[rgb]{ .761,  .855,  .941}41.88  & \cellcolor[rgb]{ .749,  .847,  .937}42.47  & \cellcolor[rgb]{ .741,  .843,  .933}52.80$_{\color{blue}(\pm 10.66)}$  \\
      10000 & \cellcolor[rgb]{ .773,  .863,  .941}65.96  & \cellcolor[rgb]{ .757,  .855,  .937}65.80  & \cellcolor[rgb]{ .769,  .863,  .941}68.10  & \cellcolor[rgb]{ .741,  .843,  .933}50.55  & \cellcolor[rgb]{ .753,  .851,  .937}48.00  & \cellcolor[rgb]{ .773,  .863,  .941}47.12  & \cellcolor[rgb]{ .749,  .851,  .937}42.83  & \cellcolor[rgb]{ .741,  .843,  .933}44.53  & \cellcolor[rgb]{ .753,  .851,  .937}41.84  & \cellcolor[rgb]{ .745,  .847,  .937}52.75$_{\color{blue}(\pm 10.14)}$  \\
      50000 & \cellcolor[rgb]{ .757,  .851,  .937}66.85  & \cellcolor[rgb]{ .765,  .859,  .941}65.33  & \cellcolor[rgb]{ .753,  .851,  .937}68.74  & \cellcolor[rgb]{ .784,  .871,  .945}47.25  & \cellcolor[rgb]{ .753,  .851,  .937}48.00  & \cellcolor[rgb]{ .761,  .855,  .941}47.89  & \cellcolor[rgb]{ .769,  .863,  .941}40.22  & \cellcolor[rgb]{ .765,  .859,  .941}41.28  & \cellcolor[rgb]{ .765,  .859,  .941}40.48  & \cellcolor[rgb]{ .757,  .851,  .937}51.78$_{\color{blue}(\pm 11.15)}$  \\
      100000 & \cellcolor[rgb]{ .788,  .871,  .945}65.23  & \cellcolor[rgb]{ .745,  .847,  .937}66.33  & \cellcolor[rgb]{ .78,  .867,  .945}67.62  & \cellcolor[rgb]{ .78,  .867,  .945}47.37  & \cellcolor[rgb]{ .741,  .843,  .933}48.86  & \cellcolor[rgb]{ .761,  .855,  .941}48.04  & \cellcolor[rgb]{ .741,  .843,  .933}43.85  & \cellcolor[rgb]{ .749,  .847,  .937}43.68  & \cellcolor[rgb]{ .745,  .847,  .937}42.68  & \cellcolor[rgb]{ .745,  .847,  .937}52.63$_{\color{blue}(\pm \phantom{0}9.95)}$  \\
      500000 & \cellcolor[rgb]{ .761,  .855,  .941}66.63  & \cellcolor[rgb]{ .749,  .847,  .937}66.20  & \cellcolor[rgb]{ .769,  .863,  .941}68.10  & \cellcolor[rgb]{ .769,  .863,  .941}48.35  & \cellcolor[rgb]{ .753,  .851,  .937}48.16  & \cellcolor[rgb]{ .761,  .855,  .941}47.81  & \cellcolor[rgb]{ .757,  .851,  .937}42.27  & \cellcolor[rgb]{ .753,  .851,  .937}43.25  & \cellcolor[rgb]{ .741,  .843,  .933}43.10  & \cellcolor[rgb]{ .745,  .847,  .937}52.65$_{\color{blue}(\pm 10.37)}$  \\
      1000000 & \cellcolor[rgb]{ .776,  .863,  .945}65.85  & \cellcolor[rgb]{ .796,  .878,  .949}63.72  & \cellcolor[rgb]{ .788,  .871,  .945}67.29  & \cellcolor[rgb]{ .792,  .875,  .949}46.52  & \cellcolor[rgb]{ .757,  .851,  .937}47.78  & \cellcolor[rgb]{ .788,  .871,  .945}45.82  & \cellcolor[rgb]{ .757,  .855,  .937}41.99  & \cellcolor[rgb]{ .765,  .859,  .941}41.03  & \cellcolor[rgb]{ .776,  .867,  .945}38.70  & \cellcolor[rgb]{ .765,  .859,  .941}50.97$_{\color{blue}(\pm 10.73)}$  \\
      \bottomrule
      \end{tabular}%
    \label{tab:S-100-Ours}%
  \end{table*}%

  \begin{table*}[h!]
    \centering
    \caption{Ablation of $S$ on ImageNet-LT (SADE's Setting)}
    \renewcommand\arraystretch{1.0}
    \small
      \begin{tabular}{ccccccccccccc}
      \toprule
      \multirow{2}[4]{*}{$S$} & \multicolumn{5}{c}{Foward-LT}         & Uni.  & \multicolumn{5}{c}{Backward-LT}       & \multirow{2}[4]{*}{Mean} \\
  \cmidrule(lr){2-6}\cmidrule(lr){7-7}\cmidrule(lr){8-12}          & 50    & 25    & 10    & 5     & 2     & 1     & 2     & 5     & 10    & 25    & 50    &  \\
      \midrule
      100   & \cellcolor[rgb]{ .98,  .831,  .733}70.24  & \cellcolor[rgb]{ .988,  .89,  .824}68.02  & \cellcolor[rgb]{ .992,  .933,  .894}64.64  & \cellcolor[rgb]{ .988,  .91,  .855}61.60  & \cellcolor[rgb]{ .996,  .953,  .925}57.17  & \cellcolor[rgb]{ .992,  .925,  .878}54.06  & \cellcolor[rgb]{ .992,  .933,  .898}49.91  & \cellcolor[rgb]{ .992,  .929,  .886}47.10  & \cellcolor[rgb]{ .992,  .941,  .906}42.88  & \cellcolor[rgb]{ .988,  .902,  .843}43.14  & \cellcolor[rgb]{ .988,  .902,  .847}41.13  & \cellcolor[rgb]{ .992,  .918,  .871}54.54$_{\color{blue}(\pm 10.04)}$  \\
      500   & \cellcolor[rgb]{ .976,  .812,  .702}70.42  & \cellcolor[rgb]{ .992,  .918,  .871}67.87  & \cellcolor[rgb]{ 1,  .988,  .98}64.12  & \cellcolor[rgb]{ 1,  .992,  .988}60.31  & 56.09  & 51.27  & 46.23  & \cellcolor[rgb]{ 1,  .984,  .973}43.07  & 37.40  & 30.67  & 27.25  & 50.43$_{\color{blue}(\pm 14.11)}$  \\
      1000  & \cellcolor[rgb]{ .973,  .796,  .678}70.53  & \cellcolor[rgb]{ .984,  .882,  .812}68.05  & 64.00  & 60.18  & \cellcolor[rgb]{ 1,  .992,  .984}56.33  & \cellcolor[rgb]{ 1,  .98,  .969}52.03  & \cellcolor[rgb]{ 1,  .984,  .976}47.13  & 41.79  & \cellcolor[rgb]{ .984,  .863,  .78}50.14  & \cellcolor[rgb]{ .984,  .859,  .773}48.37  & \cellcolor[rgb]{ .984,  .863,  .784}46.71  & \cellcolor[rgb]{ .988,  .91,  .855}55.02$_{\color{blue}(\pm \phantom{0}9.06)}$  \\
      5000  & \cellcolor[rgb]{ .984,  .867,  .792}69.94  & \cellcolor[rgb]{ .984,  .878,  .808}68.07  & \cellcolor[rgb]{ .984,  .859,  .776}65.35  & \cellcolor[rgb]{ .984,  .859,  .776}62.35  & \cellcolor[rgb]{ .98,  .847,  .761}59.42  & \cellcolor[rgb]{ .98,  .831,  .733}57.40  & \cellcolor[rgb]{ .976,  .82,  .718}56.15  & \cellcolor[rgb]{ .98,  .827,  .729}54.30  & \cellcolor[rgb]{ .976,  .816,  .71}54.31  & \cellcolor[rgb]{ .976,  .808,  .698}54.25  & \cellcolor[rgb]{ .976,  .816,  .706}53.66  & \cellcolor[rgb]{ .976,  .816,  .71}59.56$_{\color{blue}(\pm \phantom{0}5.69)}$  \\
      10000 & \cellcolor[rgb]{ .98,  .851,  .761}70.09  & \cellcolor[rgb]{ .973,  .796,  .678}68.46  & \cellcolor[rgb]{ .973,  .796,  .678}65.93  & \cellcolor[rgb]{ .976,  .804,  .686}63.22  & \cellcolor[rgb]{ .976,  .8,  .682}60.50  & \cellcolor[rgb]{ .976,  .8,  .682}58.61  & \cellcolor[rgb]{ .976,  .8,  .686}57.27  & \cellcolor[rgb]{ .976,  .816,  .706}55.27  & \cellcolor[rgb]{ .976,  .808,  .698}55.04  & \cellcolor[rgb]{ .976,  .8,  .682}55.38  & \cellcolor[rgb]{ .976,  .8,  .686}55.33  & \cellcolor[rgb]{ .973,  .796,  .678}60.46$_{\color{blue}(\pm \phantom{0}5.36)}$  \\
      50000 & \cellcolor[rgb]{ .992,  .929,  .886}69.42  & \cellcolor[rgb]{ .988,  .894,  .831}67.99  & \cellcolor[rgb]{ .98,  .839,  .745}65.54  & \cellcolor[rgb]{ .976,  .808,  .694}63.15  & \cellcolor[rgb]{ .976,  .812,  .698}60.27  & \cellcolor[rgb]{ .976,  .8,  .682}58.55  & \cellcolor[rgb]{ .976,  .8,  .682}57.36  & \cellcolor[rgb]{ .976,  .804,  .69}55.96  & \cellcolor[rgb]{ .976,  .808,  .694}55.26  & \cellcolor[rgb]{ .976,  .804,  .69}55.00  & \cellcolor[rgb]{ .976,  .804,  .686}55.23  & \cellcolor[rgb]{ .976,  .8,  .682}60.34$_{\color{blue}(\pm \phantom{0}5.11)}$  \\
      100000 & \cellcolor[rgb]{ .992,  .941,  .91}69.31  & \cellcolor[rgb]{ .996,  .965,  .945}67.65  & \cellcolor[rgb]{ .98,  .843,  .753}65.50  & \cellcolor[rgb]{ .98,  .827,  .729}62.81  & \cellcolor[rgb]{ .976,  .804,  .69}60.41  & \cellcolor[rgb]{ .976,  .8,  .682}58.56  & \cellcolor[rgb]{ .976,  .8,  .682}57.36  & \cellcolor[rgb]{ .976,  .804,  .686}56.15  & \cellcolor[rgb]{ .976,  .804,  .69}55.47  & \cellcolor[rgb]{ .976,  .804,  .686}55.11  & \cellcolor[rgb]{ .976,  .804,  .686}55.25  & \cellcolor[rgb]{ .976,  .8,  .686}60.33$_{\color{blue}(\pm \phantom{0}4.98)}$  \\
      500000 & \cellcolor[rgb]{ .992,  .941,  .91}69.31  & \cellcolor[rgb]{ .996,  .945,  .91}67.75  & \cellcolor[rgb]{ .98,  .831,  .733}65.61  & \cellcolor[rgb]{ .973,  .796,  .678}63.28  & \cellcolor[rgb]{ .973,  .796,  .678}60.53  & \cellcolor[rgb]{ .976,  .8,  .682}58.59  & \cellcolor[rgb]{ .976,  .8,  .686}57.29  & \cellcolor[rgb]{ .976,  .8,  .686}56.21  & \cellcolor[rgb]{ .976,  .8,  .686}55.66  & \cellcolor[rgb]{ .976,  .8,  .682}55.40  & \cellcolor[rgb]{ .976,  .804,  .686}55.23  & \cellcolor[rgb]{ .976,  .8,  .682}60.44$_{\color{blue}(\pm \phantom{0}4.99)}$  \\
      1000000 & 68.81  & 67.47  & \cellcolor[rgb]{ .984,  .878,  .804}65.18  & \cellcolor[rgb]{ .98,  .835,  .737}62.73  & \cellcolor[rgb]{ .976,  .8,  .686}60.45  & \cellcolor[rgb]{ .973,  .796,  .678}58.63  & \cellcolor[rgb]{ .973,  .796,  .678}57.44  & \cellcolor[rgb]{ .973,  .796,  .678}56.49  & \cellcolor[rgb]{ .973,  .796,  .678}55.98  & \cellcolor[rgb]{ .973,  .796,  .678}55.63  & \cellcolor[rgb]{ .973,  .796,  .678}55.88  & \cellcolor[rgb]{ .976,  .8,  .682}60.43$_{\color{blue}(\pm \phantom{0}4.66)}$  \\
      \bottomrule
      \end{tabular}%
    \label{tab:IMG-SADE}%
  \end{table*}%

  \begin{table*}[h!]
    \centering
    \caption{Ablation of $S$ on ImageNet-LT (Ours Setting) }
    \renewcommand\arraystretch{1.0}
    \small
      \begin{tabular}{ccccccccccc}
      \toprule
      \multirow{2}[4]{*}{$S$} & \multicolumn{3}{c}{Forward-LT} & \multicolumn{3}{c}{Uniform} & \multicolumn{3}{c}{Backward-LT} & \multirow{2}[4]{*}{Mean} \\
  \cmidrule(lr){2-4}\cmidrule(lr){5-7}\cmidrule(lr){8-10}          & 1     & 2     & 3     & 1     & 2     & 3     & 1     & 2     & 3     &  \\
      \midrule
      100   & \cellcolor[rgb]{ .98,  .855,  .769}70.21  & \cellcolor[rgb]{ .984,  .875,  .8}70.26  & \cellcolor[rgb]{ .973,  .796,  .678}70.74  & \cellcolor[rgb]{ .992,  .925,  .878}53.82  & \cellcolor[rgb]{ .996,  .961,  .937}53.36  & \cellcolor[rgb]{ .992,  .941,  .906}53.22  & \cellcolor[rgb]{ .988,  .894,  .831}42.44  & \cellcolor[rgb]{ .988,  .898,  .835}40.73  & \cellcolor[rgb]{ .988,  .89,  .824}42.38  & \cellcolor[rgb]{ .988,  .902,  .843}55.24$_{\color{blue}(\pm 11.73)}$  \\
      500   & \cellcolor[rgb]{ .98,  .843,  .749}70.30  & \cellcolor[rgb]{ .984,  .867,  .784}70.33  & \cellcolor[rgb]{ .984,  .863,  .78}70.24  & 50.90  & 51.89  & 51.10  & 27.32  & 24.87  & 24.96  & 49.10$_{\color{blue}(\pm 18.28)}$  \\
      1000  & \cellcolor[rgb]{ .973,  .796,  .678}70.65  & \cellcolor[rgb]{ .98,  .855,  .769}70.42  & \cellcolor[rgb]{ .98,  .835,  .737}70.45  & \cellcolor[rgb]{ 1,  1,  .996}51.01  & \cellcolor[rgb]{ 1,  .996,  .992}52.14  & 51.04  & \cellcolor[rgb]{ .984,  .867,  .784}46.58  & \cellcolor[rgb]{ .984,  .859,  .776}46.42  & \cellcolor[rgb]{ .984,  .863,  .78}46.42  & \cellcolor[rgb]{ .988,  .886,  .82}56.13$_{\color{blue}(\pm 10.37)}$  \\
      5000  & \cellcolor[rgb]{ .984,  .875,  .804}70.03  & \cellcolor[rgb]{ .988,  .914,  .859}69.94  & \cellcolor[rgb]{ .98,  .831,  .733}70.47  & \cellcolor[rgb]{ .98,  .847,  .757}56.75  & \cellcolor[rgb]{ .98,  .847,  .757}57.53  & \cellcolor[rgb]{ .98,  .847,  .761}56.55  & \cellcolor[rgb]{ .976,  .812,  .702}54.11  & \cellcolor[rgb]{ .976,  .812,  .702}53.72  & \cellcolor[rgb]{ .976,  .808,  .694}55.03  & \cellcolor[rgb]{ .976,  .816,  .706}60.46$_{\color{blue}(\pm \phantom{0}6.95)}$  \\
      10000 & \cellcolor[rgb]{ .984,  .863,  .784}70.13  & \cellcolor[rgb]{ .973,  .796,  .678}70.88  & \cellcolor[rgb]{ .98,  .855,  .769}70.29  & \cellcolor[rgb]{ .976,  .804,  .69}58.38  & \cellcolor[rgb]{ .976,  .812,  .698}58.85  & \cellcolor[rgb]{ .976,  .808,  .694}58.02  & \cellcolor[rgb]{ .976,  .8,  .686}55.59  & \cellcolor[rgb]{ .976,  .804,  .686}55.09  & \cellcolor[rgb]{ .973,  .796,  .678}56.25  & \cellcolor[rgb]{ .973,  .796,  .678}61.50$_{\color{blue}(\pm \phantom{0}6.43)}$  \\
      50000 & \cellcolor[rgb]{ .992,  .941,  .906}69.50  & \cellcolor[rgb]{ .996,  .961,  .933}69.55  & \cellcolor[rgb]{ .992,  .929,  .886}69.70  & \cellcolor[rgb]{ .973,  .796,  .678}58.58  & \cellcolor[rgb]{ .973,  .796,  .678}59.29  & \cellcolor[rgb]{ .973,  .796,  .678}58.37  & \cellcolor[rgb]{ .976,  .804,  .694}54.96  & \cellcolor[rgb]{ .976,  .808,  .694}54.52  & \cellcolor[rgb]{ .976,  .804,  .69}55.31  & \cellcolor[rgb]{ .976,  .804,  .69}61.09$_{\color{blue}(\pm \phantom{0}6.21)}$  \\
      100000 & \cellcolor[rgb]{ .996,  .949,  .918}69.44  & 69.20  & \cellcolor[rgb]{ .996,  .969,  .953}69.37  & \cellcolor[rgb]{ .976,  .804,  .694}58.29  & \cellcolor[rgb]{ .976,  .808,  .698}58.90  & \cellcolor[rgb]{ .976,  .8,  .682}58.36  & \cellcolor[rgb]{ .976,  .804,  .69}55.15  & \cellcolor[rgb]{ .976,  .8,  .682}55.53  & \cellcolor[rgb]{ .976,  .8,  .682}56.14  & \cellcolor[rgb]{ .976,  .804,  .69}61.15$_{\color{blue}(\pm \phantom{0}5.91)}$  \\
      500000 & \cellcolor[rgb]{ .992,  .918,  .871}69.68  & \cellcolor[rgb]{ .988,  .894,  .831}70.09  & \cellcolor[rgb]{ .988,  .902,  .847}69.90  & \cellcolor[rgb]{ .976,  .804,  .686}58.43  & \cellcolor[rgb]{ .976,  .804,  .69}59.10  & \cellcolor[rgb]{ .976,  .8,  .686}58.26  & \cellcolor[rgb]{ .976,  .804,  .694}54.95  & \cellcolor[rgb]{ .976,  .808,  .694}54.61  & \cellcolor[rgb]{ .976,  .8,  .682}55.87  & \cellcolor[rgb]{ .976,  .804,  .686}61.21$_{\color{blue}(\pm \phantom{0}6.31)}$  \\
      1000000 & 69.02  & \cellcolor[rgb]{ 1,  .996,  .996}69.24  & 69.12  & \cellcolor[rgb]{ .976,  .808,  .694}58.26  & \cellcolor[rgb]{ .976,  .82,  .718}58.47  & \cellcolor[rgb]{ .976,  .8,  .686}58.26  & \cellcolor[rgb]{ .973,  .796,  .678}56.02  & \cellcolor[rgb]{ .973,  .796,  .678}55.83  & \cellcolor[rgb]{ .976,  .8,  .682}56.18  & \cellcolor[rgb]{ .976,  .804,  .69}61.16$_{\color{blue}(\pm \phantom{0}5.72)}$  \\
      \bottomrule
      \end{tabular}%
    \label{tab:IMG-Ours}%
  \end{table*}%

  \subsection{The Effect of The Number of Experts}
  \begin{figure}[t]  
    \centering
    \subfigure[Our's Setting]{
    
      \includegraphics[width=0.44\columnwidth]{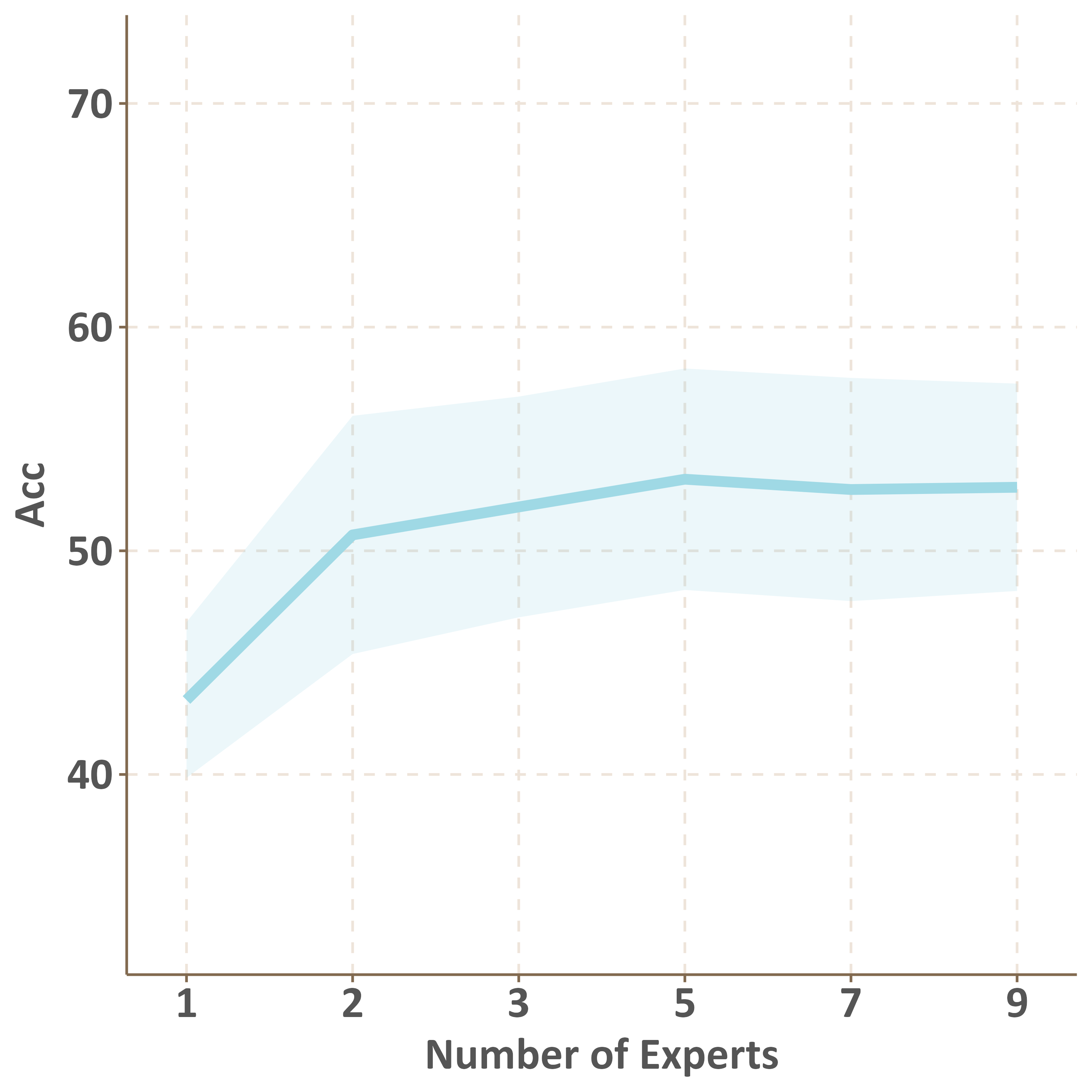} 
    }
    \subfigure[SADE's Setting]{
      \includegraphics[width=0.44\columnwidth]{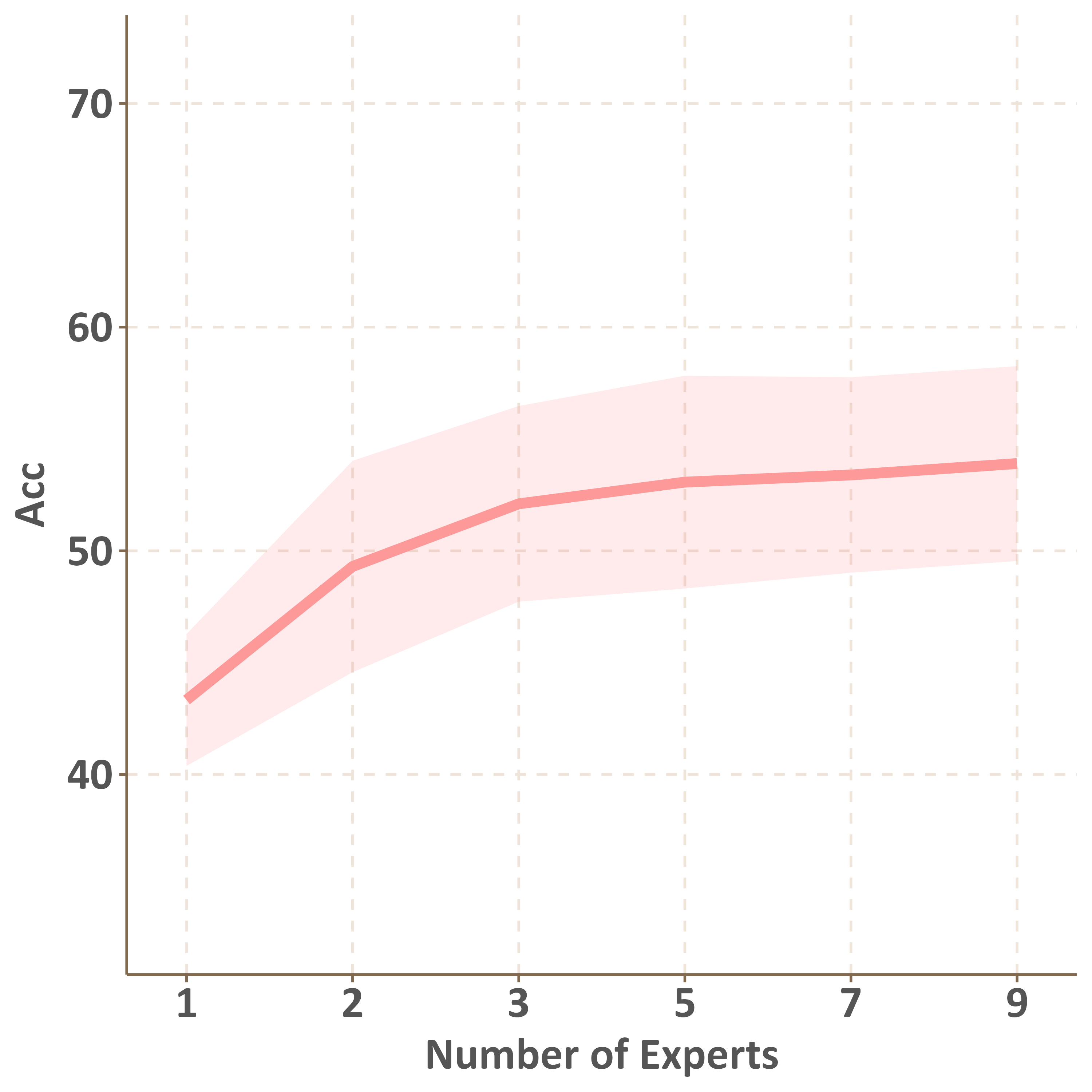} 
    }
    
    \caption{\label{fig:experts} \textbf{The Effect of Using Different Number of Experts.} }
  \end{figure}   

  In the \ours~ method, three experts are initially considered. However, the \ours~ framework can be extended to accommodate varying numbers of experts. We explore the impact of the number of experts on CIFAR 100-LT, specifically examining cases where the number is set to 1, 2, 3, 5, 7, and 9, where each expert corresponds to a distinct Dirichlet distribution component.
  
  The forward and backward components are characterized by an imbalance ratio denoted as $\rho$.
  
  Denote the training class distribution as $P_1, P_2, \ldots, P_C$, where $C$ denotes the number of classes.
  
  Specifically, we will need the following notations for experts handling different levels of imbalance ratio.
  
  The forward component $\alpha^{(f_\rho)}$ with imbalance ratio $\rho$ is set element-wisely as:
  \begin{align*}
      \alpha^{(f_\rho)}_i = \frac{S * \rho^{(i - 1) / (C - 1)} }{\sum_{j=1}^C \rho^{(j - 1) / (C - 1)}}, \ i = 1, 2, \ldots, C,
  \end{align*}
  where $S$ is the normalization factor. 
  
  The backward component $\alpha^{(b_\rho)}$ with imbalance ratio $\rho$ is set element-wisely as:
  \begin{align*}
      \alpha^{(b_\rho)}_i = \frac{S * \rho^{(C - i) / (C - 1)} }{\sum_{j=1}^C \rho^{(j - 1) / (C - 1)}}, \ i = 1, 2, \ldots, C
  \end{align*}
  where $S$ is the normalization factor. 
  
  The uniform component $\alpha^{(u)}$ is set element-wisely as:
  \begin{align*}
      \alpha^{(u)}_i = \frac{S}{C}, i = 1, 2, \ldots, C
  \end{align*}
  where $S$ is the normalization factor. 
  
  In the case of a single expert, we utilize one Dirichlet component $\alpha^{(u)}$. For two experts, we employ two components, $\alpha^{(f_{100})}$ and $\alpha^{(b_{100})}$. For three experts, we select the uniform component $\alpha^{(u)}$ and the forward and backward component with $\rho = 100$. For scenarios with more than three experts, we will symmetrically add the forward components and backward components based on the 3-experts configuration. For five experts, we select the uniform component $\alpha^{(u)}$ and the forward/backward components with $\rho = 100, 50$. For seven experts, we select the uniform component $\alpha^{(u)}$ and the forward/backward component with $\rho = 100, 50, 25$. For nine experts, we select the uniform component $\alpha^{(u)}$ and the forward/backward components with $\rho = 100, 50, 25, 10$. 
  
  The results are reported in Fig.\ref{fig:experts}. This indicates that the average ACC is monotonically increasing as more and more experts are employed.

\end{document}